\documentclass[10pt,dvipsnames]{article} 

\PassOptionsToPackage{round,sort&compress}{natbib}
\usepackage[left=1in,right=1in,top=1in,bottom=1in]{geometry}
\usepackage[dvipsnames,table]{xcolor}
\usepackage[round, sort&compress]{natbib}
\usepackage{amssymb,amsthm, amsfonts}            
\usepackage{mathtools, thmtools}          
\usepackage{mathrsfs}           
\mathtoolsset{showonlyrefs}     
\usepackage{graphicx}           
\usepackage{subcaption}         
\usepackage[space]{grffile}     
\usepackage{url}                
\usepackage{lipsum}             
\usepackage{algorithm}
\usepackage[noend]{algpseudocode}
\usepackage{booktabs}
\usepackage{enumitem}
\usepackage{hyperref}
\usepackage{titletoc}
\hypersetup{
    colorlinks=true,
    linkcolor=Cerulean,
    citecolor=Cerulean,
    urlcolor=Magenta,
}
\usepackage[T1]{fontenc}
\theoremstyle{definition}
\newtheorem{theorem}{Theorem}
\newtheorem{lemma}{Lemma}
\newtheorem{proposition}{Proposition}
\newtheorem{assumption}{Assumption}
\newtheorem{definition}{Definition}
\newtheorem{corollary}{Corollary}
\newtheorem{remark}{Remark}

\newcommand{\eps}{\varepsilon}

\definecolor{cBad}{HTML}{C0392B}
\definecolor{cTeal}{HTML}{1B9E8A}
\definecolor{accent}{HTML}{C0392B}
\usepackage{titletoc}
\definecolor{tealcol}{HTML}{0E9E67} 

\definecolor{darkgrey}{HTML}{535050}

\definecolor{lightblue}{HTML}{0B91D4}
\definecolor{darkorange}{HTML}{D4510B}
\definecolor{golden}{HTML}{BD8A09}
\definecolor{darkgreen}{HTML}{057A1A}
\definecolor{purple}{HTML}{6713BD}
\usepackage{amsmath,amsfonts,bm}
\usepackage{mathtools}
\usepackage{makecell}
\mathtoolsset{showonlyrefs=true}

\newcommand{\QFRS}{\textsc{AnchorQ}}

\newcommand{\reach}{\rho}
\newcommand{\dist}{\operatorname{dist}}
\newcommand{\ScoreFn}{\mathcal{F}}
\newcommand{\xierr}{\xi_{a}}
\newcommand{\Mom}{\mathcal{M}}
\newcommand{\ConMom}{\mathcal{K}}
\DeclareMathOperator*{\E}{\mathbb{E}}
\newcommand{\KLop}{\operatorname{kl}}

\newcommand{\R}{\mathbb{R}}

\newcommand{\KL}{D_{\mathrm{KL}}}
\newcommand{\Var}{\mathrm{Var}}

\DeclareMathOperator*{\argmax}{arg\,max}

\newcommand{\margin}{\delta}
\newcommand{\probmeas}{\mu}
\providecommand{\betatilt}{\beta} 

\def\1{\bm{1}}

\def\vmu{{\bm{\mu}}}
\def\eps{{\epsilon}}

\def\gA{{\mathcal{A}}}

\def\gC{{\mathcal{C}}}

\def\gE{{\mathcal{E}}}

\def\gM{{\mathcal{M}}}

\def\gS{{\mathcal{S}}}
\def\gT{{\mathcal{T}}}

\def\gV{{\mathcal{V}}}

\def\gZ{{\mathcal{Z}}}

\algrenewcommand\algorithmicrequire{\textbf{Input:}}
\algrenewcommand\algorithmicensure{\textbf{Output:}}

\newcommand{\Rb}{\mathbb{R}}

\newcommand{\Dc}{\mathcal{D}}

\newcommand{\Bc}{\mathcal{B}}

\newcommand{\reverseflow}{\textsc{RevFlow}}
\newcommand{\pizero}{\ensuremath{\pi_0}}
\newcommand{\bank}{\ensuremath{Z}}

\newcommand{\numcrit}{K}
\newcommand{\cnext}{c^\star}
\newcommand{\Vnext}{V^\star}
\newcommand{\znext}{z^\star}
\newcommand{\ztnext}{\tilde{z}^\star}
\newcommand{\wnext}{w^\star}
\newcommand{\numpart}{M}

\newcommand{\mF}{F}
\newcommand{\AQ}{\mathrm{A}}
\newcommand{\Lsmooth}{L}

\newcommand{\figleft}{{\em (Left)}}
\newcommand{\figcenter}{{\em (Center)}}
\newcommand{\figright}{{\em (Right)}}

\title{Sample-Efficient Reinforcement Learning for Policy Improvement via Flow-Reversed Action Anchors}

\author{%
  \shortstack{
    Carl Qi\textsuperscript{1,2}$\dagger$
    \quad Weijun Dong\textsuperscript{1,3}
    \quad Xiaojie Wang\textsuperscript{1}
    \quad Peiyi Chen\textsuperscript{1}
    \\[0.8em]
    Rahul Sajnani\textsuperscript{1,4}
    \quad Mingyuan Zhang\textsuperscript{1}
    \quad Amy Zhang\textsuperscript{1,2}
    \quad Jose Barreiros\textsuperscript{1}
    \\[0.8em]
    Manikantan Nambi\textsuperscript{1}
    \quad Lillian J.~Ratliff\textsuperscript{1,5*}
    \quad Dylan Randle\textsuperscript{1*}
  }%
}

\begin{document}

\maketitle
\begin{abstract}
Generalist robot policies offer strong behavioral priors, but achieving deployment-level reliability and throughput remains difficult. Reinforcement learning (RL) can improve these policies, yet its effectiveness is constrained by the cost of online interaction. Existing methods adapt frozen flow-matching policies by treating noise space as a continuous action space. We show that this behaviorally unstructured space makes exploration inefficient, often preventing these methods from extracting meaningful learning signals on challenging tasks. We introduce \QFRS{}, which enables coarse-to-fine policy improvement through a structured, refinable latent action representation. At each decision step, \QFRS{} uses flow reversal to map reference behaviors to state-conditioned latent anchors, samples local particles around those anchors, and applies value-guided sampling to explore promising refinements. Our analysis shows that anchors reduce identification complexity at the cost of support bias; particle refinement trades this bias against action-class complexity; and critic accuracy determines the reliability of KL-regularized, value-guided selection. Across simulated and real-world manipulation tasks, \QFRS{} outperforms state-of-the-art post-training baselines in sample efficiency, delivering a \textbf{37\%} improvement in real-world success rate and \textbf{97\%} higher simulated task throughput under the same interaction budget. Additional visualizations are available on our website~\footnote{\url{https://anchorq-rl.com/}}.
\end{abstract}
\begin{center}
\small
$^{1}$Amazon; \small
$^{2}$UT Austin; \small
$^{3}$UC San Diego; \small
$^{4}$Brown University; \small
$^{5}$University of Washington \\
$\dagger$Work done during the internship at Amazon; *Equal advising
\end{center}

\section{Introduction}
\label{sec:intro}
Generalist vision-language-action (VLA) models provide a powerful starting
point for robotic manipulation. Through extensive pretraining, these policies
can follow language instructions and transfer to objects, scenes, and tasks
beyond the training set
\citep{black2024pi0,black2025pi05,kim2024openvla,tri2025lmb}.
Yet deployment demands more than broad competence: robots must achieve high
success rates and throughput in their target environments, while adaptation
itself must remain inexpensive. Pretrained VLAs remain imperfect under changes
in geometry, clutter, object placement, and task execution, while collecting
demonstrations that densely cover these variations is costly.
Reinforcement-learning (RL) post-training offers a natural route to
deployment-specific improvement, but fine-tuning a large VLA is computationally
expensive and can degrade useful pretrained behaviors. An appealing alternative
is therefore to keep the generalist policy frozen and learn how to select and
refine its actions.

Prior methods adapts pretrained VLA policies by learning lightweight action corrections, steering latent inputs, or selecting among candidate actions \citep{yuan2025policy,wagenmaker2025dsrl,tang2026frs,nakamoto2024vgps}.
Flow-matching policies provide an additional adaptation interface through the
noise that initializes action generation. For a fixed state $s$ and numerical
sampler, an initial latent $z$ is deterministically mapped to an action
chunk, $a=\pizero(s,z)$. The latent can therefore serve directly as the action for RL: changing how $z$ is selected changes the executed behavior without modifying the VLA. For example, diffusion steering via reinforcement learning (DSRL) learns directly
in the latent noise space of a frozen policy~\citep{wagenmaker2025dsrl}.

This raises a more fundamental question: \emph{what action representation should be used for RL to adapt pretrained VLA policies?}
At one extreme, a handful of behaviorally meaningful actions, such as moving left or right, can make online learning much easier, but only by restricting the behaviors the policy can express.
At the other extreme, existing approaches~\citep{wagenmaker2025dsrl,yuan2025policy,xiao2025pld} search expressive but behaviorally unstructured continuous action or latent spaces. The learner must discover which action corrections or latent vectors produce useful behaviors in each state, without an explicit set of behavioral alternatives to guide this search. This makes efficient exploration extremely difficult: as we demonstrate in our experiments, DSRL and action-space residual RL baselines fail to learn effective policies on challenging tasks within the interaction budget.
What is missing is an \textbf{\emph{online action representation that is both structured and refinable}}: compact when coarse behaviors suffice, but able to expose richer behavior when needed.

To address this gap, we construct a structured and refinable action representation around state-conditioned noise-space anchors, which we generate at every decision by applying flow reversal~\citep{tang2026frs} to a small library of semantically meaningful reference actions.
Building on this representation, we introduce \QFRS{}, a coarse-to-fine RL method for adapting frozen flow-matching policies directly in their latent noise space.
\QFRS{} first samples local particles around each anchor and applies exponential value tilting to bias sampling toward high-value particles. It then selects an anchor and samples a refined noise latent from its tilted particle distribution, which the frozen flow policy decodes into an action. We train the critic $Q$ using temporal-difference (TD) learning from online experience to guide both particle tilting and anchor selection.

\begin{figure}[t]
    \centering
    \includegraphics[width=\textwidth]{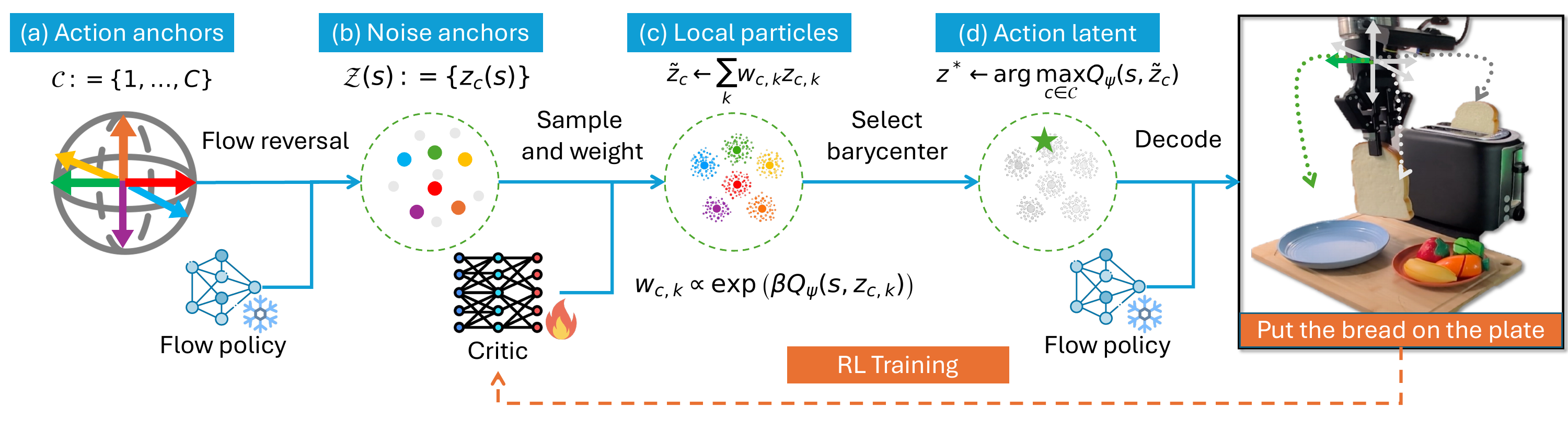}
    \vspace{2mm}
    \caption{\textbf{Overview of \QFRS{}. } (a) A small set of reference action anchors defines coarse behaviors. (b) Flow reversal maps these references into state-conditioned noise anchors. (c) \QFRS{} samples local particles around every anchor and reweights them using the learned critic to form refined barycenters. (d) It selects the highest-value barycenter, decodes it through the frozen flow policy, and executes the resulting action. The critic is trained online using Bellman updates.}
    \label{fig:anchorq_overview}
\end{figure}

\paragraph{Contributions.} The main contributions are summarized as follows:
\begin{itemize}[itemsep=-2pt,topsep=-2pt,leftmargin=*]

\item \textbf{A structured and refinable action representation for online RL.} We propose a new online action representation and a coarse-to-fine RL method that turn continuous-space exploration into selection and refinement of behaviorally meaningful actions.

\item \textbf{Guarantees for coarse-to-fine latent search.} We show that the flow-reversed noise anchors reduces identification complexity
(Theorem~\ref{thm:persistent-abstraction-main}), while particle refinement
trades support bias against the complexity of the local action class
(Theorem~\ref{thm:abstraction-frontier-main}). We further show that value tilting reliably amplifies high-value refinements under critic error (Theorem~\ref{thm:exponential-tilting-main}), while epistemic uncertainty adaptation calibrates exploitation to critic confidence (Theorem~\ref{thm:adaptive-local-tilting}, Appendix~\ref{sec:uncertainty_adaptation}).

\item \textbf{Consistent empirical gains in simulation and on real robots.}
Across three benchmarks in LIBERO \\ \citep{liu2023libero} and RoboTwin~\citep{chen2025robotwin2}, \QFRS{} outperforms the state-of-the-art steering baselines, which include continuous noise space RL and residual-based RL, with 97\% higher throughput than the closest baseline. We also deploy \QFRS{} to solve three challenging real robot manipulation tasks, where it outperforms both the frozen base policy and the RL baseline under the same interaction budget.
\end{itemize}

\section{Related Work} 
The most relevant related work is presented here; Appendix~\ref{sec:related_work} is a broader review.

\paragraph{RL and steering in frozen generative policies.}
Several methods adapt frozen policies via their latent noise space. Value-guided methods rerank or steer sampled action chunks \citep{nakamoto2024vgps,wang2025itps}, while DSRL \citep{wagenmaker2025dsrl} treats the noise itself as a continuous action and learns a policy over the full latent space. FRS \citep{tang2026frs} instead inverts the frozen flow to recover latents corresponding to reference motions, using this primitive for zero-shot steering and to initialize continuous noise-space RL. UniSteer \citep{lu2026unisteer} also combines flow inversion with noise-space RL, yet uses online human corrections and learns a continuous latent actor.  On the other hand, \QFRS{} uses inversion to 
map a fixed library of coarse reference behaviors into latent anchors that are refined locally with continuous particles, and ranked by a critic.

\paragraph{Reduced and coarse-to-fine action spaces.}
Reducing the effective action space can substantially simplify RL, through discretization \citep{tang2020discretizing}, learned primitives \citep{shafiullah2022bet,lee2024vqbet}, or hierarchical refinement \citep{seo2024cqn,seo2024cqnas}. These methods typically construct abstractions directly in the action coordinates. In contrast, \QFRS{} maps a finite set of coarse reference actions to state-conditioned latent anchors through flow inversion, then refines these anchors locally using sampled continuous perturbations.

\section{\QFRS{}: Value Learning over Online Flow-Reversed Latents} 
\label{sec:method} 
\QFRS{} adapts a frozen flow-matching VLA by learning to select actions in its latent noise space. At each decision, it flow-reverses a small library of coarse reference behaviors into state-dependent latent anchors, samples local refinements around those anchors, and uses a learned critic to select a latent for execution through the frozen VLA.

We first introduce the flow-reversed anchors, then describe local value-tilted refinement, and finally present action selection and critic training.  Figure~\ref{fig:anchorq_overview} and Algorithm~\ref{alg:anchorq-revised} summarize the method.

\subsection{Online flow-reversed anchors}
\label{sec:m-bank}
Given a state $s\in \mathcal{S}$ and latent noise $z\in\mathbb{R}^{d_z}$, the frozen
flow-matching VLA produces an action chunk
$a=\pi_{\mathrm{flow}}(s,z)\in\mathbb{R}^{d_z}$ such that $d_z:=h_a\cdot d_a$ where $h_a$ is
the chunk horizon and $d_a$  the dimension of each per-timestep action. With the ODE solver for the flow policy fixed, this
mapping is deterministic, and we therefore treat $z$ as the RL action.

\QFRS{} constructs a compact, state-dependent set of such latent actions from
a fixed library of coarse reference behaviors (e.g. cardinal Cartesian motions) indexed by $\gC:=\{1,\ldots,C\}$. At state $s\in \mathcal{S}$, a reference map $\varphi$ specifies an \emph{action anchor} $a_c^{\mathrm{ref}}(s):=\varphi(s,c)$,
a reference action chunk in the environment action space. Flow reversal (cf.~\citep{rout2024rfinversion,wang2024taming,kim2022diffusionclip}), recently
adopted  in robotics~\citep{tang2026frs}, maps each action anchor backward through the flow-matching policy,  conditioned on $s$, to obtain a corresponding \emph{noise anchor}:
\begin{equation}
    z_c(s)
    := \reverseflow\bigl(s,a_c^{\mathrm{ref}}(s)\bigr)
    \in \Rb^{d_z}.
\end{equation}
The resulting noise-anchor set is $\gZ(s):=\{z_c(s):c\in\gC\}.$ 
\QFRS{} reconstructs $\gZ(s)$ at every decision using the current state, and parameterizes the critic as $Q_\psi(s,z)$ so it evaluates  noise anchors and their continuous refinements. 

\subsection{Local particles and value-tilted refinement}
\label{sec:m-particles}
The anchor-only variant of \QFRS{}  restricts action selection to $\gZ(s)$;
the reference behaviors underlying the noise anchors may
be too coarse for fine-grained control.
We therefore refine each noise anchor using local
particles weighted by the learned critic.

For each $c\in\gC$, we retain the noise anchor as
$z_{c,0}(s):=z_c(s)$ and sample $\numpart$ additional
particles:
\begin{equation}
    z_{c,k}(s)
    := z_c(s)+\sigma_p\xi_{c,k},
    \qquad
    \xi_{c,k}\sim\mathcal N(0,I),
    \qquad
    k=1,\ldots,\numpart,
    \label{eq:m-particle}
\end{equation}
where $\sigma_p$ controls the perturbation scale.
Let
$\mathcal B_c(s):=(z_{c,0}(s),\ldots,z_{c,\numpart}(s))$
denote the resulting local particle cloud.
Within each cloud, value tilting favors particles with higher critic values. 
Indeed, denote the probability simplex $\Delta_{\numpart+1}:=
\{
p\in\mathbb{R}_{+}^{\numpart+1}:
\sum_{k=0}^\numpart p_k=1
\}$, and let $u\in \Delta_{\numpart+1}$ be the uniform element. Then, for a tilting strength $\beta>0$, define weights
\begin{equation}
    w_{c,\beta}(\cdot\,|\,s)\textstyle 
    :=\argmax_{p\in\Delta_{\numpart+1}}
    \left\{
        \sum_{k=0}^{\numpart}
        p_k Q_\psi(s,z_{c,k}(s))
        -
        \frac{1}{\beta}\KL(p\|u)
    \right\}.
    \label{eq:m-local-entropic-objective}
\end{equation}
The unique solution is $
        \exp(\beta Q_\psi(s,z_{c,k}(s)))/
        \sum_{\ell=0}^{\numpart}
        \exp(\beta Q_\psi(s,z_{c,\ell}(s)))$, yielding an exponential \emph{value tilting} of the local particles toward higher-value refinements.
These weights define a refined latent via 
    $\tilde z_{c,\beta}(s)
    :=
    \sum_{k=0}^{\numpart}
    w_{c,\beta}(k\,|\,s)\,z_{c,k}(s)$.
Since the weights are nonnegative and sum to one,
$\tilde z_{c,\beta}(s)$ lies in the convex hull of
$\mathcal B_c(s)$, so that the refined latent remains in a neighborhood of $z_{c}(s)$. 
\QFRS{} constructs a refined latent for all noise
anchors before selecting among them.

$\sigma_p$ controls the reach of local exploration, while $\beta$ controls how strongly refinement favors high-value particles. As $\sigma_p\to0$, the method recovers exact-anchor \QFRS{}. As $\beta\to0$, the weights become uniform; as $\beta\to\infty$, the refined latent converges to the highest-valued particle when the maximizer is unique. In practice, we adapt both parameters using the procedures in Appendix~\ref{sec:algorithm}.

\subsection{Action selection and critic training}
\label{sec:m-training}

During RL training, \QFRS{} explores with $\epsilon$-greedy selection among the refined latents. With probability $\epsilon$, it samples $c_t\sim\operatorname{Unif}(\gC)$; otherwise, $c_t=\argmax_{c\in\gC}Q_\psi\bigl(s_t,\tilde z_{c,\beta}(s_t)\bigr)$. We execute the action chunk $a_t=\pi_{\mathrm{flow}}(s_t,z_t)$ with $z_t:=\tilde z_{c_t,\beta}(s_t)$, observe reward $r_t$, next state $s_{t+1}$, and terminal indicator $d_t$, and store $(s_t,z_t,r_t,s_{t+1},d_t)$ in the replay buffer $\Dc$. We cache the flow-reversed anchors with each replay observation and reuse them during critic updates, while drawing fresh local particles. During evaluation, we use the same refinement procedure with greedy selection.

For critic updates, we sample batched transitions $(s,z,r,s',d)$ from $\Dc$ and use temporal difference (TD) learning. Each transition spans one latent-action decision, with reward $r$ from the executed action chunk and discount $\gamma$ per decision. We minimize
\begin{equation}
    \mathcal L(\psi):=\mathbb{E}_{(s,z,r,s',d)\sim\Dc}\left[\bigl(Q_\psi(s,z)-r-\gamma(1-d)V^\star(s')\bigr)^2\right],
    \label{eq:m-critic-loss}
\end{equation}
where $V^\star(s'):=Q_{\bar\psi}(s',\tilde z_{\cnext(s'),\beta}(s'))$ is the target critic's value of the refined latent selected by the online critic, with $\cnext(s'):=\argmax_{c\in\gC}Q_\psi(s',\tilde z_{c,\beta}(s'))$. For nonterminal transitions, all noise anchors are retrieved from cache and refined at $s'$ before selection; when $d=1$, the bootstrap term vanishes. The VLA remains frozen throughout training, and latent actions are obtained directly through critic-guided refinement and selection, without requiring a separately trained latent actor. Algorithm~\ref{alg:anchorq-revised} gives the detailed procedure, with implementation details in Appendix~\ref{sec:algorithm}.

\begin{algorithm}[ht]
\caption{\QFRS{} with local value-tilted refinement}
\label{alg:anchorq-revised}
\begin{algorithmic}[1]

\Require frozen flow policy $\pi_{\mathrm{flow}}$ and reverse map $\reverseflow$;
reference map $\varphi$; critic $Q_\psi$ and target critic $Q_{\bar\psi}$;
replay $\Dc$; particles $\numpart$, scale $\sigma_p$, tilt $\beta$,
discount $\gamma$, exploration $\epsilon$

\For{each latent-action decision $t$}

\State \textbf{Construct anchors:}\;
$z_c(s_t)\leftarrow
\reverseflow(s_t,\varphi(s_t,c))$
for all $c\in\gC$

\State \textbf{Sample particles:}\;
$z_{c,k}\leftarrow
z_c(s_t)+\sigma_p\xi_{c,k}$, $\forall c\in\gC$, $k\in[\numpart]$,
with $\xi_{c,0}=0$,
$\xi_{c,k}\sim\mathcal N(0,I)$

\State \textbf{Tilt locally:}\;
$w_{c,k}\propto
\exp\left(\beta Q_\psi(s_t,z_{c,k})\right)$
for each $c\in\gC$

\State \textbf{Refine:}\;
$\tilde z_c\leftarrow
\sum_{k=0}^{\numpart}w_{c,k}z_{c,k}$
for each $c\in\gC$

\State \textbf{Select refined action:}\;
w.p.~$\epsilon$,
$c_t\sim\operatorname{Unif}(\gC)$; otherwise
$c_t\leftarrow
\argmax_{c\in\gC}Q_\psi(s_t,\tilde z_c)$

    \State Execute $\pi_{\mathrm{flow}}(s_t,\tilde z_{c_t})$, observe $(r_t,s_{t+1},d_t)$, and store $(s_t,\tilde z_{c_t},r_t,s_{t+1},d_t)$ in $\Dc$

    \For{each critic update}
    
        \State Sample $(s,z,r,s',d)\sim\Dc$
    
        \State Retrieve cached anchors $\{z_c(s')\}_{c\in\gC}$
    
        \State Sample $\{\znext_{c,k}\}_{k=0}^{\numpart}$ around $z_c(s')$ for all $c\in\gC$;
        set $\wnext_{c,k}\propto\exp(\beta Q_\psi(s',\znext_{c,k}))$
    
        \State Set $\ztnext_c\leftarrow\sum_k\wnext_{c,k}\znext_{c,k}$ for all $c\in\gC$ and set
        $\cnext\leftarrow\argmax_c Q_\psi(s',\ztnext_c)$
    
        \State Set $\Vnext(s')\leftarrow Q_{\bar\psi}(s',\ztnext_{\cnext})$
        and set $y\leftarrow r+\gamma(1-d)\Vnext(s')$
    
        \State Update $Q_\psi$ from
        $\bigl(Q_\psi(s,z)-y\bigr)^2$ and update $Q_{\bar\psi}$
    
    \EndFor
\EndFor
\end{algorithmic}
\end{algorithm}

\subsection{Online noise anchors improve sample efficiency}
\label{sec:qfrs-theory-main}
\QFRS{} uses state-dependent noise anchors and local particle refinements to construct and select actions online.
Our results isolate three consequences of this algorithmic structure: 
\begin{itemize}[itemsep=0pt,topsep=-2pt,leftmargin=*] 
\item \textbf{Sample efficiency.} Restricting online search to the actions represented by the noise anchors can reduce the number of latent alternatives that must be identified (Theorem~\ref{thm:persistent-abstraction-main}).
\item \textbf{A bias--complexity frontier.} 
Restricting the action class introduces support bias; local particles reduce this bias by enlarging the  class, at the cost of increased statistical complexity (Theorem~\ref{thm:abstraction-frontier-main}).
\item \textbf{Risk-sensitive local refinement.}
Exponential value tilting provides an entropic, upper-tail-seeking reweighting of particles around each anchor. With a sufficiently accurate critic, it amplifies genuinely high-value refinements and, under local smoothness, yields a corresponding value guarantee for the weighted latent (Theorem~\ref{thm:exponential-tilting-main}).
\end{itemize}
Below we provide the key results in informal theorems; full statements are given in the appendix. 

\paragraph{Sample Efficiency: Online noise anchors reduce identification complexity.} 
To isolate the statistical effect of restricting the online action class, we
consider a one-decision  MDP with $N$ behaviorally distinct latent modes.
Flow reversal exposes a subset of $C<N$ modes, whereas an unrestricted learner may select
any of the $N$.

\begin{theorem}[Online noise anchors reduce complexity] \label{thm:persistent-abstraction-main} Fix $\varepsilon\in (0,1/8]$ and $\alpha\in(0,1/4]$. If the best noise anchor is $\varepsilon/2$-optimal, then restricting online
search to the $C$ noise anchors identifies an $\varepsilon$-optimal action with
probability at least $1-\alpha$ using
$\mathcal O(
\frac{C}{\varepsilon^2}\log\frac{C}{\alpha})$ episodes.
Conversely, any learner that is $(\varepsilon,\alpha)$-correct over all $N$
actions requires, on some instance whose optimal action lies in the noise-anchor set,
$\Omega(
\frac{N-C}{\varepsilon^2}\log\frac1\alpha)$
expected samples from actions outside the noise-anchor set.
\end{theorem} 
The upper bound shows that when the $C$ noise anchors contain a near-optimal action, searching only among them is sufficient to identify an $\varepsilon$-optimal action.
For the lower bound construction, each non-flow-reversed action is suboptimal in the base instance but optimal in a corresponding alternative instance. Hence, a learner searching all $N$ latent modes must still collect enough samples to rule out all $N-C$ outside modes, even when a noise anchor is globally optimal. That is,
unrestricted residual RL, even with FRS-based prefill (e.g., \cite{tang2026frs}), cannot eliminate this exploration cost.
The full result is in
Theorem~\ref{thm:persistent-abstraction-separation},
Appendix~\ref{sec:persistent-abstraction-separation}.

\paragraph{Bias--Complexity Frontier: Trading support bias for learning complexity.}
Local particles interpolate between selection among exact noise anchors and broader latent search. For the more general sequential MDP, we view the frozen VLA as a control problem whose
action is the latent $z$: choosing $z$ at state $s$ executes
$\pi_{\mathrm{flow}}(s,z)$ and induces the corresponding environment reward and transition.
Let $V^\star$ and $Q^\star$ denote the optimal value and action-value functions
when any latent $z\in\mathbb{R}^{d_z}$ may be selected.

To isolate the
bias--complexity tradeoff, we analyze an idealized local action class.  Indeed, following Section~\ref{sec:m-particles}, 
let $\rho\ge0$ denote an \emph{idealized local search radius}, where $\rho=0$ allows only the exact noise anchors, while larger $\rho$ allows any latent in $B_2(z_c(s),\rho):=\{
            \tilde z\in\R^{d_z}:
            \lVert \tilde z-z_c(s)\rVert_2\leq \reach\}$. 
Define $\mathcal Z_\rho(s) := \bigcup_{c\in\gC} B_2\left(z_c(s),\rho\right)$,
and define the corresponding one-step support gap 
$\Delta_\rho(s) := V^\star(s) - \sup_{z\in\mathcal Z_\rho(s)} Q^\star(s,z)$ and set  $\bar\Delta_\rho := \sup_s\Delta_\rho(s)$. Thus $\Delta_\rho(s)$ measures how much value is unavailable as action selection remains local to the noise anchors. 
\begin{theorem}[Bias--complexity frontier] \label{thm:abstraction-frontier-main} 
Increasing the local search radius can only reduce the support bias: $0\le\rho_1\le\rho_2$ implies $\bar\Delta_{\rho_2} \le \bar\Delta_{\rho_1}$.  Moreover, suppose a finite particle representation of the $\rho$-local action class incurs one-step representation error at most $\eps_{\mathrm{rep}}$, and its learned selector incurs one-step selection error at most $\eps_{\mathrm{sel}}(n,\rho,\alpha)$ after $n$ interactions. Let $\pi_{\rho,n}$ denote the resulting learned policy and
$V^{\pi_{\rho,n}}$ its discounted value function. With probability at least $1-\alpha$, the bound holds: $\| V^\star-V^{\pi_{\rho,n}}\|_\infty \leq ( \bar\Delta_\rho + \eps_{\mathrm{rep}} + \eps_{\mathrm{sel}}(n,\rho,\alpha))/( 1-\gamma)$. 
\end{theorem} 
This result exposes a Pareto-type frontier: \emph{larger local search  implies smaller support bias but a larger action class to learn}. 
Exact-anchor \QFRS{} gives the
strongest restriction, while local particles progressively enlarge the
available action class. Formal support and sample-complexity results are given
in Theorems~\ref{thm:general-mdp-abstraction-price} 
and~\ref{thm:particle-bias-complexity-frontier} (Appendices~\ref{sec:general-mdp-abstraction-price} and \ref{sec:particle-bias-complexity-frontier}), respectively.
Theorem~\ref{thm:abstraction-frontier-main} is stated in terms of an idealized  $\rho$, whereas the
algorithm realizes local search via Gaussian particles.
Proposition~\ref{prop:finite-gaussian-envelope-coverage}
(Appendix~\ref{subsec:gaussian-particle-cloud-connection}) connects these two
views: sufficiently many Gaussian particles around the anchors cover the
$\rho$-local action class to a prescribed resolution with high probability.
Thus $(\sigma_p,M)$ control the stochastic coverage of the idealized local
class rather than setting $\rho$ directly.

\paragraph{Risk-Sensitive Local Refinement: Value tilting exploits particle diversity.} 
The weight vector $w_{c,\beta}$ defined in Section~\ref{sec:m-particles} is exactly the Gibbs change of measure solving the entropic variational problem $\frac{1}{\beta} \log( \frac{1}{M+1} \sum_{k=0}^{M} e^{\beta Q_\psi(s,z_{c,k}(s))}) = \max_{p\in\Delta_{M+1}} \{ \sum_{k=0}^{M} p_k Q_\psi(s,z_{c,k}(s)) - \frac{1}{\beta}\KL(p\|u)\}$,  where $u\in\Delta_{M+1}$ is uniform over the local candidates. Thus, $\beta>0$ induces an upper-tail-seeking reweighting that places increasingly more mass on higher-value particles.

\begin{theorem}[Amplification of high-value local refinements]
\label{thm:exponential-tilting-main}
Fix a state $s\in \mathcal{S}$, 
anchor $c\in\mathcal{C}$, and corresponding particle cloud
$\mathcal B_c(s)$. Let
$\mathcal H\subset\{0,\ldots, M\}$ be a nonempty proper subset and define 
$\margin_{\mathcal H}
    :=
    \min_{k\in\mathcal H}
        Q^\star(s,z_{c,k}(s))
    -
    \max_{\ell\in \{0,\ldots, M\}\setminus\mathcal H}
        Q^\star(s,z_{c,\ell}(s))$. If the critic satisfies $\max_{ 0\leq k\leq M}|
        Q_\psi(s,z_{c,k}(s))
        -
        Q^\star(s,z_{c,k}(s))|
    \leq
    \eps_Q$, then
\begin{equation}
    \label{eq:local-high-value-amplification}
        w_{c,\beta}(\mathcal H\,|\, s)/
        (1-w_{c,\beta}(\mathcal H\,|\, s))
    \geq
    \exp(
        \beta(\margin_{\mathcal H}-2\eps_Q))
        |\mathcal H|/(
        M+1-|\mathcal H|),
\end{equation}
where $w_{c,\beta}(\mathcal H\,|\, s):=\sum_{k\in\mathcal H}
    w_{c,\beta}(k\,|\, s)$. Hence, whenever
$\margin_{\mathcal H}>2\eps_Q$, value tilting increases the odds assigned to the
truly high-value subset relative to uniform weighting, with amplification
growing exponentially in both $\beta$ and the value margin.
\end{theorem}
Theorem~\ref{thm:exponential-tilting-main}
shows that high-value local deviations exert greater influence on each refined
candidate when their value advantage is large relative to critic error; the formal statement is Theorem~\ref{thm:exponential-tilting},
Appendix~\ref{sec:exponential-tilting}. Under local smoothness, this guarantee extends to the refined latent: its loss relative to the best local particle is controlled by critic error, finite-temperature weighting, and weighted particle variance (Proposition~\ref{prop:barycentric-refinement-value}).
As aggressive tilting may amplify critic errors, we estimate local epistemic
uncertainty from disagreement among the critic ensemble and set $\beta(s)$
inversely proportional to this disagreement. The weighting therefore remains
diffuse where the critic is uncertain and becomes sharper as critic confidence
increases (Theorem~\ref{thm:adaptive-local-tilting}).

\section{Experiment Results}
\label{sec:result}
We evaluate \QFRS{} in simulation on the LIBERO~\citep{liu2023libero} and RoboTwin~\citep{chen2025robotwin2} benchmarks, and in the real world on manipulation tasks using a Franka DROID setup~\citep{khazatsky2024droid}. We focus the analysis on three questions:
\begin{itemize}[itemsep=0pt,topsep=5pt,leftmargin=20pt]
    \item Do flow-reversed anchors improve online RL over a continuous noise action space?
    \item How does particle tilting enable finer motions beyond those produced by the noise anchors?
    \item Does the same recipe work for steering a real robot?
\end{itemize}

\subsection{Tasks \& Setup}

\paragraph{Benchmarks.} We use the two evaluation splits of \citet{tang2026frs}, both drawn from LIBERO-90 \citep{liu2023libero}: a 15-task split and a 10-task hard split, none of them used to train the base policy. We also evaluate on 6 RoboTwin~\citep{chen2025robotwin2} benchmark tasks that requires bimanual manipulation. To evaluate the effectiveness of our approach in the real world, we additionally evaluate on three
common household tasks using a Franka arm with Robotiq gripper in a DROID setup~\citep{khazatsky2024droid}. We show a visual overview of the tasks in Figure~\ref{fig:experiment-task-collage}.

\begin{table}[t]
  \centering
  \small
  \caption{
  \textbf{Performance on manipulation tasks in simulation.}
  Success rate (\%) on the two held-out LIBERO-90 splits from
  \citet{tang2026frs} and on RoboTwin~\citep{chen2025robotwin2}.
  We report mean $\pm$ standard error over four seeds. Results marked with $^{\dagger}$ use DSRL+FRS~\citep{tang2026frs}.
  \QFRS{} achieves the highest success rate across all three suites.
  }
  \label{tab:main}

  \begin{tabular}{@{}lccccc@{}}
  \toprule
  \textbf{Method}
  & \multicolumn{2}{c}{\textbf{LIBERO-15T}}
  & \multicolumn{2}{c}{\textbf{LIBERO Hard-10}}
  & \textbf{RoboTwin-6T} \\
  \cmidrule(lr){2-3}
  \cmidrule(lr){4-5}
  \cmidrule(l){6-6}
  
  $\pi_{0.5}$ zero-shot
  & \multicolumn{2}{c}{$8.5 \pm 0.4$}
  & \multicolumn{2}{c}{$1.6 \pm 0.4$}
  & $47.7 \pm 1.6$ \\

  VLM+FRS zero-shot
  & \multicolumn{2}{c}{$21.9 \pm 0.6$}
  & \multicolumn{2}{c}{$2.0 \pm 0.3$}
  & $54.2 \pm 0.2$ \\

  \midrule

  \rowcolor{gray!15}
  \textit{Online RL env.~steps}
  & $250$k & $500$k
  & $250$k & $500$k
  & $500$k \\

  DSRL
  & $37.2 \pm 1.6$
  & $48.3 \pm 3.8$
  & ${7.8 \pm 4.6}^{\dagger}$
  & ${25.2 \pm 3.4}^{\dagger}$
  & $47.2 \pm 4.5$ \\

  Residual RL
  & $22.8 \pm 0.3$
  & $26.7 \pm 0.5$
  & $9.6 \pm 2.1$ 
  & $14.6 \pm 0.8$
  & $0.0 \pm 0.0$ \\

  \rowcolor{tealcol!18}
  \textbf{\QFRS{} (Ours)}
  & $\mathbf{86.0 \pm 0.8}$
  & $\mathbf{95.5 \pm 0.6}$
  & $\mathbf{52.1 \pm 2.7}$
  & $\mathbf{62.1 \pm 4.0}$
  & $\mathbf{70.5 \pm 3.0}$ \\

  \bottomrule
  \end{tabular}
\end{table}
\begin{figure}[t]
\centering
\includegraphics[width=0.95\textwidth]{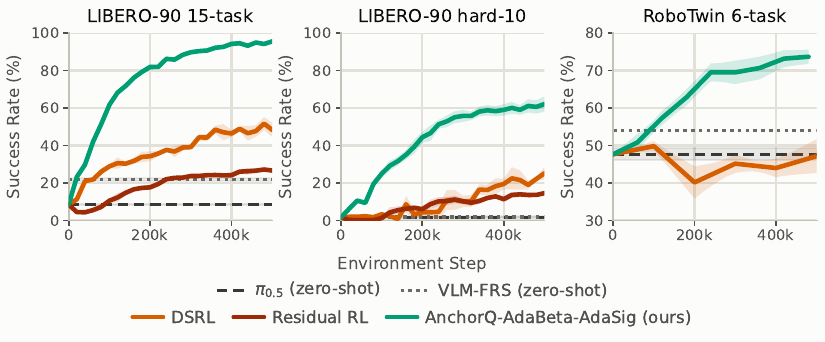}
\caption{\textbf{Success rate on manipulation tasks in simulation}; mean $\pm 1$ standard error over four seeds. \figleft{} 15-task held-out
LIBERO-90 split. \figcenter{} 10-task hard split. \figright{} 6-task bimanual suite in RoboTwin. \QFRS{} consistently outperforms the baselines across all three task suites.}
\label{fig:curves}
\end{figure}

\paragraph{Baselines.} Every method learns steering on the same frozen flow-matching VLA. In simulation we use $\pi_{0.5}$~\citep{black2025pi05}, and on the real robot we use DROID-trained $\pizero$~\citep{black2024pi0} for their out-of-the box robustness. \textbf{DSRL} \citep{wagenmaker2025dsrl} learns a noise-space policy over the continuous space that the flow policy samples from. On LIBERO 10-task split, we compare against a stronger RL variant of \citet{tang2026frs}, which pre-fills DSRL buffer with flow-reversed demonstrations. \textbf{Residual RL} follows
\citet{xiao2025pld} but only trains the residual actor on top of the frozen VLA, since we do not consider any offline learning.
\textbf{VLM-FRS} queries a VLM to propose coarse actions, which are executed via flow reversal steering~\citet{tang2026frs}.

\paragraph{Evaluation.} Each RL run trains for a fixed budget of environment steps and is evaluated periodically on held-out configurations of its task. We report the mean over four seeds and over the tasks of a split at $200K$ and $500K$ environment interactions, with one standard error across seeds in our tables and show the full training progress in our figures.

\subsection{Simulation Performance and Ablation Studies}

We compare performance and sample efficiency on our simulation benchmarks and examine how particle tilting parameters, anchor composition, and the choice of $Q$ function affect learning.

\paragraph{\QFRS{} Outperforms Steering Baselines in Simulation.} As shown in Table~\ref{tab:main} and Figure~\ref{fig:curves}, \QFRS{} outperforms other steering baselines by a wide margin. Especially on the LIBERO-90 15-task suite, it outperforms the closest baseline (DSRL) by at least $2\times$ the success rate at \textit{any} matched trajectory budget (over $250-2200$ trajectories). We attribute this huge improvement to \QFRS{}'s efficient exploration thanks to its effective anchors and the particles that elicit meaningful exploratory behaviors. Similarly, on the hard-10 suite, \QFRS{} also excels at exploring and learning the right task-solving behavior, even if the base policy performance is close to 0. Finally, on the RoboTwin bimanual suite, we demonstrate that \QFRS{} can effectively learn dexterous behaviors that require the coordination of both arms, unaffected by the complex action space induced by the additional embodiment, whereas other baselines seem to struggle to beat the base policy's performance on these tasks. \QFRS{} achieves an average throughput of $294$ completions/hr across all simulation tasks, compared to DSRL's $149$, corresponding to $97\%$ higher throughput.

\begin{figure}[t]
\centering
\begin{minipage}[t]{0.53\linewidth}
    \vspace{0pt}
    \centering
    \setlength{\tabcolsep}{3pt}

    \resizebox{\linewidth}{!}{%
    \begin{tabular}{@{}lcc@{}}
    \toprule
     & \multicolumn{2}{c}{\textbf{Success rate}} \\
    \cmidrule(l){2-3}
    \textbf{Method} & $250$k & $500$k \\
    \midrule

    AdaBeta-AdaSig (\QFRS{})
        & $\mathbf{86.0 \pm 1.0}$
        & $\mathbf{95.5 \pm 0.6}$ \\

    \midrule
    \multicolumn{3}{@{}l@{}}{\textit{Particle tilting}} \\

    \quad AdaBeta, $\sigma=0.3$
        & $80.8 \pm 2.0$
        & $94.2 \pm 0.7$ \\

    \quad $\beta=0.1$, $\sigma=0.3$
        & $68.4 \pm 1.4$
        & $85.6 \pm 0.9$ \\

    \quad No tilting, anchors only
        & $69.8 \pm 1.4$
        & $81.6 \pm 2.4$ \\

    \midrule
    \multicolumn{3}{@{}l@{}}{\textit{Discrete $Q$ (Rainbow-FRS})} \\

    \quad w. \QFRS{} anchors
        & $64.2 \pm 2.1$
        & $79.4 \pm 1.8$ \\

    \quad $+$ half-scale
        & $68.3 \pm 2.4$
        & $79.5 \pm 1.7$ \\

    \quad $+$ rotations
        & $63.4 \pm 1.7$
        & $75.7 \pm 2.9$ \\

    \bottomrule
    \end{tabular}%
    }
    \captionof{table}{
         Success rate after $250$k and
        $500$k environment steps, reported as mean $\pm$ standard error over
        four seeds for ablations on particle tilting, adaptive $\beta$ and $\sigma$,
        anchor design, and continuous versus discrete $Q$-learning on the
        15-task LIBERO-90 split. 
    }
    \label{tab:ablation}
\end{minipage}
\hfill
\begin{minipage}[t]{0.425\linewidth}
    \vspace{0pt}
    \centering
    \includegraphics[
        width=\linewidth
    ]{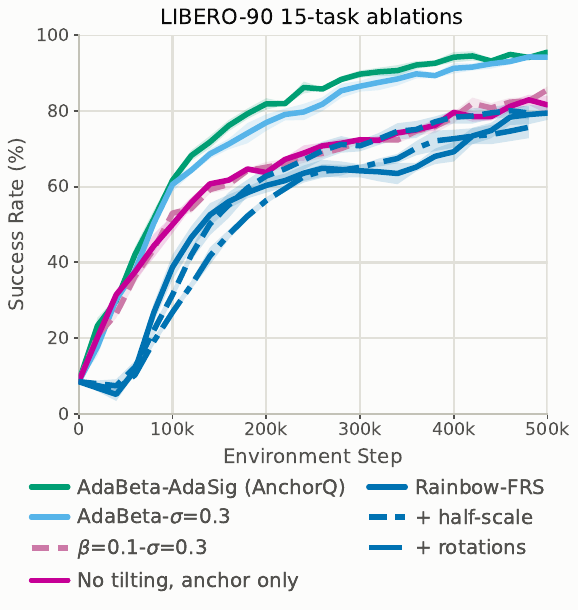}
    \vspace{0pt}
    \caption{Ablation training progress:
         mean $\pm1$ standard error over four seeds.
    }
    \label{fig:ablation}
\end{minipage}
\end{figure}

\paragraph{Ablation studies.}\label{sec:e-ablations} We consider several ablations. Tilting versus no tilting; tilting with an adaptive
$\beta$ and $\sigma$ versus fixed values; the composition of the anchor set; and
learning $Q$ in the continuous noise space versus over the discrete set of anchor
options using RAINBOW~\citep{hessel2018rainbow} (Rainbow-FRS). Figure~\ref{fig:ablation} shows the training progress of each ablation together with the corresponding success rates in Table~\ref{tab:ablation}. Tilting has greatest effect when
$\beta$ is adaptive. An adaptive $\beta$ is
worth $+8.6$ points whereas tilting at a fixed
$\beta\!=\!0.1$ recovers only $4.0$ points over not tilting at all. Adapting $\sigma$ on top adds a further $1.3$ points. Using discrete anchors achieve similar performance as no-tilt but suffers from the beginning of training and does not have the benefit of tilting. Noticeably, all of our ablations still outperforms DSRL, showing the powerful effect of leveraging flow-reversed anchors.
We provide additional results and visualizations on per-task performance in Appendix~\ref{app:additional_numerics}.

\subsection{Analysis on anchors and particle tilting.}

To illustrate how particle tilting helps with dexterous manipulation, we focus on task $81$, which requires precise placement of a book into the front compartment of a caddy and overcoming the base policy's tendency to drop it into the wrong compartment.
On this task, \QFRS{}
achieves an $84\%$ success rate, the anchor-only ablation gets $22\%$, and DSRL gets
$0\%$. Figure~\ref{fig:task81-analysis} (a) measures behavior in robot space: the realized
end-effector motion already deviates from the commanded anchor axis for both anchor-based
arms, and tilting deviates markedly more, giving it the freedom to express fine, off-axis
motions that the six discrete anchors cannot. Figure~\ref{fig:task81-analysis} (b) shows
why this steering is reachable at all. The noise anchors sit at an intermediate
magnitude, between the near-zero noise DSRL collapses to and $\pi_{0.5}$ prior, and \QFRS{} spreads its executed noise across this anchor band, tilting to refine directions that the discrete anchors alone cannot reach. DSRL is a
continuous noise-space actor regularized toward the $\pi_{0.5}$ prior; however, it collapses to a near-zero noise vector and reproduces the base policy's failure. \QFRS{} refines within the anchor-supported region and explores beyond the discrete anchors, achieving the highest success rate.

\begin{figure}[t]
\centering
\includegraphics[width=\textwidth]{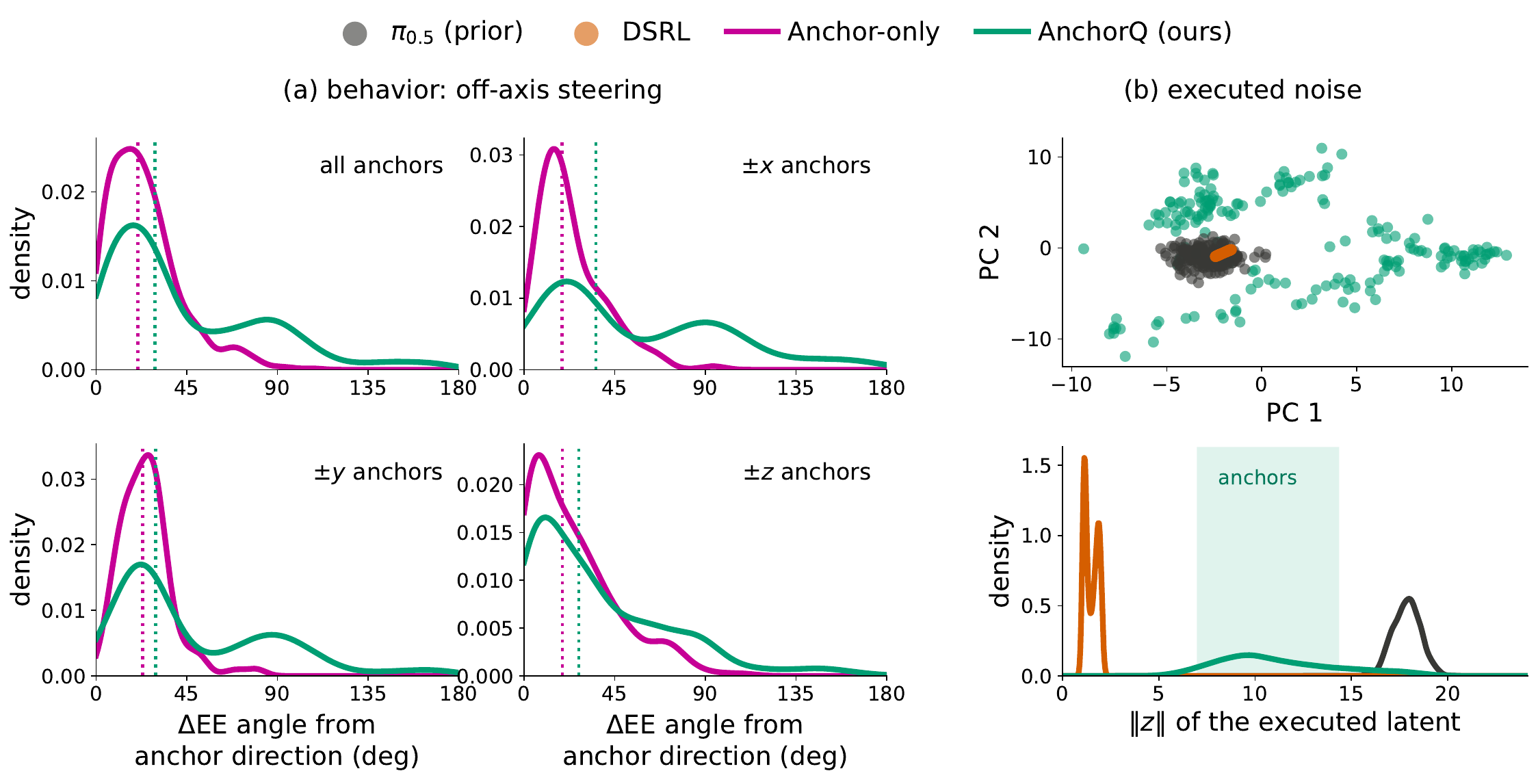}
\caption{\textbf{Anchor \& particle-tilting analysis on LIBERO task 81.} \textbf{(a)} Angle between realized end-effector displacement and anchor direction, grouped by anchor: \QFRS{} shows a heavier right tail on every axis, indicating finer steering beyond the discrete anchor.
\textbf{(b)} Executed noise projected onto the top two pooled principal components (top) and noise magnitude $\|z\|$ (bottom): reverse-flow anchors span the shaded band, \QFRS{} explores across it, DSRL concentrates near zero, and the $\pi_{0.5}$ prior lies above.}
\label{fig:task81-analysis}
\end{figure}

\subsection{Real World Steering Experiments}
\begin{figure}[ht]
\centering
\includegraphics[width=\textwidth]{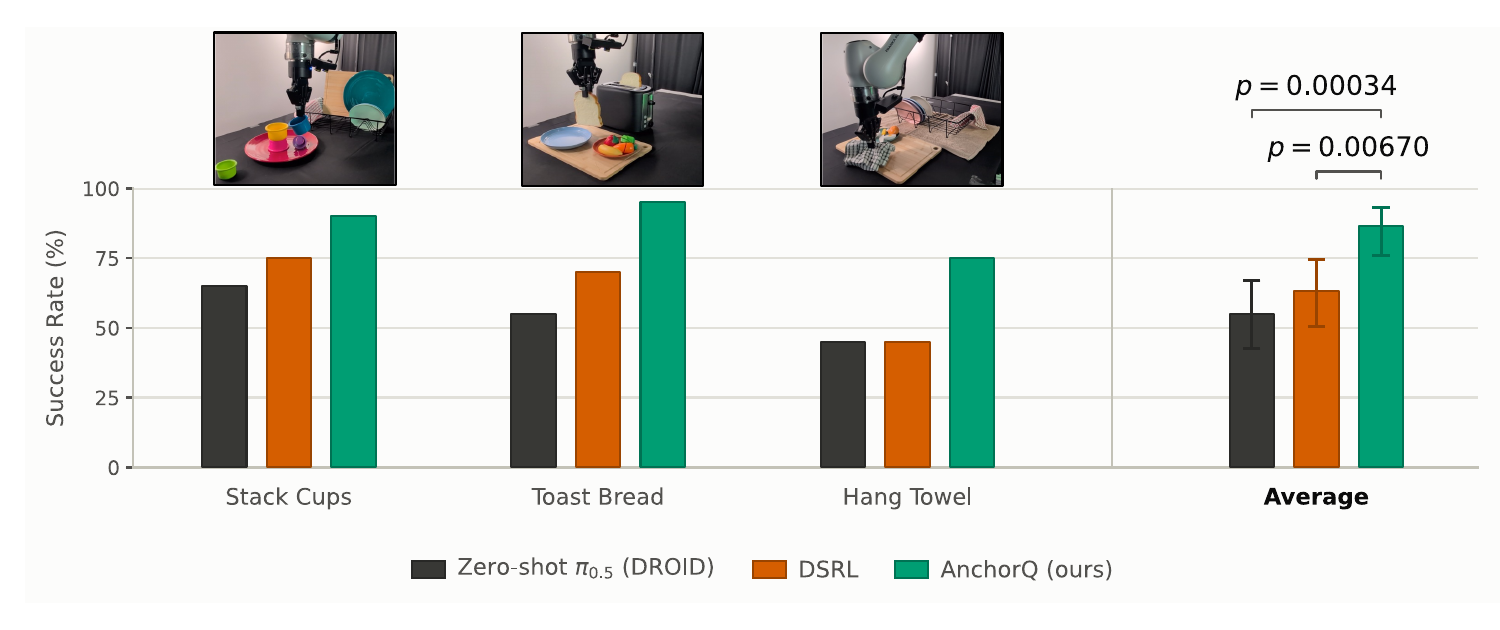}
\caption{
\textbf{Real world steering results.}
We evaluate on three manipulation tasks using a Franka robot with a DROID setup. Each RL method receives 100 online training episodes per task and is evaluated on 20 held out configurations. Error bars on the aggregate success rate show 95\% Wilson score confidence intervals. \QFRS{} achieves higher success rates than the frozen $\pi_{0.5}$ policy and DSRL on all three tasks.
}
\label{fig:realrobot}
\end{figure}
\paragraph{Setup and training.}
We evaluate \QFRS{} on three household manipulation tasks: stacking cups, toasting bread, and hanging a towel. We use a Franka arm with a DROID~\citep{khazatsky2024droid} setup, with observations from a third-person camera and a wrist camera. All methods steer the same frozen DROID pretrained $\pi_{0.5}$ policy. We compare \QFRS{} against the base policy without steering and DSRL. We give both RL methods the same interaction budget of 100 online episodes per task. A human operator observes each rollout, provides a reward at the end of the episode, and resets the scene to one of a predefined set of initial configurations. These training configurations are separate from the configurations used for evaluation. RL training runs asynchronously with data collection, updating the steering parameters while keeping the base policy frozen.

\paragraph{\QFRS{} Outperforms Steering Baselines on a Real Robot.}
We evaluate each method on 20 held out configurations per task and report individual and average task success rates, with 95\% Wilson score confidence intervals~\citep{wilson1927probable} for the average. We show statistical significance using two sided $t$-tests comparing \QFRS{} with each baseline, testing the null hypothesis of equal mean success. 
As shown in Figure~\ref{fig:realrobot}, \QFRS{} outperforms both $\pi_{0.5}$ and DSRL on all three tasks under the same online interaction budget. On the towel hanging task, it is the only method that improves over the base policy. These results demonstrate that the our approach can significantly improve manipulation performance on a real robot under limited environment interaction.

\section{Conclusion}
\label{sec:conclusion}
We introduce \QFRS{}, a coarse-to-fine RL method that adapts frozen flow-matching policies through a structured and refinable action representation. Flow reversal maps reference behaviors into state-conditioned latent anchors, while critic-guided selection and local particle tilting enable fine-grained refinement. Our analysis shows how anchoring reduces identification complexity, how refinement trades support bias against action-class complexity, and when value tilting remains reliable under critic error. Across simulated and real-world manipulation tasks, \QFRS{} improves sample efficiency over post-training baselines in challenging manipulation tasks across simulation and real world benchmarks. 
These findings highlight action representation as a key design choice for efficient adaptation of frozen policies.
  
\paragraph{Future work.} \QFRS{} requires additional memory from scoring particles at every decision. Distilling the adapted behavior into a policy could reduce the need for our sampling-based inference. Extending the current reference anchors to more complex embodiments and moving from task-specific to multitask post-training are also exciting future directions.

\newpage
\bibliography{refs}
\bibliographystyle{plainnat}

\newpage
\appendix

\startcontents[appendix]

\section*{Appendix Contents}
\printcontents[appendix]{}{1}{}

\section{Extended Related Work}
\label{sec:related_work}
\paragraph{Generalist and generative robot policies.}
Policies trained on heterogeneous robot data generalize across scenes and
instructions, from action-token transformers \citep{brohan2022rt1,brohan2023rt2}
to open checkpoints \citep{ghosh2024octo,kim2024openvla}. The current generation
couples a vision-language backbone to a flow-matching action expert
\citep{black2024pi0,black2025pi05,nvidia2025groot,geminirobotics2025}, inheriting
a generative toolkit that includes action chunking \citep{zhao2023act}, diffusion
policies \citep{chi2024diffusionpolicy}, and flow matching with rectified flow
\citep{lipman2023flowmatching,liu2023rectifiedflow}. We keep such a decoder frozen
throughout and improve its behavior only through the latent that initializes
action generation.

\paragraph{RL post-training of generative policies.}
A complementary line of work updates the generative policy itself from reward.
Q-score matching trains a diffusion policy from a learned value function
\citep{psenka2024qsm}, while DPPO applies policy optimization to diffusion-policy
parameters \citep{ren2024dppo}. Flow-based counterparts include Flow
$Q$-Learning \citep{park2025fql} and ReinFlow, which fine-tunes a flow-matching
policy with online RL \citep{zhang2025reinflow}. Flow Policy Optimization brings
flow matching into an on-policy, PPO-style policy-gradient framework without
requiring exact action likelihoods \citep{mcallister2025fmpg}. These methods retain a rich
continuous action distribution and can alter the behaviors represented by the
pretrained generator, but require updating or distilling the generative policy.
Our setting instead treats the pretrained flow model as a fixed behavioral
prior: interaction data trains only the critic, and improvement occurs by
choosing and refining the decoder's input latent. This distinction is important
for large generalist policies, where preserving pretrained behavior and avoiding
backpropagation through the policy are primary design constraints.

Between direct policy fine-tuning and frozen-policy steering, PLD \citep{xiao2025pld} separates RL-based adaptation from updating the generalist. It first trains lightweight residual actors with a frozen VLA, then collects recovery trajectories aligned with the base policy's state distribution and distills them into the VLA through supervised fine-tuning. This shares our use of a frozen behavioral prior during online RL, but ultimately transfers the improvements into the generator's parameters. We instead retain the frozen generator and express adaptation through its input noise, connecting our approach to the steering methods discussed next.

\paragraph{Steering a frozen policy in its noise space.}
Closest are methods that fix the policy and search its latent input, reranking
chunks by value \citep{nakamoto2024vgps} or steering toward human intent
\citep{wang2025itps}. DSRL \citep{wagenmaker2025dsrl} makes the noise the RL action
and trains a continuous-control agent \citep{haarnoja2018sac} in it; recent
variants add trust-region perturbations \citep{simsir2026lpds}, critic-guided
denoising \citep{ruan2026qpilots}, compositional steering
\citep{tan2026rl2vla}, and retrieval priors \citep{zhao2026retrievesteer}. FRS
\citep{tang2026frs} contributes our inversion primitive, recovering the noise
that elicits a reference motion and using reversed human demonstrations as
noise-space training data. Its RL variant prefills DSRL's replay buffer and adds
a behavior-cloning term, after which the actor still searches the continuous
noise space.

UniSteer \citep{lu2026unisteer} independently combines a frozen flow decoder,
noise-space RL, and action-to-noise inversion. Its reference actions are online
human corrections, whereas ours are a fixed library requiring no human in the
loop; moreover, it trains a continuous noise actor, while \QFRS{} learns only a
critic over state-conditioned anchors and their local neighborhoods. Thus our
use of inversion is persistent and representational: it defines the structured
action class used throughout online learning rather than only providing data or
guidance for a continuous actor.

More generally, latent-steering methods differ in whether the learned object is
an actor, a denoising-time guidance signal, or a selector over proposals. LP-DS
learns constrained perturbations in the latent space
\citep{simsir2026lpds}, while Q-Pilots injects critic guidance into the
generative process \citep{ruan2026qpilots}. Residual flow steering (RFS) jointly
optimizes latent-noise modulation for global adaptation and residual actions
for local corrections \citep{su2026rfs}. Compositional and retrieval-based
approaches instead construct guidance from reusable policy components or prior
experience \citep{tan2026rl2vla,zhao2026retrievesteer}. \QFRS{} occupies a
different point in this design space: inversion first supplies a small set of
state-conditioned semantic coordinates, and online value learning operates over
that structured set before exposing continuous local refinements. Consequently,
the critic need not discover useful modes and their latent locations
simultaneously.

\paragraph{Reduced and coarse-to-fine action spaces for RL.}
Shrinking the effective action space can aid exploration through discretization
\citep{tang2020discretizing}, per-dimension branching
\citep{tavakoli2018branching}, clustered or vector-quantized primitives
\citep{shafiullah2022bet,lee2024vqbet}, and temporal abstraction over options
\citep{sutton1999options} or chunks \citep{li2025qchunking}. CQN and its
action-sequence extension recursively refine intervals of normalized
per-dimension action coordinates \citep{seo2024cqn,seo2024cqnas}. Our hierarchy
instead operates in a learned generative latent reached by an inversion
operator, related to inversion methods used for generative editing
\citep{song2021ddim,rout2024rfinversion}.

RLMM-Flow \citep{wang2026rlmmflow} also uses the phrase ``coarse-to-fine latent
steering,'' but its hierarchy varies latent dimensionality: the coarse level
pools a continuous latent along the horizon, and the fine level relaxes it to a
full-dimensional residual. It is offline and does not use inversion. In
contrast, our coarse level is a finite set of inverted reference motions, and
our fine level is an online, $Q$-weighted local particle refinement around the
selected motion.

\paragraph{Uncertainty-aware and regularized value selection.}
Entropy- and KL-regularized RL trade value improvement against deviation from a
reference distribution. This principle underlies relative-entropy policy
improvement methods such as MPO \citep{abdolmaleki2018mpo}, maximum-entropy
control \citep{haarnoja2018sac}, and the general theory of regularized MDPs
\citep{geist2019regularizedmdp}. SAC can also tune a global entropy temperature
through a constrained objective \citep{haarnoja2018sac}. Our temperature has
a different role: it is state-dependent and determined by epistemic uncertainty
rather than by a target policy entropy.

Randomized value functions and critic ensembles provide practical surrogates
for epistemic uncertainty. Bootstrapped DQN uses randomized value functions for
deep exploration \citep{osband2016bootstrappeddqn}; randomized prior functions
preserve diversity away from observed data \citep{osband2018randomized}; and
SUNRISE uses ensemble uncertainty in both Bellman backups and upper-confidence
action selection \citep{lee2021sunrise}. Our value-tilting rule combines this
uncertainty machinery with KL regularization at the
level of a realized local particle cloud: exponential weights are the exact
optimizer of a KL-regularized value objective relative to the untilted cloud.
Rather than adding ensemble uncertainty as an exploration bonus, we use it to
set the state-dependent inverse temperature
$\betatilt(s)=1/u_{\mathrm{epi}}(s)$. The selector therefore remains close to
the reference weighting where candidate values are poorly known and becomes
sharper as epistemic uncertainty contracts. This provides an uncertainty-aware
mechanism for deciding how strongly to exploit local refinements while
preserving the support supplied by the anchors and particles.

\section{Algorithm Details}
\label{sec:algorithm}
We provide the \QFRS{} training procedure and the implementation choices used
in our experiments, including the critic architecture and benchmark-specific
hyperparameters.

\subsection{\QFRS{} Pseudocode}
Algorithm~\ref{alg:frsq} details the experience collection and critic updates
summarized in Algorithm~\ref{alg:anchorq-revised}.
Algorithms~\ref{alg:frsq-local-refine} and~\ref{alg:adaptive-scale} specify
local value-tilted refinement and adaptive particle reach, respectively.
The accompanying text defines replay conventions and the uncertainty estimates
used for adaptation.

\paragraph{Execution and replay.}
An outer decision executes at most $n_{\mathrm{exec}}$ low-level actions before replanning, with $1\leq n_{\mathrm{exec}}\leq h_a$; $\gamma$ is the discount per outer decision. After the critic warmup $W_Q$, each collected transition triggers $U$ minibatch gradient updates. Replay stores $(s,z,r,s',d)$, where $z$ is the executed latent, $r$ is the reward for the decision, $s'$ is the next observation, and $d$ indicates true termination. Each observation carries its cached candidate bank, including the defer latent when enabled. The frozen VLA allows these banks to be reused across critic updates, avoiding repeated flow reversal. At a time-limit truncation, the final observation and its bank are retained for bootstrapping before the environment is reset.

\paragraph{Local refinement and action selection.}
For each newly collected observation, flow reversal constructs the noise anchors from the fixed reference commands, and the defer candidate is sampled from the VLA's base noise prior. This bank is cached with the observation. Each refinement call draws a fresh local cloud containing the candidate itself and $M$ Gaussian perturbations; particle clouds and their weights are recomputed on demand. Exponential weights combine the particles into a refined latent. With probability $\epsilon_t$, a candidate is selected uniformly; otherwise, the online ensemble mean selects the highest-valued refined latent. The selected latent is decoded by the frozen VLA and stored as the action in replay.

\begin{algorithm}[H]
\caption{\QFRS{}: online anchors, local value tilting, and $U$ critic updates per decision}
\label{alg:frsq}
\small
\begin{algorithmic}[1]

\Require frozen $\pi_{\mathrm{flow}}$, $\reverseflow$, reference map $\varphi$;
  $\numcrit$ critics $\{Q_{\psi,i}\}$; prioritized replay $\Dc$
\Require execution length $n_{\mathrm{exec}}$, updates $U$, critic warmup $W_Q$;
  $\gamma,\eta,\tau,N_{\mathrm{targ}}$; exploration schedule $\epsilon_t$
\Require $M$ particles, particle warmup $W_P$, scale schedule $\sigma_p^{(t)}$;
  optional adaptive scale with $0\leq\sigma_{\min}\leq\sigma_{\max}$
\Require fixed $\beta$ or adaptive tilt with uncertainty floor $u_{\min}>0$

\State Initialize $\bar\psi\leftarrow\psi$ and reset the environment to $s_0$

\For{each outer decision $t=0,1,\ldots$}

  \State Retrieve cached $\bank(s_t)$; if absent, construct and cache it
    by flow reversal of $\varphi(s_t,c)$ for each reference command
  \Statex \hspace{\algorithmicindent}
    When constructing a bank, include a base-prior candidate if defer is enabled

  \If{$M=0$ or $t<W_P$}
    \State $\sigma_t\leftarrow 0$
  \ElsIf{adaptive scale is enabled}
    \State $\sigma_t\leftarrow\Call{AdaptiveScale}{s_t,\bank(s_t)}$
  \Else
    \State $\sigma_t\leftarrow\sigma_p^{(t)}$
  \EndIf

  \For{each candidate $c\in\bank(s_t)$}
    \State $\tilde z_{t,c}\leftarrow\Call{LocalRefine}{s_t,z_c(s_t),t,\sigma_t}$
  \EndFor

  \State Choose $c_t$ uniformly w.p.~$\epsilon_t$; otherwise
    $c_t\leftarrow
    \arg\max_{c\in\bank(s_t)}
    \bar Q_\psi(s_t,\tilde z_{t,c})$

  \State $\tilde z_t\leftarrow\tilde z_{t,c_t}$

  \State Decode $a_t\leftarrow\pi_{\mathrm{flow}}(s_t,\tilde z_t)$;
    execute its first $n_{\mathrm{exec}}$ actions, stopping early if the episode ends

  \State Observe outer reward $r_t$, final observation $s_{t+1}$,
    and terminal flag $d_t$

  \State Construct and cache $\bank(s_{t+1})$
    with the same anchor/defer construction

  \State Store $(s_t,\tilde z_t,r_t,s_{t+1},d_t)$ in $\Dc$
    \Comment{observations include their cached banks}

  \If{$t+1\geq W_Q$ and $\Dc$ can supply a minibatch}

    \For{$u=1,\ldots,U$}

      \State Sample
        $\Bc=\{(s_b,z_b,r_b,s'_b,d_b)\}$
        with IS weights $\omega_b$

      \State Draw target-head subset $\mathcal I$,
        $|\mathcal I|=N_{\mathrm{targ}}$

      \For{each replay row $b$ \textnormal{(batched)}}

        \State $y_b\leftarrow r_b$

        \If{$d_b=0$}

          \State Retrieve cached $\bank(s'_b)$ from replay

          \If{$M=0$ or $t<W_P$}
            \State $\sigma'_b\leftarrow 0$
          \ElsIf{adaptive scale is enabled}
            \State $\sigma'_b\leftarrow
              \Call{AdaptiveScale}{s'_b,\bank(s'_b)}$
          \Else
            \State $\sigma'_b\leftarrow\sigma_p^{(t)}$
          \EndIf

          \For{each candidate $c\in\bank(s'_b)$}
            \State $\tilde z'_{b,c}\leftarrow\Call{LocalRefine}{s'_b,z_c(s'_b),t,\sigma'_b}$
          \EndFor

          \State $c'_b\leftarrow
            \arg\max_{c\in\bank(s'_b)}
            \bar Q_\psi(s'_b,\tilde z'_{b,c})$

          \State $\tilde z'_b\leftarrow\tilde z'_{b,c'_b}$

          \State $y_b\leftarrow
            r_b+\gamma\min_{i\in\mathcal I}Q_{\bar\psi,i}(s'_b,\tilde z'_b)$

        \EndIf

      \EndFor

      \State $\mathcal L_Q\leftarrow
        (\numcrit|\Bc|)^{-1}
        \sum_{b\in\Bc}\sum_{i=1}^{\numcrit}
        \omega_b
        \bigl(
          Q_{\psi,i}(s_b,z_b)
          -\operatorname{stopgrad}(y_b)
        \bigr)^2$

      \State Take an optimizer step on $\psi$ with learning rate $\eta$

      \State Update replay priorities from absolute TD errors

      \State $\bar\psi\leftarrow(1-\tau)\bar\psi+\tau\psi$
        \Comment{frozen VLA and any random priors remain fixed}

    \EndFor

  \EndIf

  \State If the episode ended, reset before the next decision

\EndFor

\end{algorithmic}
\end{algorithm}

\paragraph{Critic ensemble and uncertainty signal.}
Write $Q_{\psi,i}$ for critic head $i$ and $\bar Q_\psi(s,z)=N^{-1}\sum_{i=1}^{N}Q_{\psi,i}(s,z)$ for the online ensemble mean used for particle weighting and selection among refined latents. Uncertainty at a candidate latent is measured by the standard deviation across heads. Its median over a particle cloud provides the local uncertainty estimate used for adaptive tilting.

\begin{algorithm}[ht]
\caption{Local refinement subroutine for Algorithm~\ref{alg:frsq}}
\label{alg:frsq-local-refine}
\small
\begin{algorithmic}[1]
\Function{LocalRefine}{$s,z_c,t,\sigma$}
  \If{$M=0$ or $t<W_P$ or $\sigma=0$}
    \State \Return $z_c$
    \Comment{exact-anchor mode}
  \EndIf

  \State Set $z_0\leftarrow z_c$; draw
    $\xi_k\sim\mathcal N(0,I)$ and set
    $z_k\leftarrow z_c+\sigma\xi_k$ for $k=1,\ldots,M$
    \Comment{fresh local cloud}

  \State Evaluate the critic ensemble on the cloud and set
    $q_k\leftarrow\bar Q_\psi(s,z_k)$ for $k=0,\ldots,M$

  \If{adaptive tilt is enabled}
    \State $\widehat u_{\mathrm{epi}}
      \leftarrow
      \operatorname{median}_{0\leq k\leq M}
      \operatorname{Std}_{i=1,\ldots,\numcrit}
      Q_{\psi,i}(s,z_k)$
    \State $\beta_{\mathrm{loc}}
      \leftarrow
      1/\max\{\widehat u_{\mathrm{epi}},u_{\min}\}$
  \Else
    \State $\beta_{\mathrm{loc}}\leftarrow\beta$
  \EndIf

  \State $w_k\leftarrow
    \exp\!\bigl(\beta_{\mathrm{loc}}(q_k-\max_j q_j)\bigr)
    \Big/
    \sum_{\ell=0}^{M}
    \exp\!\bigl(\beta_{\mathrm{loc}}(q_\ell-\max_j q_j)\bigr)$
    for $k=0,\ldots,M$

  \State \Return $\sum_{k=0}^{M}w_k z_k$
    \Comment{barycentric refinement}
\EndFunction
\end{algorithmic}
\end{algorithm}

\begin{algorithm}[ht]
\caption{Adaptive particle scale}
\label{alg:adaptive-scale}
\small
\begin{algorithmic}[1]

\Function{AdaptiveScale}{$s,\bank(s)$}

  \For{each candidate $c\in\bank(s)$}
    \State $z^{\mathrm{probe}}_{c,0}\leftarrow z_c(s)$
    \State Draw $\xi^{\mathrm{probe}}_{c,k}\sim\mathcal N(0,I)$ and set
      $z^{\mathrm{probe}}_{c,k}
      \leftarrow
      z_c(s)+\sigma_{\max}\xi^{\mathrm{probe}}_{c,k}$,
      $k=1,\ldots,M$
  \EndFor

  \State $u_{\mathrm{probe}}\leftarrow
    \operatorname{median}_{c,\,0\leq k\leq M}
    \operatorname{Std}_{i=1,\ldots,\numcrit}
    Q_{\psi,i}(s,z^{\mathrm{probe}}_{c,k})$

  \State \Return
    $\sigma_{\min}
    +(\sigma_{\max}-\sigma_{\min})
    \operatorname{clip}(u_{\mathrm{probe}},0,1)$

\EndFunction

\end{algorithmic}
\end{algorithm}

\paragraph{Adaptive particle reach.}
Particles are disabled before $W_P$. With adaptive
scale disabled, $\sigma_p^{(t)}$  follows the configured schedule---e.g., $\sigma_p^{(t)}\equiv \sigma_p$ or a ramp to $\sigma_p$
over $W_\sigma$ decisions after warmup (an immediate switch if $W_\sigma=0$).
With adaptive scale enabled, the post-warmup scale is instead chosen from
$[\sigma_{\min},\sigma_{\max}]$ using a separate critic-only probe cloud
around the current anchors at scale $\sigma_{\max}$:
\[
 \begin{aligned}
 u_{\mathrm{probe}}(s)
 &:=\operatorname{median}_{c,k}\operatorname{Std}_{i=1,\ldots,\numcrit}
       Q_{\psi,i}(s,z^{\mathrm{probe}}_{c,k}),\\
 \sigma(s)&:=\sigma_{\min}+(\sigma_{\max}-\sigma_{\min})
       \operatorname{clip}(u_{\mathrm{probe}}(s),0,1).
 \end{aligned}
\]
The clipping is fixed, not a running normalization: larger probe disagreement
widens the search, while smaller disagreement contracts it. The probe requires
no VLA decoding and serves only to set the scale; its particles do not compete
for execution. A fresh refinement cloud is then drawn around each
anchor at $\sigma(s)$. This optional reach heuristic is separate from the
fixed-cloud adaptive-tilting guarantee.

\paragraph{Adaptive tilting intensity.}
For each anchor's realized refinement cloud, set
\[
 \widehat u_{\mathrm{epi}}(s)
 :=\operatorname{median}_{0\leq k\leq M}
       \operatorname{Std}_{i=1,\ldots,\numcrit}Q_{\psi,i}(s,z_k),
 \qquad
 \beta_{\mathrm{loc}}(s)
 :=\frac{1}{\max\{\widehat u_{\mathrm{epi}}(s),u_{\min}\}}.
\]
This instantiates~\eqref{eq:adaptive-tilt-beta}; when adaptation is
disabled, the routine instead uses the fixed intensity $\beta=1/T$.
The adaptive weights compare value gaps in units of the floored uncertainty
scale: $w_j/w_k=\exp((q_j-q_k)/u_{\mathrm{epi}}(s))$, where
$q_k=\bar Q_\psi(s,z_k)$ and
$u_{\mathrm{epi}}(s)=\max\{\widehat u_{\mathrm{epi}}(s),u_{\min}\}$.
Large uncertainty therefore softens the weighting, while small uncertainty
permits stronger preference for high-scoring particles. The floor limits
$\beta_{\mathrm{loc}}$ to $1/u_{\min}$, preventing a near-zero disagreement
estimate from producing an arbitrarily sharp tilt.

Theorem~\ref{thm:adaptive-local-tilting} interprets these weights as an
uncertainty-weighted KL-regularized improvement rule. Its reliability
requires the calibration condition in
Assumption~\ref{ass:adaptive-epistemic-calibration}: local critic error
must be controlled by the uncertainty scale. Neither a small ensemble
standard deviation nor the use of random priors establishes that condition
automatically. Corollary~\ref{cor:adaptive-barycentric-refinement} additionally
accounts for executing the weighted latent, with a separate interpolation
penalty under local smoothness. Adaptive reach and adaptive tilt can be
enabled independently. Reach uses disagreement over the probe cloud, while
tilt uses disagreement recomputed on each anchor's refinement cloud.
Their realized values need
not change monotonically during learning, since the critic and clouds also
change.

\paragraph{Critic updates and next-state refinement.}
For each nonterminal replay row, we retrieve the cached next-state bank and refine all candidates with fresh particles, then select the highest-valued refined latent using the online ensemble mean. Adaptive reach, tilt, and particle weights are recomputed using the current online ensemble at the next state. At each update, we sample a subset $\mathcal I$ of $N_{\mathrm{targ}}$ target heads and form the detached target $y=r+\gamma\min_{i\in\mathcal I}Q_{\bar\psi,i}(s',\tilde z')$, where $\tilde z'$ is the selected refined latent; true terminal rows use $y=r$. Prioritized replay supplies importance-sampling weights for the squared TD loss. Each optimizer step is followed by a replay-priority refresh and a Polyak target update; the VLA and random prior networks remain frozen.

\subsection{Implementation Details and Hyperparameters}
Following FRS~\citep{tang2026frs}, we use six cardinal Cartesian reference motions and a \emph{defer} candidate sampled from the VLA's base noise prior. For bimanual RoboTwin tasks, the Cartesian product of the six directions for each arm yields $36$ reference motions; adding defer gives $37$ candidates. In simulation, \QFRS{} and DSRL use REDQ's ensemble update scheme~\citep{chen2021redq}, with multiple critic updates per decision and a minimum over randomly sampled target heads. Real-world training uses a three-step TD target (multi-step Q-learning) in place of the one-step target in Algorithm~\ref{alg:frsq}, and averages the sampled target heads. Table~\ref{tab:hyperparams} lists the settings for each benchmark.

\paragraph{Critic architecture.}
The critic evaluates $Q(s,z)$ from state features and the full latent noise trajectory. On LIBERO-90, two frozen $384$-dimensional DINOv2-S~\citep{oquab2023dinov2} embeddings from the third-person and wrist cameras form a $768$-dimensional visual input, alongside an $8$-dimensional proprioceptive vector. Separate linear layers with ReLU project the visual and proprioceptive inputs to $256$ dimensions each. Their concatenation passes through a $512\rightarrow256\rightarrow256$ state encoder. The $10\times32$ noise trajectory is flattened and encoded by a $320\rightarrow256\rightarrow256$ MLP. Input dimensions follow each benchmark's observation and latent spaces.

With randomized prior functions enabled, each of the $N=10$ heads has independent trainable featurizers and encoders. An output MLP maps the concatenated state and noise features through widths $512\rightarrow256\rightarrow256\rightarrow1$. Hidden layers use ReLU, without layer normalization or dropout. Each head adds an independently initialized, frozen prior network: $Q_{\psi,i}(s,z)=f_{\psi_i}(s,z)+\kappa p_i(s,z)$, with prior scale $\kappa$ given in Table~\ref{tab:hyperparams}. Disagreement across heads supplies the uncertainty signal for local adaptation.

For our DROID policy~\citep{khazatsky2024droid}, we stack four frames of proprioceptive and DINO features and use shared trainable encoders with independent scalar output heads, without random priors. The critic applies layer normalization to the flattened $15\times32$ noise input; execution uses the original latent.

\begin{table*}[ht]
\caption{\textbf{Hyperparameters by benchmark.}
Settings for LIBERO-90~\citep{liu2023libero}, RoboTwin~\citep{chen2025robotwin2}, and the real-world experiments.}
\label{tab:hyperparams}
\centering
\small
\setlength{\tabcolsep}{3pt}
\resizebox{\textwidth}{!}{
\begin{tabular}{@{}llcccl@{}}
\toprule
\textbf{Symbol} & \textbf{Short Key} & \textbf{LIBERO-90} & \textbf{RoboTwin} & \textbf{Real world} & \textbf{Meaning} \\
\midrule
\multicolumn{6}{@{}l@{}}{\textit{Environment and action space}}\\
$n_{\mathrm{exec}}$ & \texttt{pi0\_replan\_steps} & $10$ & $30$ & $10$ & substeps per outer decision \\
$K$ & \texttt{pi0\_flow\_steps} & $10$ & $10$ & $10$ & Euler steps in the flow ODE \\
$C$ & \texttt{num\_anchors} & $7$ & $37$ & $7$ & reference candidates including defer \\
$(H,D)$ & \texttt{noise\_shape} & $(10,32)$ & $(50,32)$ & $(15,32)$ & noise trajectory shape \\
\midrule
\multicolumn{6}{@{}l@{}}{\textit{Critic, replay, and exploration}}\\
$\gamma$ & \texttt{gamma} & $0.999^{10}$ & $0.999^{30}$ & $0.99$ & per-outer-decision discount \\
$\eta$ & \texttt{learning\_rate} & $3\!\times\!10^{-4}$ & $3\!\times\!10^{-4}$ & $3\!\times\!10^{-4}$ & Adam step size \\
 $B$ & \texttt{batch\_size} & $16$ & $16$ & $64$ & transitions per gradient step \\
$N$ & \texttt{num\_critics} & $10$ & $10$ & $10$ & critic heads \\
$N_{\mathrm{targ}}$ & \texttt{n\_critics\_to\_sample} & $2$ & $2$ & $2$ & random target subset size \\
$\tau$ & \texttt{tau} & $0.005$ & $0.005$ & $0.005$ & Polyak rate \\
$U$ & \texttt{num\_updates\_per\_train\_step} & $10$ & $10$ & $5$ & gradient steps / decision \\
$W_Q$ & \texttt{learning\_starts} & $1000$ & $1000$ & $0$ & critic warmup (outer decisions) \\
$\epsilon$ & \texttt{eps\_start} $\to$ \texttt{eps\_end} & $1.0 \to 0.05$ & $1.0 \to 0.05$ & $0.6 \to 0.05$ & exploration schedule \\
 & \texttt{eps\_decay\_steps} & $10{,}000$ & $5000$ & $2000$ & decisions for exploration decay \\
$\kappa$ & \texttt{rpf\_prior\_scale} & $2$ & $20$ & $2$ & random-prior scale \\
\midrule
\multicolumn{6}{@{}l@{}}{\textit{Particle tilting}}\\
$M$ & \texttt{particle\_count} & $32$ & $10$ & $8$ & sampled particles per anchor \\
$\sigma_p$ & \texttt{particle\_sigma} & $0.5$ & $0.3$ & $0.0$ & particle scale \\
$\beta$ & \texttt{tilt\_beta} & $1.0$ & $0.1$ & $1.0$ & fixed tilting strength \\
$W_P$ & \texttt{particle\_warmup} & $10{,}000$ & $2000$ & $2000$ & particle warmup (outer decisions) \\
$(\sigma_{\min},\sigma_{\max})$ & \texttt{sigma\_min/\_max} & $(0.05,0.3)$ & $(0.05,0.3)$ & $(0.05,0.3)$ & adaptive particle-scale range \\
\bottomrule
\end{tabular}
}
\end{table*}

\section{Additional Experiment Results}
\label{app:additional_numerics}
\begin{figure}[ht]
\centering
\includegraphics[width=0.85\textwidth]{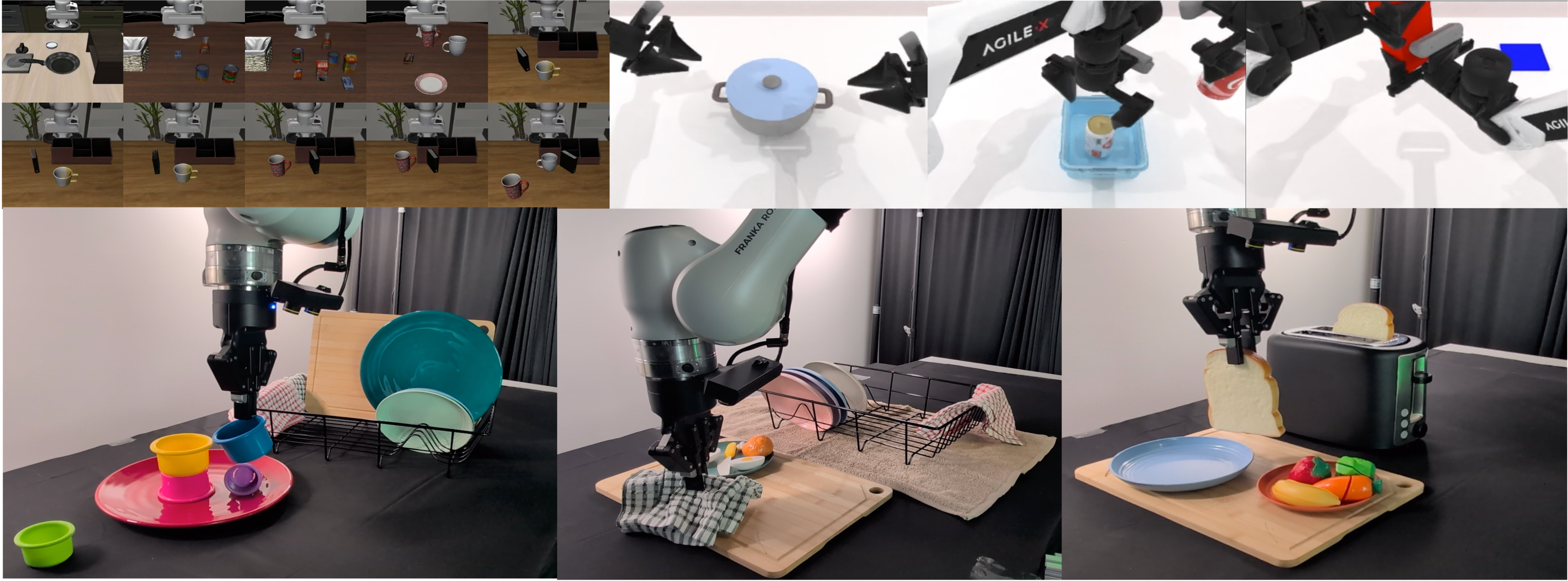}
\caption{\textbf{Overview of the experimental tasks.}
Representative tasks from the LIBERO simulation benchmark, the bimanual
RoboTwin benchmark, and our real-world Franka setup.}
\label{fig:experiment-task-collage}
\end{figure}

We supplement the experimental results in the main paper with per-task
analyses of the simulation benchmarks and, in
Figure~\ref{fig:experiment-task-collage}, a visual overview of the tasks from
LIBERO, RoboTwin, and the real-world robot experiments. Across these settings,
\QFRS{} consistently outperforms the steering baselines and improves the
frozen base policy, including on bimanual and real-world manipulation tasks.
The following sections report the simulation results by task and examine the
LIBERO tasks where particle tilting most clearly separates from anchor-only
learning and baseline RL algorithms.

\subsection{Per-Task Results}
\label{sec:appendix-per-task}

Figures~\ref{fig:libero15-per-task}--\ref{fig:robotwin-per-task} show the
training progress for each task in the three simulation suites.
On the LIBERO-90 15-task suite, the adaptive variants learn reliably across
tasks, while fixed and no-tilting variants are less consistent.
On the hard-10 suite, adaptive tilting succeeds on most tasks, although the
hardest cases remain challenging for every method.
On RoboTwin, at least one particle-tilting variant improves over the base
policy on every task, with particularly clear gains on handover and
object-placement tasks.

\begin{figure*}[ht]
\centering
\includegraphics[width=\textwidth]{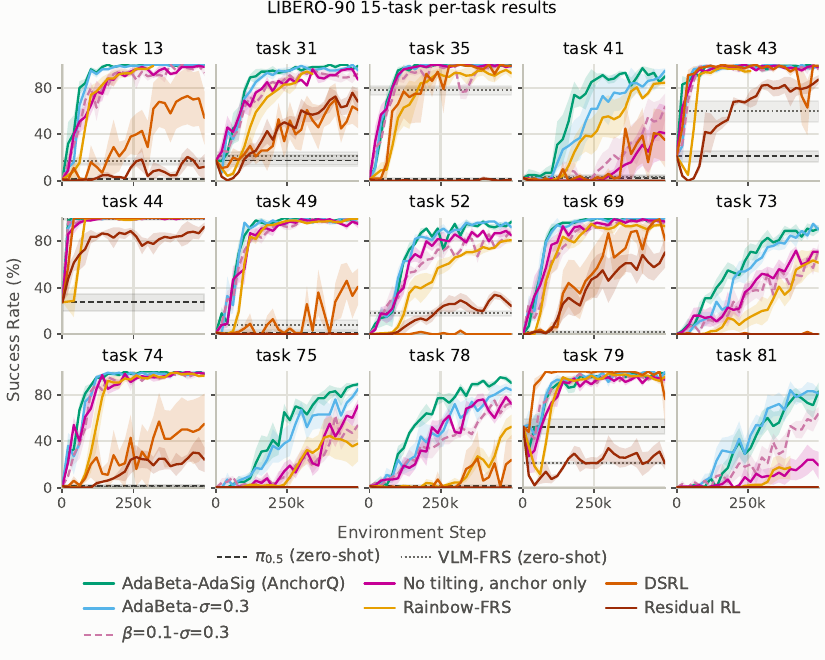}
\caption{\textbf{Per-task performance on the LIBERO-90 15-task suite.}
Success rate over training for the particle-tilting ablations and steering
baselines. Curves show the mean and shaded regions show $\pm1$ standard error
across seeds.}
\label{fig:libero15-per-task}
\end{figure*}

\begin{figure*}[ht]
\centering
\includegraphics[width=\textwidth]{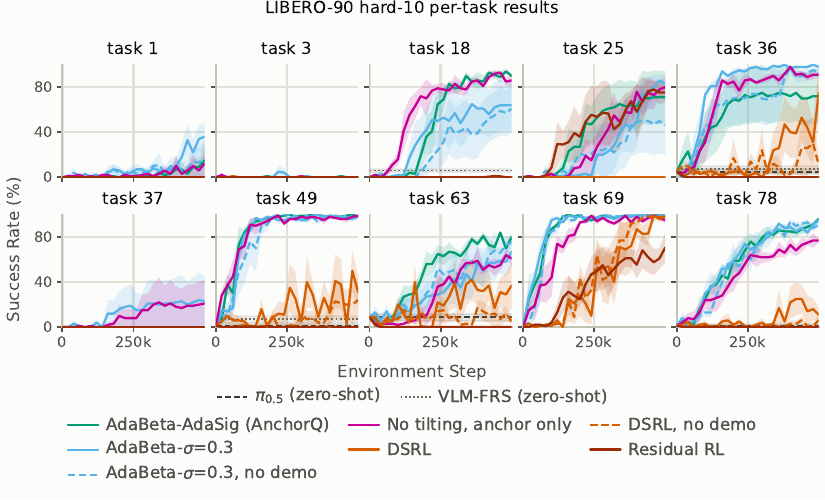}
\caption{\textbf{Per-task performance on the LIBERO-90 hard-10 suite.}
Success rate over training for the adaptive-tilting, anchor-only, DSRL, and
Residual RL variants. Curves show the mean and shaded regions show $\pm1$
standard error across seeds.}
\label{fig:hard10-per-task}
\end{figure*}

\begin{figure*}[ht]
\centering
\includegraphics[width=\textwidth]{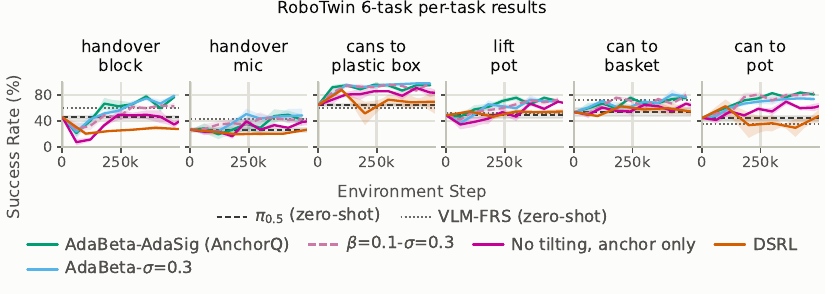}
\caption{\textbf{Per-task performance on the RoboTwin 6-task suite.}
Success rate over training for the particle-tilting ablations and steering
baselines. Curves show the mean and shaded regions show $\pm1$ standard error
across seeds.}
\label{fig:robotwin-per-task}
\end{figure*}

\subsection{Additional LIBERO-90 Results}
\label{sec:appendix-libero-results}

Figure~\ref{fig:libero15-four-task} and
Table~\ref{tab:libero15-four-task} report the four tasks with the clearest
differences among particle tilting, anchor-only learning, and the steering
baselines. Task $41$ requires placing a frying pan on top of a cabinet, tasks
$73$ and $81$ require placing a book into the front compartment of a caddy in
different study scenes, and task $75$ requires placing the book into the right
compartment. The base policy performs poorly on all four tasks, which require
precise end-effector motions to align the object with the target placement.
These fine corrections are difficult to express using only the discrete
flow-reversed anchors, giving particle tilting a larger advantage over the
anchor-only ablation. Consequently, the adaptive tilting variants perform
consistently across the four tasks, while anchor-only learning is
task-dependent and the baselines (DSRL and Residual RL) make little progress.

Figure~\ref{fig:task81-notilt-vs-ada} illustrates this difference on task $81$.
The anchor-only rollout repeatedly approaches the caddy but ends with the book
outside it, whereas \QFRS{} with adaptive tilting adjusts the approach and
successfully inserts the book into the front compartment.

\begin{figure*}[ht]
\centering
\includegraphics[width=0.6\textwidth]{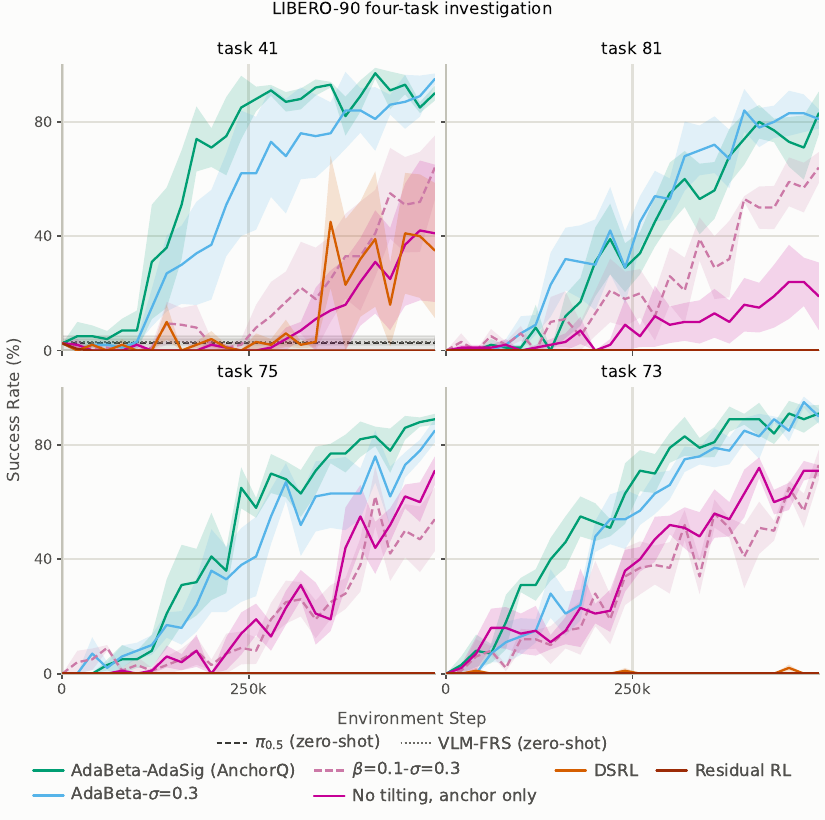}
\caption{\textbf{Per-task ablations on four LIBERO-90 tasks.}
\QFRS{} with adaptive tilting learns consistently across all four tasks.
Fixed tilting, anchor-only learning, and the steering baselines show
substantially greater variation across tasks. Curves show the mean and shaded
regions show $\pm1$ standard error across seeds.}
\label{fig:libero15-four-task}
\end{figure*}

\begin{figure*}[ht]
\centering
\includegraphics[width=\textwidth]{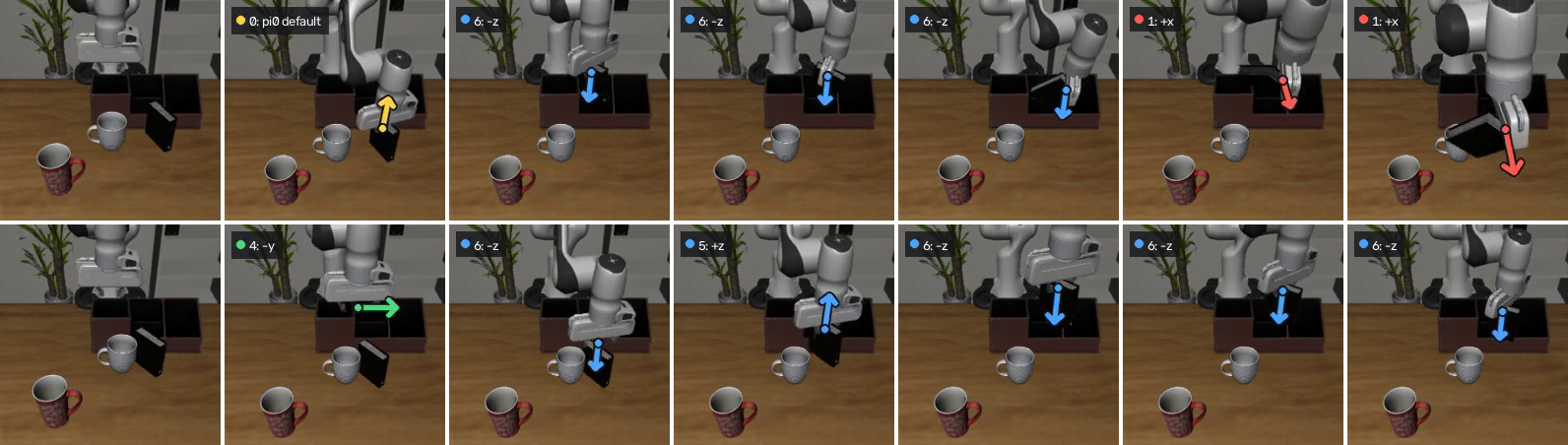}
\caption{\textbf{Anchor-only learning versus adaptive particle tilting on LIBERO-90 task $81$.}
The task is to place the book into the front compartment of the caddy.
\textbf{Top:} the no-tilting, anchor-only ablation fails to place the book.
\textbf{Bottom:} \QFRS{} completes the insertion. Colored overlays identify the nearest anchor and its reference direction;
yellow denotes a base-policy action.}
\label{fig:task81-notilt-vs-ada}
\end{figure*}

\begin{table*}[ht]
\centering
\small
\setlength{\tabcolsep}{6pt}
\caption{\textbf{Per-task success rate after $500$k environment steps.}
Success rate (\%), reported as mean $\pm$ standard error across seeds. The best
result on each task is bold.}
\label{tab:libero15-four-task}
\begin{tabular}{@{}lcccc@{}}
\toprule
\textbf{Method} & \textbf{Task 41} & \textbf{Task 81} &
\textbf{Task 75} & \textbf{Task 73} \\
\midrule
$\pi_{0.5}$ zero-shot
    & $2.5 \pm 2.5$ & $0.0 \pm 0.0$ & $0.0 \pm 0.0$ & $0.0 \pm 0.0$ \\
VLM-FRS zero-shot
    & $3.0 \pm 1.0$ & $0.0 \pm 0.0$ & $0.0 \pm 0.0$ & $0.0 \pm 0.0$ \\
DSRL
    & $35.0 \pm 23.6$ & $0.0 \pm 0.0$ & $0.0 \pm 0.0$ & $0.0 \pm 0.0$ \\
Residual RL
    & $0.0 \pm 0.0$ & $0.0 \pm 0.0$ & $0.0 \pm 0.0$ & $0.0 \pm 0.0$ \\
\midrule
\QFRS{}-AdaBeta-AdaSig
    & $90.0 \pm 2.6$ & $\mathbf{83.0 \pm 7.5}$ &
      $\mathbf{89.0 \pm 1.9}$ & $\mathbf{91.0 \pm 3.0}$ \\
\quad AdaBeta, $\sigma{=}0.3$
    & $\mathbf{95.0 \pm 1.9}$ & $81.0 \pm 3.4$ &
      $85.0 \pm 4.1$ & $90.0 \pm 2.6$ \\
\quad $\beta{=}0.1$, $\sigma{=}0.3$
    & $64.0 \pm 11.2$ & $64.0 \pm 5.4$ &
      $54.0 \pm 11.1$ & $73.0 \pm 5.3$ \\
\quad No tilting, anchors only
    & $41.0 \pm 24.0$ & $19.0 \pm 11.8$ &
      $71.0 \pm 5.0$ & $71.0 \pm 1.9$ \\
\bottomrule
\end{tabular}
\end{table*}

\section{Overview of Supplemental Theoretical Results}
\label{sec:appendix-theory-overview}
The theoretical results in this appendix develop a unified view of \QFRS{} as a
\emph{structured online action abstraction with local value-guided refinement}.
The results answer a sequence of complementary questions: why retaining flow-reversed action anchors as
the online action class can reduce the statistical burden of learning; what
 price is paid for that restriction in terms of representation bias; how local particles trade this
representation bias against increased learning complexity; how exponential
value-tilting exploits the resulting candidate support; and, finally, how the
tilting intensity can be adapted to epistemic uncertainty in the learned
critic.  Taken together, the theory separates three distinct roles:
\begin{itemize}
    \item \emph{Flow-reversal action anchr abstraction}: determines the structured online action class;
    \item \emph{Local particle refinement}: relax that class and enlarge the available latent support;
    \item \emph{Value tilting}: reallocate posterior sampling probability toward high-value candidates. 
\end{itemize}
The uncertainty-adaptive rule then determines how aggressively the learned
critic should be trusted when performing the final reweighting.

\paragraph{Online flow-reversal can reduce identification complexity.}
We begin in Appendix~\ref{sec:persistent-abstraction-separation} by isolating the statistical effect of retaining the structured latent action class induced by the flow-reversed action anchors throughout online learning, compared with learning over a broader latent class even when \emph{a priori} information is given about the restricted set. 

Theorem~\ref{thm:persistent-abstraction-separation} considers a stylized
one-decision latent-mode problem with $N$ behaviorally distinct latent modes,
of which only $C<N$ belong to the online, state-conditioned flow-reversed action anchor abstraction.  When the flow-reversed action
class is sufficiently accurate, a learner restricted to this class requires
\[\mathcal{O}\left(\frac{C}{\epsilon^2}
    \log\frac{C}{\alpha}\right)\] online samples.
On the other hand, a learner that retains a uniform guarantee over the
larger class must, on an instance for which the flow-reversal action anchor abstraction has zero
bias, spend additional online samples ruling out the $N-C$ non-flow-reversed  modes.
Importantly, this lower bound persists even if the learner is given the exact
reward laws of all flow-reversed modes before online interaction.  

Thus the statistical
benefit is attributable to the \emph{persistent online restriction} itself: information about flow-reversed behaviors, whether supplied through prefill or even their exact reward laws, cannot 
cannot
certify that unexplored non-flow-reversed behaviors are suboptimal.

\paragraph{The statistical gain has an explicit support price.}
The smaller action class is beneficial only when it contains sufficiently
good actions. In Appendix~\ref{sec:general-mdp-abstraction-price},  Theorem~\ref{thm:general-mdp-abstraction-price} therefore moves
to a general discounted latent MDP and defines the one-step flow-reversal support gap
\[
    \Delta_{\mathrm F}(s)
    :=
    V^\star(s)
    -
    \max_{c\in\mathcal C}
    Q^\star\bigl(s,z_c(s)\bigr), \quad\text{and}\quad \bar{\Delta}_{\mathrm F}
        :=
        \sup_{s\in\gS}
        \Delta_{\mathrm F}(s).
\]
It shows that the value price of remaining inside the flow-reversed anchor bank at every
decision is controlled by this local support deficiency,
\[
    \Delta_{\mathrm F}(s)
    \le
    V^\star(s)-V_{\mathrm F}^\star(s)
    \le
    \frac{\bar{\Delta}_{\mathrm F}}{1-\gamma},
\]
and, more importantly, gives the exact decomposition
\[
    V^\star-V^\omega
    =
    \underbrace{
        V^\star-V_{\mathrm F}^\star
    }_{\text{persistent flow-reversal support price}}
    +
    \underbrace{
        (I-\gamma P^\omega)^{-1}
        \operatorname{err}_{\mathrm{sel}}(\omega;\cdot)
    }_{\text{selection error within the flow-reversed class}}.
\]
The first term is structural: no amount of additional critic data can recover
a useful behavior that is absent from the flow-reversed action support.  The second is
statistical and can decrease as the learned critic and selector improve.
Together with Theorem~\ref{thm:persistent-abstraction-separation}, this
identifies the fundamental tradeoff behind  \QFRS{}: the smaller
online class can lower identification complexity, but only at the cost of an
explicit and measurable support bias.

\paragraph{Particles trace a bias--complexity frontier.}
In Appendix~\ref{sec:particle-bias-complexity-frontier}, we show that \QFRS{} provides a controlled way to relax the exact-anchor
restriction without immediately reverting to unrestricted latent-space
control.  Theorem~\ref{thm:particle-bias-complexity-frontier} considers a
local action class of deterministic local reach $\reach$ around the state-dependent flow-reversed
anchors.  Increasing $\reach$ can only reduce the corresponding support gap
$\bar{\Delta}_\reach$; under smooth local value geometry, the remaining gap
also admits an explicit local quadratic bound.  At the same time, increasing
$\reach$ enlarges the number of actions
$A_\reach(\epsilon_{\mathrm B})$
 that are distinguishable at resolution $\epsilon_{\mathrm B}$ in
the Bellman behavioral pseudometric.  The finite-sample result makes both
effects explicit:
\[
    \left\|
        V^\star
        -
        V^{\widehat\pi_{\reach,\epsilon_{\mathrm B},J}}
    \right\|_\infty
    \le
    \frac{
        \bar{\Delta}_\reach
        +
        \epsilon_{\mathrm B}
        +
        2\epsilon_Q
    }{
        1-\gamma
    },
\]
while the required generative-model sample count grows with
$A_\reach(\epsilon_{\mathrm B})$.  Hence local expansion is not free exploration:
it trades reduced structural bias against a larger behavioral action class
that must be represented and valued.  The deterministic reach $\reach$ used in
this analysis is an analytical local envelope and is not identified with the
Gaussian particle scale used by the implementation; Appendix~\ref{subsec:gaussian-particle-cloud-connection} separately
connects a finite Gaussian particle cloud to coverage of such a local
envelope.

\paragraph{Exponential tilting exploits the available support.} Once the candidate class has been generated, the remaining problem is how to
use value information to allocate execution probability within it.

Indeed, the preceding results complete the coarse-to-fine picture developed in
Theorems~\ref{thm:persistent-abstraction-separation}--%
\ref{thm:particle-bias-complexity-frontier}.  Those results characterize the
action representation made available to online learning; Appendix~\ref{sec:exponential-tilting}
characterizes how value information is used to refine that representation after
a coarse anchor has been selected.

Theorem~\ref{thm:exponential-tilting} gives the exact population-level characterization of this operation. Conditional on a state $s$ and selected anchor $c$, the exponentially tilted law is the unique optimizer of a KL-regularized value-improvement problem relative to the untilted local particle law. The theorem further shows that increasing the tilt intensity $\betatilt$ monotonically increases both the expected local-particle score and the KL departure from the reference law, while preserving its support. Thus the particle construction determines which local refinements are available, whereas $\betatilt$ determines how aggressively value information reallocates weight among those refinements. Corollary~\ref{cor:entropic-mean-variance-expansion} makes the corresponding upper-tail preference explicit: for small $\betatilt$, the entropic objective equals the mean local score plus a positive variance correction to first order.

The same analysis quantifies when this value-directed weighting remains reliable under an imperfect learned critic.  Theorem~\ref{thm:exponential-tilting}.{\it e.}  shows that a truly high-value local region with population-score separation $\Xi_{\nu_{s,c}}(A;q_{\mathrm{pop}})$---i.e., worst score inside $A$ minus the best score outside $A$, ignoring $\nu_{s,c}$-measure-zero exceptions---has its odds amplified by at least the factor
\[
\exp\left(
\betatilt
\left(
\Xi_{\nu_{s,c}}(A;q_{\mathrm{pop}})
-
2\eps_{\mathrm{crit}}
\right)
\right).
\]
 Consequently, strict amplification is guaranteed whenever the true local value separation exceeds $2\eps_{\mathrm{crit}}$. 

Finite-cloud results connect these population identities to the implemented particle weighting rule.
With critic error, Corollary~\ref{cor:finite-cloud-selection-error} gives the corresponding population-score gap
$\log(M+1)/\betatilt
+
2\eps_{\mathrm{crit}}$.
This still does not by itself characterize the action executed by \QFRS{}, because the implementation forms the barycentric latent rather than sampling one of the weighted particles. Proposition~\ref{prop:barycentric-refinement-value} closes this gap under local smoothness of the population value function: the loss from replacing the weighted particle values by the value of their weighted-average latent is at most
$
L_{\mathrm{sm}}
\mathcal V_{c,\betatilt}(s)/2,
$
where $\mathcal V_{c,\betatilt}(s)$ is the weighted particle variance. Corollary~\ref{cor:local-reach-barycenter-error} therefore separates the local refinement error into finite-particle representation error, finite-temperature weighting error, critic error, and barycentric interpolation error.

Finally, Theorem~\ref{thm:hierarchical-anchorq-control} propagates these local guarantees through the complete hierarchical \QFRS{} policy. The resulting bound
\begin{equation}
\label{eq:hierarchical-anchorq-control-bound-summary}
    \left\|V^\star-V^{\pi_{\AQ}}\right\|_\infty
    \leq
    \frac{
        \bar\Delta_\rho
        +\bar\eps_{\mathrm{sel}}
        +\bar\eps_{\mathrm{rep}}
        +\betatilt^{-1}\log(M+1)
        +2\bar\eps_{\mathrm{crit}}
        +\tfrac{L_{\mathrm{sm}}}{2}\bar{\mathcal V}_{\betatilt}
    }{1-\gamma}
\end{equation}
decomposes the sequential-control loss into six distinct terms: the structural support gap $\bar\Delta_\rho$, the coarse-anchor selection error $\bar\eps_{\mathrm{sel}}$, the finite-particle representation error $\bar\eps_{\mathrm{rep}}$, the finite-temperature term $\betatilt^{-1}\log(M+1)$, the critic-error term $2\bar\eps_{\mathrm{crit}}$, and the barycentric interpolation penalty $(L_{\mathrm{sm}}/2)\bar{\mathcal V}_{\betatilt}$. 
The decomposition makes explicit where \QFRS{} can lose value. The anchor and particle representation determines whether sufficiently good actions are available, through the support gap $\bar\Delta_\rho$ and representation error $\bar\eps_{\mathrm{rep}}$; learning determines whether those actions are identified reliably, through $\bar\eps_{\mathrm{sel}}$ and $\bar\eps_{\mathrm{crit}}$; and the value-tilted refinement mechanism introduces the finite-temperature and interpolation terms. 

Thus, the final guarantee connects the method's coarse-to-fine representation directly to its online decision rule: \emph{richer local search reduces structural bias, while reliable value information is needed to exploit that additional flexibility.}

\paragraph{Epistemic uncertainty determines how aggressively to tilt.}
The final result addresses this reliability--aggressiveness tradeoff directly.
In Section~\ref{sec:uncertainty_adaptation}, the tilting intensity is chosen
statewise according to
\[
    \betatilt(s)
    =
    1/u_{\mathrm{epi}}(s),
\]
where $u_{\mathrm{epi}}(s)$ is an ensemble-based surrogate for local epistemic
uncertainty.  Theorem~\ref{thm:adaptive-local-tilting} isolates the benefit of adapting the tilting intensity to epistemic uncertainty.   The resulting particle distribution is
the exact optimizer of an uncertainty-weighted KL-regularized improvement
problem,
\[
    \max_{\pi}
    \left\{
        \mathbb{E}_{\pi}[\widehat q]
        -
        u_{\mathrm{epi}}(s)
        \KL(\pi\|\mu)
    \right\}.
\]
Thus epistemic uncertainty directly determines the penalty for departing from
the flow-reversed  reference distribution.

Under the finite-bank calibration condition that critic error is proportional
to $u_{\mathrm{epi}}(s)$, the same theorem gives a true-value selection bound
whose error contracts linearly with the epistemic scale.  It also shows that
reliable preference between two candidates is governed by the standardized
value gap
$(q_j-q_k)/u_{\mathrm{epi}}(s)$,
rather than by the unnormalized critic difference alone.  Moreover, when
critic perturbations occur on the same uncertainty scale, adaptive
$\betatilt(s)=1/u_{\mathrm{epi}}(s)$ yields a scale-free stability bound:
sensitivity of the tilted weights depends on relative critic error rather than
growing with the absolute uncertainty level.  
Corollary~\ref{cor:adaptive-finite-cloud} gives the most direct finite-cloud interpretation of adaptive $\betatilt$.  The gap between the best represented local particle and the population value achieved by the learned weighting is proportional to the epistemic uncertainty scale:
\[
    \text{adaptive weighting error}
    =
    \mathcal O
    \left(
        u_{\mathrm{epi}}(s)
    \right).
\]
Thus, when the critic is poorly known, the method remains deliberately soft; as critic uncertainty contracts, the finite-temperature and critic-estimation terms contract together.

\paragraph{Overall picture.}
The results therefore form a single structured refinement hierarchy:
\begin{table}[ht]
\centering
\renewcommand{\arraystretch}{1.3}
\begin{tabular}{|l|l|}
\Xhline{1.2pt}
\textbf{Mechanism} & \textbf{Effect} \\
\Xhline{1.2pt}
online flow-reversal
  & small structured online action class \\
local particles
  & controlled relaxation of flow-reversed anchor support \\
exponential tilting
  & value-directed exploitation of that support \\
adaptive $\betatilt$
  & epistemically calibrated exploitation \\
\Xhline{1.2pt}
\end{tabular}
\end{table}

The first two results characterize the benefit and cost of the online flow-reversed action anchor
abstraction; the particle result describes how that restriction can be
relaxed; and the tilting results describe how the resulting capacity should be
used once it is available.  In this sense, flow-reversal of the coarse action anchors determines the structured
hypothesis class, particles determine how much local expressiveness is exposed
to learning, and epistemically adaptive value tilting determines how strongly
the learned critic should exploit that expressiveness.

\section{Online Flow-Reversal Abstraction and Latent Search Cost}
\label{sec:persistent-abstraction-separation}
We first isolate the statistical consequence of using the state-conditioned bank of flow-reversed action anchors as a \emph{persistent action
abstraction}.  The result is stated for a one-decision episodic latent Markov decision process (MDP).  This is exactly a
fixed-confidence Bernoulli best-arm identification problem; a one-step episodic MDP (see, e.g., \cite{lattimore2020bandit,kaufmann2016complexity}).  The construction separates the cost of identifying a useful latent mode from
state-distribution shift, temporal credit assignment, and function approximation.

\paragraph{Latent-mode MDPs.}
Fix integers $C$ and $N$ such that
$1 \le C < N$.
Let
$[N] := \{1,\ldots,N\}$
and $\gC := [C]$.
The set $\gC$ indexes the $C$ latent modes generated by the flow-reversed action anchors abstraction, while
$[N]\setminus\gC$ indexes the remaining latent modes available to a learner that retains the
full latent action class. It suffices to analyze this finite latent-mode subclass: any learner that is
uniformly correct over the full continuous latent space must also be correct
on this $N$-point restriction, so any lower bound here carries to unrestricted
continuous latent-space learning.

Choose distinct latent actions
$z_1,\ldots,z_N \in \gZ$.
For each mean vector
\[
    \vmu=(\mu_1,\ldots,\mu_N)\in[0,1]^N,
\]
define the one-decision episodic latent MDP---namely, the learner chooses one latent action, observes binary success or failure, and the episode ends---$\gM_{\vmu}$ as follows.  The state space consists of
an initial state $s_0$ and a terminal state $s_{\mathrm{term}}$.  At $s_0$, selecting latent
$z_i$ produces a reward
$R \sim \operatorname{Bernoulli}(\mu_i)$
and transitions deterministically to $s_{\mathrm{term}}$.  Each episode therefore produces one
online reward observation.  

At the unique decision state $s_0$, define the flow-reversed anchor map $F$ by
\[F(s_0,c)=z_c,\qquad c\in\gC, \]
so that $z_1,\ldots,z_C$ are precisely the latent actions available under the online flow-reversal abstraction.

For this MDP, set
\[
    Q_{\vmu}^\star(s_0,z_i)=\mu_i
    \qquad\text{and}\qquad
    V_{\vmu}^\star(s_0)=\mu^\star(\vmu),
\quad \text{where}\;\;
\mu^\star(\vmu)
    :=
    \max_{i\in[N]}\mu_i.\]
The best mean available within the \emph{online flow-reversal abstraction} is
$\mu_{\mF}^\star(\vmu)
    :=
    \max_{c\in\gC}\mu_c$.
The flow-reversal abstraction bias on $\gM_{\vmu}$ is
\begin{equation}
    \label{eq:mode-frs-bias}
    \Delta_{\mF}(\vmu)
    :=
    \mu^\star(\vmu)
    -
    \mu_{\mF}^\star(\vmu)
    =
    V_{\vmu}^\star(s_0)
    -
    \max_{c\in\gC}
    Q_{\vmu}^\star(s_0,F(s_0,c)).
\end{equation}

\paragraph{Online learners and the correctness criterion.}
An online learner is specified by a sequence of action-selection kernels and a terminal decision
rule.  Before online interaction, the learner may be given the exact flow-reversed mode Bernoulli laws,
equivalently the vector
$W_{\mF}(\vmu)
    :=
    (\mu_c)_{c\in\gC}$.
This side information is at least as informative as any finite FRS prefill consisting only
of samples from the flow-reversed modes.  In Theorem~\ref{thm:persistent-abstraction-separation} below, part~\textnormal{\it a.}~uses no such side information;
part~\textnormal{\it b.}~grants it to the full-class learner in order to strengthen the lower bound.

At online round $t\in[n]$, the learner selects an index $I_t\in[N]$ as a possibly
randomized function of $W_{\mF}(\vmu)$ and the previous online observations, and then
observes
$R_t\,|\, I_t
    \sim
    \operatorname{Bernoulli}(\mu_{I_t})$.
Conditionally on the selected indices, the rewards are independent across rounds.  After $n$
online rounds, the learner returns an index
$\widehat I_n\in[N]$.
For each $i\in[N]$, define the online sample count
\[
    N_i(n)
    :=
    \sum_{t=1}^n
    \1\{I_t=i\}.
\]

A learner is \emph{flow-reversal-restricted} if
$I_t\in\gC$ for every $t$ and $\widehat I_n\in \gC$
almost surely.  A learner is \emph{full-class} if it may select and return any index in $[N]$.

Let $\mathfrak M$ be a class of latent-mode MDPs.  For accuracy $\eps>0$ and failure
probability $\alpha\in(0,1)$, a learner is $(\eps,\alpha)$-correct on $\mathfrak M$ if, for
every $\gM_{\vmu}\in\mathfrak M$, the guarantee holds:
\begin{equation}
    \label{eq:pac-correctness}
    \mathbb P_{\vmu}
    \left(
        \mu^\star(\vmu)
        -
        \mu_{\widehat I_n}
        \le
        \eps
    \right)
    \ge
    1-\alpha.
\end{equation}

\subsection{Technical lemmas}
We need some technical lemmas before we prove the first main result.

For $p,q\in[0,1]$, write
\[
    \KLop(p,q)
    :=
    p\log(p/q)
    +(1-p)\log((1-p)/(1-q))
\]
for the binary relative entropy, with the standard conventions
\[
    0\log\frac{0}{q}:=0,
    \qquad
    p\log\frac{p}{0}:=+\infty\ \text{for }p>0,
\]
and the analogous conventions for the second term.

The following is a slight extension of the standard adaptive KL divergence
decomposition to include instance-independent side information and a terminal
randomized output; see
\citet[Lemma~15.1]{lattimore2020bandit} and
\citet[Lemma~1]{kaufmann2016complexity}.
We state it here in an abstract form and specialize it to the
Bernoulli latent-mode construction below.

\begin{lemma}[Adaptive change of measure]
\label{lem:adaptive-change-of-measure}
Let $\theta$ and $\theta'$ be two problem instances with a common finite
action set $[N]:=\{1,\ldots,N\}$. For each action $i\in[N]$, let
$P_{\theta,i}$ and $P_{\theta',i}$ denote the corresponding observation
distributions under the two instances.
Consider an adaptive learner run for exactly $n$ rounds. Before the first
round, the learner receives side information $W$ having the same distribution
under $\theta$ and $\theta'$. At round $t$, the learner selects an action
$A_t\in[N]$ according to a possibly randomized kernel depending on the
history
\[
    H_{t-1}
    :=
    \bigl(
        W,
        A_1,Y_1,\ldots,A_{t-1},Y_{t-1}
    \bigr),
\]
and then observes
$Y_t\sim P_{\theta,A_t}$
under instance $\theta$, and
$Y_t\sim P_{\theta',A_t}$
under instance $\theta'$.
After round $n$, the learner may produce a possibly randomized terminal
output $\widehat U$ based on the interaction history.
Assume that the learner uses the same action-selection and terminal-output
kernels under the two instances. Define the complete transcript
\[
    \mathcal T_n
    :=
    \bigl(
        W,
        A_1,Y_1,\ldots,A_n,Y_n,
        \widehat U
    \bigr),
\]
and let $\mathbb P_\theta^{(n)}$ and
$\mathbb P_{\theta'}^{(n)}$ denote the distributions of $\mathcal T_n$
under $\theta$ and $\theta'$, respectively. Also define
$N_i(n)
    :=
    \sum_{t=1}^n \1\{A_t=i\}$.
Then
\begin{equation}
    \label{eq:adaptive-kl-decomposition}
    \KL
    \left(
        \mathbb P_\theta^{(n)}
        \,\middle\|\,
        \mathbb P_{\theta'}^{(n)}
    \right)
    =
    \sum_{i=1}^N
    \mathbb{E}_\theta[N_i(n)]
    \KL
    \left(
        P_{\theta,i}
        \,\middle\|\,
        P_{\theta',i}
    \right),
\end{equation}
so that, for every event $\mathcal E$ measurable with respect to
$\mathcal T_n$, the lower bound holds:
\begin{equation}
    \label{eq:event-change-of-measure}
    \sum_{i=1}^N
    \mathbb{E}_\theta[N_i(n)]
    \KL
    \left(
        P_{\theta,i}
        \,\middle\|\,
        P_{\theta',i}
    \right)
    \geq
    \KLop
    \left(
        \mathbb P_\theta^{(n)}(\mathcal E),
        \mathbb P_{\theta'}^{(n)}(\mathcal E)
    \right).
\end{equation}
\end{lemma}

\begin{proof}
Fix a round $t$. Conditional on a realized history $H_{t-1}=h$, the
learner uses the same conditional distribution for $A_t$ under instances
$\theta$ and $\theta'$. Hence the action-selection step contributes zero
conditional relative entropy.

Conditional on $(H_{t-1},A_t=i)$, the observation distribution is
$P_{\theta,i}$ under instance $\theta$ and $P_{\theta',i}$ under
instance $\theta'$. The chain rule for relative entropy therefore gives
\begin{align}
    \KL
    \left(
        \mathbb P_\theta^{(n)}
        \,\middle\|\,
        \mathbb P_{\theta'}^{(n)}
    \right)
    &=
    \sum_{t=1}^n
    \mathbb{E}_\theta
    \left[
        \KL
        \left(
            P_{\theta,A_t}
            \,\middle\|\,
            P_{\theta',A_t}
        \right)
    \right].
    \label{eq:adaptive-kl-chain-rule}
\end{align}
The initial side information contributes zero relative entropy because it
has the same distribution under the two instances. Likewise, the terminal
output contributes zero conditional relative entropy because, conditional
on the preceding interaction history, the learner uses the same
terminal-output kernel under $\theta$ and $\theta'$.

Expanding the expectation in
\eqref{eq:adaptive-kl-chain-rule} over the selected action gives
\begin{align*}
    \sum_{t=1}^n
    \mathbb{E}_\theta
    \left[
        \KL
        \left(
            P_{\theta,A_t}
            \,\middle\|\,
            P_{\theta',A_t}
        \right)
    \right]
    &=
    \sum_{t=1}^n
    \sum_{i=1}^N
    \mathbb P_\theta(A_t=i)
    \KL
    \left(
        P_{\theta,i}
        \,\middle\|\,
        P_{\theta',i}
    \right)
    \\
    &\qquad=
    \sum_{i=1}^N
    \mathbb E_\theta
    \left[
        \sum_{t=1}^n \1\{A_t=i\}
    \right]
    \KL
    \left(
        P_{\theta,i}
        \,\middle\|\,
        P_{\theta',i}
    \right)
    \\
    &\qquad=
    \sum_{i=1}^N
    \mathbb E_\theta[N_i(n)]
    \KL
    \left(
        P_{\theta,i}
        \,\middle\|\,
        P_{\theta',i}
    \right),
\end{align*}
which proves \eqref{eq:adaptive-kl-decomposition}.

Now fix any event $\mathcal E$ measurable with respect to the complete
transcript $\mathcal T_n$. Mapping the transcript to the binary indicator
$\1\{\mathcal E\}$ is a measurable transformation. By the data-processing
inequality,
\begin{align*}
    \KL
    \left(
        \mathbb P_\theta^{(n)}
        \,\middle\|\,
        \mathbb P_{\theta'}^{(n)}
    \right)
    &\geq
    \KLop
    \left(
        \mathbb P_\theta^{(n)}(\mathcal E),
        \mathbb P_{\theta'}^{(n)}(\mathcal E)
    \right),
\end{align*}
because under the two instances the indicator $\1\{\mathcal E\}$ is
Bernoulli with parameters
$\mathbb P_\theta^{(n)}(\mathcal E)$ and
$\mathbb P_{\theta'}^{(n)}(\mathcal E)$, respectively.
Combining this inequality with
\eqref{eq:adaptive-kl-decomposition} proves
\eqref{eq:event-change-of-measure}.
\end{proof}

The next technical lemma are elementary bounds on binary relative entropy, and are standard;
see, e.g., \citet[Lemma~10.2]{lattimore2020bandit} for its basic
monotonicity properties and \citet{kaufmann2016complexity} for their use
in fixed-confidence bandit lower bounds.
\begin{lemma}[Binary relative-entropy bounds]
    \label{lem:binary-kl-bounds}
    Let $\alpha\in(0,1/2)$.  If $p\ge1-\alpha$ and $q\le\alpha$, then
    \begin{equation}
        \label{eq:binary-kl-monotonicity}
        \KLop(p,q)
        \ge
        \KLop(1-\alpha,\alpha).
    \end{equation}
    If, in addition, $\alpha\le1/4$, then
    \begin{equation}
        \label{eq:binary-kl-alpha-lower}
        \KLop(1-\alpha,\alpha)
        \ge
        \log(3/(4\alpha))/2.
    \end{equation}
    Finally, for every $\eps\in(0,1/8]$, the estimate holds:
    \begin{equation}
        \label{eq:bernoulli-instance-kl-upper}
        \KLop
        \left(
            1/2-2\eps, 1/2+2\eps
        \right)
        \leq 128\eps^2/3.
    \end{equation}
\end{lemma}

\begin{proof}
    First suppose $p,q\in(0,1)$.  Directly differentiating  $\KLop(p,q)$ yields
    \[
        \frac{\partial}{\partial p}
        \KLop(p,q)
        =
        \log
        \frac{p(1-q)}{q(1-p)}
    \quad\text{
    and}
    \quad
        \frac{\partial}{\partial q}
        \KLop(p,q)
        =
        \frac{q-p}{q(1-q)}.
    \]
    On the region $p>q$, the first derivative is positive and the second derivative is
    negative.  Since $p\ge1-\alpha>\alpha\ge q$, the minimum over interior points satisfying
    $p\ge1-\alpha$ and $q\le\alpha$ is attained at
    $(p,q)=(1-\alpha,\alpha)$.  If $p=1$ or $q=0$, the same inequality follows by taking
    limits from the interior; in the case $q=0<p$, the left-hand side is $+\infty$.  This
    proves \eqref{eq:binary-kl-monotonicity} for all $p,q\in[0,1]$ covered by the statement.

    A direct calculation yields
    $\KLop(1-\alpha,\alpha)
        =
        (1-2\alpha)
        \log\frac{1-\alpha}{\alpha}$.
    If $\alpha\leq1/4$, then
    $1-2\alpha\geq\frac12$
       and $1-\alpha\geq\frac34$.
    Therefore
    \[
        \KLop(1-\alpha,\alpha)
        \ge
        \frac12
        \log\frac{3}{4\alpha},
    \]
    proving \eqref{eq:binary-kl-alpha-lower}.

    Set
    $x:=4\eps$.
    Since $\eps\le1/8$, we have $x\in(0,1/2]$.  Using
    \[
        1-\left(1/2-2\eps\right)
        =
        1/2+2\eps
    \quad
    \text{and}
    \quad
        1-\left(1/2+2\eps\right)
        =1/2-2\eps,
    \]
    we obtain
    \begin{align*}
        \KLop
        \left(
            1/2-2\eps,1/2+2\eps
        \right)
        &=
        4\eps
        \log\frac{1+4\eps}{1-4\eps}=
        x\log\frac{1+x}{1-x}.
    \end{align*}
    For $x\in[0,1)$, define
    \[
        g(x)
        :=
        2x/(1-x^2)
        -
        \log((1+x)/(1-x)).
    \]
    Then $g(0)=0$ and
    $g'(x)
        =4x^2/(1-x^2)^2
        \geq0$.
    Hence
    \[
        \log((1+x)/(1-x))
        \leq
     2x/(1-x^2).
    \]
    It follows that
    \[
        x\log((1+x)/(1-x))
        \leq
        2x^2/(1-x^2)
        \leq
       8x^2/3
        =
        128\eps^2/3,
    \]
    where the second inequality uses $x^2\leq1/4$.  This proves
    \eqref{eq:bernoulli-instance-kl-upper}.
\end{proof}
\subsection{Main result on online flow-reversal abstraction sample complexity}
With the technical lemmas in hand, we now state the main sample complexity separation result. 
\begin{theorem}[Online flow-reversed action anchor abstraction sample-complexity]
    \label{thm:persistent-abstraction-separation}
    Fix integers $1\le C<N$, accuracy $\eps\in(0,1/8]$, and failure probability
    $\alpha\in(0,1/4]$.

    \begin{enumerate}[label={\it \alph*.}, topsep=0pt,itemsep=0pt,leftmargin=*]
        \item \textbf{Upper bound for an online flow-reversal abstraction.}
        Consider the class
   \[\mathfrak M_{\mF}(\eps/2)
            :=\left\{
                \gM_{\vmu}:
                \Delta_{\mF}(\vmu)
                \leq
                \eps/2
            \right\}.
            \]
        There exists an flow-reversal-restricted learner that uses no prior knowledge and is $(\eps,\alpha)$-correct on
        $\mathfrak M_{\mF}(\eps/2)$ after
        \begin{equation}
            \label{eq:qfrs-upper-sample-complexity}
            n_{\mathrm{FR}}
            :=
            C
            \left\lceil
                \frac{8}{\eps^2}
                \log\frac{2C}{\alpha}
            \right\rceil\quad\text{episodes}.
        \end{equation}

        \item \textbf{Lower bound for retaining the full latent class.}
        Define a base instance $\gM_{\vmu^{(0)}}$ by
        \begin{equation}
            \label{eq:base-instance}
            \mu_1^{(0)}
            :=
            1/2,
            \qquad
            \mu_i^{(0)}
            :=
   1/2-2\eps,
            \quad
            i\in[N]\setminus\{1\}.
        \end{equation}
        For each $j\in[N]\setminus\gC$, define an alternative instance
        $\gM_{\vmu^{(j)}}$ by
        \begin{equation}
            \label{eq:alternative-instance}
            \mu_i^{(j)}
            :=
            \begin{cases}
                1/2+2\eps,
                & i=j,\\[1mm]
                \mu_i^{(0)},
                & i\ne j.
            \end{cases}
        \end{equation}
        Let
        $\mathfrak M_{\mathrm{full}}
            :=
            \left\{
                \gM_{\vmu^{(0)}}
            \right\}
            \cup
            \left\{
                \gM_{\vmu^{(j)}}:
                j\in[N]\setminus\gC
            \right\}$.
        Consider any full-class learner that is given the exact flow-reversed-mode laws
        $W_{\mF}(\vmu)$ before online interaction and is
        $(\eps,\alpha)$-correct on $\mathfrak M_{\mathrm{full}}$ after $n$ online
        episodes.  Then, on the zero-bias base instance $\gM_{\vmu^{(0)}}$, every
        non-flow-reversed mode satisfies
        \begin{equation}
            \label{eq:per-mode-lower-bound}
            \mathbb{E}_{\vmu^{(0)}}[N_j(n)]
            \geq
                \KLop(1-\alpha,\alpha)
           /\KLop
                \left(
                    1/2-2\eps,1/2+2\eps
                \right),
            \qquad
            j\in[N]\setminus\gC,
        \end{equation}
        so that 
                \begin{equation}
            \label{eq:full-latent-total-lower-bound}
            n
            \ge
            \frac{3(N-C)}{256\eps^2}
            \log\frac{3}{4\alpha}.
        \end{equation}

        \item \textbf{Zero-bias separation.}
        The base instance satisfies
        $\Delta_{\mF}(\vmu^{(0)})=0$.
        Thus the lower bound in part \textnormal{\it b.} holds on an instance for which restricting
        the action class to the online flow-reversal abstraction incurs no control bias.  The additional
        samples are required solely because a learner that is correct on
        $\mathfrak M_{\mathrm{full}}$ must rule out every alternative in which a non-flow-reversed mode is
        superior.
    \end{enumerate}
\end{theorem}

\begin{proof}
    We prove the three parts in order.

    \paragraph{Proof of part \textnormal{\it a}.}
    Let  $m
        :=
        \left\lceil
            \frac{8}{\eps^2}
            \log\frac{2C}{\alpha}
        \right\rceil$.
    The learner samples each flow-reversed mode $c\in\gC$ exactly $m$ times.  Let
    $\widehat\mu_c:=
        \frac{1}{m}
        \sum_{\ell=1}^m
        R_{c,\ell}$
    be the empirical mean of the $m$ rewards obtained from flow-reversed mode $c$.  The learner returns
    $\widehat I
        \in
        \argmax_{c\in\gC}
        \widehat\mu_c$,
    with any fixed deterministic tie-breaking rule.

    Define the event
    \[
        \gE
        :=
        \bigcap_{c\in\gC}
        \left\{
            |\widehat\mu_c-\mu_c|
            \leq\eps/4
        \right\}.
    \]
    By H\"oeffding's inequality, for each $c\in\gC$, we have that 
    \[
        \mathbb P_{\vmu}
        \left(
            |\widehat\mu_c-\mu_c|
            >
            \eps/4
        \right)
        \leq
        2\exp
        \left(
            -2m\eps^2/16
        \right)
        =
        2\exp
        \left(
            -m\eps^2/8
        \right).
    \]
    A union bound and the definition of $m$ yield
    \[
        \mathbb P_{\vmu}(\gE^c)
        \leq2C \exp
        \left(
            -m\eps^2/8
        \right)\leq\alpha.
    \]
    Fix
    $c_{\mF}^\star
        \in
        \argmax_{c\in\gC}\mu_c$.
    On the event $\gE$, we have that 
    \begin{align*}
        \mu_{\widehat I}
        &\ge
        \widehat\mu_{\widehat I}
        -
        \frac{\eps}{4}
        \ge
        \widehat\mu_{c_{\mF}^\star}
        -
        \frac{\eps}{4}
        \ge
        \mu_{c_{\mF}^\star}
        -
        \frac{\eps}{2}
        =
        \mu_{\mF}^\star(\vmu)
        -
        \frac{\eps}{2}.
    \end{align*}
    Therefore, on $\gE$, we have that 
    \begin{align*}
        \mu^\star(\vmu)
        -
        \mu_{\widehat I}
        &\le
        \mu^\star(\vmu)
        -
        \mu_{\mF}^\star(\vmu)
        +
        \frac{\eps}{2}
=
        \Delta_{\mF}(\vmu)
        +
        \frac{\eps}{2}
\leq
        \eps,
    \end{align*}
    where the final inequality uses
    $\Delta_{\mF}(\vmu)\le\eps/2$.  Since
    $\mathbb P_{\vmu}(\gE)\ge1-\alpha$, the learner is
    $(\eps,\alpha)$-correct on $\mathfrak M_{\mF}(\eps/2)$.  It uses exactly
    $Cm=n_{\mathrm{FR}}$ online episodes.

    \paragraph{Proof of part \textnormal{\it b}.}
  Since $\eps\le1/8$, all means in
    \eqref{eq:base-instance}--\eqref{eq:alternative-instance} belong to $[0,1]$.  Moreover,
    all instances in $\mathfrak M_{\mathrm{full}}$ agree on the flow-reversed modes---namely, we have that 
    \[
        W_{\mF}(\vmu^{(0)})
        =
        W_{\mF}(\vmu^{(j)})
        \qquad
        \text{for every }j\in[N]\setminus\gC.
    \]
    Thus the exact flow-reversal side information is identical under the base and every alternative.

    On the base instance, mode $1$ has mean $1/2$, while every other mode has mean
    $1/2-2\eps$.  Hence mode $1$ is the unique $\eps$-optimal mode, since every other
    mode has suboptimality gap $2\eps>\eps$.  Since the learner is
    $(\eps,\alpha)$-correct, we have that 
    \begin{equation}
        \label{eq:base-output-probability}
        \mathbb P_{\vmu^{(0)}}
        \left(
            \widehat I_n=1
        \right)
        \geq
        1-\alpha.
    \end{equation}
    Fix any $j\in[N]\setminus\gC$.  On the alternative instance
    $\gM_{\vmu^{(j)}}$, mode $j$ has mean $1/2+2\eps$, mode $1$ has mean $1/2$, and
    every remaining mode has mean $1/2-2\eps$.  Thus mode $j$ is the unique
    $\eps$-optimal mode.  Correctness on $\gM_{\vmu^{(j)}}$ therefore implies
    $\mathbb P_{\vmu^{(j)}}
        (\widehat I_n=j)
        \geq
        1-\alpha$,
    and hence
    \begin{equation}
        \label{eq:alternative-output-probability}
        \mathbb P_{\vmu^{(j)}}
        \left(
            \widehat I_n=1
        \right)
        \le
        \alpha.
    \end{equation}
    Now, we apply Lemma~\ref{lem:adaptive-change-of-measure} with
    $\vmu=\vmu^{(0)}$,
        $\bm{\lambda}=\vmu^{(j)}$,
        and
        $\mathcal{E}=\{\widehat I_n=1\}$.
    The two mean vectors differ only at coordinate $j$.  Therefore
    \begin{align*}
        \mathbb{E}_{\vmu^{(0)}}[N_j(n)]
        \KLop
        \left(
            \frac12-2\eps,
            \frac12+2\eps
        \right)\geq
        \KLop
        \left(
            \mathbb P_{\vmu^{(0)}}(\mathcal{E}),
            \mathbb P_{\vmu^{(j)}}(\mathcal{E})
        \right).
    \end{align*}
    Combining expressions \eqref{eq:base-output-probability} and 
    \eqref{eq:alternative-output-probability} with 
    Lemma~\ref{lem:binary-kl-bounds} yields
    \[
        \KLop
        \left(
            \mathbb P_{\vmu^{(0)}}(\mathcal{E}),
            \mathbb P_{\vmu^{(j)}}(\mathcal{E})
        \right)
        \ge
        \KLop(1-\alpha,\alpha).
    \]
    Dividing by the positive Bernoulli relative entropy proves
    \eqref{eq:per-mode-lower-bound}.

    Summing \eqref{eq:per-mode-lower-bound} over the $N-C$ indices in
    $[N]\setminus\gC$ yields 
        \begin{align}
           \textstyle \mathbb{E}_{\vmu^{(0)}}
            \left[
                \sum_{j\in[N]\setminus\gC}
                N_j(n)
            \right]
            &\geq
                (N-C)
                \KLop(1-\alpha,\alpha)/
                \KLop(
                    1/2-2\eps,
                    1/2+2\eps)
            \label{eq:exact-nonfrs-lower-bound}
        \end{align}
        Applying
    \eqref{eq:bernoulli-instance-kl-upper}, we further deduce that 
\begin{equation}\label{eq:explicit-nonfrs-lower-bound}\textstyle
     \mathbb{E}_{\vmu^{(0)}}
            \left[
                \sum_{j\in[N]\setminus\gC}
                N_j(n)
            \right]\geq\displaystyle \frac{3(N-C)}{128\eps^2}
           \KLop(1-\alpha,\alpha)\geq \frac{3(N-C)}{256\eps^2}
            \log\frac{3}{4\alpha}
\end{equation}
where the last inequality follows by  applying
    \eqref{eq:binary-kl-alpha-lower}.

    Finally, on every realized transcript, we have that 
$\sum_{j\in[N]\setminus\gC}
        N_j(n)
        \leq
        \sum_{i=1}^N N_i(n)
        =n$.
    Taking expectations under $\vmu^{(0)}$ and combining with
    \eqref{eq:explicit-nonfrs-lower-bound} proves
    \eqref{eq:full-latent-total-lower-bound}.

    \paragraph{Proof of part \textnormal{\it c}.}
    On the base instance, mode $1$ belongs to $\gC$ and has mean $1/2$, while every other
    mode has mean $1/2-2\eps$.  Consequently, we have that 
$\mu^\star(\vmu^{(0)})
        =
        \mu_{\mF}^\star(\vmu^{(0)})
        =
        \frac12$.
    By \eqref{eq:mode-frs-bias}, we have that 
  $\Delta_{\mF}(\vmu^{(0)})=0$.
    The lower bound in part \textnormal{\it b.} therefore holds on a zero-bias instance.  Each
    alternative $\vmu^{(j)}$ agrees with the base instance on every flow-reversed mode and differs only on
    the non-flow-reversed mode $j$.  Hence neither the exact flow-reversed-mode laws nor any prefill consisting only
    of flow-reversed-mode samples can distinguish the base instance from alternative $j$.  A learner that
    is correct on the full family must obtain the information required to rule out that alternative
    through online samples of mode $j$.  This is precisely the per-mode requirement in
    \eqref{eq:per-mode-lower-bound}.
\end{proof}

\begin{corollary}[Full-class learning with bank-restricted prefill]
    \label{cor:dsrl-prefill-lower-bound}
    Consider any learner whose online action class contains all $N$ latent modes and whose
    prefill, if present, is generated using only interactions with modes in $\gC$.  If the
    learner is $(\eps,\alpha)$-correct on $\mathfrak M_{\mathrm{full}}$, then the lower bounds
    \eqref{eq:per-mode-lower-bound}--\eqref{eq:full-latent-total-lower-bound} apply.  
\end{corollary}
    This result covers both residual RL  without prefill and learning initialized with bank-restricted data; pure DSRL and FRS-prefill followed by DSRL are respective examples \citep{wagenmaker2025dsrl, tang2026frs}.
\begin{proof}
    Suppose, for contradiction, that such a learner violates one of the lower bounds.  Construct a
    learner that is given the exact flow-reversed-mode laws $W_{\mF}(\vmu)$.  Using independent
    internal randomness and these exact Bernoulli laws, it first simulates the original learner's
    entire prefill-generation procedure, including any adaptive choices among modes in $\gC$,
    and then runs the original learner's online procedure on the resulting simulated prefill.  The
    simulated prefill has exactly the same distribution as the original prefill under every
    instance in $\mathfrak M_{\mathrm{full}}$.  The constructed learner therefore has the same
    correctness probability and the same distribution of online sample counts as the original
    learner.  This contradicts Theorem~\ref{thm:persistent-abstraction-separation}\textnormal{\it b.},
    which applies to every learner given the exact flow-reversed-mode laws.  Hence the stated lower bounds
    also apply to the original learner.
\end{proof}
This is a worst-case result for latent-mode problems with no assumed
smoothness or generalization between modes. It does not imply that every
practical DSRL run must explicitly test $N-C$ additional modes. Rather, it
shows that FRS-only prefill cannot, without further structural assumptions,
guarantee removal of the exploration cost associated with retaining the full
latent action class.
\begin{remark}[What the separation establishes]
    \label{rem:separation-interpretation}
    Theorem~\ref{thm:persistent-abstraction-separation} is a nested-model or adaptation result.
    The flow-reversal-restricted learner commits to the hypothesis that the online flow-reversal action abstraction has
    bias at most $\eps/2$.  The full-class learner retains a correctness guarantee on instances
    in which a non-flow-reversed latent mode is better.  The lower bound shows that retaining this broader
    guarantee has a statistical cost even on the base instance, where the flow-reversal abstraction is
    exactly sufficient.  A full-class learner can avoid that cost only by declining to rule out the
    non-flow-reversed alternatives, in which case it no longer has the stated correctness guarantee on
    $\mathfrak M_{\mathrm{full}}$.
\end{remark}

\begin{remark}[Why buffer prefill with flow-reversed action anchors cannot remove the lower bound]
    \label{rem:prefill-lower-bound}
    The lower bound grants the learner the exact Bernoulli reward law of every flow-reversed mode before any
    online interaction.  This is more information than any finite flow-reversal prefill generated solely by
    those modes. Since the base and alternative instances have identical reward laws on all flow-reversed
    modes, no amount of such prefill changes the likelihood ratio between them.  The lower bound
    therefore applies, \emph{a fortiori}, to FRS-prefill followed by DSRL \cite{tang2026frs}.
\end{remark}

\begin{remark}[Relation to continuous latent spaces]
    \label{rem:continuous-latent-embedding}
    The finite latent set $\{z_1,\ldots,z_N\}$ may be embedded in any continuous latent space
    $\gZ\subseteq\R^{d_z}$ containing $N$ distinct points.  More generally, each index may
    represent a measurable latent region whose elements induce the same reward law.  The theorem
    therefore models a latent space with $N$ behaviorally distinguishable modes and makes no
    smoothness assumption linking the rewards of different modes.  Since the construction is a
    one-state, one-decision episodic MDP, it is also a standard Bernoulli multi-armed bandit and can
    be inserted as a decision state in a longer-horizon MDP without changing the identification
    lower bound at that state.
\end{remark}

\section{The general-MDP price of the persistent FRS abstraction}
\label{sec:general-mdp-abstraction-price}
Theorem~\ref{thm:persistent-abstraction-separation} isolates the statistical benefit of solving a
smaller online latent-action problem.  We now quantify the control price of that restriction in a
general discounted latent MDP.  The result separates two errors: an irreducible support error caused
by the flow-reversal action anchor abstraction and a reducible selection error caused by finite-data learning within
that abstraction.

\subsection{General latent MDP and the persistent flow-reversal action class}
\label{subsec:general-latent-mdp}

Let
$\gM
    =
    (\gS,\gZ,P,r,\gamma)$
be a discounted latent-action MDP.  Assume that $(\gS,\Sigma_{\gS})$ and
$(\gZ,\Sigma_{\gZ})$ are standard Borel spaces, that
$\gamma\in[0,1),$
that
$r:\gS\times\gZ\to\R$
is bounded and measurable, and that
$P(\cdot\,|\, s,z)$
is a stochastic kernel on $\gS$ given $(s,z)\in\gS\times\gZ$.

For a bounded measurable function $V:\gS\to\R$, define
\[
    Q_V(s,z)
    :=
    r(s,z)
    +
    \gamma
    \int_{\gS}
        V(s')
        P(\mathrm{d}s'\mid s,z),
\]
and define the full latent-action Bellman optimality operator
\[
    (\gT V)(s)
    :=
    \sup_{z\in\gZ}
    Q_V(s,z).
\]
Assume that $\gT$ maps bounded measurable functions to bounded measurable functions and that the
full optimal value $V^\star$ is its unique bounded fixed point---namely,
$V^\star
    =
    \gT V^\star$.
Define the corresponding full optimal action-value function by
\[
    Q^\star(s,z)
    :=
    Q_{V^\star}(s,z)
    =
    r(s,z)
    +
    \gamma
    \int_{\gS}
        V^\star(s')
        P(\mathrm{d}s'\mid s,z).
\]

Recall that
$\gC
    =
    \{1,\ldots,C\}$
is the finite set of coarse commands, and let
$F:\gS\times\gC\to\gZ$
be the measurable online flow-reversal map.  As in the method definition, write
\[
    z_c(s)
    :=
    F(s,c)
    =
    H_\theta^K\bigl(s,\varphi(s,c)\bigr).
\]
Recall the state-dependent flow-reversed action bank is
$\gZ(s)
    :=
    \{z_c(s):c\in\gC\}$.

For a bounded measurable $V$, define the flow-reversal-restricted Bellman optimality operator
\begin{equation}
    \label{eq:frs-restricted-bellman-operator}
    (\gT_{\mF}V)(s)
    :=
    \max_{c\in\gC}
    Q_V\bigl(s,z_c(s)\bigr).
\end{equation}
Since $\gC$ is finite, the maximum in
\eqref{eq:frs-restricted-bellman-operator} is attained at every state.

A stationary flow-reversal selector is a measurable Markov kernel
$\omega(\cdot\,|\, s)$
on $\gC$.  The selector draws a command $c\sim\omega(\cdot\,|\, s)$ and executes the actual latent
$z_c(s)$.  Define its reward  by
\[
    r^\omega(s)
    :=
    \sum_{c\in\gC}
    \omega(c\,|\, s)
    r\bigl(s,z_c(s)\bigr),
\]
and transition kernel by 
\[
    P^\omega(B\,|\, s)
    :=
    \sum_{c\in\gC}
    \omega(c\,|\, s)
    P\bigl(B\,|\, s,z_c(s)\bigr),
    \quad\text{where}\quad
    B\in\Sigma_{\gS}.
\]
For a bounded measurable $V$, write
\[
    (P^\omega V)(s)
    :=
    \int_{\gS}
        V(s')
        P^\omega(\mathrm{d}s'\,|\, s).
\]
For an initial state $s$, let
$\E_s^\omega[\cdot]$
denote expectation under the Markov chain $(S_t)_{t\ge0}$ satisfying $S_{t+1}\sim P^\omega(\cdot\,|\,S_t)$ with 
$S_0=s$.
Define the policy Bellman operator
\begin{align}
    (\gT^\omega V)(s)
    &:=
    r^\omega(s)
    +
    \gamma(P^\omega V)(s)=
    \sum_{c\in\gC}
    \omega(c\,|\, s)
    Q_V\bigl(s,z_c(s)\bigr).
\end{align}
The value $V^\omega$ is the unique bounded fixed point of $\gT^\omega$.

\begin{lemma}[Restricted Bellman fixed point and optimal stationary selector]
    \label{lem:restricted-bellman-fixed-point}
    The operator $\gT_{\mF}$ and, for every stationary flow-reversal selector $\omega$, the operator
    $\gT^\omega$ are monotone $\gamma$-contractions in the supremum norm.  Hence
    $\gT_{\mF}$ has a unique bounded fixed point $V_{\mF}^\star$, and
    $\gT^\omega$ has a unique bounded fixed point $V^\omega$---namely, 
    \[
        V_{\mF}^\star
        =
        \gT_{\mF}V_{\mF}^\star,
        \qquad\text{and}\qquad
        V^\omega
        =
        \gT^\omega V^\omega.
    \]
    There exists a measurable deterministic stationary flow-reversal selector
    $\omega_{\mF}^\star$ such that
   $V^{\omega_{\mF}^\star}
        =
        V_{\mF}^\star$.
    Moreover, for every stationary flow-reversal selector $\omega$, the bounds hold:
    \[
        V^\omega(s)
        \le
        V_{\mF}^\star(s)
        \le
        V^\star(s)
        \qquad
        \text{for every }s\in\gS.
    \]
\end{lemma}

\begin{proof}
    Let $V,W$ be bounded measurable functions.  For each $s\in\gS$, we have that 
    \begin{align*}
        \left|
            (\gT_{\mF}V)(s)
            -
            (\gT_{\mF}W)(s)
        \right|
        &\le
        \max_{c\in\gC}
        \left|
            Q_V\bigl(s,z_c(s)\bigr)
            -
            Q_W\bigl(s,z_c(s)\bigr)
        \right|
        \\
        &=
        \gamma
        \max_{c\in\gC}
        \left|
            \int_{\gS}
                \bigl(V(s')-W(s')\bigr)
                P(\mathrm{d}s'\mid s,z_c(s))
        \right|
        \\
        &\le
        \gamma
        \lVert V-W\rVert_\infty.
    \end{align*}
    Taking the supremum over $s$ proves
    \[
        \lVert
            \gT_{\mF}V
            -
            \gT_{\mF}W
        \rVert_\infty
        \le
        \gamma
        \lVert V-W\rVert_\infty.
    \]
    Thus $\gT_{\mF}$ is a contraction.  It is monotone because
    $V\le W$ pointwise implies
    $Q_V(s,z)\le Q_W(s,z)$ for every $(s,z)$, and taking a maximum over $c$ preserves this
    inequality.  The contraction mapping theorem gives a unique bounded fixed point
    $V_{\mF}^\star$.

    Fix a stationary flow-reversal selector $\omega$.  For bounded measurable $V$ and $W$, and every
    state $s$,
    \begin{align*}
        \left|
            (\gT^\omega V)(s)
            -
            (\gT^\omega W)(s)
        \right|
        &\le
        \sum_{c\in\gC}
        \omega(c\mid s)
        \left|
            Q_V\bigl(s,z_c(s)\bigr)
            -
            Q_W\bigl(s,z_c(s)\bigr)
        \right|\leq        \gamma\lVert V-W\rVert_\infty.
    \end{align*}
    Taking the supremum over $s$ shows that $\gT^\omega$ is a $\gamma$-contraction.  The same
    pointwise-order argument used for $\gT_{\mF}$ shows that $\gT^\omega$ is monotone.
    The contraction mapping theorem therefore gives a unique bounded fixed point $V^\omega$.

    For every bounded $V$ and state $s$, we have that 
    \[
        (\gT^\omega V)(s)
        =
        \sum_{c\in\gC}
        \omega(c\mid s)
        Q_V\bigl(s,z_c(s)\bigr)
        \le
        \max_{c\in\gC}
        Q_V\bigl(s,z_c(s)\bigr)
        =
        (\gT_{\mF}V)(s).
    \]
    In particular, we have that
    $V^\omega
        =
        \gT^\omega V^\omega
        \le
        \gT_{\mF}V^\omega$.
    Iterating the monotone operator $\gT_{\mF}$ yields
    \[
        V^\omega
        \le
        \gT_{\mF}V^\omega
        \le
        \gT_{\mF}^2V^\omega
        \le
        \cdots.
    \]
    Since $\gT_{\mF}$ is a contraction, the sequence
    $\gT_{\mF}^kV^\omega$ converges in supremum norm to its unique fixed point
    $V_{\mF}^\star$.  Passing to the limit yields
    $V^\omega
        \le
        V_{\mF}^\star$.

    Define
    \[
        c_{\mF}^\star(s)
        :=
        \min
        \argmax_{c\in\gC}
        Q_{V_{\mF}^\star}\bigl(s,z_c(s)\bigr).
    \]
    For each $c\in\gC$, set
    $q_c^{\mF}(s)
        :=
        Q_{V_{\mF}^\star}\bigl(s,z_c(s)\bigr)$.
    Every function $q_c^{\mF}$ is measurable.  For a fixed $c\in\gC$,
    \[
        \{s:c_{\mF}^\star(s)=c\}
        =
        \bigg(
            \bigcap_{j\in\gC}
            \{s:q_c^{\mF}(s)\ge q_j^{\mF}(s)\}
        \bigg)
        \cap
        \bigg(
            \bigcap_{\ell<c}
            \{s:q_\ell^{\mF}(s)<q_c^{\mF}(s)\}
        \bigg).
    \]
    This is a finite intersection of measurable sets.  Hence
    $c_{\mF}^\star$ is measurable and defines a deterministic stationary selector
    $\omega_{\mF}^\star$.  By construction, we have that 
    \begin{align*}
        (\gT^{\omega_{\mF}^\star}V_{\mF}^\star)(s)
        &=
        \max_{c\in\gC}
        Q_{V_{\mF}^\star}\bigl(s,z_c(s)\bigr)
        =
        (\gT_{\mF}V_{\mF}^\star)(s)
        =
        V_{\mF}^\star(s).
    \end{align*}
    Therefore $V_{\mF}^\star$ is a fixed point of
    $\gT^{\omega_{\mF}^\star}$.  By uniqueness of the policy fixed point, we have that
    $V^{\omega_{\mF}^\star}
        =
        V_{\mF}^\star$.

    Finally, for every bounded $V$, the bound 
    $\gT_{\mF}V
        \le
        \gT V$ holds
    pointwise since the flow-reversed latent action bank is a subset of the full latent action space.  Hence we deduce that 
    \[
        \gT_{\mF}V^\star
        \le
        \gT V^\star
        =
        V^\star.
    \]
    Applying $\gT_{\mF}$ repeatedly and using monotonicity gives
    \[
        \gT_{\mF}^kV^\star
        \le
        V^\star
        \qquad
        \text{for every }k\ge1.
    \]
    The left-hand side converges in supremum norm to $V_{\mF}^\star$, so
    $V_{\mF}^\star
        \le
        V^\star$.
\end{proof}

Define the optimal restricted action-value function on the actual flow-reversed latents by
\begin{equation}
    \label{eq:restricted-optimal-q}
    Q_{\mF}^\star\bigl(s,z_c(s)\bigr)
    :=
    r\bigl(s,z_c(s)\bigr)
    +
    \gamma
    \int_{\gS}
        V_{\mF}^\star(s')
        P(\mathrm{d}s'\mid s,z_c(s)).
\end{equation}
Then the optimal value is given by 
\[
    V_{\mF}^\star(s)
    =
    \max_{c\in\gC}
    Q_{\mF}^\star\bigl(s,z_c(s)\bigr).
\]

\subsection{Main Result: Support gap and selection error}
\label{subsec:support-gap-selection-error}
We now define some key quantities that help decompose the error. 
\begin{definition}[One-step flow-reversal support gap]
    \label{def:one-step-frs-support-gap}
    For each state $s\in\gS$, define
        $\Delta_{\mF}(s)
        :=
        V^\star(s)
        -
        \max_{c\in\gC}
        Q^\star\bigl(s,z_c(s)\bigr)$, and define the worst-case support gap by
        $\bar{\Delta}_{\mF}
        :=
        \sup_{s\in\gS}
        \Delta_{\mF}(s)$.
\end{definition}

Since $V^\star(s)
    =
    \sup_{z\in\gZ}Q^\star(s,z)$
and every $z_c(s)$ belongs to $\gZ$, we have that
$
    \Delta_{\mF}(s)\ge0$.
The quantity $\Delta_{\mF}(s)$ forces only the current latent to lie in the latent bank generated via flow-reversal; the
continuation value inside $Q^\star$ still assumes globally optimal future control over the full
latent space.

\begin{definition}[Restricted selection error]
    \label{def:restricted-selection-error}
    For a stationary flow-reversal selector $\omega$, define its pointwise selection error by
    \begin{equation}
                \label{eq:pointwise-selection-error}\operatorname{err}_{\mathrm{sel}}(\omega;s)
        :=
        V_{\mF}^\star(s)
        -
        (\gT^\omega V_{\mF}^\star)(s)
 =
        V_{\mF}^\star(s)
        -
        \sum_{c\in\gC}
        \omega(c\mid s)
        Q_{\mF}^\star\bigl(s,z_c(s)\bigr).
    \end{equation}
    Its uniform selection error is
    \begin{equation}
        \label{eq:uniform-selection-error}
        \operatorname{err}_{\mathrm{sel}}(\omega)
        :=
        \sup_{s\in\gS}
        \operatorname{err}_{\mathrm{sel}}(\omega;s).
    \end{equation}
\end{definition}

Since $V_{\mF}^\star(s)$ is the maximum of
$Q_{\mF}^\star(s,z_c(s))$ over $c\in\gC$ and $\omega(\cdot\mid s)$ is a probability
distribution, both quantities in
\eqref{eq:pointwise-selection-error}--\eqref{eq:uniform-selection-error} are nonnegative.

The following theorem characterizes the price of abstraction and finite-sample selection error. 
\begin{theorem}[General-MDP abstraction price and finite-sample selection error]
    \label{thm:general-mdp-abstraction-price}
 The following statements hold.
    \begin{enumerate}[label=\textnormal{\it\alph*.}, leftmargin=25pt, topsep=0pt,itemsep=0pt]
        \item \textbf{Oracle price of the online flow-reversal abstraction.}
        For every state $s\in\gS$,
       the bounds hold: \begin{equation}
            \label{eq:oracle-abstraction-sandwich}
            \Delta_{\mF}(s)
            \le
            V^\star(s)-V_{\mF}^\star(s)
            \le
            \bar{\Delta}_{\mF}/(1-\gamma).
        \end{equation}

\item \textbf{Exact decomposition for an arbitrary learned flow-reversal selector.}
Fix any stationary flow-reversal selector $\omega$. Then, as an identity between
bounded functions, we have the decomposition
\begin{equation}
    \label{eq:exact-support-selection-decomposition}
    V^\star-V^\omega
    =
    \bigl(V^\star-V_{\mathrm F}^\star\bigr)
    +
    (I-\gamma P^\omega)^{-1}
    \operatorname{err}_{\mathrm{sel}}(\omega;\cdot).
\end{equation}
Consequently, for every $s\in\gS$, the bounds hold:
\begin{equation}
    \label{eq:general-policy-abstraction-selection-bound}
    \Delta_{\mF}(s)
    \le
    V^\star(s)-V^\omega(s)
    \leq
    \frac{
        \bar{\Delta}_{\mF}
        +\operatorname{err}_{\mathrm{sel}}(\omega)
    }{1-\gamma}.
\end{equation}

        \item \textbf{Finite-sample consequence.}
        Let $\widehat\omega_n$ be any random stationary flow-reversal selector returned after $n$ online
        environment interactions.  Suppose that, for some
        $\alpha\in(0,1)$ and deterministic number
        $\eps_{\mathrm{sel}}(n,\alpha)\ge0$, the bound holds:
        \begin{equation}
            \label{eq:finite-sample-selection-event}
            \mathbb P
            \left(
                \operatorname{err}_{\mathrm{sel}}(\widehat\omega_n)
                \le
                \eps_{\mathrm{sel}}(n,\alpha)
            \right)
            \ge
            1-\alpha.
        \end{equation}
        Then, with probability at least $1-\alpha$, simultaneously for every
        $s\in\gS$,
        \begin{equation}
            \label{eq:finite-sample-general-mdp-bound}
            \Delta_{\mF}(s)
            \le
            V^\star(s)-V^{\widehat\omega_n}(s)
            \leq
                (\bar{\Delta}_{\mF}
                +
                \eps_{\mathrm{sel}}(n,\alpha))/(1-\gamma).
        \end{equation}
    \end{enumerate}
\end{theorem}

\begin{proof}
    We prove the three parts in order.

\noindent\textbf{Proof of part {\it a.}: lower bound.}
    By Lemma~\ref{lem:restricted-bellman-fixed-point}, we have that $V_{\mF}^\star
        \le
        V^\star$.
    The operator $\gT_{\mF}$ is monotone, so that
  $\gT_{\mF}V_{\mF}^\star
        \le
        \gT_{\mF}V^\star$.
    Using the fixed-point equation for $V_{\mF}^\star$ and the definition of
    $\Delta_{\mF}(s)$, we deduce that
    \begin{align*}
        V_{\mF}^\star(s)
        &=
        (\gT_{\mF}V_{\mF}^\star)(s)
        \le
        (\gT_{\mF}V^\star)(s)
        =
        \max_{c\in\gC}
        Q^\star\bigl(s,z_c(s)\bigr)
        =
        V^\star(s)-\Delta_{\mF}(s).
    \end{align*}
    Rearranging yields
  $\Delta_{\mF}(s)
        \le
        V^\star(s)-V_{\mF}^\star(s)$.

\noindent\textbf{Proof of part \textit{a.}: resolvent identity and upper bound.}
    For each $c\in\gC$, define the measurable function
  $q_c^\dagger(s)
        :=
        Q^\star\bigl(s,z_c(s)\bigr)$. Moreover, let
\[c^\dagger(s)
            :=
            \min\left\{c\in \gC\,:\, 
            Q^\star\bigl(s,z_c(s)\bigr)=\max_{j\in \gC} Q^\star(s,z_j(s))\right\}.\]
    The smallest-maximizer selector is measurable since, for every $c\in\gC$, we have that
    \[
        \{s:c^\dagger(s)=c\}
        =
        \Bigg(
            \bigcap_{j\in\gC}
            \{s:q_c^\dagger(s)\ge q_j^\dagger(s)\}
        \Bigg)
        \cap
        \Bigg(
            \bigcap_{\ell<c}
            \{s:q_\ell^\dagger(s)<q_c^\dagger(s)\}
        \Bigg),
    \]
    which is a finite intersection of measurable sets.  Let $\omega^\dagger$ be the associated
    deterministic FRS selector that assigns probability one to
        $c^\dagger(s)$.  
    By construction, we have that 
    \begin{align}
        (\gT^{\omega^\dagger}V^\star)(s)
        &=
        Q^\star\bigl(s,z_{c^\dagger(s)}(s)\bigr)
=
        \max_{c\in\gC}
        Q^\star\bigl(s,z_c(s)\bigr)
        =
        V^\star(s)-\Delta_{\mF}(s).
        \label{eq:oracle-full-q-greedy-residual}
    \end{align}

    Define $D^\dagger
        :=
        V^\star-V^{\omega^\dagger}$.
    Using the Bellman fixed-point equation
    $V^{\omega^\dagger}=\gT^{\omega^\dagger}V^{\omega^\dagger}$, adding and subtracting
    $\gT^{\omega^\dagger}V^\star$, and applying
    \eqref{eq:oracle-full-q-greedy-residual}, we obtain, for every state $s$,
    \begin{align*}
        D^\dagger(s)
        &=
        V^\star(s)
        -
        (\gT^{\omega^\dagger}V^{\omega^\dagger})(s)
        \\
        &=
        V^\star(s)
        -
        (\gT^{\omega^\dagger}V^\star)(s)+
        (\gT^{\omega^\dagger}V^\star)(s)
        -
        (\gT^{\omega^\dagger}V^{\omega^\dagger})(s)
        \\
        &=
        \Delta_{\mF}(s)
        +
        \gamma
        \bigl(P^{\omega^\dagger}D^\dagger\bigr)(s).
    \end{align*}
    Hence, as an identity between bounded functions, we have that 
    \begin{equation}
        \label{eq:oracle-resolvent-equation}
        (I-\gamma P^{\omega^\dagger})D^\dagger
        =
        \Delta_{\mF}(\cdot).
    \end{equation}

    Recall also that a Markov operator is nonexpansive in the supremum norm---i.e., $\lVert P^{\omega^\dagger}V\rVert_\infty
        \le
        \lVert V\rVert_\infty$.
    Therefore the Neumann series
$\sum_{t=0}^{\infty}
        \gamma^t
        (P^{\omega^\dagger})^t$
    converges in operator norm and is the inverse of
    $I-\gamma P^{\omega^\dagger}$.  Applying this inverse to
    \eqref{eq:oracle-resolvent-equation} yields
    \[
        D^\dagger
        =
        (I-\gamma P^{\omega^\dagger})^{-1}
        \Delta_{\mF}(\cdot)
        =
        \sum_{t=0}^{\infty}
        \gamma^t
        (P^{\omega^\dagger})^t
        \Delta_{\mF}(\cdot).
    \]
For every initial state $s$, the $t$-th term
    satisfies
   $\bigl((P^{\omega^\dagger})^t\Delta_{\mF}(\cdot)\bigr)(s)
        =
        \mathbb{E}_s^{\omega^\dagger}
        \left[
            \Delta_{\mF}(S_t)
        \right]$.
    The support-gap function is bounded and nonnegative.  Therefore the monotone convergence
    theorem permits interchange of the expectation and the increasing sequence of nonnegative
    partial sums.  Since $\Delta_{\mF}(s)\leq\bar{\Delta}_{\mF}$ for every state, we have that 
    \begin{align*}
        D^\dagger(s)
        &=
        \mathbb{E}_s^{\omega^\dagger}
        \left[
            \sum_{t=0}^{\infty}
            \gamma^t
            \Delta_{\mF}(S_t)
        \right]\leq
        \sum_{t=0}^{\infty}
        \gamma^t
        \bar{\Delta}_{\mF}=
        \frac{\bar{\Delta}_{\mF}}{1-\gamma}.
    \end{align*}
    The selector $\omega^\dagger$ is an admissible flow-reversal selector.  By
    Lemma~\ref{lem:restricted-bellman-fixed-point}, the bound
 $V^{\omega^\dagger}
        \le
        V_{\mF}^\star$ holds, so that 
    \[
        V^\star(s)-V_{\mF}^\star(s)
        \le
        V^\star(s)-V^{\omega^\dagger}(s)
        =
        D^\dagger(s).
    \]
    Combining this inequality with the preceding occupancy and uniform bounds proves
    \eqref{eq:oracle-state-dependent-upper} and the upper bound in
    \eqref{eq:oracle-abstraction-sandwich}.

\noindent\textbf{Proof of part \textnormal{\it b}.}
    Fix an arbitrary stationary flow-reversal selector $\omega$, and define
$D_{\mF}^\omega
        :=
        V_{\mF}^\star-V^\omega$.
    Using the Bellman fixed-point equation
    $V^\omega=\gT^\omega V^\omega$, adding and subtracting
    $\gT^\omega V_{\mF}^\star$, and using
    Definition~\ref{def:restricted-selection-error}, we obtain, for every state $s$, we deduce that 
    \begin{align*}
        D_{\mF}^\omega(s)
        &=
        V_{\mF}^\star(s)
        -
        (\gT^\omega V^\omega)(s)
        \\
        &=
        V_{\mF}^\star(s)
        -
        (\gT^\omega V_{\mF}^\star)(s)+
        (\gT^\omega V_{\mF}^\star)(s)
        -
        (\gT^\omega V^\omega)(s)
        \\
        &=
        \operatorname{err}_{\mathrm{sel}}(\omega;s)
        +
        \gamma
        \bigl(P^\omega D_{\mF}^\omega\bigr)(s).
    \end{align*}
    Hence
 $(I-\gamma P^\omega)D_{\mF}^\omega
        =
        \operatorname{err}_{\mathrm{sel}}(\omega;\cdot)$.
    As above, the Neumann series yields
    \begin{align}
        D_{\mF}^\omega
        &=
        (I-\gamma P^\omega)^{-1}
        \operatorname{err}_{\mathrm{sel}}(\omega;\cdot)=
        \sum_{t=0}^{\infty}
        \gamma^t
        (P^\omega)^t
        \operatorname{err}_{\mathrm{sel}}(\omega;\cdot).
        \label{eq:selection-resolvent-identity}
    \end{align}        
    The selection-error function is bounded and nonnegative.  Thus the monotone convergence
    theorem again permits interchange of expectation and the nonnegative discounted series.
    Therefore, for every initial state $s$, we have that 
    \begin{align}
        V_{\mF}^\star(s)-V^\omega(s)
        &=
        \mathbb{E}_s^\omega
        \left[
            \sum_{t=0}^{\infty}
            \gamma^t
            \operatorname{err}_{\mathrm{sel}}(\omega;S_t)
        \right]\leq
            \operatorname{err}_{\mathrm{sel}}(\omega)
        /(1-\gamma).
        \label{eq:selection-value-upper}
    \end{align}

    The algebraic identity $V^\star-V^\omega
        =
        \bigl(V^\star-V_{\mF}^\star\bigr)
        +
        \bigl(V_{\mF}^\star-V^\omega\bigr)$
    combined with \eqref{eq:selection-resolvent-identity} proves \eqref{eq:exact-support-selection-decomposition}.
    Its occupancy form---namely,
        for every $s\in\gS$, the form 
            $V^\star(s)-V^\omega(s)
            =
            V^\star(s)-V_{\mF}^\star(s)
            +
            \mathbb{E}_s^\omega
            \left[
                \sum_{t=0}^{\infty}
                \gamma^t
                \operatorname{err}_{\mathrm{sel}}(\omega;S_t)
            \right]$---follows from the first equality in
    \eqref{eq:selection-value-upper}.

    Since $V^\omega\le V_{\mF}^\star$, part~\textnormal{\it a.} implies that 
$V^\star(s)-V^\omega(s)
        \ge
        V^\star(s)-V_{\mF}^\star(s)
        \ge
        \Delta_{\mF}(s)$.
    For the upper bound, part~\textnormal{\it a.} and
    \eqref{eq:selection-value-upper} together yield
    \begin{align*}
        V^\star(s)-V^\omega(s)
        &=
        V^\star(s)-V_{\mF}^\star(s)
        +
        V_{\mF}^\star(s)-V^\omega(s)\leq
        \frac{
           \bar{\Delta}_{\mF}
            +
            \operatorname{err}_{\mathrm{sel}}(\omega)
        }{1-\gamma}.
    \end{align*}
    This proves \eqref{eq:general-policy-abstraction-selection-bound}.

\noindent\textbf{Proof of part \textnormal{\it c}.}
    Define the event $\gE_n
        :=
        \left\{
            \operatorname{err}_{\mathrm{sel}}(\widehat\omega_n)
            \le
            \eps_{\mathrm{sel}}(n,\alpha)
        \right\}$.
    By assumption, we have  $\mathbb P(\gE_n)
        \ge
        1-\alpha$.
    On every outcome in $\gE_n$, the realized selector $\widehat\omega_n$ is a stationary flow-reversal
    selector satisfying
$\operatorname{err}_{\mathrm{sel}}(\widehat\omega_n)
        \leq
        \eps_{\mathrm{sel}}(n,\alpha)$.
    Applying the deterministic bound
    \eqref{eq:general-policy-abstraction-selection-bound} to that realized selector yields,
    simultaneously for every state $s$, that 
    \[
        \Delta_{\mF}(s)
        \le
        V^\star(s)-V^{\widehat\omega_n}(s)
        \leq
            (\bar{\Delta}_{\mF}
            +
            \eps_{\mathrm{sel}}(n,\alpha))
    /(1-\gamma).
    \]
    Since this statement holds on an event of probability at least $1-\alpha$, it proves
    \eqref{eq:finite-sample-general-mdp-bound}.
\end{proof}

\subsection{From critic accuracy to selection error}
\label{subsec:critic-to-selection-error}

Theorem~\ref{thm:general-mdp-abstraction-price} does not assume a particular critic objective or
candidate-selection rule.  The following corollary shows how a conventional finite-sample critic
bound can be inserted into the abstract selection term.

In particular, Theorem~\ref{thm:general-mdp-abstraction-price} has an abstract term $\operatorname{err}_{\text{sel}}$; the next corollary shows  what that abstract error means for an ordinary critic-based implementation. Indeed, it shows
\[\text{
selection error}\leq 2 \cdot \text{critic error}+\text{failure to maximize the critic}.\]
The last sentence is simply identifying the common special case where the final term vanishes.

\begin{corollary}[Approximate critic and approximate candidate selection]
    \label{cor:critic-to-general-mdp-control}
    Let $\widehat Q_n$ be any random score defined on the state-dependent  latent bank generated via flow-reversal, and let
    $\widehat\omega_n$ be any random stationary flow-reversal selector.  Suppose that, with probability at
    least $1-\alpha$, the following two inequalities hold simultaneously:
    \begin{equation}
        \label{eq:restricted-critic-uniform-error}
        \sup_{s\in\gS}
        \max_{c\in\gC}
        \left|
            \widehat Q_n\bigl(s,z_c(s)\bigr)
            -
            Q_{\mF}^\star\bigl(s,z_c(s)\bigr)
        \right|
        \le
        \eps_Q(n,\alpha),
    \end{equation}
    and, for every $s\in\gS$,
    \begin{equation}
        \label{eq:approximate-score-maximization}
        \sum_{c\in\gC}
        \widehat\omega_n(c\mid s)
        \widehat Q_n\bigl(s,z_c(s)\bigr)
        \ge
        \max_{c\in\gC}
        \widehat Q_n\bigl(s,z_c(s)\bigr)
        -
        \eps_{\mathrm{opt}}(n,\alpha).
    \end{equation}
    Then, on the same event,
    \begin{equation}
        \label{eq:critic-implies-selection-error}
        \operatorname{err}_{\mathrm{sel}}(\widehat\omega_n)
        \le
        2\eps_Q(n,\alpha)
        +
        \eps_{\mathrm{opt}}(n,\alpha).
    \end{equation}
    Consequently, with probability at least $1-\alpha$, simultaneously for every state $s$,
    \begin{equation}
        \label{eq:critic-general-mdp-control-bound}
        \Delta_{\mF}(s)
        \le
        V^\star(s)-V^{\widehat\omega_n}(s)
        \le
        \frac{
            \bar{\Delta}_{\mF}
            +
            2\eps_Q(n,\alpha)
            +
            \eps_{\mathrm{opt}}(n,\alpha)
        }{1-\gamma}.
    \end{equation}
    In particular, exact greedy selection with respect to $\widehat Q_n$ corresponds to
    $\eps_{\mathrm{opt}}(n,\alpha)=0$.
\end{corollary}

\begin{proof}
    Work on an outcome on which
    \eqref{eq:restricted-critic-uniform-error} and
    \eqref{eq:approximate-score-maximization} both hold.  Fix an arbitrary state $s\in\gS$.
    For brevity, write
    \[
        q_c
        :=
        Q_{\mF}^\star\bigl(s,z_c(s)\bigr)
        \qquad\text{and}\qquad
        \widehat q_c
        :=
        \widehat Q_n\bigl(s,z_c(s)\bigr).
    \]
    By \eqref{eq:restricted-critic-uniform-error}, we have that 
    \[
        q_c
        \le
        \widehat q_c
        +
        \eps_Q(n,\alpha)
        \qquad
        \text{for every }c\in\gC.
    \]
    Taking the maximum over $c$ gives us that 
    \begin{equation}
        \label{eq:true-max-to-estimated-max}
        \max_{c\in\gC}q_c
        \le
        \max_{c\in\gC}\widehat q_c
        +
        \eps_Q(n,\alpha).
    \end{equation}
    By \eqref{eq:approximate-score-maximization}, we have that 
    \begin{equation}
        \label{eq:estimated-max-to-selected-estimated}
        \max_{c\in\gC}\widehat q_c
        \le
        \sum_{c\in\gC}
        \widehat\omega_n(c\mid s)
        \widehat q_c
        +
        \eps_{\mathrm{opt}}(n,\alpha).
    \end{equation}
    The other direction of the uniform critic bound yields
    $\widehat q_c
        \le
        q_c
        +
        \eps_Q(n,\alpha)$
        for every $c\in\gC$.
    Averaging with respect to $\widehat\omega_n(\cdot\mid s)$ then gives us that 
    \begin{equation}
        \label{eq:selected-estimated-to-selected-true}
        \sum_{c\in\gC}
        \widehat\omega_n(c\mid s)
        \widehat q_c
        \le
        \sum_{c\in\gC}
        \widehat\omega_n(c\mid s)
        q_c
        +
        \eps_Q(n,\alpha).
    \end{equation}
    Combining
    \eqref{eq:true-max-to-estimated-max}--\eqref{eq:selected-estimated-to-selected-true}
     yields
    \[
        \max_{c\in\gC}q_c
        \le
        \sum_{c\in\gC}
        \widehat\omega_n(c\mid s)
        q_c
        +
        2\eps_Q(n,\alpha)
        +
        \eps_{\mathrm{opt}}(n,\alpha).
    \]
    Since
$\max_{c\in\gC}q_c
        =
        V_{\mF}^\star(s)$,
    rearranging the above inequality proves
    \[
        \operatorname{err}_{\mathrm{sel}}(\widehat\omega_n;s)
        \le
        2\eps_Q(n,\alpha)
        +
        \eps_{\mathrm{opt}}(n,\alpha).
    \]
    The state $s$ is arbitrary, so taking the supremum over $s$ proves
    \eqref{eq:critic-implies-selection-error}.  Applying 
    Theorem~\ref{thm:general-mdp-abstraction-price}.{\it c.}~with
$\eps_{\mathrm{sel}}(n,\alpha)
        =
        2\eps_Q(n,\alpha)
        +
        \eps_{\mathrm{opt}}(n,\alpha)$
    proves \eqref{eq:critic-general-mdp-control-bound}.
\end{proof}

The corollary is basically a bridge from \emph{my critic is accurate} to \emph{my learned \QFRS{} policy has good control performance.} 

It says the total error comes from:
\[\bar\Delta_{\rm F}+2\eps_Q(n,\alpha)+\eps_{\text{opt}}(n,\alpha),\]
where the first term captures that 
flow-reversing action anchors does not contain the right behavior, the second term captures that 
critic has not learned the available behaviors accurately, and the last term captures that 
we do not fully exploit the learned critic.
Those are three very different failure modes.
The first is a 
structural error.
More data does not fix this.
If the right behavior is missing from the latent bank generated by flow-reversing anchors, even a perfect critic and perfect selector cannot recover it.
The second 
is a statistical learning error.
More/better data can reduce this.
The right candidate is available, but the critic cannot yet tell which candidate is good.
And, the last term is the candidate-selection error.
This depends on how aggressively you optimize the learned scores.
Exact greedy selection removes it completely:
$
\eps_\text{opt}=0$. 

\subsection{Sharpness and regime consequences}
\label{subsec:abstraction-price-sharpness}
Theorem~\ref{thm:general-mdp-abstraction-price} bounds the loss of the
flow-reversal-restricted policy using a one-step support gap and a one-step selection error.
A natural question is whether the resulting factors are merely artifacts of the
Bellman argument.  The following proposition shows that they are not: the
one-step support gap can equal the entire control loss in a one-decision problem,
and both support error and selection error can accumulate over the full effective
discounted horizon $1/(1-\gamma)$.

\begin{proposition}[Worst-case tightness of the abstraction and selection bounds]
\label{prop:abstraction-price-sharpness}
Each of the three bounds appearing in
Theorem~\ref{thm:general-mdp-abstraction-price} is attained by some discounted
latent MDP:

\begin{enumerate}[label=\textnormal{\it \alph*.}, itemsep=0pt,
                  topsep=-2pt, leftmargin=20pt]

    \item \textbf{Tightness of the one-step lower bound.}
    For every $d>0$, there exists a one-decision latent MDP and an
    flow-reversal action set for which
    \[
        \Delta_{\mF}(s_0)
        =
        V^\star(s_0)-V_{\mF}^\star(s_0)
        =
        d.
    \]

    \item \textbf{Tightness of support-error accumulation.}
    For every $\gamma\in[0,1)$ and $d>0$, there exists a one-state
    discounted latent MDP and an flow-reversal action set for which
    \[
        \Delta_{\mF}(s)=d,
        \qquad
        V^\star(s)-V_{\mF}^\star(s)
        =
        \frac{d}{1-\gamma}.
    \]

    \item \textbf{Tightness of selection-error accumulation.}
    For every $\gamma\in[0,1)$ and $\xi>0$, there exists a one-state
    flow-reversal-restricted latent MDP and an flow-reversal selector $\omega$ for which
    \[
        \operatorname{err}_{\mathrm{sel}}(\omega;s)=\xi,
        \qquad
        V_{\mF}^\star(s)-V^\omega(s)
        =
        \frac{\xi}{1-\gamma}.
    \]
\end{enumerate}
\end{proposition}

\begin{proof}
We prove each part separately. 

\noindent\textbf{Proof of part \textnormal{\it a}.}
    Consider an MDP with one decision state $s_0$ and one terminal state.  At $s_0$, the full
    latent action set contains two latents, $z^\star$ and $z_{\mF}$.  Executing either latent
    transitions immediately to the terminal state.  Set
    \[
        r(s_0,z^\star)=d
        \qquad\text{and}\qquad
        r(s_0,z_{\mF})=0,
    \]
    and let the latent bank generated via flow-reversal contain only $z_{\mF}$.  Then
    \[
        V^\star(s_0)=d,
        \qquad
        V_{\mF}^\star(s_0)=0,
        \qquad\text{and}\;\;
        Q^\star(s_0,z_{\mF})=0.
    \]
    Therefore
    \[
        \Delta_{\mF}(s_0)
        =
        d-0
        =
        d
        =
        V^\star(s_0)-V_{\mF}^\star(s_0).
    \]

\noindent\textbf{Proof of part \textnormal{\it b}.}
    Consider an MDP with a single nonterminal state $s$.  Both available latents return
    deterministically to $s$.  Let the full action set contain $z^\star$ and $z_{\mF}$, set
    \[
        r(s,z^\star)=d
        \qquad\text{and}\qquad
        r(s,z_{\mF})=0,
    \]
    and let the flow-reversed latent bank contain only $z_{\mF}$.  The full optimal policy repeatedly executes
    $z^\star$, so
    \[
        V^\star(s)
        =
        d+\gamma V^\star(s)
        =
        \frac{d}{1-\gamma}.
    \]
    The restricted policy repeatedly executes $z_{\mF}$, so
    \[
        V_{\mF}^\star(s)
        =
        0+\gamma V_{\mF}^\star(s)
        =
        0.
    \]
    When the current latent is forced to be $z_{\mF}$ but the continuation is valued by the
    full optimum,
    \[
        Q^\star(s,z_{\mF})
        =
        0+\gamma V^\star(s)
        =
        \frac{\gamma d}{1-\gamma}.
    \]
    Hence
    \[
        \Delta_{\mF}(s)
        =
        V^\star(s)-Q^\star(s,z_{\mF})
        =
        \frac{d}{1-\gamma}
        -
        \frac{\gamma d}{1-\gamma}
        =
        d,
    \]
    while
    \[
        V^\star(s)-V_{\mF}^\star(s)
        =
        \frac{d}{1-\gamma}.
    \]

\noindent\textbf{Proof of part \textnormal{\it c}.}
    Consider a one-state restricted MDP whose flow-reversed  bank contains two latents $z_1$ and $z_2$.  Both
    return deterministically to the same state $s$.  Set
    \[
        r(s,z_1)=\xi
        \qquad\text{and}\qquad
        r(s,z_2)=0.
    \]
    The restricted optimal policy repeatedly executes $z_1$, and hence
    \[
        V_{\mF}^\star(s)
        =
        \xi+\gamma V_{\mF}^\star(s)
        =
        \frac{\xi}{1-\gamma}.
    \]
    Let $\omega$ select $z_2$ with probability one.  Then
   $V^\omega(s)=0$.
    Moreover, we have that 
    \[
        Q_{\mF}^\star(s,z_2)
        =
        0+\gamma V_{\mF}^\star(s)
        =
        \frac{\gamma \xi}{1-\gamma},
    \]
    so that 
    \begin{align*}
        \operatorname{err}_{\mathrm{sel}}(\omega;s)
        &=
        V_{\mF}^\star(s)
        -
        Q_{\mF}^\star(s,z_2)
=
        \frac{\xi }{1-\gamma}
        -
        \frac{\gamma \xi}{1-\gamma}=
        \xi,
    \end{align*}
    while
    $V_{\mF}^\star(s)-V^\omega(s)
        =\xi/(1-\gamma)$.
\end{proof}
Next we investigate the small and large bias regimes. 
\begin{corollary}[Bias regimes]
    \label{cor:small-large-bias-regimes}
    Fix a target control accuracy $\eps_{\mathrm{ctrl}}>0$ and initial state $s_0$.

    \begin{enumerate}[label=\textnormal{\it \alph*.},itemsep=0pt, topsep=-2pt, leftmargin=20pt]
        \item \textbf{Large bias.} If
$\Delta_{\mF}(s_0)
            >
            \eps_{\mathrm{ctrl}}$,
        then no policy that always executes a latent from the flow-reversed  bank can be
        $\eps_{\mathrm{ctrl}}$-optimal at $s_0$.

\item \textbf{Small bias.} Suppose that, for some $\alpha\in(0,1)$ and deterministic
$\eps_{\mathrm{sel}}(n,\alpha)\geq0$, the bound
$\mathbb P\left(
        \operatorname{err}_{\mathrm{sel}}(\widehat\omega_n)
        \leq \eps_{\mathrm{sel}}(n,\alpha)
    \right)
    \geq 1-\alpha$ holds.
If, in addition,
$\bar{\Delta}_{\mF}
    +\eps_{\mathrm{sel}}(n,\alpha)
    \leq
    (1-\gamma)\eps_{\mathrm{ctrl}}$,
then, with probability at least $1-\alpha$,
\[
    V^\star(s)-V^{\widehat\omega_n}(s)
    \leq \eps_{\mathrm{ctrl}}
    \qquad\text{for every }s\in\gS.
\]
    \end{enumerate}
\end{corollary}

\begin{proof}
For part~\textnormal{\it a.}, let $\pi$ be any possibly randomized or
history-dependent policy that always selects a latent from the FRS bank.
Let $Z_0$ denote its latent action at the initial state $s_0$. Since
$V^\star$ upper-bounds the value of any continuation policy, we have
\begin{align*}
    V^\pi(s_0)
    &\leq
    \mathbb E^\pi\left[
        r(s_0,Z_0)+\gamma V^\star(S_1)
        \,\middle|\, S_0=s_0
    \right]\\
    &=
    \mathbb E^\pi\left[
        Q^\star(s_0,Z_0)
        \,\middle|\, S_0=s_0
    \right]\\
    &\leq
    \max_{c\in\gC}
        Q^\star\bigl(s_0,z_c(s_0)\bigr)
    =
    V^\star(s_0)-\Delta_{\mF}(s_0).
\end{align*}
Therefore we deduce that 
\[
    V^\star(s_0)-V^\pi(s_0)
    \geq
    \Delta_{\mF}(s_0)
    >
    \eps_{\mathrm{ctrl}}.
\]
Thus no policy restricted to the flow-reversed bank can be
$\eps_{\mathrm{ctrl}}$-optimal at $s_0$.

    For part~\textnormal{\it b.}, apply
    \eqref{eq:finite-sample-general-mdp-bound} and the assumed inequality to get that
    \[
        V^\star(s)-V^{\widehat\omega_n}(s)
        \le
        \frac{
            \bar{\Delta}_{\mF}
            +
            \eps_{\mathrm{sel}}(n,\alpha)
        }{1-\gamma}
        \le
        \eps_{\mathrm{ctrl}}
    \]
    simultaneously for all states on the event in
    \eqref{eq:finite-sample-selection-event}, which has probability at least $1-\alpha$.
\end{proof}

\subsection{Interpretation and novelty}
\label{subsec:abstraction-price-interpretation}

\textbf{The support gap is an oracle diagnostic of the FRS abstraction.}
The quantity
\[
    \Delta_{\mF}(s)
    =
    V^\star(s)
    -
    \max_{c\in\gC}
    Q^\star\bigl(s,z_c(s)\bigr)
\]
asks a deliberately local question: if future decisions could still use the globally optimal full
latent policy, how much is lost by requiring only the current latent to come from the online flow-reversed
bank?  It therefore measures the quality of the state-dependent flow-reversed action support independently
of critic estimation, replay coverage, hard versus soft backups, or actor parameterization.

\textbf{One-step support error and repeated restriction are not the same quantity.}
The actual oracle \QFRS{} comparator must remain inside the flow-reversed bank at every future state.  Its value
loss is therefore generally larger than the one-step gap.  Part~\textnormal{(a)} shows exactly how
the local support errors may accumulate:
\[
    \Delta_{\mF}(s)
    \le
    V^\star(s)-V_{\mF}^\star(s)
    \leq\bar{\Delta}_{\mF}/(1-\gamma).
\]
The state-dependent occupancy expression in
        \begin{equation}
            \label{eq:oracle-state-dependent-upper}
            V^\star(s)-V_{\mF}^\star(s)
            \leq
            \mathbb{E}_s^{\omega^\dagger}
            \left[
                \sum_{t=0}^{\infty}
                \gamma^t
                \Delta_{\mF}(S_t)
            \right]
            \leq\bar{\Delta}_{\mF}/(1-\gamma).
        \end{equation}
is more informative than the uniform bound: large support
gaps matter only when they occur at states visited by the full-$Q^\star$-greedy flow-reversal selector.
Proposition~\ref{prop:abstraction-price-sharpness} shows that neither the lower coefficient nor the
factor $1/(1-\gamma)$ can be improved uniformly over discounted MDPs.

\textbf{The theorem separates structural bias from finite-sample learning error.}
The following expression (also given in~\eqref{eq:exact-support-selection-decomposition}) is the central decomposition:
\[
    \underbrace{V^\star-V^\omega}_{\text{total full-MDP loss}}
    =
    \underbrace{V^\star-V_{\mF}^\star}_{\text{online flow-reversal support price}}
    +
    \underbrace{
        (I-\gamma P^\omega)^{-1}
        \operatorname{err}_{\mathrm{sel}}(\omega;\cdot)
    }_{\text{discounted selection error within the bank}}.
\]
The first term is irreducible for anchor-only \QFRS{}: more data cannot remove a behavior that is not
present in the bank.  The second term is reducible: it vanishes as the learned selector approaches
the optimal selector for the restricted MDP.  Corollary~\ref{cor:critic-to-general-mdp-control}
shows that critic estimation and candidate optimization enter only through this second term.

\textbf{The result is backend agnostic.}
Theorem~\ref{thm:general-mdp-abstraction-price} makes no assumption about hard versus soft Bellman
backups, critic ensembles,  subsampling scheme (e.g., REDQ~\cite{chen2021redq}), entropy regularization, or how the selector is
parameterized.  Any learning analysis that establishes
\[
    \operatorname{err}_{\mathrm{sel}}(\widehat\omega_n)
    \le
    \eps_{\mathrm{sel}}(n,\alpha)
\]
can be substituted directly into
\eqref{eq:finite-sample-general-mdp-bound}.  These implementation choices affect how quickly the
selection error decreases; they do not change the structural support term
$\bar{\Delta}_{\mF}$.

\textbf{Connection to the sample-complexity separation.}
Theorem~\ref{thm:persistent-abstraction-separation} shows that committing to a $C$-mode persistent
FRS abstraction can reduce online identification complexity.
Theorem~\ref{thm:general-mdp-abstraction-price} bounds the value loss caused
by the FRS restriction and selection error. Corollary~\ref{cor:small-large-bias-regimes}
translates this bound into conditions for achieving a target value-suboptimality
tolerance $\eps_{\mathrm{ctrl}}$ relative to the optimal full-latent policy.
In the small-bias regime,
\[
    \bar{\Delta}_{\mF}
    +
    \eps_{\mathrm{sel}}(n,\alpha)
    \le
    (1-\gamma)\eps_{\mathrm{ctrl}},
\]
the restricted learner can meet the desired full-MDP control tolerance while retaining the smaller
action problem.  In the large-bias regime, if
$\Delta_{\mF}(s_0)
    >
    \eps_{\mathrm{ctrl}}$,
no amount of \QFRS{} data can reach the target accuracy at $s_0$.

\textbf{Persistent action abstraction versus continuous noise-space RL.}
The distinction concerns the online action class, rather than a particular
initialization scheme. In the anchor-restricted setting analyzed here,
\QFRS{} receives the statistical benefit of a smaller action class and pays
the persistent support price quantified above. Continuous noise-space RL
instead learns an actor or perturbation that selects latents beyond a fixed
anchor bank. Examples include DSRL \citep{wagenmaker2025dsrl}, FRS-prefill
followed by DSRL \citep{tang2026frs}, human-guided noise-actor learning in
UniSteer \citep{lu2026unisteer}, and the trust-region noise perturbations of
LP-DS \citep{simsir2026lpds}. These methods use different sources of guidance
and constraints, but do not impose the persistent finite FRS action
restriction studied here.

When the online class retains all latent modes, it need not incur the fixed
FRS-bank support gap. Theorem~\ref{thm:persistent-abstraction-separation}
shows that guaranteeing correctness over this broader class can nevertheless
require additional online exploration, even when the FRS bank is sufficient
on the instance encountered. FRS-only prefill is a concrete example: it
provides information about bank modes without ruling out better modes outside
the bank. This is a worst-case statement under the theorem's information and
model assumptions, not a method-specific lower bound for every noise-space
RL algorithm; additional guidance, structural assumptions, or restrictions
on reachable latents require separate analysis. Together, the two results
identify the tradeoff between lower identification complexity with explicit
support bias and broader expressiveness with a potentially higher cost of
verifying that no better latent mode exists.

\textbf{Technical Novelty.}
The contraction, Bellman-residual, and resolvent arguments used in the proof are standard tools.
The contribution is their organization around the state-dependent online flow-reversal abstraction.  In
particular, the theorem introduces a full-$Q^\star$ one-step support gap for the latent bank generated via flow-reversal, proves a
sharp bridge from that local geometric quantity to the infinite-horizon value price of persistent
flow-reversal restriction, and separates it exactly from the finite-sample error of selecting among the
current-state flow-reversed latents.  This decomposition is the general-MDP counterpart to the adaptation
separation in Theorem~\ref{thm:persistent-abstraction-separation} and provides the bias coordinate
for the particle bias--complexity frontier developed next.

\section{\QFRS{} traces a bias--complexity frontier}
\label{sec:particle-bias-complexity-frontier}
Theorem~\ref{thm:persistent-abstraction-separation} shows that a online flow-reversal action anchor
abstraction can reduce online identification complexity when a small structured action class is
sufficient.  Theorem~\ref{thm:general-mdp-abstraction-price} shows that the same restriction incurs
an explicit support price.  Exact-anchor \QFRS{} therefore represents one endpoint of a statistical
tradeoff: it has the smallest persistent action class, but it can exclude useful latents that lie
near an flow-reversed anchor rather than exactly on the anchor bank.

Particle \QFRS{} relaxes the abstraction locally.  The purpose of this section is to formalize the
resulting frontier.  We first define a deterministic local envelope (whose size is parameterized by  $\reach$) around the
state-dependent flow-reversed anchors.  We then show that increasing $\reach$ monotonically reduces the
one-step support gap while monotonically increasing, at every fixed behavioral resolution, the maximal number
of pairwise behaviorally distinguishable local candidates.  Finally, under a finite-state generative model, we give an explicit finite-sample
bound in which both quantities appear.

We use $\reach\ge0$ for the deterministic \emph{local-reach radius} in the analysis. 
The symbol $\eps_{\mathrm B}>0$ denotes
resolution in the Bellman behavioral pseudometric, whereas $\eps_z>0$ denotes Euclidean
resolution in latent space.
The deterministic local reach $\reach$ around flow-reversed action anchors is also not identical to the Gaussian particle scale
$\sigma_{p}$ used by the implementation.  A Gaussian scale determines a distribution of
finitely many random candidates and has unbounded support.  A separate result at the end of this
section gives a formal connection between a finite Gaussian cloud and a bounded local envelope of
reach $\reach$.

\subsection{Continuous local action abstractions}
\label{subsec:continuous-local-action-abstractions}

Throughout this section, retain the discounted latent-action MDP
$\gM=(\gS,\gZ,P,r,\gamma)$
from Section~\ref{sec:general-mdp-abstraction-price}, with the following additional assumptions.

\begin{assumption}[Finite-state continuous latent model]
    \label{ass:finite-state-continuous-latent}
    The state space $\gS=\{1,\ldots,S\}$ is finite with 
        $S<\infty$,
    the latent action space is
    $\gZ=\R^{d_z}$,
    and for some finite $r_{\max}>0$, rewards satisfy $r(s,z)\in [0,r_{\max}]$ for every $(s,z)\in\gS\times\gZ$.
    For every state $s$ and next state $s'$, the maps
    \[
        z\mapsto r(s,z)
        \qquad\text{and}\qquad
        z\mapsto P(s'\mid s,z)
    \]
    are continuous.  The full optimal value $V^\star$ is the unique bounded fixed point of the
    full Bellman operator $\gT$, and
    $V^\star(s)=\sup_{z\in\gZ}Q^\star(s,z)$
    for every state $s$.
\end{assumption}

Recall that  $\gC=\{1,\ldots,C\}$
is the fixed coarse-command set and let
$z_c(s)=F(s,c)$
be the corresponding online flow-reversed anchor, as in Theorem~\ref{thm:general-mdp-abstraction-price}.

\begin{definition}[Local-reach-$\reach$ flow-reversed action set]
    \label{def:radius-reach-frs-action-set}
    For every reach $\reach\ge0$ and state $s\in\gS$, define
    \begin{equation}
        \label{eq:radius-reach-frs-action-set}
        \gZ_\reach(s)
        :=
        \bigcup_{c\in\gC}
        B_2\bigl(z_c(s),\reach\bigr),
    \end{equation}
    where
$B_2(z,\reach)
        :=
        \left\{
            \tilde z\in\R^{d_z}:
            \lVert \tilde z-z\rVert_2\le \reach
        \right\}$.
    For a bounded function $V:\gS\to\R$, define the reach-$\reach$ Bellman optimality operator
    \begin{equation}
        \label{eq:radius-reach-bellman-operator}
        (\gT_\reach V)(s)
        :=
        \max_{z\in\gZ_\reach(s)}
        \left\{
            r(s,z)
            +
            \gamma
            \sum_{s'\in\gS}
            P(s'\mid s,z)V(s')
        \right\}.
    \end{equation}
\end{definition}

Since $\gZ_\reach(s)$ is a finite union of compact Euclidean balls and the maximand in
\eqref{eq:radius-reach-bellman-operator} is continuous in $z$, the maximum is attained.

A stationary reach-$\reach$ policy is a Markov kernel $\pi(\mathrm dz\,|\, s)$ supported on
$\gZ_\reach(s)$ at every state.  Let $V^\pi$ denote its value.  Write $V_\reach^\star$ for the unique
fixed point of $\gT_\reach$.

\begin{lemma}[Well-posedness of the reach-$\reach$ control problem]
    \label{lem:radius-reach-well-posedness} Suppose  Assumption~\ref{ass:finite-state-continuous-latent} holds.
    For every $\reach\ge0$, the operator $\gT_\reach$ is a monotone
    $\gamma$-contraction in the supremum norm.  Hence it has a unique bounded fixed point
    $V_\reach^\star$.  There exists a deterministic stationary reach-$\reach$ policy
    $\pi_\reach^\star$ satisfying
    $V^{\pi_\reach^\star}=V_\reach^\star$.
    Moreover, the bounds 
   $V^\pi(s)
        \leq
        V_\reach^\star(s)
        \leq
        V^\star(s)$
    hold
    for every stationary reach-$\reach$ policy $\pi$ and every state $s\in \mathcal{S}$.
\end{lemma}

\begin{proof}
    Fix bounded functions $V,W:\gS\to\R$.  For any state $s\in \mathcal{S}$,
    \begin{align*}
        \left|
            (\gT_\reach V)(s)
            -
            (\gT_\reach W)(s)
        \right|
        &\leq
        \max_{z\in\gZ_\reach(s)}
        \gamma
        \left|
            \sum_{s'\in\gS}
            P(s'\,|\, s,z)
            \bigl(V(s')-W(s')\bigr)
        \right|
        \leq
        \gamma\lVert V-W\rVert_\infty.
    \end{align*}
    Taking the maximum over $s$ yeilds
    \[
        \lVert \gT_\reach V-\gT_\reach W\rVert_\infty
        \le
        \gamma\lVert V-W\rVert_\infty.
    \]
    Thus $\gT_\reach$ is a contraction.  If $V\le W$ pointwise, then the Bellman maximand for
    $V$ is no larger than the Bellman maximand for $W$ for every $(s,z)$, so
    $\gT_\reach V\le\gT_\reach W$.  Hence $\gT_\reach$ is monotone.  The contraction mapping theorem
    gives a unique bounded fixed point $V_\reach^\star$.

    For each state $s$, continuity and compactness imply that the set
    \[
        \argmax_{z\in\gZ_\reach(s)}
        \left\{
            r(s,z)
            +
            \gamma
            \sum_{s'\in\gS}
            P(s'\mid s,z)V_\reach^\star(s')
        \right\}
    \]
    is nonempty.  Since $\gS$ is finite, select one maximizer $z_\reach^\star(s)$ at every state and
    define the deterministic stationary policy
    \[\pi_\reach^\star(\cdot\,|\, s)
        :=
        \delta_{z_\reach^\star(s)}(\cdot).\]
    Its Bellman operator $\gT^{\pi_\reach^\star}$ satisfies
    $\gT^{\pi_\reach^\star}V_\reach^\star
        =
        \gT_\reach V_\reach^\star
        =
        V_\reach^\star$.
    The policy Bellman operator is also a $\gamma$-contraction, so its bounded fixed point is
    unique.  Therefore
    $V^{\pi_\reach^\star}=V_\reach^\star$.

    For an arbitrary stationary reach-$\reach$ policy $\pi$ and any bounded $V$,
    \[
        \gT^\pi V
        \le
        \gT_\reach V
    \]
    pointwise, since $\gT^\pi V$ averages Bellman values over actions in $\gZ_\reach(s)$ whereas
    $\gT_\reach V$ takes their maximum.  Starting from $V_\reach^\star$ and iterating
    $\gT^\pi$ gives
    \[
        (\gT^\pi)^kV_\reach^\star
        \le
        V_\reach^\star
        \qquad
        \text{for every }k\ge1.
    \]
    The left-hand side converges in supremum norm to $V^\pi$, so
    $V^\pi\le V_\reach^\star$.

    Finally, for every bounded $V$, we have that 
    $\gT_\reach V
        \le
        \gT V$
    as $\gZ_\reach(s)\subseteq\gZ$.  In particular, we have 
    \[
        \gT_\reach V^\star
        \le
        \gT V^\star
        =
        V^\star.
    \]
    Monotonicity of $\gT_\reach$ gives
    $\gT_\reach^kV^\star
        \le
        V^\star$
for every $k\ge1$.
    The left-hand side converges to $V_\reach^\star$, proving
    $V_\reach^\star\le V^\star$.
\end{proof}

\begin{definition}[Reach-$\reach$ one-step support gap]
    \label{def:radius-reach-support-gap}
    For every $\reach\ge0$ and state $s\in\gS$, define
        \[\Delta_\reach(s)
        :=
        V^\star(s)
        -
        \max_{z\in\gZ_\reach(s)}Q^\star(s,z),\]
    and define the corresponding uniform support gap by
        $\bar\Delta_\reach
        :=
        \max_{s\in\gS}\Delta_\reach(s)$.
\end{definition}

Since $\gZ_\reach(s)\subseteq\gZ$, every $\Delta_\reach(s)$ is nonnegative.  At zero reach, we have 
$\gZ_0(s)
    =
    \{z_c(s):c\in\gC\}$,
so that 
\begin{equation}
    \label{eq:reach-zero-recovers-frs}
    \Delta_0(s)=\Delta_{\mF}(s),
    \qquad
    \bar\Delta_0=\bar\Delta_{\mF},
    \qquad\text{and}\qquad
    V_0^\star=V_{\mF}^\star.
\end{equation}

\begin{proposition}[Monotone expressiveness under local reach]
    \label{prop:monotone-local-reach}
    Let $0\le \reach_1\le \reach_2$.  Then, for every state $s$, the following hold:
    \begin{equation*}
        \gZ_{\reach_1}(s)
        \subseteq
        \gZ_{\reach_2}(s),
\quad 
        \Delta_{\reach_2}(s)
        \leq
        \Delta_{\reach_1}(s),
        \quad\text{and}\quad 
        V_{\reach_1}^\star(s)
        \leq
        V_{\reach_2}^\star(s)
        \leq
        V^\star(s),
    \end{equation*}
so that 
    $\bar\Delta_{\reach_2}
        \leq
        \bar\Delta_{\reach_1}$.
    In addition, for every state $s$, the limits are such that 
    \begin{equation}
        \label{eq:reach-gap-converges-pointwise}
        \lim_{\reach\to\infty}\Delta_\reach(s)=0,\quad\text{and}\quad 
    \lim_{\reach\to\infty}\bar\Delta_\reach=0, \end{equation}
   since $\gS$ is finite.
\end{proposition}

\begin{proof}
    For each anchor $z_c(s)$, we have that
    \[
        B_2\bigl(z_c(s),\reach_1\bigr)
        \subseteq
        B_2\bigl(z_c(s),\reach_2\bigr).
    \]
    Taking the union over $c\in\gC$ proves
   $\gZ_{\reach_1}(s)
        \subseteq
        \gZ_{\reach_2}(s)$.  Maximizing the same function
    $Q^\star(s,\cdot)$ over a larger set cannot decrease its maximum.  Therefore
    \[
        \max_{z\in\gZ_{\reach_2}(s)}Q^\star(s,z)
        \ge
        \max_{z\in\gZ_{\reach_1}(s)}Q^\star(s,z),
    \]
    which proves  $\Delta_{\reach_2}(s)
        \leq
        \Delta_{\reach_1}(s)$ after subtracting both sides from
    $V^\star(s)$.  Taking the maximum over the finite state set proves the uniform inequality.

    The admissible policy class at radius $\reach_1$ is contained in the admissible policy class at
    radius $\reach_2$.  Therefore optimizing over the larger class gives
    $V_{\reach_1}^\star(s)
        \le
        V_{\reach_2}^\star(s)$.
    The upper bound $V_{\reach_2}^\star\le V^\star$ follows from
    Lemma~\ref{lem:radius-reach-well-posedness}.

    To prove \eqref{eq:reach-gap-converges-pointwise}, fix a state $s$ and an arbitrary command
    $c_0\in\gC$.  Since $C\ge1$, such a command exists.  For every latent
    $z\in\R^{d_z}$,
   we have  $
        z
        \in
        B_2
        \left(
            z_{c_0}(s),
            \lVert z-z_{c_0}(s)\rVert_2
        \right)$.
    Hence
    \[
        \bigcup_{\reach\ge0}\gZ_\reach(s)
        =
        \R^{d_z}.
    \]
    Define
    $q_\reach(s)
        :=
        \max_{z\in\gZ_\reach(s)}Q^\star(s,z)$.
    The quantity $q_\reach(s)$ is nondecreasing in $\reach$ and bounded above by $V^\star(s)$.
    Fix $\eta>0$.  By
    $V^\star(s)=\sup_{z\in\R^{d_z}}Q^\star(s,z)$,
    there exists $z_\eta\in\R^{d_z}$ such that
    $Q^\star(s,z_\eta)
        \ge
        V^\star(s)-\eta$.
    Whenever
    $\reach
        \ge
        \lVert z_\eta-z_{c_0}(s)\rVert_2$,
    the point $z_\eta$ belongs to $\gZ_\reach(s)$, so
    \[
        q_\reach(s)
        \ge
        V^\star(s)-\eta.
    \]
    Since also $q_\reach(s)\le V^\star(s)$, it follows that
    $\lim_{\reach\to\infty}q_\reach(s)
        =
        V^\star(s)$.
    Therefore
    \[
        \lim_{\reach\to\infty}\Delta_\reach(s)
        =
        \lim_{\reach\to\infty}
        \bigl(V^\star(s)-q_\reach(s)\bigr)
        =0,
    \]
    proving \eqref{eq:reach-gap-converges-pointwise}.

    Finally, fix $\eta>0$.  For every $s\in\gS$, pointwise convergence gives a finite
    $\reach_s(\eta)$ such that
    \[
        \Delta_\reach(s)\leq\eta
        \qquad
        \text{whenever }\reach\geq\reach_s(\eta).
    \]
    Since $\gS$ is finite, define 
    \[
        \reach_{\eta}
        :=
        \max_{s\in\gS}\reach_s(\eta)
        <\infty.
    \]
    Thus, for every $\reach\geq\reach_{\eta}$, we have that 
    \[
        \bar\Delta_\reach
        =
        \max_{s\in\gS}\Delta_\reach(s)
        \leq
        \eta.
    \]
    Since $\eta>0$ was arbitrary, $\lim_{\reach\to\infty}\bar\Delta_\reach=0$
   follows.
\end{proof}

\begin{proposition}[Reach-$\reach$ abstraction price]
    \label{prop:radius-reach-abstraction-price}
    For every $\reach\ge0$ and state $s$,
    \begin{equation}
        \label{eq:radius-reach-abstraction-sandwich}
        \Delta_\reach(s)
        \le
        V^\star(s)-V_\reach^\star(s)
        \le
        \bar\Delta_\reach/(1-\gamma).
    \end{equation}
    More precisely, let $z_\reach^\dagger(s)$ be any maximizer of
    $Q^\star(s,\cdot)$ over $\gZ_\reach(s)$ and let $\pi_\reach^\dagger$ be the deterministic policy
    selecting $z_\reach^\dagger(s)$.  Then
    \[
        V^\star-V^{\pi_\reach^\dagger}
        =
        (I-\gamma P^{\pi_\reach^\dagger})^{-1}
        \Delta_\reach(\cdot)=
        \sum_{t=0}^{\infty}
        \gamma^t
        (P^{\pi_\reach^\dagger})^t
        \Delta_\reach(\cdot),\]
    and hence
    \begin{equation}
        \label{eq:radius-reach-occupancy-bound}
        V^\star(s)-V_\reach^\star(s)
        \leq
        \mathbb{E}_s^{\pi_\reach^\dagger}
        \left[
            \sum_{t=0}^{\infty}
            \gamma^t
            \Delta_\reach(S_t)
        \right].
    \end{equation}
\end{proposition}

\begin{proof}
    By Lemma~\ref{lem:radius-reach-well-posedness}, we have that
    $V_\reach^\star\leq V^\star$.
    Monotonicity of $\gT_\reach$ therefore yields
    $\gT_\reach V_\reach^\star
        \leq
        \gT_\reach V^\star$.
    Using the fixed-point identity for $V_\reach^\star$ and the definition of
    $\Delta_\reach$,
    \begin{align*}
        V_\reach^\star(s)
        &=
        (\gT_\reach V_\reach^\star)(s)
  \leq
        (\gT_\reach V^\star)(s)=
        \max_{z\in\gZ_\reach(s)}Q^\star(s,z)=
        V^\star(s)-\Delta_\reach(s).
    \end{align*}
    Rearranging proves the lower bound in
    \eqref{eq:radius-reach-abstraction-sandwich}.

    For every state $s$, the maximum defining $z_\reach^\dagger(s)$ is attained by compactness and
    continuity.  Since $\gS$ is finite, the deterministic selector is well defined.  By
    construction,
    \begin{equation}
        \label{eq:radius-reach-greedy-residual}
        (\gT^{\pi_\reach^\dagger}V^\star)(s)
        =
        V^\star(s)-\Delta_\reach(s).
    \end{equation}
    Define
    $D_\reach^\dagger
        :=
        V^\star-V^{\pi_\reach^\dagger}$.
    Using the Bellman fixed-point identity
    $V^{\pi_\reach^\dagger}=\gT^{\pi_\reach^\dagger}V^{\pi_\reach^\dagger}$, adding and subtracting
    $\gT^{\pi_\reach^\dagger}V^\star$, and applying
    \eqref{eq:radius-reach-greedy-residual}, we obtain
    \begin{align*}
        D_\reach^\dagger(s)
        &=
        V^\star(s)
        -
        (\gT^{\pi_\reach^\dagger}V^{\pi_\reach^\dagger})(s)\\
&   =
        V^\star(s)
        -
        (\gT^{\pi_\reach^\dagger}V^\star)(s)
       +
        (\gT^{\pi_\reach^\dagger}V^\star)(s)
        -
        (\gT^{\pi_\reach^\dagger}V^{\pi_\reach^\dagger})(s)\\
       & =        \Delta_\reach(s)
        +
        \gamma
        (P^{\pi_\reach^\dagger}D_\reach^\dagger)(s).
    \end{align*}
    Hence
    $(I-\gamma P^{\pi_\reach^\dagger})D_\reach^\dagger
        =
        \Delta_\reach(\cdot)$.
    A Markov operator has supremum-norm operator norm at most one.  Since $\gamma<1$, the Neumann
    series
    $\sum_{t=0}^{\infty}
        \gamma^t(P^{\pi_\reach^\dagger})^t$
    converges in operator norm and equals
    $(I-\gamma P^{\pi_\reach^\dagger})^{-1}$.  This proves $ V^\star-V^{\pi_\reach^\dagger}
        =
        (I-\gamma P^{\pi_\reach^\dagger})^{-1}
        \Delta_\reach(\cdot)$.

    Evaluating the Neumann series at state $s$ gives
    \[
        V^\star(s)-V^{\pi_\reach^\dagger}(s)
        =
        \mathbb{E}_s^{\pi_\reach^\dagger}
        \left[
            \sum_{t=0}^{\infty}
            \gamma^t\Delta_\reach(S_t)
        \right].
    \]
    Since $\pi_\reach^\dagger$ is an admissible reach-$\reach$ policy,
    $V_\reach^\star(s)
        \ge
        V^{\pi_\reach^\dagger}(s)$.
    Therefore
    \[
        V^\star(s)-V_\reach^\star(s)
        \le
        V^\star(s)-V^{\pi_\reach^\dagger}(s),
    \]
    proving \eqref{eq:radius-reach-occupancy-bound}.  Finally,
    $\Delta_\reach(S_t)
        \le
        \bar\Delta_\reach$
    for every $t$, so that
    \[
        V^\star(s)-V_\reach^\star(s)
        \le
        \sum_{t=0}^{\infty}
        \gamma^t\bar\Delta_\reach
        =
        \bar\Delta_\reach/(1-\gamma).
    \]
\end{proof}

\subsection{Selection error within the \texorpdfstring{reach-$\reach$}{radius-reach} abstraction}
\label{subsec:radius-reach-selection-error}

For $z\in\gZ_\reach(s)$, define the optimal action-value function of the reach-$\reach$ restricted
MDP by
\begin{equation}
    \label{eq:radius-reach-restricted-qstar}
    Q_\reach^\star(s,z)
    :=
    r(s,z)
    +
    \gamma
    \sum_{s'\in\gS}
    P(s'\mid s,z)V_\reach^\star(s').
\end{equation}
Then
$V_\reach^\star(s)
    =
    \max_{z\in\gZ_\reach(s)}Q_\reach^\star(s,z)$.
For a stationary reach-$\reach$ policy $\pi$, define
\begin{align}
    r^\pi(s)&:=\int_{\gZ_\reach(s)}r(s,z)\pi(\mathrm dz\mid s)
    \quad\text{and}\quad
    P^\pi(s'\mid s)
    :=
    \int_{\gZ_\reach(s)}P(s'\mid s,z)\pi(\mathrm dz\mid s),
\end{align}
and define its pointwise selection error inside the reach-$\reach$ abstraction by
\begin{equation}
    \label{eq:radius-reach-selection-error}
    e_{\mathrm{sel},\reach}(\pi;s)
    :=
    V_\reach^\star(s)
    -
    \int_{\gZ_\reach(s)}
    Q_\reach^\star(s,z)
    \pi(\mathrm dz\mid s).
\end{equation}
Let
$\overline e_{\mathrm{sel},\reach}(\pi)
    :=
    \max_{s\in\gS}e_{\mathrm{sel},\reach}(\pi;s)$.
Since the integral in \eqref{eq:radius-reach-selection-error} is an average of values no larger
than their maximum, $e_{\mathrm{sel},\reach}(\pi;s)\ge0$.

\begin{proposition}[Exact support--selection decomposition at reach $\reach$]
    \label{prop:radius-reach-support-selection-decomposition}
    For every stationary reach-$\reach$ policy $\pi$,
    \begin{equation}
        \label{eq:radius-reach-support-selection-decomposition}
        V^\star-V^\pi
        =
        \bigl(V^\star-V_\reach^\star\bigr)
        +
        (I-\gamma P^\pi)^{-1}
        e_{\mathrm{sel},\reach}(\pi;\cdot).
    \end{equation}
    Equivalently, for every state $s$, the decomposition 
    \begin{align}
        V^\star(s)-V^\pi(s)
        ={}&
        V^\star(s)-V_\reach^\star(s)
        +
        \mathbb{E}_s^\pi
        \left[
            \sum_{t=0}^{\infty}
            \gamma^t
            e_{\mathrm{sel},\reach}(\pi;S_t)
        \right],
        \label{eq:radius-reach-support-selection-occupancy}
    \end{align}
 holds so that 
    \begin{equation}
        \label{eq:radius-reach-support-selection-uniform}
        V^\star(s)-V^\pi(s)
        \in [\Delta_\reach(s),
            (\bar\Delta_\reach
            +
            \overline e_{\mathrm{sel},\reach}(\pi))/(1-\gamma)]
    \end{equation} 
    for every state $s$.
\end{proposition}

\begin{proof}
    Let $\gT^\pi$ denote the Bellman operator of $\pi$:
    \[
        (\gT^\pi V)(s)
        =
        r^\pi(s)
        +
        \gamma
        \sum_{s'\in\gS}
        P^\pi(s'\mid s)V(s').
    \]
    By \eqref{eq:radius-reach-restricted-qstar} and
    \eqref{eq:radius-reach-selection-error}, we have that 
    \begin{align*}
        e_{\mathrm{sel},\reach}(\pi;s)
        &={}
        V_\reach^\star(s)
        -
        (\gT^\pi V_\reach^\star)(s).
    \end{align*}
    Since $V^\pi=\gT^\pi V^\pi$, adding and subtracting
    $\gT^\pi V_\reach^\star$ gives
    \begin{align*}
        V_\reach^\star(s)-V^\pi(s)
        &={}
        V_\reach^\star(s)
        -
        (\gT^\pi V_\reach^\star)(s)+
        (\gT^\pi V_\reach^\star)(s)
        -
        (\gT^\pi V^\pi)(s)
        \\
        &={}
        e_{\mathrm{sel},\reach}(\pi;s)
        +
        \gamma
        \sum_{s'\in\gS}
        P^\pi(s'\mid s)
        \bigl(V_\reach^\star(s')-V^\pi(s')\bigr).
    \end{align*}
    Hence, in vector notation,
    $(I-\gamma P^\pi)
        \bigl(V_\reach^\star-V^\pi\bigr)
        =
        e_{\mathrm{sel},\reach}(\pi;\cdot)$.
    Since $P^\pi$ is a Markov operator and $\gamma<1$, we have that 
    $(I-\gamma P^\pi)^{-1}
        =
        \sum_{t=0}^{\infty}\gamma^t(P^\pi)^t$
    in the supremum-norm operator topology.  Therefore
    \begin{equation}
        \label{eq:radius-reach-restricted-policy-gap}
        V_\reach^\star-V^\pi
        =
        (I-\gamma P^\pi)^{-1}
        e_{\mathrm{sel},\reach}(\pi;\cdot).
    \end{equation}
    Adding $V^\star-V_\reach^\star$ to both sides of
    \eqref{eq:radius-reach-restricted-policy-gap} proves
    \eqref{eq:radius-reach-support-selection-decomposition}.  Evaluating the Neumann series at
    state $s$ proves \eqref{eq:radius-reach-support-selection-occupancy}.

    The lower bound in \eqref{eq:radius-reach-support-selection-uniform} follows since
    $V^\pi\le V_\reach^\star$ by Lemma~\ref{lem:radius-reach-well-posedness} and
    \[
        V^\star(s)-V_\reach^\star(s)
        \ge
        \Delta_\reach(s)
    \]
    by Proposition~\ref{prop:radius-reach-abstraction-price}.  For the upper bound, the same
    proposition yields
    \[
        V^\star(s)-V_\reach^\star(s)
        \leq
      \bar\Delta_\reach/(1-\gamma),
    \]
    while the Neumann series and nonnegativity give
    \[
        \bigl(
            (I-\gamma P^\pi)^{-1}
            e_{\mathrm{sel},\reach}(\pi;\cdot)
        \bigr)(s)
        \leq
        \sum_{t=0}^{\infty}
        \gamma^t
        \overline e_{\mathrm{sel},\reach}(\pi)
        =
   \overline e_{\mathrm{sel},\reach}(\pi)/(1-\gamma).
    \]
    Combining the two inequalities proves the result.
\end{proof}

\begin{remark}[Recovery of Theorem~\ref{thm:general-mdp-abstraction-price}]
    At zero reach, $\reach=0$, the local action set is the exact FRS anchor bank by
    \eqref{eq:reach-zero-recovers-frs}.  Proposition~\ref{prop:radius-reach-support-selection-decomposition}
    therefore reduces to the support--selection decomposition of
    Theorem~\ref{thm:general-mdp-abstraction-price}.
\end{remark}

\subsection{A geometric rate for the support gap}
\label{subsec:geometric-support-gap-rate}

Monotonicity alone does not quantify how quickly the support gap decreases.  The following
assumption gives a local quadratic rate when a full-MDP optimal latent lies near the flow-reversed anchor bank.

\begin{assumption}[Smooth local optimal-value geometry]
    \label{ass:smooth-local-optimal-value-geometry}
    For every state $s$, there exists a latent
    \[
        z^\star(s)
        \in
        \argmax_{z\in\gZ}Q^\star(s,z).
    \]
    The map $z\mapsto Q^\star(s,z)$ is differentiable on every line segment joining
    $z^\star(s)$ to an flow-reversed anchor $z_c(s)$ and has an $\Lsmooth$-Lipschitz gradient on those
    segments.  Moreover, 
    \[
        \nabla_zQ^\star\bigl(s,z^\star(s)\bigr)=0.
    \]
\end{assumption}
Here $z^\star(s)$ is an optimizer over the full latent space $\gZ$, not over
the reach-$\reach$ action class.  The latter is the constrained approximation
whose loss is measured by $\Delta_\reach(s)$.  Since in our setting
$\gZ=\R^{d_z}$, differentiability at an attained full-space optimum implies
$\nabla_z Q^\star(s,z^\star(s))=0$.

Define the distance from a full-MDP optimal latent to the nearest flow-reversed anchor by
\begin{equation}
    \label{eq:optimal-latent-anchor-distance}
    \dist_{\mF}(s)
    :=
    \min_{c\in\gC}
    \left\lVert
        z^\star(s)-z_c(s)
    \right\rVert_2.
\end{equation}

\begin{proposition}[Quadratic decay of local support bias]
    \label{prop:quadratic-local-support-bias}
    Under Assumption~\ref{ass:smooth-local-optimal-value-geometry}, for every $\reach\ge0$ and state
    $s$, the upper bound
    \begin{equation}
        \label{eq:quadratic-local-support-bias}
        \Delta_\reach(s)
        \leq
        \Lsmooth
        \bigl(\dist_{\mF}(s)-\reach\bigr)_+^2/2,
    \end{equation} holds
    where $(x)_+:=\max\{x,0\}$, so that  
    \begin{equation}
        \label{eq:uniform-quadratic-local-support-bias}
        \bar\Delta_\reach
        \leq \Lsmooth
        \max_{s\in\gS}
        \bigl(\dist_{\mF}(s)-\reach\bigr)_+^2/2.
    \end{equation}
\end{proposition}

\begin{proof}
    Fix a state $s\in \mathcal{S}$.  Let $c^\star(s)$ be an index attaining the minimum in
    \eqref{eq:optimal-latent-anchor-distance}, and abbreviate
    \[
        z^\star:=z^\star(s),
        \qquad
        z_c:=z_{c^\star(s)}(s),
        \qquad
        d:=\dist_{\mF}(s)=\lVert z^\star-z_c\rVert_2.
    \]

    If $d=0$, then $z^\star=z_c\in\gZ_\reach(s)$ for every $\reach\ge0$.  Hence
    \[
        \max_{z\in\gZ_\reach(s)}Q^\star(s,z)
        =
        V^\star(s)
    \]
    and $\Delta_\reach(s)=0$.  The claimed bound follows.

    Suppose $d>0$.  If $\reach\ge d$, then $z^\star\in B_2(z_c,\reach)$, so again
    $\Delta_\reach(s)=0$, which equals the right-hand side of
    \eqref{eq:quadratic-local-support-bias}.

    It remains to consider $0\le \reach<d$.  Define
    $z_\reach
        :=
        z_c
        +
        \frac{\reach}{d}
        \bigl(z^\star-z_c\bigr)$.
    Then
   $\lVert z_\reach-z_c\rVert_2
        =
        \reach$,
    so $z_\reach\in\gZ_\reach(s)$.  Also, we have
    $z_\reach-z^\star
        =
        -\left(1-\frac{\reach}{d}\right)
        \bigl(z^\star-z_c\bigr)$,
    and therefore
    $\lVert z_\reach-z^\star\rVert_2
        =
        d-\reach$.

    The standard smoothness inequality, applied to the function
    $z\mapsto Q^\star(s,z)$ at $z^\star$ and $z_\reach$, gives
    \begin{align*}
        Q^\star(s,z_\reach)
        &\ge
        Q^\star(s,z^\star)
        +
        \left\langle
            \nabla_zQ^\star(s,z^\star),
            z_\reach-z^\star
        \right\rangle
        -
        \frac{\Lsmooth}{2}
        \lVert z_\reach-z^\star\rVert_2^2
        ={}
        V^\star(s)
        -
        \frac{\Lsmooth}{2}
        (d-\reach)^2,
    \end{align*}
    where the equality uses
    $\nabla_zQ^\star(s,z^\star)=0$.  Since $z_\reach\in\gZ_\reach(s)$, the lower bound holds:
    \[
        \max_{z\in\gZ_\reach(s)}Q^\star(s,z)
        \geq
        Q^\star(s,z_\reach).
    \]
    Subtracting from $V^\star(s)$ yields
    $\Delta_\reach(s)
        \leq
        \Lsmooth
        (d-\reach)^2/2$.
    This is exactly \eqref{eq:quadratic-local-support-bias} in the case $\reach<d$.
    Taking the maximum over the finite state set proves
    \eqref{eq:uniform-quadratic-local-support-bias}.
\end{proof}

\subsection{Behaviorally distinguishable local candidates}
\label{subsec:behaviorally-distinguishable-candidates}

A Euclidean perturbation need not correspond to a distinct closed-loop behavior.  We therefore
measure local action complexity using a Bellman behavioral pseudometric rather than by counting
raw latent coordinates.

Let
$V_{\max}
    :=
    r_{\max}/(1-\gamma)$
and define
$\gV_{\max}
    :=
    \left\{
        V:\gS\to[0,V_{\max}]
    \right\}$.

\begin{definition}[Bellman behavioral pseudometric]
    \label{def:bellman-behavioral-pseudometric}
    For every state $s$ and latents $z,\tilde z\in\gZ$, define
    \begin{equation}
        \label{eq:bellman-behavioral-pseudometric}
        d_s^{\mathrm B}(z,\tilde z)
        :=
        \sup_{V\in\gV_{\max}}
        \left|
            Q_V(s,z)-Q_V(s,\tilde z)
        \right|.
    \end{equation}
\end{definition}

Thus two latents are close in $d_s^{\mathrm B}$ when no bounded continuation value can distinguish
their one-step reward and transition consequences by more than the prescribed resolution.

\begin{lemma}[Basic properties of the behavioral pseudometric]
    \label{lem:behavioral-pseudometric-properties}
    For each state $s$, the function $d_s^{\mathrm B}$ is a pseudometric on $\gZ$.  For
    every $V\in\gV_{\max}$, the  bound holds:
    \begin{equation}
        \label{eq:behavioral-pseudometric-controls-bellman}
        \left|
            Q_V(s,z)-Q_V(s,\tilde z)
        \right|
        \le
        d_s^{\mathrm B}(z,\tilde z).
    \end{equation}
    In particular, the bound holds:
    \begin{equation}
        \label{eq:behavioral-pseudometric-controls-qstar}
        \left|
            Q^\star(s,z)-Q^\star(s,\tilde z)
        \right|
        \le
        d_s^{\mathrm B}(z,\tilde z).
    \end{equation}
    Moreover, the distance is upper bounded as 
    \begin{equation}
        \label{eq:behavioral-pseudometric-explicit-upper}
        d_s^{\mathrm B}(z,\tilde z)
        \le
        |r(s,z)-r(s,\tilde z)|
        +
        \gamma V_{\max}
        \sum_{s'\in\gS}
        |P(s'\mid s,z)-P(s'\mid s,\tilde z)|,
    \end{equation}
 so that $d_s^{\mathrm B}(z_n,z)\to0$ whenever
    $\lVert z_n-z\rVert_2\to0$.  Every compact Euclidean subset of $\gZ$ therefore has a finite
    behavioral $\eps_{\mathrm B}$-net for every $\eps_{\mathrm B}>0$.
\end{lemma}

\begin{proof}
    Nonnegativity and symmetry follow immediately from the absolute value in
    \eqref{eq:bellman-behavioral-pseudometric}.  Also,
    $d_s^{\mathrm B}(z,z)=0$.
    Fix $z_1,z_2,z_3\in\gZ$.  For every $V\in\gV_{\max}$, we have that 
    \begin{align*}
        \left|
            Q_V(s,z_1)-Q_V(s,z_3)
        \right|
        &\le
        \left|
            Q_V(s,z_1)-Q_V(s,z_2)
        \right|+
        \left|
            Q_V(s,z_2)-Q_V(s,z_3)
        \right|
        \\
        &\le
        d_s^{\mathrm B}(z_1,z_2)
        +
        d_s^{\mathrm B}(z_2,z_3).
    \end{align*}
    Taking the supremum over $V\in\gV_{\max}$ proves the triangle inequality.  Distinct latents
    may have identical reward and transition laws, in which case their distance is zero; hence the
    object is a pseudometric rather than necessarily a metric.

    Inequality \eqref{eq:behavioral-pseudometric-controls-bellman} follows directly because the
    left-hand side is one of the quantities over which the supremum is taken.  Since rewards are
    in $[0,r_{\max}]$, the full optimal value satisfies $V^\star(s)\in [0,V_{\max}]$
    for every state, so $V^\star\in\gV_{\max}$.  Substituting $V=V^\star$ into
    \eqref{eq:behavioral-pseudometric-controls-bellman} proves
    \eqref{eq:behavioral-pseudometric-controls-qstar}.

    For every $V\in\gV_{\max}$, we have that 
    \begin{align*}
        |Q_V(s,z)-Q_V(s,\tilde z)|
        &\le
        |r(s,z)-r(s,\tilde z)|+
        \gamma
        \sum_{s'\in\gS}
        V(s')
        |P(s'\mid s,z)-P(s'\mid s,\tilde z)|
        \\
        &\le
        |r(s,z)-r(s,\tilde z)|+
        \gamma V_{\max}
        \sum_{s'\in\gS}
        |P(s'\mid s,z)-P(s'\mid s,\tilde z)|.
    \end{align*}
    Taking the supremum over $V\in\gV_{\max}$ proves
    \eqref{eq:behavioral-pseudometric-explicit-upper}.  By
    Assumption~\ref{ass:finite-state-continuous-latent}, the right-hand side converges to zero
    whenever $\tilde z\to z$ in Euclidean norm.  Thus the identity map from any Euclidean
    compact subset of $\gZ$ to the same set equipped with $d_s^{\mathrm B}$ is continuous.
    A continuous image of a compact set is compact and therefore totally bounded in the
    pseudometric quotient.  Hence the compact set admits a finite behavioral $\eps_{\mathrm B}$-net for every
    $\eps_{\mathrm B}>0$.
\end{proof}

For a pseudometric space $(A,d)$ and resolution $\eps_{\mathrm B}>0$, define its $\eps_{\mathrm B}$-packing number by
\begin{equation}
    \label{eq:behavioral-packing-number-general}
    \mathcal M(A,d,\eps_{\mathrm B})
    :=
    \sup
    \left\{
        |\mathcal P|:
            \mathcal P\subseteq A\; \text{ is finite, and}\;
            d(x,y)>\eps_{\mathrm B}\; \text{for all distinct }x,y\in\mathcal P
    \right\}.
\end{equation}
The value may be infinite in a general pseudometric space.  A set counted by
\eqref{eq:behavioral-packing-number-general} consists of candidates that are pairwise
behaviorally distinguishable at resolution $\eps_{\mathrm B}$.

Define the maximal behavioral packing complexity of the reach-$\reach$ action abstraction by
\begin{equation}
    \label{eq:behavioral-packing-complexity}
    A_\reach(\eps_{\mathrm B})
    :=
    \max_{s\in\gS}
    \mathcal M
    \left(
        \gZ_\reach(s),
        d_s^{\mathrm B},
        \eps_{\mathrm B}
    \right).
\end{equation}

\begin{proposition}[Monotone behavioral complexity]
    \label{prop:monotone-behavioral-complexity}
    For every fixed resolution $\eps_{\mathrm B}>0$, the map
    $\reach\mapsto A_\reach(\eps_{\mathrm B})$
    is nondecreasing.  In particular, if $0\le \reach_1\le \reach_2$, then
    $A_{\reach_1}(\eps_{\mathrm B})
        \le
        A_{\reach_2}(\eps_{\mathrm B})$.
\end{proposition}

\begin{proof}
    Fix a state $s\in \mathcal{S}$.  By Proposition~\ref{prop:monotone-local-reach}, the inclusion holds:
    \[
        \gZ_{\reach_1}(s)
        \subseteq
        \gZ_{\reach_2}(s).
    \]
    Every finite $\eps_{\mathrm B}$-separated subset of $\gZ_{\reach_1}(s)$ is also an $\eps_{\mathrm B}$-separated subset of
    $\gZ_{\reach_2}(s)$.  Taking the supremum over all such finite subsets gives
    \[
        \mathcal M
        \left(
            \gZ_{\reach_1}(s),d_s^{\mathrm B},\eps_{\mathrm B}
        \right)
        \le
        \mathcal M
        \left(
            \gZ_{\reach_2}(s),d_s^{\mathrm B},\eps_{\mathrm B}
        \right).
    \]
    Taking the maximum over the finite state set proves the result.
\end{proof}

The monotonicity is weak.  If newly admitted latents are behaviorally redundant at resolution
$\eps_{\mathrm B}$, then $A_\reach(\eps_{\mathrm B})$ may remain constant.

\begin{assumption}[Behavioral Lipschitzness]
    \label{ass:behavioral-lipschitzness}
    There exists $L_{\mathrm B}>0$ such that, for every state $s$ and all
    $z,\tilde z\in\mathbb{R}^{d_z}$, the bound holds:
    \begin{equation}
        \label{eq:behavior-upper-lipschitz}
        d_s^{\mathrm B}(z,\tilde z)
        \le
        L_{\mathrm B}
        \lVert z-\tilde z\rVert_2.
    \end{equation}
\end{assumption}

\begin{lemma}[A sufficient condition for behavioral Lipschitzness]
    \label{lem:sufficient-behavior-lipschitz}
    Suppose there are constants $L_r,L_P\ge0$ such that, for all states $s$ and latents
    $z,\tilde z$, the estimates hold:
    \begin{align*}
        |r(s,z)-r(s,\tilde z)|
        &\leq
        L_r\lVert z-\tilde z\rVert_2,
        \quad\text{and}\quad
        \sum_{s'\in\gS}
        \left|
            P(s'\mid s,z)-P(s'\mid s,\tilde z)
        \right|
        \leq
        L_P\lVert z-\tilde z\rVert_2.
    \end{align*}
    Then Assumption~\ref{ass:behavioral-lipschitzness} is satisfied with
    $L_{\mathrm B}
        :=
        L_r+\gamma V_{\max}L_P$.
\end{lemma}

\begin{proof}
    Fix $V\in\gV_{\max}$.  Applying the triangle inequality, we have that 
    \begin{align*}
        |Q_V(s,z)-Q_V(s,\tilde z)|
        &\le
        |r(s,z)-r(s,\tilde z)|
  +
        \gamma
        \left|
            \sum_{s'\in\gS}
            V(s')
            \bigl(
                P(s'\mid s,z)
                -
                P(s'\mid s,\tilde z)
            \bigr)
        \right|
        \\
        &\le
        L_r\lVert z-\tilde z\rVert_2
  +
        \gamma V_{\max}
        \sum_{s'\in\gS}
        \left|
            P(s'\mid s,z)-P(s'\mid s,\tilde z)
        \right|
        \\
        &\le
        \bigl(L_r+\gamma V_{\max}L_P\bigr)
        \lVert z-\tilde z\rVert_2.
    \end{align*}
    Taking the supremum over $V\in\gV_{\max}$ proves the claim.
\end{proof}

\begin{assumption}[Local behavioral richness]
    \label{ass:local-behavioral-richness}
    There exist a state $s_0$, a command $c_0$, a radius $\reach_{\mathrm{rich}}>0$, an integer
    $d_{\mathrm{eff}}\in\{1,\ldots,d_z\}$, a $d_{\mathrm{eff}}$-dimensional linear subspace
    $U\subseteq\R^{d_z}$, and a constant $\ell_{\mathrm B}>0$ such that
    \begin{equation}
        \label{eq:behavior-lower-lipschitz}
        d_{s_0}^{\mathrm B}
        \bigl(z_{c_0}(s_0)+u,z_{c_0}(s_0)+v\bigr)
        \geq
        \ell_{\mathrm B}\lVert u-v\rVert_2
    \end{equation}
    for all $u,v\in U$ satisfying
    $\lVert u\rVert_2\leq \reach_{\mathrm{rich}}$ and $\lVert v\rVert_2\leq \reach_{\mathrm{rich}}$.
\end{assumption}

\begin{proposition}[Polynomial upper and lower bounds on local behavioral complexity]
    \label{prop:behavioral-packing-growth}
    Under Assumption~\ref{ass:behavioral-lipschitzness}, for every $\reach\ge0$ and $\eps_{\mathrm B}>0$, the estimate holds:
    \begin{equation}
        \label{eq:behavioral-packing-upper}
        A_\reach(\eps_{\mathrm B})
        \le
        C
        \left(
            1+\frac{2L_{\mathrm B}\reach}{\eps_{\mathrm B}}
        \right)^{d_z}.
    \end{equation}
    If Assumption~\ref{ass:local-behavioral-richness} also holds, then, for every
    $\reach\in(0,\reach_{\mathrm{rich}}]$ and $\eps_{\mathrm B}>0$, we additionally have the lower bound
    \begin{equation}
        \label{eq:behavioral-packing-lower}
        A_\reach(\eps_{\mathrm B})
        \ge
        \left(
            \frac{\ell_{\mathrm B}\reach}{\eps_{\mathrm B}}
        \right)^{d_{\mathrm{eff}}}.
    \end{equation}
    The lower bound is nontrivial when $\ell_{\mathrm B}\reach>\eps_{\mathrm B}$.
\end{proposition}

\begin{proof}
    We first prove the upper bound.  Fix a state $s$ and let
    $\mathcal P\subseteq\gZ_\reach(s)$
    be any finite set whose distinct points have pairwise behavioral distance strictly larger
    than $\eps_{\mathrm B}$.  For each $z\in\mathcal P$, choose one command $c(z)\in\gC$ such that
    $z\in B_2\bigl(z_{c(z)}(s),\reach\bigr)$.
    This is possible by the definition of $\gZ_\reach(s)$.  For $c\in\gC$, let
    $\mathcal P_c
        :=
        \{z\in\mathcal P:c(z)=c\}$ so that
    the sets $\mathcal P_c$ partition $\mathcal P$.

    If $z,\tilde z\in\mathcal P_c$ are distinct, then
    \[
        \eps_{\mathrm B}
        <
        d_s^{\mathrm B}(z,\tilde z)
        \le
        L_{\mathrm B}\lVert z-\tilde z\rVert_2,
    \]
    so
    $\lVert z-\tilde z\rVert_2
        >
        \eps_{\mathrm B}/L_{\mathrm B}$.
    Hence the Euclidean balls centered at points of $\mathcal P_c$ with radius
    $\eps_{\mathrm B}/(2L_{\mathrm B})$ are pairwise disjoint.  They are all contained in
    \[
        B_2
        \left(
            z_c(s),
            \reach+\eps_{\mathrm B}/(2L_{\mathrm B})
        \right).
    \]
    Comparing $d_z$-dimensional Euclidean volumes gives
    \begin{align*}
        |\mathcal P_c|
        \operatorname{vol}
        \left(
            B_2\left(0,\frac{\eps_{\mathrm B}}{2L_{\mathrm B}}\right)
        \right)
        &\le
        \operatorname{vol}
        \left(
            B_2\left(0,\reach+\frac{\eps_{\mathrm B}}{2L_{\mathrm B}}\right)
        \right).
    \end{align*}
    Therefore
    $|\mathcal P_c|
        \le
        \left(
            1+\frac{2L_{\mathrm B}\reach}{\eps_{\mathrm B}}
        \right)^{d_z}$.
    Summing over $c\in\gC$ yields
    $|\mathcal P|
        \le
        C
        \left(
            1+\frac{2L_{\mathrm B}\reach}{\eps_{\mathrm B}}
        \right)^{d_z}$.
    Since $\mathcal P$ was arbitrary, the same upper bound holds for the packing number at state
    $s$, and then for $A_\reach(\eps_{\mathrm B})$ after maximizing over $s$.

    We now prove the lower bound.  Fix $\reach\in(0,\reach_{\mathrm{rich}}]$ and work at the state, anchor, and
    subspace in Assumption~\ref{ass:local-behavioral-richness}.  Let
    \[
        D_\reach
        :=
        \{u\in U:\lVert u\rVert_2\leq \reach\}.
    \]
    Choose an inclusion-maximal finite subset
    $\{u_1,\ldots,u_N\}
        \subseteq
        D_\reach$
    whose distinct points have pairwise Euclidean distance strictly larger than
    $\eps_{\mathrm B}/\ell_{\mathrm B}$.  Such a finite maximal set exists because $D_\reach$ is compact and has
    finite Euclidean packing number at every positive resolution.  Maximality implies that the
    closed Euclidean balls in $U$ of radius $\eps_{\mathrm B}/\ell_{\mathrm B}$ centered at the $u_i$ cover
    $D_\reach$: otherwise an uncovered point could be added while preserving the strict separation.

    Let $v_{d_{\mathrm{eff}}}$ denote the volume of the unit Euclidean ball in $U$.  By
    subadditivity of volume, we have 
    \[
        v_{d_{\mathrm{eff}}}\reach^{d_{\mathrm{eff}}}
        \le
        N
        v_{d_{\mathrm{eff}}}
        \left(
           \eps_{\mathrm B}/\ell_{\mathrm B}
        \right)^{d_{\mathrm{eff}}}.
    \]
    Canceling the unit-ball volume gives
    $N
        \geq
        \left(
           \ell_{\mathrm B}\reach/\eps_{\mathrm B}
        \right)^{d_{\mathrm{eff}}}$.
For distinct $i,j$, Assumption~\ref{ass:local-behavioral-richness} yields
    \begin{align*}
        d_{s_0}^{\mathrm B}
        \bigl(
            z_{c_0}(s_0)+u_i,
            z_{c_0}(s_0)+u_j
        \bigr)
        &\ge
        \ell_{\mathrm B}\lVert u_i-u_j\rVert_2>\eps_{\mathrm B}.
    \end{align*}
    The translated points lie in $\gZ_\reach(s_0)$ and form a behavioral $\eps_{\mathrm B}$-packing.  Hence
    $A_\reach(\eps_{\mathrm B})
        \geq
        N$,
    proving \eqref{eq:behavioral-packing-lower}.
\end{proof}

For a pseudometric space $(A,d)$, let
$\mathcal N_{\mathrm{int}}(A,d,\eps_{\mathrm B})$
denote the smallest cardinality of an \emph{internal} $\eps_{\mathrm B}$-net: a set
$\mathcal A\subseteq A$ such that every $z\in A$ has some $a\in\mathcal A$ with
$d(z,a)\le \eps_{\mathrm B}$.

\begin{lemma}[Packing--net sandwich]
    \label{lem:packing-net-sandwich}
    Suppose $\mathcal M(A,d,\eps_{\mathrm B})<\infty$.  Then
    \begin{equation}
        \label{eq:packing-net-sandwich}
        \mathcal M(A,d,2\eps_{\mathrm B})
        \leq
        \mathcal N_{\mathrm{int}}(A,d,\eps_{\mathrm B})
        \leq
        \mathcal M(A,d,\eps_{\mathrm B}).
    \end{equation}
    In particular, under Assumption~\ref{ass:behavioral-lipschitzness}, for every state $s$ there
    exists a behavioral $\eps_{\mathrm B}$-net
    $\gA_{\reach,\eps_{\mathrm B}}(s)
        \subseteq
        \gZ_\reach(s)$
    satisfying
   $|\gA_{\reach,\eps_{\mathrm B}}(s)|
        \le
        A_\reach(\eps_{\mathrm B})$.
    If Assumption~\ref{ass:local-behavioral-richness} also holds and
    $\reach\in(0,\reach_{\mathrm{rich}}]$, then every behavioral $\eps_{\mathrm B}$-net of
    $\gZ_\reach(s_0)$ has cardinality at least
    \begin{equation}
        \label{eq:behavioral-net-lower-bound}
        \left(
        \ell_{\mathrm B}\reach/(2\eps_{\mathrm B})
        \right)^{d_{\mathrm{eff}}}.
    \end{equation}
\end{lemma}

\begin{proof}
    Let $\mathcal P\subseteq A$ be any finite set whose distinct points have pairwise distance
    strictly larger than $2\eps_{\mathrm B}$, and let $\mathcal A\subseteq A$ be an internal $\eps_{\mathrm B}$-net.  No single
    point $a\in\mathcal A$ can be within distance $\eps_{\mathrm B}$ of two distinct points of $\mathcal P$;
    otherwise the triangle inequality would make those two points have distance at most $2\eps_{\mathrm B}$.
    Consequently $|\mathcal A|\ge|\mathcal P|$.  Taking the supremum over all such
    $\mathcal P$ and then the infimum over all internal $\eps_{\mathrm B}$-nets proves the left inequality in
    \eqref{eq:packing-net-sandwich}.

    For the right inequality, construct an inclusion-maximal $\eps_{\mathrm B}$-separated subset
    $\mathcal P_{\max}\subseteq A$.  Because the $\eps_{\mathrm B}$-packing number is finite, this construction
    terminates after finitely many points.  Maximality implies that $\mathcal P_{\max}$ is an
    internal $\eps_{\mathrm B}$-net: if some $x\in A$ were farther than $\eps_{\mathrm B}$ from every point of
    $\mathcal P_{\max}$, then $x$ could be added while preserving $\eps_{\mathrm B}$-separation.  Moreover, we have that
    $|\mathcal P_{\max}|
        \le
        \mathcal M(A,d,\eps_{\mathrm B})$.
    This proves the right inequality.

    Under Assumption~\ref{ass:behavioral-lipschitzness},
    Proposition~\ref{prop:behavioral-packing-growth} makes the relevant packing numbers finite,
    so the existence claim follows after maximizing over states.  Finally,
    \eqref{eq:behavioral-net-lower-bound} follows from the left inequality in
    \eqref{eq:packing-net-sandwich} and
    Proposition~\ref{prop:behavioral-packing-growth} applied at resolution $2\eps_{\mathrm B}$.
\end{proof}

\subsection{Finite behavioral discretization}
\label{subsec:finite-behavioral-discretization}

Fix $\reach\ge0$ and $\eps_{\mathrm B}>0$.  For every state $s$, choose a behavioral $\eps_{\mathrm B}$-net
$\gA_{\reach,\eps_{\mathrm B}}(s)
    \subseteq
    \gZ_\reach(s)$
with
$|\gA_{\reach,\eps_{\mathrm B}}(s)|
    \le
    A_\reach(\eps_{\mathrm B})$.
Define the finite-net Bellman operator
\begin{equation}
    \label{eq:finite-net-bellman-operator}
    (\gT_{\reach,\eps_{\mathrm B}}V)(s)
    :=
    \max_{a\in\gA_{\reach,\eps_{\mathrm B}}(s)}
    Q_V(s,a),
\end{equation}
and let $V_{\reach,\eps_{\mathrm B}}^\star$ denote its unique fixed point. Define the one-step support gap of the finite net by
\begin{equation}
    \label{eq:finite-net-support-gap}
    \Delta_{\reach,\eps_{\mathrm B}}(s)
    :=
    V^\star(s)
    -
    \max_{a\in\gA_{\reach,\eps_{\mathrm B}}(s)}
    Q^\star(s,a).
\end{equation}

\begin{lemma}[Discretization adds at most $\eps_{\mathrm B}$ one-step error]
    \label{lem:finite-net-support-gap}
    For every state $s$, we have that 
    \begin{equation}
        \label{eq:finite-net-support-gap-bound}
        \Delta_{\reach,\eps_{\mathrm B}}(s)
        \leq
        \Delta_\reach(s)+\eps_{\mathrm B},
    \end{equation}
    so that
    \begin{equation}
        \label{eq:finite-net-value-approximation}
        V^\star(s)-V_{\reach,\eps_{\mathrm B}}^\star(s)
        \in [0,
        (\bar\Delta_\reach+\eps_{\mathrm B})/(1-\gamma)]
    \end{equation}
    for every state $s\in \mathcal{S}$.
\end{lemma}

\begin{proof}
    Fix a state $s\in \mathcal{S}$.  Since $\gZ_\reach(s)$ is compact and $Q^\star(s,\cdot)$ is continuous, choose
    \[
        z_\reach^\star(s)
        \in
        \argmax_{z\in\gZ_\reach(s)}Q^\star(s,z).
    \]
    Since $\gA_{\reach,\eps_{\mathrm B}}(s)$ is an $\eps_{\mathrm B}$-net under $d_s^{\mathrm B}$, there exists
    $a_{\reach,\eps_{\mathrm B}}(s)\in\gA_{\reach,\eps_{\mathrm B}}(s)$ satisfying
    \[
        d_s^{\mathrm B}
        \bigl(
            z_\reach^\star(s),a_{\reach,\eps_{\mathrm B}}(s)
        \bigr)
        \leq \eps_{\mathrm B}.
    \]
    Lemma~\ref{lem:behavioral-pseudometric-properties} gives us that 
    \[
        Q^\star
        \bigl(
            s,a_{\reach,\eps_{\mathrm B}}(s)
        \bigr)
        \geq
        Q^\star
        \bigl(
            s,z_\reach^\star(s)
        \bigr)
        -\eps_{\mathrm B}.
    \]
    Therefore
    \begin{align*}
        \max_{a\in\gA_{\reach,\eps_{\mathrm B}}(s)}Q^\star(s,a)
        &\geq
        \max_{z\in\gZ_\reach(s)}Q^\star(s,z)-\eps_{\mathrm B}.
    \end{align*}
    Subtracting both sides from $V^\star(s)$ proves
    \eqref{eq:finite-net-support-gap-bound}.

    The finite sets $\gA_{\reach,\eps_{\mathrm B}}(s)$ define a persistent state-dependent action abstraction of
    the form analyzed in Theorem~\ref{thm:general-mdp-abstraction-price}.  Applying the oracle
    upper bound from that theorem with support gap $\Delta_{\reach,\eps_{\mathrm B}}$ gives
    \[
        V^\star(s)-V_{\reach,\eps_{\mathrm B}}^\star(s)
        \leq
            \max_{x\in\gS}\Delta_{\reach,\eps_{\mathrm B}}(x)/(1-\gamma).
    \]
    By \eqref{eq:finite-net-support-gap-bound}, we have that
    \[
        \max_{x\in\gS}\Delta_{\reach,\eps_{\mathrm B}}(x)
        \leq
        \bar\Delta_\reach+\eps_{\mathrm B}.
    \]
    This proves the upper bound in
    \eqref{eq:finite-net-value-approximation}.  The lower bound follows from
    $V_{\reach,\eps_{\mathrm B}}^\star\le V^\star$.
\end{proof}

    For $a\in\gA_{\reach,\eps_{\mathrm B}}(s)$, define
    \begin{equation}
        \label{eq:finite-net-optimal-q-definition}
        Q_{\reach,\eps_{\mathrm B}}^\star(s,a)
        :=
        r(s,a)
        +
        \gamma
        \sum_{s'\in\gS}
        P(s'\mid s,a)
        V_{\reach,\eps_{\mathrm B}}^\star(s').
    \end{equation}
    Let $\omega(\cdot\mid s)$ be any stationary Markov kernel supported on
    $\gA_{\reach,\eps_{\mathrm B}}(s)$, and define its pointwise finite-net selection error by
    \begin{equation}
        \label{eq:finite-net-selection-error}
        e_{\reach,\eps_{\mathrm B}}(\omega;s)
        :=
        V_{\reach,\eps_{\mathrm B}}^\star(s)
        -
        \sum_{a\in\gA_{\reach,\eps_{\mathrm B}}(s)}
        \omega(a\mid s)
        Q_{\reach,\eps_{\mathrm B}}^\star(s,a).
    \end{equation}
    
\begin{proposition}[Finite-net support--selection decomposition]
    \label{prop:finite-net-support-selection-decomposition}
    Let
    \begin{equation}
        \label{eq:uniform-finite-net-selection-error}
        \bar{e}_{\reach,\eps_{\mathrm B}}(\omega)
        :=
        \max_{s\in\gS}
        e_{\reach,\eps_{\mathrm B}}(\omega;s).
    \end{equation}
    Then $e_{\reach,\eps_{\mathrm B}}(\omega;s)\ge0$ for every state and
    \begin{equation}
        \label{eq:finite-net-exact-support-selection-decomposition}
        V^\star-V^\omega
        =
        \bigl(V^\star-V_{\reach,\eps_{\mathrm B}}^\star\bigr)
        +
        (I-\gamma P^\omega)^{-1}
        e_{\reach,\eps_{\mathrm B}}(\omega;\cdot),
    \end{equation}
so that
    \begin{equation}
        \label{eq:finite-net-support-selection-bound}
        0
        \leq
        V^\star(s)-V^\omega(s)
        \leq
            (
            \bar\Delta_\reach
            +\eps_{\mathrm B}
            +\bar{e}_{\reach,\eps_{\mathrm B}}(\omega))/(1-\gamma)
    \end{equation}
    for every state $s$.
    In particular, if a learned selector $\widehat\omega_n$ satisfies
    \begin{equation}
        \label{eq:finite-net-selection-event}
        \mathbb P
        \left(
            \bar{e}_{\reach,\eps_{\mathrm B}}(\widehat\omega_n)
            \leq
            \eps_{\mathrm{sel}}(n,\alpha)
        \right)
        \geq
        1-\alpha,
    \end{equation}
    then, with probability at least $1-\alpha$, simultaneously for every state,
    \begin{equation}
        \label{eq:finite-net-high-probability-control}
        V^\star(s)-V^{\widehat\omega_n}(s)
        \leq
            (\bar\Delta_\reach
            +\eps_{\mathrm B}
            +\eps_{\mathrm{sel}}(n,\alpha)
        )/(1-\gamma).
    \end{equation}
\end{proposition}

\begin{proof}
    Since
    \[
        V_{\reach,\eps_{\mathrm B}}^\star(s)
        =
        \max_{a\in\gA_{\reach,\eps_{\mathrm B}}(s)}
        Q_{\reach,\eps_{\mathrm B}}^\star(s,a),
    \]
    the average in \eqref{eq:finite-net-selection-error} cannot exceed
    $V_{\reach,\eps_{\mathrm B}}^\star(s)$.  Hence
    $e_{\reach,\eps_{\mathrm B}}(\omega;s)\ge0$.

    Let $\gT^\omega$ be the policy Bellman operator.  By
    \eqref{eq:finite-net-selection-error},
   we have that  \[
        e_{\reach,\eps_{\mathrm B}}(\omega;\cdot)
        =
        V_{\reach,\eps_{\mathrm B}}^\star
        -
        \gT^\omega V_{\reach,\eps_{\mathrm B}}^\star.
    \]
    Since $V^\omega=\gT^\omega V^\omega$,
   we have  \begin{align*}
        V_{\reach,\eps_{\mathrm B}}^\star-V^\omega
        &=
        V_{\reach,\eps_{\mathrm B}}^\star
        -
        \gT^\omega V_{\reach,\eps_{\mathrm B}}^\star
        +
        \gT^\omega V_{\reach,\eps_{\mathrm B}}^\star
        -
        \gT^\omega V^\omega=
        e_{\reach,\eps_{\mathrm B}}(\omega;\cdot)
        +
        \gamma P^\omega
        \bigl(V_{\reach,\eps_{\mathrm B}}^\star-V^\omega\bigr).
    \end{align*}
    The Markov operator $P^\omega$ has supremum-norm operator norm at most one, so
    $I-\gamma P^\omega$ is invertible with Neumann-series inverse
    $(I-\gamma P^\omega)^{-1}
        =
        \sum_{t=0}^{\infty}
        \gamma^t(P^\omega)^t$.
    Therefore
    \begin{equation}
        \label{eq:finite-net-selection-resolvent}
        V_{\reach,\eps_{\mathrm B}}^\star-V^\omega
        =
        (I-\gamma P^\omega)^{-1}
        e_{\reach,\eps_{\mathrm B}}(\omega;\cdot).
    \end{equation}
    Adding $V^\star-V_{\reach,\eps_{\mathrm B}}^\star$ to both sides proves
    \eqref{eq:finite-net-exact-support-selection-decomposition}.

    Lemma~\ref{lem:finite-net-support-gap} yields
   $0
        \leq
        V^\star(s)-V_{\reach,\eps_{\mathrm B}}^\star(s)
        \leq
        (\bar\Delta_\reach+\eps_{\mathrm B})/(1-\gamma)$.
    Moreover, evaluating the Neumann series in
    \eqref{eq:finite-net-selection-resolvent} at state $s$ yields
    \begin{align*}
        V_{\reach,\eps_{\mathrm B}}^\star(s)-V^\omega(s)
        &=
        \mathbb{E}_s^\omega
        \left[
            \sum_{t=0}^{\infty}
            \gamma^t
            e_{\reach,\eps_{\mathrm B}}(\omega;S_t)
        \right]\leq
        \frac{\overline e_{\reach,\eps_{\mathrm B}}(\omega)}{1-\gamma}.
    \end{align*}
    Combining the two bounds proves
    \eqref{eq:finite-net-support-selection-bound}.  Finally, on the event in
    \eqref{eq:finite-net-selection-event}, substitute
    \[
        \overline e_{\reach,\eps_{\mathrm B}}(\widehat\omega_n)
        \leq
        \eps_{\mathrm{sel}}(n,\alpha)
    \]
    into \eqref{eq:finite-net-support-selection-bound}.  The event has probability at least
    $1-\alpha$, proving \eqref{eq:finite-net-high-probability-control}.
\end{proof}

\subsection{A finite-sample frontier under a generative model}
\label{subsec:finite-sample-particle-frontier}

We now instantiate the finite-sample selection term using a synchronous empirical Q-iteration
scheme.  This specialization is used only to make the dependence on the number of behaviorally
distinguishable local candidates explicit.

\begin{assumption}[Generative model]
    \label{ass:particle-frontier-generative-model}
    For every state $s$ and latent $z$, a call to the generative model returns a random pair
    $(\mathsf R,\mathsf S')$
    satisfying
    \[
        0\le\mathsf R\le r_{\max},
        \qquad
        \E[\mathsf R]=r(s,z),
        \qquad
        \mathbb P(\mathsf S'=s')=P(s'\mid s,z).
    \]
    All generative-model calls are independent conditional on their queried state--latent pairs.
\end{assumption}

For the finite action sets $\gA_{\reach,\eps_{\mathrm B}}(s)$, define the optimal Q-Bellman operator
\begin{equation}
    \label{eq:finite-net-q-bellman-operator}
    (\gT_{\reach,\eps_{\mathrm B}}Q)(s,a)
    :=
    r(s,a)
    +
    \gamma
    \sum_{s'\in\gS}
    P(s'\mid s,a)
    \max_{a'\in\gA_{\reach,\eps_{\mathrm B}}(s')}
    Q(s',a').
\end{equation}
Its unique fixed point is the function $Q_{\reach,\eps_{\mathrm B}}^\star$ defined in
\eqref{eq:finite-net-optimal-q-definition}.

Initialize
$\widehat Q_0(s,a):=0$
for every finite-net state--action pair.  For an integer $J\ge1$, at each iteration $k\in\{0,\ldots,J-1\}$ and for every
$(s,a)$, draw $m$ fresh generative samples
\[
    \bigl(
        \mathsf R_{k,i}(s,a),
        \mathsf S'_{k,i}(s,a)
    \bigr),
    \qquad
    i=1,\ldots,m,
\]
and set
\begin{equation}
    \label{eq:empirical-q-iteration}
    \widehat Q_{k+1}(s,a)
    :=
    \frac1m
    \sum_{i=1}^m
    \left[
        \mathsf R_{k,i}(s,a)
        +
        \gamma
        \max_{a'\in\gA_{\reach,\eps_{\mathrm B}}(\mathsf S'_{k,i}(s,a))}
        \widehat Q_k
        \bigl(
            \mathsf S'_{k,i}(s,a),a'
        \bigr)
    \right].
\end{equation}
After $J$ iterations, let $\widehat\pi_{\reach,\eps_{\mathrm B},J}$ be a deterministic policy satisfying
\begin{equation}
    \label{eq:empirical-q-greedy-policy}
    \widehat\pi_{\reach,\eps_{\mathrm B},J}(s)
    \in
    \argmax_{a\in\gA_{\reach,\eps_{\mathrm B}}(s)}
    \widehat Q_J(s,a).
\end{equation}

\begin{theorem}[QFRS{} bias--complexity frontier]
    \label{thm:particle-bias-complexity-frontier}
    Suppose Assumptions~\ref{ass:finite-state-continuous-latent},
    \ref{ass:behavioral-lipschitzness}, and
    \ref{ass:particle-frontier-generative-model} hold.  Fix the following constants:
    \[
        \reach\geq0,
        \quad
        \eps_{\mathrm B}>0,
        \quad
        \eps_Q>0,\quad\text{and}
        \quad
        \alpha\in(0,1).
    \]
    For every state, choose a behavioral $\eps_{\mathrm B}$-net $\gA_{\reach,\eps_{\mathrm B}}(s)$ with cardinality at most
    $A_\reach(\eps_{\mathrm B})$.  Let $J\geq1$ satisfy
        $\gamma^J V_{\max}
        \leq
        \eps_Q/2$,
    and set
    \begin{equation}
        \label{eq:frontier-samples-per-pair}
        m
        :=
        \left\lceil
        \frac{2V_{\max}^2}{(1-\gamma)^2\eps_Q^2}
        \log
        \left(
            \frac{2SA_\reach(\eps_{\mathrm B})J}{\alpha}
        \right)
        \right\rceil.
    \end{equation}
    Then, with probability at least $1-\alpha$, we have that 
    \begin{equation}
        \label{eq:frontier-control-bound}
        \lVert
            V^\star
            -
            V^{\widehat\pi_{\reach,\eps_{\mathrm B},J}}
        \rVert_\infty
        \leq
        \frac{
            \bar\Delta_\reach
            +
            \eps_{\mathrm B}
            +
            2\eps_Q
        }{1-\gamma}.
    \end{equation}
    The number of generative-model samples used by
    \eqref{eq:empirical-q-iteration} is
    \begin{equation}
        \label{eq:frontier-total-samples}
        N_{\mathrm{gen}}
        =
        Jm
        \sum_{s\in\gS}
        |\gA_{\reach,\eps_{\mathrm B}}(s)|
        \le
        SJA_\reach(\eps_{\mathrm B})
        \left\lceil
        \frac{2V_{\max}^2}{(1-\gamma)^2\eps_Q^2}
        \log
        \left(
            \frac{2SA_\reach(\eps_{\mathrm B})J}{\alpha}
        \right)
        \right\rceil.
    \end{equation}
Furthermore, we have the following consequences:
\begin{enumerate}[label={\it\alph*.},itemsep=0pt, topsep=0pt, leftmargin=20pt]
    \item  For every fixed behavioral resolution $\eps_{\mathrm B}$, estimation accuracy $\eps_Q$, confidence level $\alpha$, and
    admissible iteration count $J$, increasing the reach $\reach$ moves the two structural
quantities in opposite directions: $\bar{\Delta}_\reach$ is
nonincreasing in $\reach$, whereas
$A_\reach(\eps_{\mathrm B})$ is nondecreasing in $\reach$.
    Consequently, for fixed $\eps_{\mathrm B}$, $\eps_Q$, $\alpha$, and $J$, the sample-complexity upper bound
    on the right-hand side of \eqref{eq:frontier-total-samples} is nondecreasing in $\reach$.
    \item  If Assumption~\ref{ass:smooth-local-optimal-value-geometry} also holds, then
    \begin{equation}
        \label{eq:explicit-frontier-control-bound}
        \left\lVert
            V^\star
            -
            V^{\widehat\pi_{\reach,\eps_{\mathrm B},J}}
        \right\rVert_\infty
        \le
        \frac{1}{1-\gamma}
        \left[
            \frac{\Lsmooth}{2}
            \max_{s\in\gS}
            \bigl(\dist_{\mF}(s)-\reach\bigr)_+^2
            +\eps_{\mathrm B}+2\eps_Q
        \right]
    \end{equation}
    with probability at least $1-\alpha$.  Under the standing behavioral Lipschitzness
    assumption, the bound holds:
    \begin{equation}
        \label{eq:explicit-frontier-action-count}
        A_\reach(\eps_{\mathrm B})
        \le
        C
        \left(
            1+\frac{2L_{\mathrm B}\reach}{\eps_{\mathrm B}}
        \right)^{d_z}.
    \end{equation}
  \item   If Assumption~\ref{ass:local-behavioral-richness} also holds and
    $\reach\in(0,\reach_{\mathrm{rich}}]$, then every behavioral $\eps_{\mathrm B}$-net at the distinguished state $s_0$
    contains at least
        $\left(
            \ell_{\mathrm B}\reach/(2\eps_{\mathrm B})
        \right)^{d_{\mathrm{eff}}}$
    candidates.
\end{enumerate}

\end{theorem}

\begin{proof}
    We divide the proof into five steps.

\noindent\textbf{Step 1: boundedness of the empirical iterates and targets.}
    We prove by induction that
    \begin{equation}
        \label{eq:empirical-q-boundedness}
        0
        \leq
        \widehat Q_k(s,a)
        \leq
        V_{\max}
    \end{equation}
    for every $k\in\{0, \ldots, J\}$ and every finite-net state--action pair.
    The claim holds for $k=0$ because $\widehat Q_0=0$.  Suppose it holds at iteration $k$.
    Every summand in \eqref{eq:empirical-q-iteration} is nonnegative.  It is also bounded above by
    \[
        r_{\max}+\gamma V_{\max}
        =
        r_{\max}
        +
        \gamma\frac{r_{\max}}{1-\gamma}
        =
        \frac{r_{\max}}{1-\gamma}
        =
        V_{\max}.
    \]
    Therefore their average lies in $[0,V_{\max}]$, proving
    \eqref{eq:empirical-q-boundedness} at iteration $k+1$.

\noindent\textbf{Step 2: a uniform one-step empirical Bellman bound.}
    Let $\mathcal F_k$ be the sigma-algebra generated by all samples drawn before iteration $k$.
    Then $\widehat Q_k$ is $\mathcal F_k$-measurable.  By the fresh-sample assumption, conditional
    on $\mathcal F_k$, the $m$ summands in \eqref{eq:empirical-q-iteration} for a fixed pair
    $(s,a)$ are independent, lie in $[0,V_{\max}]$, and have common conditional expectation
    $(\gT_{\reach,\eps_{\mathrm B}}\widehat Q_k)(s,a)$.
    H\"oeffding's inequality therefore gives, for every $\xi>0$, the bound
    \begin{align}
        &\mathbb P
        \left(
            \left|
                \widehat Q_{k+1}(s,a)
                -
                (\gT_{\reach,\eps_{\mathrm B}}\widehat Q_k)(s,a)
            \right|
            >\xi
            \,\middle|\,
            \mathcal F_k
        \right)\leq
        2\exp
        \left(
            -2m\xi^2/V_{\max}^2
        \right).
        \label{eq:conditional-hoeffding-frontier}
    \end{align}
    Taking expectations of both sides gives the same unconditional bound.  Across all iterations,
    there are at most
    $SJA_\reach(\eps_{\mathrm B})$
    state--action--iteration triples.  A union bound applied to
    \eqref{eq:conditional-hoeffding-frontier} yields
    \begin{align}
        &\mathbb P
        \left(
            \max_{0\le k<J}
            \left\lVert
                \widehat Q_{k+1}
                -
                \gT_{\reach,\eps_{\mathrm B}}\widehat Q_k
            \right\rVert_\infty
            >\xi
        \right)\le
        2SJA_\reach(\eps_{\mathrm B})
        \exp
        \left(
            -2m\xi^2/V_{\max}^2
        \right).
        \label{eq:uniform-empirical-bellman-probability}
    \end{align}

    Set
        $\xi
        :=
        (1-\gamma)\eps_Q/2$.
    By \eqref{eq:frontier-samples-per-pair}, we have that 
    \begin{align*}
        \frac{2m\xi^2}{V_{\max}^2}
        &=
        \frac{m(1-\gamma)^2\eps_Q^2}{2V_{\max}^2}
\ge
        \log
        \left(
            \frac{2SA_\reach(\eps_{\mathrm B})J}{\alpha}
        \right).
    \end{align*}
    Substituting into \eqref{eq:uniform-empirical-bellman-probability} shows that the event
    \begin{equation}
        \label{eq:good-empirical-bellman-event}
        \gE_{\mathrm{Bell}}
        :=
        \left\{
            \max_{0\leq k<J}
            \left\lVert
                \widehat Q_{k+1}
                -
                \gT_{\reach,\eps_{\mathrm B}}\widehat Q_k
            \right\rVert_\infty
            \leq
            \frac{(1-\gamma)\eps_Q}{2}
        \right\}
    \end{equation}
    has probability at least $1-\alpha$.

\noindent\textbf{Step 3: error of empirical Q-iteration on the finite net.}
    The operator $\gT_{\reach,\eps_{\mathrm B}}$ is a $\gamma$-contraction in supremum norm.  Indeed, for
    any two finite-net Q-functions $Q$ and $\widetilde Q$,
    \begin{align*}
        \left|
            (\gT_{\reach,\eps_{\mathrm B}}Q)(s,a)
            -
            (\gT_{\reach,\eps_{\mathrm B}}\widetilde Q)(s,a)
        \right|&\leq
        \gamma
        \sum_{s'\in\gS}
        P(s'\mid s,a)
        \left|
            \max_{a'\in\gA_{\reach,\eps_{\mathrm B}}(s')}Q(s',a')
            -
            \max_{a'\in\gA_{\reach,\eps_{\mathrm B}}(s')}\widetilde Q(s',a')
        \right|
        \\
        &\qquad\le
        \gamma\lVert Q-\widetilde Q\rVert_\infty.
    \end{align*}
    Taking the maximum over $(s,a)$ proves the contraction property.  Its fixed point is
    $Q_{\reach,\eps_{\mathrm B}}^\star$.

    Work on the event $\gE_{\mathrm{Bell}}$.  Define
    $e_k
        :=
    \lVert
            \widehat Q_k-Q_{\reach,\eps_{\mathrm B}}^\star
\rVert_\infty$.
    For $k<J$, we have that 
    \begin{align*}
        e_{k+1}
        &\leq\lVert
            \widehat Q_{k+1}
            -
            \gT_{\reach,\eps_{\mathrm B}}\widehat Q_k
\rVert_\infty+
        \lVert
            \gT_{\reach,\eps_{\mathrm B}}\widehat Q_k
            -
            \gT_{\reach,\eps_{\mathrm B}}Q_{\reach,\eps_{\mathrm B}}^\star
\rVert_\infty
\leq
        \frac{(1-\gamma)\eps_Q}{2}
        +
        \gamma e_k.
    \end{align*}
    Iterating this inequality yields
    \begin{align*}
        e_J
        &\leq
        \gamma^J e_0
        +
        \frac{(1-\gamma)\eps_Q}{2}
        \sum_{j=0}^{J-1}\gamma^j\le
        \gamma^J e_0
        +
        \frac{\eps_Q}{2}.
    \end{align*}
    Since $\widehat Q_0=0$ and
    $0\leq Q_{\reach,\eps_{\mathrm B}}^\star\leq V_{\max}$, we have that 
    $e_0
        \leq
        V_{\max}$.
    Since $\gamma^J V_{\max}
        \leq
        \frac{\eps_Q}{2}$,
   we have $\gamma^J e_0
        \leq
        \frac{\eps_Q}{2}$.
    Therefore
    \begin{equation}
        \label{eq:finite-net-q-error}
     \lVert
            \widehat Q_J-Q_{\reach,\eps_{\mathrm B}}^\star
\rVert_\infty
        \leq
        \eps_Q.
    \end{equation}

\noindent\textbf{Step 4: performance of the greedy finite-net policy.}
    Fix a state $s$ and let
    \[
        a^\star(s)
        \in
        \argmax_{a\in\gA_{\reach,\eps_{\mathrm B}}(s)}
        Q_{\reach,\eps_{\mathrm B}}^\star(s,a).
    \]
    By the definition of $\widehat\pi_{\reach,\eps_{\mathrm B},J}$ and
    \eqref{eq:finite-net-q-error}, we have that 
    \begin{align*}
        Q_{\reach,\eps_{\mathrm B}}^\star
        \bigl(
            s,\widehat\pi_{\reach,\eps_{\mathrm B},J}(s)
        \bigr)
        &\geq
        \widehat Q_J
        \bigl(
            s,\widehat\pi_{\reach,\eps_{\mathrm B},J}(s)
        \bigr)
        -\eps_Q
        \\
        &={}
        \max_{a\in\gA_{\reach,\eps_{\mathrm B}}(s)}
        \widehat Q_J(s,a)
        -\eps_Q
        \\
        &\geq
        \widehat Q_J\bigl(s,a^\star(s)\bigr)
        -\eps_Q
        \\
        &\geq
        Q_{\reach,\eps_{\mathrm B}}^\star\bigl(s,a^\star(s)\bigr)
        -2\eps_Q
        \\
        &=
        V_{\reach,\eps_{\mathrm B}}^\star(s)-2\eps_Q.
    \end{align*}
    Define the finite-net greedy residual
    $e_{\mathrm{greedy}}(s)
        :=
        V_{\reach,\eps_{\mathrm B}}^\star(s)
        -
        Q_{\reach,\eps_{\mathrm B}}^\star
        \bigl(
            s,\widehat\pi_{\reach,\eps_{\mathrm B},J}(s)
        \bigr)$.
    The preceding inequality yields
    $0\le e_{\mathrm{greedy}}(s)\le2\eps_Q$.

    Let
$D_{\mathrm{greedy}}
        :=
        V_{\reach,\eps_{\mathrm B}}^\star
        -
        V^{\widehat\pi_{\reach,\eps_{\mathrm B},J}}$.
    Adding and subtracting the finite-net policy Bellman operator applied to
    $V_{\reach,\eps_{\mathrm B}}^\star$ gives
    \begin{align*}
        D_{\mathrm{greedy}}(s)
        &=
        V_{\reach,\eps_{\mathrm B}}^\star(s)
        -
        (\gT^{\widehat\pi_{\reach,\eps_{\mathrm B},J}}
        V^{\widehat\pi_{\reach,\eps_{\mathrm B},J}})(s)
        \\
        &=
        V_{\reach,\eps_{\mathrm B}}^\star(s)
        -
        (\gT^{\widehat\pi_{\reach,\eps_{\mathrm B},J}}
        V_{\reach,\eps_{\mathrm B}}^\star)(s)+
        \gamma
        (P^{\widehat\pi_{\reach,\eps_{\mathrm B},J}}
        D_{\mathrm{greedy}})(s)
        \\
        &=
        e_{\mathrm{greedy}}(s)
        +
        \gamma
        (P^{\widehat\pi_{\reach,\eps_{\mathrm B},J}}
        D_{\mathrm{greedy}})(s).
    \end{align*}
    Taking the supremum norm yields
    \[
        \lVert D_{\mathrm{greedy}}\rVert_\infty
        \leq
        2\eps_Q
        +
        \gamma\lVert D_{\mathrm{greedy}}\rVert_\infty,
    \]
   so that rearranging gives us that 
    \begin{equation}
        \label{eq:finite-net-greedy-performance}
        \left\lVert
            V_{\reach,\eps_{\mathrm B}}^\star
            -
            V^{\widehat\pi_{\reach,\eps_{\mathrm B},J}}
        \right\rVert_\infty
        \leq
        2\eps_Q/(1-\gamma).
    \end{equation}

\textbf{Step 5: combine approximation, estimation, and monotonicity.}
    By Lemma~\ref{lem:finite-net-support-gap}, we have that
    \[
        \left\lVert
            V^\star-V_{\reach,\eps_{\mathrm B}}^\star
        \right\rVert_\infty
        \leq(\bar\Delta_\reach+\eps_{\mathrm B})/(1-\gamma).
    \]
    Combining this with \eqref{eq:finite-net-greedy-performance} and the triangle inequality gives
    \begin{align*}
        \left\lVert
            V^\star
            -
            V^{\widehat\pi_{\reach,\eps_{\mathrm B},J}}
        \right\rVert_\infty
        &\leq
        \left\lVert
            V^\star-V_{\reach,\eps_{\mathrm B}}^\star
        \right\rVert_\infty+
        \left\lVert
            V_{\reach,\eps_{\mathrm B}}^\star
            -
            V^{\widehat\pi_{\reach,\eps_{\mathrm B},J}}
        \right\rVert_\infty
        \leq
            (\bar\Delta_\reach+\eps_{\mathrm B}+2\eps_Q)/(1-\gamma).
    \end{align*}
    This proves \eqref{eq:frontier-control-bound} on the event
    $\gE_{\mathrm{Bell}}$, which has probability at least $1-\alpha$.

    The sample-count identity follows because iteration $k$ draws $m$ samples for every finite-net
    state--action pair.  Therefore
    \[
        N_{\mathrm{gen}}
        =
        Jm
        \sum_{s\in\gS}|\gA_{\reach,\eps_{\mathrm B}}(s)|
        \le
        SJmA_\reach(\eps_{\mathrm B}).
    \]

    Proposition~\ref{prop:monotone-local-reach} proves that
    $\bar\Delta_\reach$ is nonincreasing in $\reach$, while
    Proposition~\ref{prop:monotone-behavioral-complexity} proves that
    $A_\reach(\eps_{\mathrm B})$ is nondecreasing.  This proves
  that $\bar{\Delta}_\reach$ and $A_\reach(\eps_\mathrm{B})$ are monotonic.  The right-hand side of
    \eqref{eq:frontier-total-samples} is a nondecreasing function of $A_\reach(\eps_{\mathrm B})$, so its
    monotonicity in $\reach$ follows as well.

    Finally, substituting
    \eqref{eq:uniform-quadratic-local-support-bias} into
    \eqref{eq:frontier-control-bound} proves
    \eqref{eq:explicit-frontier-control-bound}, and
    \eqref{eq:explicit-frontier-action-count} follows from
    Proposition~\ref{prop:behavioral-packing-growth}.  The lower bound
    $\left(
            \ell_{\mathrm B}\reach/(2\eps_{\mathrm B})
        \right)^{d_{\mathrm{eff}}}$ is
    \eqref{eq:behavioral-net-lower-bound} from
    Lemma~\ref{lem:packing-net-sandwich}.
\end{proof}

\begin{corollary}[A target-accuracy reach choice]
    \label{cor:target-accuracy-radius-choice}
    Fix a desired full-MDP control tolerance $\eps_{\mathrm{ctrl}}>0$ and a confidence level $\alpha\in(0,1)$.  Suppose there exist
    $\reach\geq0$, $\eps_{\mathrm B}>0$ and $\eps_Q>0$ such that
    \begin{equation}
        \label{eq:target-accuracy-error-allocation}
        \bar\Delta_\reach
        \leq
      (1-\gamma)\eps_{\mathrm{ctrl}}/3,
        \qquad
        \eps_{\mathrm B}
        \leq
    (1-\gamma)\eps_{\mathrm{ctrl}}/3,
        \qquad
        2\eps_Q
        \leq
     (1-\gamma)\eps_{\mathrm{ctrl}}/3.
    \end{equation}
    If $J$ and $m$ satisfy
    $\gamma^J V_{\max}
        \leq
        \eps_Q/2$ and \eqref{eq:frontier-samples-per-pair}, then with
    probability at least $1-\alpha$,
    \[
       \lVert
            V^\star
            -
            V^{\widehat\pi_{\reach,\eps_{\mathrm B},J}}
\rVert_\infty
        \le
        \eps_{\mathrm{ctrl}}.
    \]
    Among two reaches $\reach_1\le \reach_2$ that both satisfy the support-bias requirement in
    \eqref{eq:target-accuracy-error-allocation}, the smaller reach has no larger behavioral packing complexity:
    \[
        A_{\reach_1}(\eps_{\mathrm B})
        \leq
        A_{\reach_2}(\eps_{\mathrm B}).
    \]
    Thus, if a smallest radius meeting a prescribed support-bias budget exists, it minimizes the
    behavioral packing complexity among all larger feasible radii.
\end{corollary}

\begin{proof}
    The performance conclusion follows by substituting the three inequalities in
    \eqref{eq:target-accuracy-error-allocation} into
    \eqref{eq:frontier-control-bound}.  The monotonicity conclusion follows from
    Proposition~\ref{prop:monotone-behavioral-complexity}.
\end{proof}

\begin{corollary}[Nested-radius adaptation cost in the latent-mode specialization]
    \label{cor:nested-reach-adaptation-cost}
    Consider the one-decision Bernoulli latent-mode construction of
    Theorem~\ref{thm:persistent-abstraction-separation}.  Suppose a resolution-$\eps_{\mathrm B}$ representation
    of a smaller reach $\reach_1$ contains $A_1$ behaviorally distinct modes and a representation of
    a larger reach $\reach_2\ge \reach_1$ contains a nested set of $A_2>A_1$ modes.  Fix
    $\eps\in(0,1/8]$ and $\alpha\in(0,1/4]$.
    There exists a zero-bias instance for the smaller reach on which any learner that remains
    $(\eps,\alpha)$-correct over the larger $A_2$-mode class must collect, in expectation, at least
    \begin{equation}
        \label{eq:nested-reach-lower-bound}
        \frac{3(A_2-A_1)}{256\eps^2}
        \log\frac{3}{4\alpha}
    \end{equation}
    online samples from modes admitted by the larger reach but excluded by the smaller reach.
\end{corollary}

\begin{proof}
    Identify the $A_1$ smaller-reach modes with the persistent abstraction in
    Theorem~\ref{thm:persistent-abstraction-separation}, and identify the remaining
    $A_2-A_1$ larger-reach modes with the non-abstraction modes in that theorem.  The base instance
    used in Theorem~\ref{thm:persistent-abstraction-separation} has its optimal mode inside the
    smaller set, so the smaller reach has zero abstraction bias.  For each of the
    $A_2-A_1$ additional modes, the corresponding alternative instance changes only that mode and
    makes it uniquely $\eps$-optimal.  Part~\textnormal{(b)} of
    Theorem~\ref{thm:persistent-abstraction-separation}, with
    $C=A_1$,
and 
        $N=A_2$,
    gives exactly \eqref{eq:nested-reach-lower-bound}.
\end{proof}

\subsection{Relation to a finite Gaussian particle cloud}
\label{subsec:gaussian-particle-cloud-connection}

The implemented particle construction uses
\begin{equation}
    \label{eq:implemented-gaussian-cloud}
    z_{c,k}(s)
    =
    z_c(s)
    +
    \sigma_{ p}\xi_{c,k},
    \qquad
    \xi_{c,k}\sim\mathcal N(0,I_{d_z}),
\end{equation}
with an exact center particle $z_{c,0}(s)=z_c(s)$.  Since a Gaussian distribution has full
support for every $\sigma_{ p}>0$, the scale $\sigma_{ p}$ does not define a hard
support radius.  The following proposition, which leverages standard covering arguments \citep{wainwright2019high,vershynin2019high}, instead quantifies how many particles are sufficient
to form a Euclidean net of a prescribed reach-$\reach$ envelope.

\begin{proposition}[Finite Gaussian coverage of a bounded local envelope]
    \label{prop:finite-gaussian-envelope-coverage}
    Fix constants
    \[
        \reach>0,
        \qquad
        \eps_z>0,
        \qquad
        \sigma_{ p}>0,
        \qquad
        \alpha\in(0,1).
    \]
    Fix a state $s\in \mathcal{S}$.  For every command $c\in \mathcal{C}$, draw $M$ independent particles according to
    \eqref{eq:implemented-gaussian-cloud}.  Define
    \begin{equation}
        \label{eq:gaussian-minimum-small-ball-mass}
        p_{\min}(\reach,\eps_z,\sigma_{p })
        :=
        \inf_{\lVert u\rVert_2\leq \reach}
        \mathbb P
        \left(
            \left\lVert
                \sigma_{p}\xi-u
            \right\rVert_2
            \leq
            \frac{\eps_z}{2}
        \right),
        \qquad
        \xi\sim\mathcal N(0,I_{d_z}).
    \end{equation}
    Then $p_{\min}(\reach,\eps_z,\sigma_{ p})>0$.  Let
    $N_{\reach,\eps_z}
        :=
        \left(
            1+4\reach/\eps_z
        \right)^{d_z}$.
    If
    \begin{equation}
        \label{eq:gaussian-particle-count-coverage}
        M
        \ge
        \frac{1}{p_{\min}(\reach,\eps_z,\sigma_{p })}
        \log
        \left(
            CN_{\reach,\eps_z}/\alpha
        \right),
    \end{equation}
    then, with probability at least $1-\alpha$, simultaneously for every command $c$, the sampled
    particles around $z_c(s)$ form a Euclidean $\eps_z$-net of
    $B_2(z_c(s),\reach)$.
    Moreover, the lower bound holds:
    \begin{equation}
        \label{eq:explicit-gaussian-small-ball-lower}
        p_{\min}(\reach,\eps_z,\sigma_{p })
        \geq
        \operatorname{vol}
        \left(
            B_2\left(0,\frac{\eps_z}{2}\right)
        \right)
        (2\pi\sigma_{p }^2)^{-d_z/2}
        \exp
        \left(
            -
            \frac{(\reach+\eps_z/2)^2}{2\sigma_{p }^2}
        \right).
    \end{equation}
\end{proposition}

\begin{proof}
    Let
   $\{u_1,\ldots,u_N\}$
    be a Euclidean $\eps_z/2$-net of $B_2(0,\reach)$.  A standard Euclidean volumetric argument, with resolution $\eps_z/2$, permits the
    choice
    \[
        N
        \le
        \left(
            1+\frac{4\reach}{\eps_z}
        \right)^{d_z}
        =
        N_{\reach,\eps_z}.
    \]

    Fix a command $c$ and net point $u_j$.  Each perturbation
    $\sigma_{p }\xi_{c,k}$ falls in
    $B_2(u_j,\eps_z/2)$ with probability at least
    $p_{\min}(\reach,\eps_z,\sigma_{p })$.  Therefore the probability that none of the $M$
    particles falls in this ball is at most
    \[
        \bigl(1-p_{\min}\bigr)^M
        \leq
        \exp(-Mp_{\min}),
    \]
    where the final inequality uses $1-x\leq e^{-x}$ for $x\in[0,1]$.  A union bound over all
    $C N$ command--net-point pairs gives failure probability at most
    \[
        CN\exp(-Mp_{\min})
        \leq
        CN_{\reach,\eps_z}\exp(-Mp_{\min}).
    \]
    Under \eqref{eq:gaussian-particle-count-coverage}, this quantity is at most $\alpha$.

    On the complementary event, fix a command $c$ and arbitrary
    $u\in B_2(0,\reach)$.  Choose a net point $u_j$ satisfying
   $\lVert u-u_j\rVert_2\leq\frac{\eps_z}{2}$.
    By the event just established, there exists a particle perturbation
    $\sigma_{p }\xi_{c,k}$ satisfying
    $\left\lVert
            \sigma_{p }\xi_{c,k}-u_j
        \right\rVert_2
        \leq
        \frac{\eps_z}{2}$.
    The triangle inequality yields
 $\left\lVert
            \sigma_{p }\xi_{c,k}-u
        \right\rVert_2
        \leq \eps_z$.
    Translating by $z_c(s)$ proves that the candidate particles form an $\eps_z$-net of
    $B_2(z_c(s),\reach)$.

    It remains to prove \eqref{eq:explicit-gaussian-small-ball-lower}.  Fix any
    $u$ with $\lVert u\rVert_2\leq \reach$.  The density of
    $\sigma_{p }\xi$ at $x\in\R^{d_z}$ is
    \[
        f_{\sigma_{p }}(x)
        =
        (2\pi\sigma_{p }^2)^{-d_z/2}
        \exp
        \left(
            -\frac{\lVert x\rVert_2^2}{2\sigma_{p }^2}
        \right).
    \]
    If $x\in B_2(u,\eps_z/2)$, then
    \[
        \lVert x\rVert_2
        \leq
        \lVert u\rVert_2+\lVert x-u\rVert_2
        \leq
        \reach+\frac{\eps_z}{2}.
    \]
    Hence throughout $B_2(u,\eps_z/2)$,
    \[
        f_{\sigma_{p }}(x)
        \geq
        (2\pi\sigma_{p }^2)^{-d_z/2}
        \exp
        \left(
            -\frac{(\reach+\eps_z/2)^2}{2\sigma_{p }^2}
        \right).
    \]
    Integrating this lower bound over $B_2(u,\eps_z/2)$ yields
    \begin{align*}
        \mathbb P
        \left(
            \lVert\sigma_{p }\xi-u\rVert_2
            \le
            \frac{\eps_z}{2}
        \right)
        &\ge
        \operatorname{vol}
        \left(
            B_2\left(0,\frac{\eps_z}{2}\right)
        \right)
\cdot
        (2\pi\sigma_{p }^2)^{-d_z/2}
        \exp
        \left(
            -\frac{(\reach+\eps_z/2)^2}{2\sigma_{p }^2}
        \right).
    \end{align*}
    The right-hand side is strictly positive and independent of $u$.  Taking the infimum over
    $\lVert u\rVert_2\le \reach$ proves both positivity and
    \eqref{eq:explicit-gaussian-small-ball-lower}.
\end{proof}

\begin{corollary}[From Euclidean particle coverage to behavioral coverage]
    \label{cor:gaussian-behavioral-coverage}
    Suppose Assumption~\ref{ass:behavioral-lipschitzness} holds, so that
    \[
        d_s^{\mathrm B}(z,\tilde z)
        \le
        L_{\mathrm B}\lVert z-\tilde z\rVert_2
    \]
    for every state and pair of latents.  Fix a behavioral resolution
    $\eps_{\mathrm B}>0$ and choose a Euclidean resolution $\eps_z>0$ satisfying
$L_{\mathrm B}\eps_z
        \le
        \eps_{\mathrm B}$.
    If the particle count satisfies
    \eqref{eq:gaussian-particle-count-coverage} with this $\eps_z$, then, with probability at least
    $1-\alpha$, the sampled particles form a behavioral $\eps_{\mathrm B}$-net of $\gZ_\reach(s)$.
\end{corollary}

\begin{proof}
    On the event in Proposition~\ref{prop:finite-gaussian-envelope-coverage}, for every
    $z\in\gZ_\reach(s)$ there exists a sampled particle $\tilde z$ satisfying
    \[
        \lVert z-\tilde z\rVert_2
        \le
        \eps_z.
    \]
    Assumption~\ref{ass:behavioral-lipschitzness} therefore gives
    \[
        d_s^{\mathrm B}(z,\tilde z)
        \le
        L_{\mathrm B}\eps_z
        \le
        \eps_{\mathrm B}.
    \]
    This is exactly the definition of a behavioral $\eps_{\mathrm B}$-net.
\end{proof}

\textbf{From behavioral coverage to representation error.}
The preceding corollary shows that the finite Gaussian cloud approximates the
idealized reach-$\reach$ class geometrically.  To connect this coverage statement
to the control bounds above, we next translate behavioral approximation into
value approximation.  Because the Gaussian sampler has unbounded support, a
realized particle may lie outside $\gZ_\reach(s)$ and may even have value larger
than the best action in the reach-$\reach$ class.  We therefore measure only the
one-sided loss incurred when the finite bank fails to represent the best action
available in $\gZ_\reach(s)$.

For a realized particle bank $\mathcal B_P(s)$, define its one-step
representation error relative to the reach-$\reach$ class by
\begin{equation}
    \label{eq:gaussian-representation-error}
    \eps_{\mathrm{rep}}
    \bigl(s,\mathcal B_P(s);\reach\bigr)
    :=
    \left[
        \sup_{z\in\gZ_\reach(s)} Q^\star(s,z)
        -
        \max_{\tilde z\in\mathcal B_P(s)}
        Q^\star(s,\tilde z)
    \right]_+ ,
\end{equation}
where $[x]_+:=\max\{x,0\}$.

\begin{corollary}[Gaussian particles imply bounded representation error]
    \label{cor:gaussian-representation-error}
    Suppose the conditions of
    Corollary~\ref{cor:gaussian-behavioral-coverage} hold, so that with
    probability at least $1-\alpha$ the sampled particle bank is a behavioral
    $\eps_{\mathrm B}$-net of $\gZ_\reach(s)$.  
    Then, with probability at least $1-\alpha$, the bound holds:
        $\eps_{\mathrm{rep}}
        \bigl(s,\mathcal B_P(s);\reach\bigr)
        \leq
        \eps_{\mathrm B}$, which is 
equivalent to 
    \begin{equation}
        \label{eq:gaussian-value-coverage}
        \max_{\tilde z\in\mathcal B_P(s)}
        Q^\star(s,\tilde z)
        \geq
        \sup_{z\in\gZ_\reach(s)}
        Q^\star(s,z)-\eps_{\mathrm B}.
    \end{equation}
\end{corollary}

\begin{proof}
    Let $\mathcal E$ denote the event that
    $\mathcal B_P(s)$ is a behavioral $\eps_{\mathrm B}$-net of
    $\gZ_\reach(s)$.  By
    Corollary~\ref{cor:gaussian-behavioral-coverage},
    we have that $\mathbb P(\mathcal E)\ge 1-\alpha$.
    We prove the claimed value bound on $\mathcal E$.

    Fix an arbitrary $\eta>0$.  By the definition of the supremum, there
    exists $z_\eta\in\gZ_\reach(s)$ such that
    \begin{equation}
        \label{eq:eta-optimal-reach-action}
        Q^\star(s,z_\eta)
        \ge
        \sup_{z\in\gZ_\reach(s)}Q^\star(s,z)-\eta.
    \end{equation}
    Since $\mathcal E$ holds and the particle bank is a behavioral
    $\eps_{\mathrm B}$-net of $\gZ_\reach(s)$, there exists
    $\tilde z_\eta\in\mathcal B_P(s)$ satisfying
    $d_s^{\mathrm B}(z_\eta,\tilde z_\eta)
        \leq
        \eps_{\mathrm B}$. By Lemma~\ref{lem:behavioral-pseudometric-properties}, we have that 
\[
    Q^\star(s,\widetilde z_\eta)
    \geq
    Q^\star(s,z_\eta)
    -
    d_s^{\mathrm B}(z_\eta,\widetilde z_\eta)
\geq
    Q^\star(s,z_\eta)
    -
    \eps_{\mathrm B}.\]
    Combining this inequality with
    \eqref{eq:eta-optimal-reach-action} yields
    \[
        Q^\star(s,\tilde z_\eta)
        \geq
        \sup_{z\in\gZ_\reach(s)}Q^\star(s,z)-\eta-\eps_{\mathrm B}.
    \]
    Because $\tilde z_\eta$ is a member of the realized particle bank,
    \[
        \max_{\tilde z\in\mathcal B_P(s)}
        Q^\star(s,\tilde z)
        \geq
        \sup_{z\in\gZ_\reach(s)}Q^\star(s,z) -
        \eta-\eps_{\mathrm B}.
    \]
    Since $\eta>0$ was arbitrary, letting $\eta\downarrow0$ gives
    \[
        \max_{\tilde z\in\mathcal B_P(s)}
        Q^\star(s,\tilde z)
        \geq
        \sup_{z\in\gZ_\reach(s)}Q^\star(s,z)-\eps_{\mathrm B}.
    \]
    Rearranging and taking the positive part proves
    $\eps_{\mathrm{rep}}
        \bigl(s,\mathcal B_P(s);\reach\bigr)
        \leq\eps_{\mathrm B}$.
    The result therefore holds on an event of probability at least
    $1-\alpha$.
\end{proof}

\begin{remark}[What the Gaussian coverage result does and does not imply]
\label{rem:gaussian-coverage-interpretation}
Proposition~\ref{prop:finite-gaussian-envelope-coverage},
Corollary~\ref{cor:gaussian-behavioral-coverage}, and
Corollary~\ref{cor:gaussian-representation-error} provide the bridge from the
implemented Gaussian particle mechanism to the deterministic reach-$\reach$
analysis above.
In particular, under the stated regularity conditions, the representation-error
term  may be taken to satisfy
$\eps_{\mathrm{rep}}\leq\eps_{\mathrm B}$ with probability at least
$1-\alpha$. These coverage and representation-error guarantees are stated for a
single fixed state $s$.  A simultaneous guarantee over all  states
requires an additional union bound, for example by replacing $\alpha$
with $\alpha/S$ in the preceding particle-count condition.
    These results do not identify $\reach$ with $\sigma_{p }$ or $\eps_z$ with $\eps_{\mathrm B}$.  The former
    pair controls envelope size versus random sampling scale; the latter pair is connected only
    through a behavioral regularity condition such as
    $d_s^{\mathrm B}(z,\tilde z)\leq L_{\mathrm B}\lVert z-\tilde z\rVert_2$.

    The result also shows why finite-particle coverage
    can be difficult in a high-dimensional isotropic latent space: the lower bound on
    $p_{\min}$, and hence the required particle count, can deteriorate rapidly with dimension.
    The parameter $d_{\mathrm{eff}}$ in
Proposition~\ref{prop:behavioral-packing-growth} measures local behavioral
richness through a packing lower bound. It does not by itself establish
a low-dimensional upper bound on behavioral complexity or improve the
Gaussian coverage guarantee.
\end{remark}

\subsection{Interpretation, novelty, and connection to the preceding results}
\label{subsec:particle-frontier-interpretation}

Theorem~\ref{thm:particle-bias-complexity-frontier} describes a
reach-parameterized bias--complexity tradeoff connecting the two regimes
identified by Theorems~\ref{thm:persistent-abstraction-separation} and
\ref{thm:general-mdp-abstraction-price}.

At zero reach, $\reach=0$, the local action set is exactly the online flow-reversed anchor bank, so the structural bias is
$\bar\Delta_{\mF}$ and the action class has at most $C$ candidates before any further
discretization.  Increasing $\reach$ admits local residual corrections.  The oracle support gap
$\bar\Delta_\reach$ can only decrease, and under smooth local geometry it is bounded above
by a quadratic function of the residual distance from the nearest anchor to an optimal latent.  At the same time, the maximal number of pairwise behaviorally distinguishable actions at
resolution $\eps_{\mathrm B}$ can only increase.  Under nondegenerate local geometry, every behavioral $\eps_{\mathrm B}$-net
must contain at least order $(\reach/\eps_{\mathrm B})^{d_{\mathrm{eff}}}$ candidates.

The resulting finite-sample sub-optimality guarantee~\eqref{eq:frontier-control-bound} separates three error contributions, each scaled by $1/(1-\gamma)$:
\[
\bar\Delta_\reach+
\eps_{\mathrm B}+2\eps_Q,
\]
namely, the remaining local support bias, behavior discretization error, and finite-sample value-selection error, respectively. 
The same reach that reduces the first term enlarges the behavioral action class appearing in the
sample count.  This is the precise sense in which \QFRS{} traces a bias--complexity
frontier rather than simply providing more exploration.

The connection to Theorem~\ref{thm:persistent-abstraction-separation} is statistical.  In a
one-decision specialization, enlarging the local reach adds latent modes that must be ruled out if
the learner remains correct over the larger class.  Corollary~\ref{cor:nested-reach-adaptation-cost}
shows that the additional verification cost is proportional to the number of newly admitted modes,
even on a zero-bias instance for the smaller reach.

The connection to Theorem~\ref{thm:general-mdp-abstraction-price} is structural.  That theorem is
the zero-reach endpoint $\reach=0$.  Proposition~\ref{prop:radius-reach-abstraction-price} applies the same
one-step $V^\star-Q^\star$ support-gap logic to every local reach.  Particle expansion therefore
changes the support term itself; it does not merely improve selection among the original anchors.

The contraction, packing, H\"oeffding, and approximate-value-iteration arguments are
standard.  The Q-FRS-specific contribution is their organization around the state-dependent local
FRS action family.  In particular, the theorem
\begin{enumerate}[label={\it\roman*.},itemsep=0pt, topsep=0pt]
    \item identifies a monotone path from exact-anchor
Q-FRS toward broader latent control; 
\item measures complexity in terms of behaviorally
distinguishable candidates rather than raw particle count; 
\item places support bias and finite-data
selection error in the same full-MDP performance bound;
\item distinguishes deterministic local
reach from the stochastic Gaussian scale used to instantiate a finite particle cloud.
\end{enumerate}

The finite-sample guarantee above applies to synchronous empirical
Q-iteration on fixed behavioral nets under a generative model.
The Gaussian coverage results connect finite particle clouds to the
representation-error term, but do not by themselves provide a
sample-complexity guarantee for the implemented hierarchical policy.
Theorem~\ref{thm:hierarchical-anchorq-control} separately bounds the
suboptimality of that policy, accounting for coarse-anchor selection,
particle weighting, critic error, and barycentric refinement.

The reach parameter and the value-tilting parameter have different roles.  Reach determines which
local behaviors are made available; value tilting reallocates probability within the behaviors that
are available.  The next section studies the latter operation and its connection to entropic risk.

\section{Entropic weighting for local particle refinement}
\label{sec:exponential-tilting}
Theorem~\ref{thm:particle-bias-complexity-frontier} separates the coarse
behavioral representation supplied by the flow-reversed anchors from the local degrees of
freedom supplied by particles.  In the implemented method, a coarse anchor is
selected first.  Exponential tilting is then applied only to the local particle
cloud around that anchor, and the resulting weights are used to construct a
weighted latent refinement.  This section studies this local weighting operation
and its effect on the executed barycentric refinement.

Throughout this section, $\betatilt\geq 0$ has one meaning only: it is the exponential-tilting
intensity, equivalently an inverse temperature.  It is distinct from the confidence parameter
$\alpha$, the MDP discount $\gamma$, the deterministic local reach $\reach$, the Gaussian particle
scale $\sigma_{p}$, and the flow-integration step size $h$.  We analyze only the
positive, reward-seeking tilt $\betatilt>0$; no negative or lower-tail tilt is considered.

\textbf{Notation introduced in this section.}
To avoid collisions with the preceding results, the probability law before value reweighting is
$\nu_s$, its exponential tilt is $\nu_{s,\betatilt}^{q}$,
$\mathrm{CE}_{\betatilt}(q;s)$ is the entropic candidate evaluation, and
$\eps_{\mathrm{crit}}$ is a uniform score-estimation error.  A realized finite bank is indexed by
a finite set $\mathcal J_s$; its cardinality is written explicitly as $|\mathcal J_s|$ rather than
introducing another symbol.  None of this notation is used for the deterministic reach $R$,
Gaussian particle scale $\sigma_{p}$, behavioral resolution $\eps_{\mathrm B}$,
Euclidean latent resolution $\eps_z$, flow step size $h$, discount factor $\gamma$, confidence
level $\alpha$, or sample count $n$ defined earlier.

Throughout, we will use the standard notation for moment generating functions. For instance, $\Mom_X(a)=\E[e^{a X}]$,
where $\Mom_X$
 is the moment-generating function of the random variable $X$.
For a conditional moment-generating function, we use the notation $\Mom_{X|Y}(a|y):=\E[e^{a X}|Y=y]$ for random variables $X$ and $Y$.

\subsection{The local particle law and its entropic tilt}
\label{subsec:tilt-candidate-law}

Fix a state $s\in\gS$ and a coarse anchor $c\in\gC$ selected by the
anchor-level \QFRS{} rule.  We study only the subsequent local refinement around
$z_c(s)$.  As in the particle-refinement analysis, let
$\gZ=\R^{d_z}$ equipped with its Borel $\sigma$-algebra
$\mathcal B(\gZ)$.

Let $\nu_{s,c}$
be a probability measure on $(\gZ,\mathcal B(\gZ))$ describing the local
particle law around the selected anchor before value-based reweighting.  For
example, the population particle construction
\[
    Z_{s,c}
    =
    z_c(s)+\sigma_{p}\xierr,
    \qquad
    \xierr\sim\mathcal N(0,I_{d_z}),
\]
induces a Gaussian local law centered at $z_c(s)$.  For the implemented finite
particle cloud, after conditioning on a realized bank
\[
    b
    =
    \bigl(z_{c,0}(s),\ldots,z_{c,M}(s)\bigr),
\]
the corresponding reference law is the uniform law over its indexed
particles,
\[
    \nu_{s,c,b}
    :=
    \frac{1}{M+1}
    \sum_{k=0}^{M}
    \delta_{z_{c,k}(s)}.
\]
The abstract results below are stated for $\nu_{s,c}$; the finite realized-bank
specialization is developed in
Section~\ref{subsec:finite-cloud-tilting}.

\begin{assumption}[Bounded measurable local score]
    \label{ass:bounded-tilting-score}
    The score
    $q:\gS\times\gZ\to\R$
    is measurable.  There exist finite constants
    $q_{\min}<q_{\max}$ such that, for every state and selected anchor under
    consideration,
    \begin{equation}
        \label{eq:bounded-tilting-score}
        q_{\min}
        \leq
        q(s,z)
        \leq
        q_{\max}
        \qquad
        \text{for $\nu_{s,c}$-almost every }z.
    \end{equation}
\end{assumption}

Boundedness is natural when $q$ is an action value in the bounded-reward
discounted MDP of
Theorems~\ref{thm:general-mdp-abstraction-price}
and~\ref{thm:particle-bias-complexity-frontier}.  It guarantees that the
exponential moments and derivatives used below are finite.  The same results
extend under weaker exponential-moment assumptions, but boundedness avoids
unnecessary domain qualifications.

\begin{definition}[Entropic certainty equivalent\protect\footnote{The notation $\mathrm{CE}$ denotes a certainty equivalent.  If
$Y_{s,c}^{q}:=q(s,Z_{s,c})$ with $Z_{s,c}\sim\nu_{s,c}$, then
$\mathrm{CE}_{\beta}(q;s,c)$ is the deterministic scalar satisfying
$\exp\bigl(
        \beta\,\mathrm{CE}_{\beta}(q;s,c)
    \bigr)
    =
    \mathbb{E}\left[
        \exp\bigl(\beta Y_{s,c}^{q}\bigr)
    \right]$.
Solving for this scalar gives exactly
$\mathrm{CE}_{\beta}(q;s,c)
=
\beta^{-1}
\log
\mathbb{E}[
\exp(\beta Y_{s,c}^{q})
]$.} of local-particle value]
    \label{def:entropic-candidate-evaluation}
    For $\betatilt>0$, define
    \begin{equation}
        \label{eq:entropic-candidate-evaluation}
        \mathrm{CE}_{\betatilt}(q;s,c)
        :=
        \frac{1}{\betatilt}
        \log
        \int_{\gZ}
        \exp \bigl(\betatilt q(s,z)\bigr)
        \,\nu_{s,c}(\mathrm dz).
    \end{equation}
    At zero tilt, define
        $\mathrm{CE}_{0}(q;s,c)
        :=
        \int_{\gZ}
        q(s,z)\,\nu_{s,c}(\mathrm dz)$.
\end{definition}

Corollary~\ref{cor:entropic-mean-variance-expansion} below shows that the
zero-tilt definition is the continuous extension as
$\betatilt\downarrow0$.  The quantity
$\mathrm{CE}_{\betatilt}(q;s,c)$ is the reward-oriented entropic certainty
equivalent of the values of local refinements around the selected anchor.
Positive $\betatilt$ places increasing emphasis on the upper tail of those
local scores.

\begin{definition}[Exponentially tilted local weighting law]
    \label{def:exponentially-tilted-law}
    For $\betatilt>0$, define the probability measure
    $\nu_{s,c,\betatilt}^{q}$ by
    \begin{equation}
        \label{eq:tilted-radon-nikodym}
        \frac{
            \mathrm d\nu_{s,c,\betatilt}^{q}
        }{
            \mathrm d\nu_{s,c}
        }(z)
        :=
        \exp
        \left(
            \betatilt
            \left(
                q(s,z)
                -
                \mathrm{CE}_{\betatilt}(q;s,c)
            \right)
        \right).
    \end{equation}
    At zero tilt, set
    $\nu_{s,c,0}^{q}
        :=
        \nu_{s,c}$.
\end{definition}

The normalization in~\eqref{eq:tilted-radon-nikodym} is exact:
\[
\begin{aligned}
    &
    \int_{\gZ}
    \exp
    \left(
        \betatilt
        \left[
            q(s,z)
            -
            \mathrm{CE}_{\betatilt}(q;s,c)
        \right]
    \right)
    \nu_{s,c}(\mathrm dz)\\
    &\qquad=
    \exp \left(
        -\betatilt
        \mathrm{CE}_{\betatilt}(q;s,c)
    \right)
    \int_{\gZ}
    e^{\betatilt q(s,z)}
    \nu_{s,c}(\mathrm dz)
    =
    1.
\end{aligned}
\]

Let
$Z_{s,c}\sim\nu_{s,c}$
and define the scalar local-score random variable
$Y_{s,c}^{q}
    :=
    q(s,Z_{s,c})$.
Then
\begin{equation}
    \label{eq:entropic-evaluation-random-variable}
    \mathrm{CE}_{\betatilt}(q;s,c)
    =
    \frac{1}{\betatilt}
    \log
    \E
    \left[
        e^{\betatilt Y_{s,c}^{q}}
    \right].
\end{equation}
Thus the randomness in the entropic evaluation is the randomness of the local
particle around the already-selected anchor.  Unless $q$ is itself defined
through a conditional entropic transform of realized return,
$Y_{s,c}^{q}$ is a random local candidate score, not the environment return.

\paragraph{Risk-sensitive interpretation.}
Equation~\eqref{eq:entropic-evaluation-random-variable} is the
reward-oriented entropic certainty equivalent of the local particle score
$Y_{s,c}^{q}$.  Under the reward convention used here,
$\betatilt>0$ is upper-tail seeking: increasing $\betatilt$ places greater
emphasis on high-value local refinements.

Importantly, $\nu_{s,c,\betatilt}^{q}$ is an auxiliary
\emph{weighting law}, not the distribution of the action executed by
\QFRS{}.  In the implemented finite-cloud method, its probability masses are
the exponential weights
\[
    w_{c,\betatilt}(k\mid s)
    \propto
    \exp \left(
        \betatilt Q_\psi(s,z_{c,k}(s))
    \right),
\]
and these weights are used to construct the executed barycentric refinement
\[
    \tilde z_{c,\betatilt}(s)
    =
    \sum_{k=0}^{M}
    w_{c,\betatilt}(k\mid s)
    z_{c,k}(s).
\]
Thus the entropic transform governs how strongly each local particle
influences the executed refinement; it does not define a policy that samples
one of those particles for execution.

Theorem~\ref{thm:exponential-tilting} makes the variational interpretation
exact.  The entropic certainty equivalent and the exponentially tilted
weighting law are two views of the same optimization:
\begin{equation}
    \label{eq:entropic-dual-preview}
    \mathrm{CE}_{\betatilt}(q;s,c)
    =
    \sup_{\probmeas\ll\nu_{s,c}}
    \left\{
        \int_{\gZ}
        q(s,z)\,\probmeas(\mathrm dz)
        -
        \frac{1}{\betatilt}
        \KL(\probmeas\|\nu_{s,c})
    \right\},
\end{equation}
whose unique optimizer is
$\nu_{s,c,\betatilt}^{q}$.
Hence the exponential weights are the optimal change of measure in the
variational representation of the entropic risk functional: they favor
high-value local deviations while the KL penalty limits departure from the
untilted local particle law.

\subsection{The exact exponential-tilting theorem}
\label{subsec:exact-exponential-tilting-theorem}

For probability measures $\probmeas$ and $\nu$ on $\gZ$, define the extended
relative entropy
\[
    \KL(\probmeas\|\nu)
    :=
    \begin{cases}
        \displaystyle
        \int_{\gZ}
        f(z)\log f(z)\,\nu(\mathrm dz),
        &
        \text{if $\probmeas\ll\nu$ with
        $f=\mathrm d\probmeas/\mathrm d\nu$},
        \\
        +\infty,
        &
        \text{otherwise}.
    \end{cases}
\]
Here $0\log0:=0$.  For $\betatilt>0$, define the local KL-regularized
score functional
\begin{equation}
    \label{eq:kl-regularized-score-functional}
     \ScoreFn_{\betatilt,s,c}(\probmeas;q)
    :=
    \int_{\gZ}q(s,z)\,\probmeas(\mathrm dz)
    -
    \frac{1}{\betatilt}
    \KL(\probmeas\|\nu_{s,c}),
\end{equation}
with the convention
$ \ScoreFn_{\betatilt,s,c}(\probmeas;q)=-\infty$
whenever $\KL(\probmeas\|\nu_{s,c})=+\infty$.

For a Borel set $A\subseteq\gZ$ and probability measure $\probmeas$
satisfying $0<\probmeas(A)<1$, define
\begin{equation}
    \label{eq:odds-definition}
    \operatorname{Odds}_{\probmeas}(A)
    :=
    \frac{\probmeas(A)}{1-\probmeas(A)}.
\end{equation}
For $0<\nu_{s,c}(A)<1$, define the essential score separation
\begin{equation}
    \label{eq:essential-score-separation}
    \Xi_{\nu_{s,c}}(A;q)
    :=
    \operatorname{ess\,inf}_{z\in A}^{\nu_{s,c}}q(s,z)
    -
    \operatorname{ess\,sup}_{z\in A^c}^{\nu_{s,c}}q(s,z).
\end{equation}

\begin{theorem}[
    Entropic local weighting: duality, support, monotonicity, and robust
    high-value amplification
]
\label{thm:exponential-tilting}
Suppose Assumption~\ref{ass:bounded-tilting-score} holds.
Fix a state $s\in\gS$ and an anchor $c\in\gC$.

\begin{enumerate}[
    label=\textnormal{\it\alph*.},
    itemsep=0pt,
    topsep=-2pt,
    leftmargin=20pt
]

\item \textbf{Exact KL-regularized optimizer.}
For every $\betatilt>0$ and every probability measure $\probmeas$ on
$\gZ$ satisfying $\KL(\probmeas\|\nu_{s,c})<\infty$, the local KL-regularized score functional satisfies
\begin{equation}
    \label{eq:gibbs-pythagorean-identity}
     \ScoreFn_{\betatilt,s,c}(\probmeas;q)
    =
    \mathrm{CE}_{\betatilt}(q;s,c)
    -
    \frac{1}{\betatilt}
    \KL
    \left(
        \probmeas
        \,\middle\|\,
        \nu_{s,c,\betatilt}^{q}
    \right),
\end{equation}
so that 
\begin{equation}
    \label{eq:gibbs-variational-principle}
    \mathrm{CE}_{\betatilt}(q;s,c)
    =
    \sup_{\probmeas}
    \left\{
        \int_{\gZ}q(s,z)\,\probmeas(\mathrm dz)
        -
        \frac{1}{\betatilt}
        \KL(\probmeas\|\nu_{s,c})
    \right\},
\end{equation}
and the unique optimizer is
$\nu_{s,c,\betatilt}^{q}$.  In particular,
\begin{equation}
    \label{eq:optimizer-value-identity}
    \mathrm{CE}_{\betatilt}(q;s,c)
    =
    \int_{\gZ}
        q(s,z)\,
        \nu_{s,c,\betatilt}^{q}(\mathrm dz)
    -
    \frac{1}{\betatilt}
    \KL
    \left(
        \nu_{s,c,\betatilt}^{q}
        \,\middle\|\,
        \nu_{s,c}
    \right).
\end{equation}

\item \textbf{Support preservation of the weighting law.}
For every $\betatilt>0$,  the measures $\nu_{s,c,\betatilt}^{q}$
and
    $\nu_{s,c}$ are mutually absolutely continuous (i.e., $\nu_{s,c,\betatilt}^{q}
    \sim
    \nu_{s,c}$) 
and hence
\begin{equation}
    \label{eq:support-preservation}
    \operatorname{supp}
(
        \nu_{s,c,\betatilt}^{q}
)
    =
    \operatorname{supp}(\nu_{s,c}).
\end{equation}
Thus exponential tilting neither removes nor adds particles to the local
weighting support.  This statement concerns the weighting law: the barycentric
latent executed by \QFRS{} may lie in the convex hull of this support rather
than in the support itself.

\item \textbf{Monotone score and KL path.}
Define
\begin{align*}
    \mathcal{S}_{q,s,c}(\betatilt)
    &:=
    \int_{\gZ}
        q(s,z)\,
        \nu_{s,c,\betatilt}^{q}(\mathrm dz),
    \quad\text{and}\quad
    \mathsf{KL}_{q,s,c}(\betatilt)
    :=
    \KL
    \left(
        \nu_{s,c,\betatilt}^{q}
        \,\middle\|\,
        \nu_{s,c}
    \right).
\end{align*}
Then the following bounds hold:
\begin{align}
    \label{eq:mean-score-derivative}
    \frac{\mathrm d}{\mathrm d\betatilt}
    \mathcal{S}_{q,s,c}(\betatilt)
    &=
    \Var_{\nu_{s,c,\betatilt}^{q}}
    \left(
        q(s,Z)
    \right)
    \ge0,
    \\
    \label{eq:kl-derivative}
    \frac{\mathrm d}{\mathrm d\betatilt}
    \mathsf{KL}_{q,s,c}(\betatilt)
    &=
    \betatilt
    \Var_{\nu_{s,c,\betatilt}^{q}}
    \left(
        q(s,Z)
    \right)
    \ge0,
    \\
    \label{eq:ce-derivative}
    \frac{\mathrm d}{\mathrm d\betatilt}
    \mathrm{CE}_{\betatilt}(q;s,c)
    &=
    \frac{
        \mathsf{KL}_{q,s,c}(\betatilt)
    }{
        \betatilt^2
    }
    \ge0.
\end{align}
Moreover, we have the mean score improvement bound:
\begin{equation}
    \label{eq:mean-score-improvement}
    \mathcal{S}_{q,s,c}(\betatilt)
    -
    \mathcal{S}_{q,s,c}(0)
    \ge
    \frac{1}{\betatilt}
    \mathsf{KL}_{q,s,c}(\betatilt)
    \ge0.
\end{equation}
If $q(s,\cdot)$ is not constant
$\nu_{s,c}$-almost everywhere, then all three maps are strictly increasing
on $(0,\infty)$.

\item \textbf{Exponential amplification of high-score weight.}
Let $A\in\mathcal B(\gZ)$ satisfy
$0<\nu_{s,c}(A)<1$. Then
\begin{equation}
    \label{eq:exact-odds-ratio-identity}
    \frac{
        \operatorname{Odds}_{\nu_{s,c,\betatilt}^{q}}(A)
    }{
        \operatorname{Odds}_{\nu_{s,c}}(A)
    }
    =
    \frac{
        \E_{\nu_{s,c}(\cdot\mid A)}
        \left[
            e^{\betatilt q(s,Z)}
        \right]
    }{
        \E_{\nu_{s,c}(\cdot\mid A^c)}
        \left[
            e^{\betatilt q(s,Z)}
        \right]
    },
\end{equation}
and therefore
\begin{equation}
    \label{eq:high-score-odds-amplification}
    \operatorname{Odds}_{\nu_{s,c,\betatilt}^{q}}(A)
    \ge
    e^{
        \betatilt
        \Xi_{\nu_{s,c}}(A;q)
    }
    \operatorname{Odds}_{\nu_{s,c}}(A),
\end{equation}
so that 
\begin{equation}
    \label{eq:high-score-probability-amplification}
    \nu_{s,c,\betatilt}^{q}(A)
    \ge
    \frac{
        \nu_{s,c}(A)
        e^{\betatilt\Xi_{\nu_{s,c}}(A;q)}
    }{
        1-\nu_{s,c}(A)
        +
        \nu_{s,c}(A)
        e^{\betatilt\Xi_{\nu_{s,c}}(A;q)}
    }.
\end{equation}
Hence a positive score margin
$\Xi_{\nu_{s,c}}(A;q)>0$ yields strict exponential amplification of the
weight assigned to $A$ relative to its complement.

\item \textbf{Reliability under learned-score error.}
Let $q_{\mathrm{pop}}$ be a bounded population score and let $\hat q$
be a bounded learned score satisfying
\begin{equation}
    \label{eq:uniform-critic-error}
    \left\|
        \hat q(s,\cdot)
        -
        q_{\mathrm{pop}}(s,\cdot)
    \right\|_{L^\infty(\nu_{s,c})}
    \le
    \eps_{\mathrm{crit}}.
\end{equation}
Then
\begin{equation}
    \label{eq:entropic-value-robustness}
    \left|
        \mathrm{CE}_{\betatilt}(\hat q;s,c)
        -
        \mathrm{CE}_{\betatilt}(q_{\mathrm{pop}};s,c)
    \right|
    \le
    \eps_{\mathrm{crit}}.
\end{equation}
Moreover, the bounds hold:
\begin{equation}
    \label{eq:tilted-density-ratio-robustness}
    e^{-2\betatilt\eps_{\mathrm{crit}}}
    \le
    \frac{
        \mathrm d\nu_{s,c,\betatilt}^{\hat q}
    }{
        \mathrm d\nu_{s,c,\betatilt}^{q_{\mathrm{pop}}}
    }(z)
    \le
    e^{2\betatilt\eps_{\mathrm{crit}}}
\end{equation}
for $\nu_{s,c}$-almost every $z$, and consequently
\begin{align}
    \label{eq:forward-tilted-kl-robustness}
    \KL
    \left(
        \nu_{s,c,\betatilt}^{\hat q}
        \,\middle\|\,
        \nu_{s,c,\betatilt}^{q_{\mathrm{pop}}}
    \right)
    &\le
    2\betatilt\eps_{\mathrm{crit}}
    \quad\text{and}\quad
    \KL
    \left(
        \nu_{s,c,\betatilt}^{q_{\mathrm{pop}}}
        \,\middle\|\,
        \nu_{s,c,\betatilt}^{\hat q}
    \right)
    \le
    2\betatilt\eps_{\mathrm{crit}}.
\end{align}
The learned weighting law is also
$2\eps_{\mathrm{crit}}$-optimal for the population KL-regularized
objective:
\begin{equation}
    \label{eq:approximate-regularized-optimality}
     \ScoreFn_{\betatilt,s,c}
    \left(
        \nu_{s,c,\betatilt}^{q_{\mathrm{pop}}};
        q_{\mathrm{pop}}
    \right)
    -
     \ScoreFn_{\betatilt,s,c}
    \left(
        \nu_{s,c,\betatilt}^{\hat q};
        q_{\mathrm{pop}}
    \right)
    \le
    2\eps_{\mathrm{crit}}.
\end{equation}

Finally, let $A$ satisfy $0<\nu_{s,c}(A)<1$ and define its population-score
separation\footnote{The quantity
$\Xi_{\nu_{s,c}}(A;q_{\mathrm{pop}})$
measures the population-score margin separating $A$ from its complement under
the local proposal law $\nu_{s,c}$: it is the lowest population score attained
on $A$ minus the highest population score attained on $A^c$, ignoring sets of
$\nu_{s,c}$-measure zero.  More precisely, the essential infimum is the largest
value $a$ such that
$q_{\mathrm{pop}}(s,z)\ge a$
for $\nu_{s,c}$-almost every $z\in A$, while the essential supremum is the
smallest value $b$ such that
$q_{\mathrm{pop}}(s,z)\le b$
for $\nu_{s,c}$-almost every $z\in A^c$.
Thus exceptional subsets of $A$ or $A^c$ having zero probability under the
local proposal law do not determine the separation margin.  Consequently,
$\Xi_{\nu_{s,c}}(A;q_{\mathrm{pop}})>0$
means that, $\nu_{s,c}$-almost surely, every candidate in $A$ has larger
population score than every candidate outside $A$.  For a finite particle bank with positive mass on every particle,
the essential infimum over $A$ is the minimum population score among
bank particles lying in $A$, and the essential supremum over $A^c$
is the maximum population score among bank particles lying in $A^c$.
Thus $\Xi_{\nu_{s,c}}(A;q_{\mathrm{pop}})$ compares only the represented
particles; latent points outside the bank do not affect this separation.}

\begin{equation}
    \label{eq:population-high-score-separation}
    \Xi_{\nu_{s,c}}(A;q_{\mathrm{pop}})
    :=
    \operatorname{ess\,inf}_{z\in A}^{\nu_{s,c}}
        q_{\mathrm{pop}}(s,z)
    -
    \operatorname{ess\,sup}_{z\in A^c}^{\nu_{s,c}}
        q_{\mathrm{pop}}(s,z).
\end{equation}
Then
\begin{equation}
    \label{eq:robust-high-score-odds}
    \operatorname{Odds}_{\nu_{s,c,\betatilt}^{\hat q}}(A)
    \ge
    \exp
    \left(
        \betatilt
        \left[
            \Xi_{\nu_{s,c}}(A;q_{\mathrm{pop}})
            -
            2\eps_{\mathrm{crit}}
        \right]
    \right)
    \operatorname{Odds}_{\nu_{s,c}}(A).
\end{equation}
Therefore strict high-value weight amplification is guaranteed whenever
\begin{equation}
    \label{eq:separation-dominates-critic-error}
    \Xi_{\nu_{s,c}}(A;q_{\mathrm{pop}})
    >
    2\eps_{\mathrm{crit}}.
\end{equation}

\end{enumerate}
\end{theorem}

\begin{proof}
We prove each part separately.

\paragraph{Part a: exact KL-regularized optimizer.}
Fix $\betatilt>0$.  Boundedness gives
\[
    q_{\min}
    \le
    \mathrm{CE}_{\betatilt}(q;s,c)
    \le
    q_{\max}.
\]
Hence
\[
    0
    <
    e^{-\betatilt(q_{\max}-q_{\min})}
    \le
    \frac{
        \mathrm d\nu_{s,c,\betatilt}^{q}
    }{
        \mathrm d\nu_{s,c}
    }(z)
    \le
    e^{\betatilt(q_{\max}-q_{\min})}
    <
    \infty
\]
for $\nu_{s,c}$-almost every $z$.
Thus
$\nu_{s,c,\betatilt}^{q}\sim\nu_{s,c}$.

Let $\probmeas$ satisfy
$\KL(\probmeas\|\nu_{s,c})<\infty$.  Then
$\probmeas\ll\nu_{s,c,\betatilt}^{q}$ and
\begin{align*}
    \log
    \frac{
        \mathrm d\probmeas
    }{
        \mathrm d\nu_{s,c,\betatilt}^{q}
    }(z)
    &=
    \log
    \frac{
        \mathrm d\probmeas
    }{
        \mathrm d\nu_{s,c}
    }(z)
    -
    \betatilt q(s,z)
    +
    \betatilt
    \mathrm{CE}_{\betatilt}(q;s,c).
\end{align*}
Integrating with respect to $\probmeas$ gives
\[
    \KL
    \left(
        \probmeas
        \,\middle\|\,
        \nu_{s,c,\betatilt}^{q}
    \right)
    =
    \KL(\probmeas\|\nu_{s,c})
    -
    \betatilt
    \int q(s,z)\,\probmeas(\mathrm dz)
    +
    \betatilt
    \mathrm{CE}_{\betatilt}(q;s,c).
\]
Rearranging proves
\eqref{eq:gibbs-pythagorean-identity}.
Nonnegativity of KL gives the variational principle, and equality occurs
if and only if
$\probmeas=\nu_{s,c,\betatilt}^{q}$.
This proves uniqueness and
\eqref{eq:optimizer-value-identity}.

\paragraph{Part b: support preservation of the weighting law.}
The density bound above is strictly positive and finite
$\nu_{s,c}$-almost everywhere, so
$\nu_{s,c,\betatilt}^{q}
    \sim
    \nu_{s,c}$.
Mutually absolutely continuous Borel measures have the same null open sets,
and therefore the same topological support, proving
\eqref{eq:support-preservation}.

\paragraph{Part c: monotone score and KL path.}
Let
$Y_{s,c}^{q}
    :=
    q(s,Z_{s,c})$, with 
    $Z_{s,c}\sim\nu_{s,c}$,
and define its cumulant-generating function
\begin{equation}
    \label{eq:log-moment-function}
    \ConMom_{Y_{s,c}^{q}}(\betatilt)
    :=
    \log
    \Mom_{Y_{s,c}^{q}}(\betatilt)
    =
    \log
    \int_{\gZ}
        e^{\betatilt q(s,z)}
        \nu_{s,c}(\mathrm dz).
\end{equation}
Boundedness justifies differentiation under the integral.  Hence
\begin{align}
    \frac{\mathrm d}{\mathrm d\betatilt}
    \ConMom_{Y_{s,c}^{q}}(\betatilt)
    &=
    \mathcal{S}_{q,s,c}(\betatilt),
    \label{eq:log-moment-first-derivative}
    \\
    \frac{\mathrm d^2}{\mathrm d\betatilt^2}
    \ConMom_{Y_{s,c}^{q}}(\betatilt)
    &=
    \Var_{\nu_{s,c,\betatilt}^{q}}
    \left(
        q(s,Z)
    \right).
    \label{eq:log-moment-second-derivative}
\end{align}
Differentiating the first identity proves
\eqref{eq:mean-score-derivative}.

Furthermore, basic calculus yields
$\log
    \frac{
        \mathrm d\nu_{s,c,\betatilt}^{q}
    }{
        \mathrm d\nu_{s,c}
    }(z)
    =
    \betatilt q(s,z)
    -
    \ConMom_{Y_{s,c}^{q}}(\betatilt)$,
so that
\begin{equation}
    \label{eq:kl-log-moment-identity}
    \mathsf{KL}_{q,s,c}(\betatilt)
    =
    \betatilt
    \mathcal{S}_{q,s,c}(\betatilt)
    -
    \ConMom_{Y_{s,c}^{q}}(\betatilt).
\end{equation}
Differentiating this expression yields
$\frac{\mathrm d}{\mathrm d\betatilt}
    \mathsf{KL}_{q,s,c}(\betatilt)
    =
    \betatilt
    \Var_{\nu_{s,c,\betatilt}^{q}}
    \left(
        q(s,Z)
    \right)$,
proving \eqref{eq:kl-derivative}.

Since
$\mathrm{CE}_{\betatilt}(q;s,c)
    =
\ConMom_{Y_{s,c}^{q}}(\betatilt)
/ \betatilt$,
the quotient rule together with
\eqref{eq:kl-log-moment-identity} yields
$\frac{\mathrm d}{\mathrm d\betatilt}
    \mathrm{CE}_{\betatilt}(q;s,c)
    =
        \mathsf{KL}_{q,s,c}(\betatilt)/\betatilt^2$.
This proves \eqref{eq:ce-derivative}.

Finally, applying the variational principle with
$\probmeas=\nu_{s,c}$ yields
$\mathrm{CE}_{\betatilt}(q;s,c)
    \geq
    \mathcal{S}_{q,s,c}(0)$.
Combining this with
\eqref{eq:optimizer-value-identity} proves
\eqref{eq:mean-score-improvement}.
Strictness follows whenever $q(s,\cdot)$ is nonconstant
$\nu_{s,c}$-almost everywhere.

\paragraph{Part d: exponential amplification of high-score weight.}
Let $A$ satisfy $0<\nu_{s,c}(A)<1$.  By definition, we immediately have that
\[
    \operatorname{Odds}_{\nu_{s,c,\betatilt}^{q}}(A)
    =
    \frac{
        \int_A
        e^{\betatilt q(s,z)}
        \nu_{s,c}(\mathrm dz)
    }{
        \int_{A^c}
        e^{\betatilt q(s,z)}
        \nu_{s,c}(\mathrm dz)
    }.
\]
Factoring out the base masses yields
\[
    \operatorname{Odds}_{\nu_{s,c,\betatilt}^{q}}(A)
    =
    \operatorname{Odds}_{\nu_{s,c}}(A)
        \mathbb{E}_{\nu_{s,c}(\cdot\mid A)}
        [e^{\betatilt q(s,Z)}]
/\mathbb{E}_{\nu_{s,c}(\cdot\mid A^c)}
        [e^{\betatilt q(s,Z)}],
\]
which proves
\eqref{eq:exact-odds-ratio-identity}.

On $A$, we have the lower bound
$q(s,z)
    \ge
    \operatorname{ess\,inf}_{u\in A}^{\nu_{s,c}}q(s,u)$
almost everywhere, whereas on $A^c$, we have the upper bound
$q(s,z)
    \le
    \operatorname{ess\,sup}_{u\in A^c}^{\nu_{s,c}}q(s,u)$
almost everywhere.  Substitution gives
$\operatorname{Odds}_{\nu_{s,c,\betatilt}^{q}}(A)
    \ge
    e^{\betatilt\Xi_{\nu_{s,c}}(A;q)}
    \operatorname{Odds}_{\nu_{s,c}}(A)$,
proving \eqref{eq:high-score-odds-amplification}.
Solving this odds inequality for
$\nu_{s,c,\betatilt}^{q}(A)$ yields
\eqref{eq:high-score-probability-amplification}.

\paragraph{Part e: reliability under learned-score error.}
Condition~\eqref{eq:uniform-critic-error} gives
\[
    q_{\mathrm{pop}}(s,z)-\eps_{\mathrm{crit}}
    \le
    \hat q(s,z)
    \le
    q_{\mathrm{pop}}(s,z)+\eps_{\mathrm{crit}}
\]
for $\nu_{s,c}$-almost every $z$.  Therefore
\[
\begin{aligned}
    e^{-\betatilt\eps_{\mathrm{crit}}}
    \int e^{\betatilt q_{\mathrm{pop}}(s,z)}
        \nu_{s,c}(\mathrm dz)
    &\leq
    \int e^{\betatilt\hat q(s,z)}
        \nu_{s,c}(\mathrm dz)
\leq
    e^{\betatilt\eps_{\mathrm{crit}}}
    \int e^{\betatilt q_{\mathrm{pop}}(s,z)}
        \nu_{s,c}(\mathrm dz).
\end{aligned}
\]
Taking logarithms and dividing by $\betatilt$ proves
\eqref{eq:entropic-value-robustness}.

Next, basic calculus yields
\begin{align*}
    \log
    \frac{
        \mathrm d\nu_{s,c,\betatilt}^{\hat q}
    }{
        \mathrm d\nu_{s,c,\betatilt}^{q_{\mathrm{pop}}}
    }(z)
    &=
    \betatilt
    \left(
        \hat q(s,z)-q_{\mathrm{pop}}(s,z)
    \right)-
    \betatilt
    \left(
        \mathrm{CE}_{\betatilt}(\hat q;s,c)
        -
        \mathrm{CE}_{\betatilt}(q_{\mathrm{pop}};s,c)
    \right).
\end{align*}
Both differences have absolute value at most
$\eps_{\mathrm{crit}}$, so that
\[
    -2\betatilt\eps_{\mathrm{crit}}
    \le
    \log
    \frac{
        \mathrm d\nu_{s,c,\betatilt}^{\hat q}
    }{
        \mathrm d\nu_{s,c,\betatilt}^{q_{\mathrm{pop}}}
    }(z)
    \le
    2\betatilt\eps_{\mathrm{crit}}.
\]
Exponentiating proves
\eqref{eq:tilted-density-ratio-robustness}.
Integrating the log-density ratio under either measure gives the two KL
bounds.

For any $\probmeas$ of finite KL relative to $\nu_{s,c}$, we have that 
\[
    \left|
        \ScoreFn_{\betatilt,s,c}
        (\probmeas;\hat q)
        -
        \ScoreFn_{\betatilt,s,c}
        (\probmeas;q_{\mathrm{pop}})
    \right|
    \le
    \eps_{\mathrm{crit}}.
\]
Applying this first to the population optimizer and then to the learned
optimizer yields
\[
\begin{aligned}
    &
    \ScoreFn_{\betatilt,s,c}
    \left(
        \nu_{s,c,\betatilt}^{q_{\mathrm{pop}}};
        q_{\mathrm{pop}}
    \right)\leq
    \ScoreFn_{\betatilt,s,c}
    \left(
        \nu_{s,c,\betatilt}^{\hat q};
        q_{\mathrm{pop}}
    \right)
    +
    2\eps_{\mathrm{crit}},
\end{aligned}
\]
which proves
\eqref{eq:approximate-regularized-optimality}.

Finally, we have that
$\operatorname{ess\,inf}_{A}^{\nu_{s,c}}\hat q
    \ge
    \operatorname{ess\,inf}_{A}^{\nu_{s,c}}q_{\mathrm{pop}}
    -
    \eps_{\mathrm{crit}}$,
and
$\operatorname{ess\,sup}_{A^c}^{\nu_{s,c}}\hat q
    \leq
    \operatorname{ess\,sup}_{A^c}^{\nu_{s,c}}q_{\mathrm{pop}}
    +
    \eps_{\mathrm{crit}}$.
Hence we deduce that 
\begin{equation}
    \label{eq:learned-separation-lower-bound}
    \Xi_{\nu_{s,c}}(A;\hat q)
    \ge
    \Xi_{\nu_{s,c}}(A;q_{\mathrm{pop}})
    -
    2\eps_{\mathrm{crit}}.
\end{equation}
Applying part~\textnormal{\it d} to $\hat q$ proves
\eqref{eq:robust-high-score-odds}.  If
$\Xi_{\nu_{s,c}}(A;q_{\mathrm{pop}})
>2\eps_{\mathrm{crit}}$,
the multiplicative factor is strictly larger than one.
\end{proof}

\subsection{Risk-sensitive local weighting and the correct random variable}
\label{subsec:tilt-risk-interpretation}

Theorem~\ref{thm:exponential-tilting} shows that, conditional on an
anchor $c$, exponential weighting is exactly the optimizer associated with an
entropic risk-sensitive objective over the random local-particle score
\[
    Y_{s,c}^q
    :=
    q(s,Z_{s,c}),
    \qquad
    Z_{s,c}\sim\nu_{s,c}.
\]
Under our reward convention, $\betatilt>0$ is upper-tail seeking: increasing
$\betatilt$ places greater emphasis on local particles having large score.
Importantly, this statement concerns the auxiliary weighting law
$\nu_{s,c,\betatilt}^q$.  \QFRS{} does not sample its executed latent from this
law; in the finite-cloud implementation, the corresponding weights are used to
construct the barycentric refinement
\[
    \tilde z_{c,\betatilt}(s)
    =
    \sum_{k=0}^{M}
    w_{c,\betatilt}(k\mid s)\,
    z_{c,k}(s).
\]
The small-$\betatilt$ expansion below makes the upper-tail preference of the
weighting objective explicit.

\begin{corollary}[Mean--variance expansion with an explicit remainder]
    \label{cor:entropic-mean-variance-expansion}
    Under Assumption~\ref{ass:bounded-tilting-score}, let
    $\operatorname{osc}(q)
        :=
        q_{\max}-q_{\min}$.
    Then, for every $\betatilt>0$,
    \begin{equation}
        \label{eq:entropic-expansion-remainder}
        \bigl|
            \mathrm{CE}_{\betatilt}(q;s,c)
            -
            \mathbb{E}_{\nu_{s,c}}[q(s,Z)]
            -
            \frac{\betatilt}{2}
            \Var_{\nu_{s,c}}(q(s,Z))\bigr|
        \leq
        \frac{
            \betatilt^2
            \operatorname{osc}(q)^3
        }{
            6
        }.
    \end{equation}
    Consequently, as $\betatilt\downarrow0$, the expansion holds:
    \begin{equation}
        \label{eq:entropic-small-beta-expansion}
        \mathrm{CE}_{\betatilt}(q;s,c)
        =
        \mathbb{E}_{\nu_{s,c}}[q(s,Z)]
        +
        \frac{\betatilt}{2}
        \Var_{\nu_{s,c}}(q(s,Z))
        +
        \mathcal O(\betatilt^2).
    \end{equation}
\end{corollary}

\begin{proof}
    Retain the cumulant-generating function
    $\ConMom_{Y_{s,c}^q}$ introduced in
    \eqref{eq:log-moment-function}.  Its first two derivatives at zero are
    \[
        \ConMom_{Y_{s,c}^q}'(0)
        =
        \mathbb{E}_{\nu_{s,c}}[q(s,Z)],
        \qquad
        \ConMom_{Y_{s,c}^q}''(0)
        =
        \Var_{\nu_{s,c}}(q(s,Z)).
    \]

    For any $\theta\ge0$, differentiating once more, justified by boundedness,
    gives
    \[
        \ConMom_{Y_{s,c}^q}'''(\theta)
        =
        \mathbb{E}_{\nu_{s,c,\theta}^{q}}
        \left[
            \left(
                q(s,Z)
                -
                \mathcal{S}_{q,s,c}(\theta)
            \right)^3
        \right].
    \]
    Both $q(s,Z)$ and its expectation under every tilted weighting law lie in
    $[q_{\min},q_{\max}]$.  Hence
    $\left|
            q(s,Z)
            -
            \mathcal{S}_{q,s,c}(\theta)
        \right|
        \leq
        \operatorname{osc}(q)$,
    and therefore
    $| \ConMom_{Y_{s,c}^q}'''(\theta)|
        \leq
        \operatorname{osc}(q)^3$.

   Applying Taylor's theorem, for some
    $\theta_{\betatilt}\in(0,\betatilt)$, yields 
    \[
    \begin{aligned}
        \ConMom_{Y_{s,c}^q}(\betatilt)
        &=
        \ConMom_{Y_{s,c}^q}(0)
        +
        \betatilt
        \ConMom_{Y_{s,c}^q}'(0)
        +
        \frac{\betatilt^2}{2}
        \ConMom_{Y_{s,c}^q}''(0)+
        \frac{\betatilt^3}{6}
        \ConMom_{Y_{s,c}^q}'''(\theta_{\betatilt}).
    \end{aligned}
    \]
    Since
    $\ConMom_{Y_{s,c}^q}(0)=0$ and
        $\mathrm{CE}_{\betatilt}(q;s,c)
        =
            \ConMom_{Y_{s,c}^q}(\betatilt)/
            \betatilt$,
    dividing by $\betatilt$ and applying the third-derivative bound proves
    \eqref{eq:entropic-expansion-remainder}.  The asymptotic expansion
    \eqref{eq:entropic-small-beta-expansion} follows immediately.
\end{proof}

\paragraph{Risk preference as a local refinement mechanism.}
Conditional on the selected anchor $c$, the law $\nu_{s,c}$ describes the
local variation supplied by the particle construction.  The random variable
\[
    Y_{s,c}^q=q(s,Z_{s,c}),
    \qquad
    Z_{s,c}\sim\nu_{s,c},
\]
therefore describes variation in score across possible local deviations around
that coarse behavior.  Equation~\eqref{eq:entropic-small-beta-expansion}
shows that positive $\betatilt$ does more than optimize their mean score: to
first order in $\betatilt$, it additionally rewards variance in local scores,
and hence the presence of upper-tail value opportunities.

The resulting exponential change of measure converts an undirected local
particle cloud into a value-directed weighting of that cloud.  In the
implemented method, larger weight does not correspond to larger probability of
executing an individual particle.  Rather, it gives that particle greater
influence on the barycentric refinement
$\tilde z_{c,\betatilt}(s)$.  Thus the three stages of \QFRS{} have distinct
roles:
\begin{itemize}[itemsep=0pt, topsep=0pt]
   \item \emph{FRS anchor generation}: reverse coarse behavior modes to latent anchors. 
    \item \emph{Local particles}: expose local degrees of freedom. 
    \item \emph{Risk-sensitive weighting}: direct the local refinement, and select coarse behavioral mode.
\end{itemize}
As $\betatilt\downarrow0$, the weighting approaches the untilted local law.
As $\betatilt$ increases, the weighting concentrates progressively on the
highest-scoring particles; in a finite cloud with a unique score maximizer, it
converges to unit weight on that particle as
$\betatilt\to\infty$.

\paragraph{Score risk and critic uncertainty are distinct.}
The upper-tail preference across local particle scores should be distinguished
from uncertainty in the value estimate of a fixed particle.  The random
variable in the entropic objective is
$q(s,Z_{s,c})$
with
    $Z_{s,c}\sim\nu_{s,c}$,
so its variability is across local candidate behaviors.  

A common set of heuristics in practice involve leveraging a critic ensemble. 
If the score supplied to the weighting rule is itself constructed
conservatively from a critic ensemble, for example
\[
    \hat q_{\mathrm{pes}}(s,z)
    :=
    \min_{i\in\mathcal I}
    Q_i(s,z),
\]
then the two operations act along different axes:
\begin{itemize}[itemsep=0pt, topsep=4pt]
    \item pessimism across value estimates at fixed $z$: guard against overestimation;
    \item upper-tail weighting across local particles: favor promising refinements. 
\end{itemize}
The reliability condition from
Theorem~\ref{thm:exponential-tilting}---namely,
$\Xi_{\nu_{s,c}}(A;q_{\mathrm{pop}})
    >
    2\eps_{\mathrm{crit}}$---formalizes the corresponding tradeoff: aggressive local weighting is
certifiably aligned with the population score only when the value separation of
the high-score region dominates critic error.

\paragraph{Beyond entropic weighting.}
The entropic functional is one member of a broader family of local
risk-sensitive objectives.  Let
$\mathfrak R_{\nu_{s,c}}(q)$ be a reward-oriented functional of the random
local score $q(s,Z_{s,c})$.  Whenever it admits a variational representation
\begin{equation}
    \label{eq:general-risk-variational-representation}
    \mathfrak R_{\nu_{s,c}}(q)
    =
    \sup_{\probmeas\ll\nu_{s,c}}
    \left\{
        \int_{\gZ}
            q(s,z)\,
            \probmeas(\mathrm dz)
        -
        \mathcal P_{\mathfrak R}
        (\probmeas,\nu_{s,c})
    \right\},
\end{equation}
an optimizer $\probmeas^\star$ defines a corresponding local weighting rule.
The entropic case is recovered by
\[
    \mathcal P_{\mathfrak R}
    (\probmeas,\nu_{s,c})
    =
    \frac{1}{\betatilt}
    \KL(\probmeas\|\nu_{s,c}),
\]
for which
$\probmeas^\star=\nu_{s,c,\betatilt}^{q}$ has the closed-form exponential
density derived in Theorem~\ref{thm:exponential-tilting}.

In a population version of the local-refinement construction, such a weighting
law naturally induces the barycentric refinement
\[
    \tilde z_{\mathfrak R}(s,c)
    :=
    \int_{\gZ}
        z\,
        \probmeas^\star(\mathrm dz),
\]
whenever the integral exists.  Thus alternative risk functionals can change
how the same local particle law is combined without changing how FRS constructs
the coarse anchor or how the local particle cloud is generated.

\paragraph{Return-level risk.}
As noted above, the variance in
\eqref{eq:entropic-small-beta-expansion} is the variance of local candidate
scores under $Z_{s,c}\sim\nu_{s,c}$.  It is not automatically return variance,
and it is not uncertainty across critic heads.  The following proposition
identifies the exact score transformation under which the local entropic
functional becomes an entropic functional of realized return.

\begin{proposition}[
    Exact return-level entropic interpretation under a conditional transform
]
\label{prop:conditional-entropic-return-transform}
Fix a state $s$, a selected anchor $c$, and $\betatilt>0$.
Let $(\Omega,\mathcal F,\mathbb P)$ carry a pair
$(Z_{s,c},G_{s,c})$, where
$Z_{s,c}\sim\nu_{s,c}$
and $G_{s,c}$ is a real-valued discounted return satisfying
$\E
    \left[
        e^{\betatilt G_{s,c}}
    \right]
    <
    \infty$.
Let
$\Mom_{G_{s,c}\mid Z_{s,c}}
    (\betatilt\mid z)
    :=
    \E
    \left[
        e^{\betatilt G_{s,c}}
        \,\middle|\,
        Z_{s,c}=z
    \right]$
denote the conditional moment-generating function, and define
\begin{equation}
    \label{eq:conditional-entropic-return-score}
    q_{\mathrm{ret},\betatilt}^{\,c}(s,z)
    :=
    \frac{1}{\betatilt}
    \log
    \Mom_{G_{s,c}\mid Z_{s,c}}
    (\betatilt\mid z)
\end{equation}
for $\nu_{s,c}$-almost every $z$.
Then
\begin{equation}
    \label{eq:joint-return-entropic-identity}
    \mathrm{CE}_{\betatilt}
    \left(
        q_{\mathrm{ret},\betatilt}^{\,c};
        s,c
    \right)
    =
    \frac{1}{\betatilt}
    \log
    \E
    \left[
        e^{\betatilt G_{s,c}}
    \right].
\end{equation}
Moreover, with  the joint exponentially tilted measure
$\mathbb P_{s,c,\betatilt}^{G}$ defined by
$\mathrm d\mathbb P_{s,c,\betatilt}^{G}/
        \mathrm d\mathbb P
    :=
        e^{\betatilt G_{s,c}}/
        \E[e^{\betatilt G_{s,c}}]$,  the marginal law of $Z_{s,c}$ under this measure is
    $\mathcal L_{\mathbb P_{s,c,\betatilt}^{G}}
    (Z_{s,c})
    =
    \nu_{s,c,\betatilt}^{
        q_{\mathrm{ret},\betatilt}^{\,c}
    }$.
\end{proposition}

\begin{proof}
Since $e^{\betatilt G_{s,c}}>0$ almost surely,
we have that \[
    \Mom_{G_{s,c}\mid Z_{s,c}}
    (\betatilt\mid Z_{s,c})
    =
    \E
    \left[
        e^{\betatilt G_{s,c}}
        \,\middle|\,
        Z_{s,c}
    \right]
\]
is strictly positive almost surely.  Furthermore,
\[
\begin{aligned}
    \E
    \left[
        \Mom_{G_{s,c}\mid Z_{s,c}}
        (\betatilt\mid Z_{s,c})
    \right]
    &=
    \E
    \left[
        e^{\betatilt G_{s,c}}
    \right]
    <
    \infty,
\end{aligned}
\]
so the conditional moment is finite almost surely.

By
\eqref{eq:conditional-entropic-return-score},
we have $
    e^{
        \betatilt
        q_{\mathrm{ret},\betatilt}^{\,c}
        (s,Z_{s,c})
    }
    =
    \E
    \left[
        e^{\betatilt G_{s,c}}
        \,\middle|\,
        Z_{s,c}
    \right]$
almost surely.  Therefore, by the tower property,
\[
\begin{aligned}
    \int_{\gZ}
        e^{
            \betatilt
            q_{\mathrm{ret},\betatilt}^{\,c}(s,z)
        }
        \nu_{s,c}(\mathrm dz)
    &=
    \E
    \left[
        e^{
            \betatilt
            q_{\mathrm{ret},\betatilt}^{\,c}
            (s,Z_{s,c})
        }
    \right]=
    \E
    \left[
        \E
        \left[
            e^{\betatilt G_{s,c}}
            \,\middle|\,
            Z_{s,c}
        \right]
    \right]=
    \E
    \left[
        e^{\betatilt G_{s,c}}
    \right].
\end{aligned}
\]
Taking logarithms and dividing by $\betatilt$ proves
\eqref{eq:joint-return-entropic-identity}.

For any Borel set $A\subseteq\gZ$, elementry calculations yield
\begin{align*}
    \mathbb P_{s,c,\betatilt}^{G}
    (Z_{s,c}\in A)
    &=
    \frac{
        \E
        \left[
            \1\{Z_{s,c}\in A\}
            e^{\betatilt G_{s,c}}
        \right]
    }{
        \E[e^{\betatilt G_{s,c}}]
    }
    \\
    &=
    \frac{
        \E
        \left[
            \1\{Z_{s,c}\in A\}
            \E[
                e^{\betatilt G_{s,c}}
                \mid
                Z_{s,c}
            ]
        \right]
    }{
        \E[e^{\betatilt G_{s,c}}]
    }
    \\
    &=
    \frac{
        \displaystyle
        \int_A
        e^{
            \betatilt
            q_{\mathrm{ret},\betatilt}^{\,c}(s,z)
        }
        \nu_{s,c}(\mathrm dz)
    }{
        \displaystyle
        \int_{\gZ}
        e^{
            \betatilt
            q_{\mathrm{ret},\betatilt}^{\,c}(s,z)
        }
        \nu_{s,c}(\mathrm dz)
    }
    \\
    &=
    \nu_{s,c,\betatilt}^{
        q_{\mathrm{ret},\betatilt}^{\,c}
    }(A).
\end{align*}
Since this holds for every Borel set $A$, the marginal laws agree.
\end{proof}

Proposition~\ref{prop:conditional-entropic-return-transform} identifies the
precise distinction between local candidate-value risk and realized-return
risk.  If the score is the ordinary conditional mean return, the outer entropic
functional is applied to variation in conditional mean values across local
particles.  If the score is
$q_{\mathrm{ret},\betatilt}^{\,c}$, the same outer transform is exactly the
entropic certainty equivalent of the realized return.

This identity concerns the risk functional and its associated exponentially tilted 
weighting law. In the bounded-return setting, the transformed score is bounded, and Theorem~\ref{thm:exponential-tilting} also establishes KL-regularized optimality of this law. It does \emph{not} imply that the actual barycentric latent
executed by \QFRS{} has the return distribution
$\mathbb P_{s,c,\betatilt}^{G}$, because \QFRS{} executes
$\tilde z_{c,\betatilt}$ rather than sampling
$Z_{s,c}$ from the tilted law.  Relating the value of the executed barycenter
to the weighted candidate values requires an additional local regularity
condition; this connection is developed in
Section~\ref{subsec:finite-cloud-tilting}.

\begin{corollary}[Mean-return candidate risk is generally not return risk]
    \label{cor:mean-return-versus-return-entropic}
    In the setting of
    Proposition~\ref{prop:conditional-entropic-return-transform}, assume
    additionally that
    $\E[|G_{s,c}|]<\infty$,
    and define
$q_{\mathrm{mean}}^{\,c}(s,z)
        :=
        \E
        \left[
            G_{s,c}
            \,\middle|\,
            Z_{s,c}=z
        \right]$.
    Then
    \begin{equation}
        \label{eq:mean-return-entropic-jensen}
        \mathrm{CE}_{\betatilt}
        \left(
            q_{\mathrm{mean}}^{\,c};
            s,c
        \right)
        \le
        \frac1{\betatilt}
        \log
        \E
        \left[
            e^{\betatilt G_{s,c}}
        \right].
    \end{equation}
    Equality holds if and only if, for
    $\nu_{s,c}$-almost every $z$, the conditional return
    $G_{s,c}\mid Z_{s,c}=z$ is almost surely constant.
\end{corollary}

\begin{proof}
Applying Jensen's inequality yields
\[
    e^{
        \betatilt
        q_{\mathrm{mean}}^{\,c}
        (s,Z_{s,c})
    }
    =
    e^{
        \betatilt
        \E[G_{s,c}\mid Z_{s,c}]
    }
    \le
    \E
    \left[
        e^{\betatilt G_{s,c}}
        \,\middle|\,
        Z_{s,c}
    \right]
\]
almost surely.  Taking expectations,
\[
    \int_{\gZ}
        e^{
            \betatilt
            q_{\mathrm{mean}}^{\,c}(s,z)
        }
        \nu_{s,c}(\mathrm dz)
    \le
    \E
    \left[
        e^{\betatilt G_{s,c}}
    \right].
\]
Taking logarithms and dividing by $\betatilt>0$ proves
\eqref{eq:mean-return-entropic-jensen}.

Because $x\mapsto e^{\betatilt x}$ is strictly convex, equality in conditional
Jensen holds almost surely if and only if $G_{s,c}$ is conditionally almost
surely constant given $Z_{s,c}$, which gives the stated condition.
\end{proof}

\subsection{Sharpness of the critic-separation condition}
\label{subsec:separation-condition-sharpness}

The factor $2\eps_{\mathrm{crit}}$ in
\eqref{eq:separation-dominates-critic-error} is not an artifact of the proof.
A uniform critic error of size $\eps_{\mathrm{crit}}$ can underestimate the
score of a truly high-value local particle by $\eps_{\mathrm{crit}}$ while
simultaneously overestimating a lower-value particle by
$\eps_{\mathrm{crit}}$, reducing their apparent separation by
$2\eps_{\mathrm{crit}}$.  The following two-point construction shows that this
loss is worst-case sharp.

\begin{proposition}[Worst-case sharpness of the local value-separation threshold]
    \label{prop:critic-separation-threshold-sharpness}
    Fix a state $s$, a selected anchor $c$, a tilt intensity
    $\betatilt>0$, a critic-error level
    $\eps_{\mathrm{crit}}>0$, and a population-score separation
    $\vartheta_{\mathrm{score}}\ge0$.
    There exists a two-point local particle space
    $\gZ_{\mathrm{loc}}
        =
        \{z_{\mathrm{hi}},z_{\mathrm{lo}}\}$,
    a local reference law $\nu_{s,c}$ assigning strictly positive mass to both
    points, a population score $q_{\mathrm{pop}}$, and a learned score
    $\hat q$ such that
    \begin{equation}
        \label{eq:sharp-uniform-score-error}
        \left\|
            \hat q(s,\cdot)
            -
            q_{\mathrm{pop}}(s,\cdot)
        \right\|_{L^\infty(\nu_{s,c})}
        =
        \eps_{\mathrm{crit}},
    \end{equation}
    the point $z_{\mathrm{hi}}$ has population-score separation exactly
    $\vartheta_{\mathrm{score}}$ from $z_{\mathrm{lo}}$, and, for
    $A:=\{z_{\mathrm{hi}}\}$,
    \begin{equation}
        \label{eq:sharp-two-point-odds-factor}
        \frac{
            \operatorname{Odds}_{
                \nu_{s,c,\betatilt}^{\hat q}
            }(A)
        }{
            \operatorname{Odds}_{\nu_{s,c}}(A)
        }
        =
        \exp
        \left(
            \betatilt
            \left[
                \vartheta_{\mathrm{score}}
                -
                2\eps_{\mathrm{crit}}
            \right]
        \right).
    \end{equation}
    Consequently, if
    $\vartheta_{\mathrm{score}}<2\eps_{\mathrm{crit}}$, an admissible learned
    critic can reverse the ordering of the two local particles and cause
    exponential weighting to \emph{decrease} the relative weight assigned to
    the truly high-value particle.  At
    $\vartheta_{\mathrm{score}}=2\eps_{\mathrm{crit}}$, an admissible critic
    can eliminate any strict preference between them.
\end{proposition}

\begin{proof}
    Fix the two-point local particle space
    $\gZ_{\mathrm{loc}}
        =
        \{z_{\mathrm{hi}},z_{\mathrm{lo}}\}$,
    and let $\nu_{s,c}$ assign strictly positive probability to both points.
    Set
    $A
        :=
        \{z_{\mathrm{hi}}\}$.
    Define the population score by
    \[
        q_{\mathrm{pop}}(s,z_{\mathrm{hi}})
        :=
        \vartheta_{\mathrm{score}},
        \qquad
        q_{\mathrm{pop}}(s,z_{\mathrm{lo}})
        :=
        0.
    \]
    Hence the population-score separation between the designated high-value
    particle and its complement is exactly
    $\Xi_{\nu_{s,c}}
        (A;q_{\mathrm{pop}})
        =
        \vartheta_{\mathrm{score}}$.

    Now define the learned score by
    \[
        \hat q(s,z_{\mathrm{hi}})
        :=
        \vartheta_{\mathrm{score}}
        -
        \eps_{\mathrm{crit}}
        \quad\text{and}\quad
        \hat q(s,z_{\mathrm{lo}})
        :=
        \eps_{\mathrm{crit}}.
    \]
    The absolute score error at each point is exactly
    $\eps_{\mathrm{crit}}$, so
    \eqref{eq:sharp-uniform-score-error} holds.
    The learned score difference is
  $
        \hat q(s,z_{\mathrm{hi}})
        -
        \hat q(s,z_{\mathrm{lo}})
        =
        \vartheta_{\mathrm{score}}
        -
        2\eps_{\mathrm{crit}}$.
    On a two-point space, the exact odds identity
    \eqref{eq:exact-odds-ratio-identity} gives
    \[
    \begin{aligned}
        \frac{
            \operatorname{Odds}_{
                \nu_{s,c,\betatilt}^{\hat q}
            }(A)
        }{
            \operatorname{Odds}_{\nu_{s,c}}(A)
        }
        &=
        \frac{
            e^{\betatilt
                \hat q(s,z_{\mathrm{hi}})}
        }{
            e^{\betatilt
                \hat q(s,z_{\mathrm{lo}})}
        }=
        \exp
        \left(
            \betatilt
            \left[
                \hat q(s,z_{\mathrm{hi}})
                -
                \hat q(s,z_{\mathrm{lo}})
            \right]
        \right)=
        \exp
        \left(
            \betatilt
            \left[
                \vartheta_{\mathrm{score}}
                -
                2\eps_{\mathrm{crit}}
            \right]
        \right),
    \end{aligned}
    \]
    proving \eqref{eq:sharp-two-point-odds-factor}.

    If
    $\vartheta_{\mathrm{score}}<2\eps_{\mathrm{crit}}$, the exponent is
    negative, so the tilted weighting law assigns smaller odds to the truly
    high-value point than the untilted local law.  If
    $\vartheta_{\mathrm{score}}=2\eps_{\mathrm{crit}}$, the exponent is zero,
    so the odds are unchanged.  This proves the final claims.
\end{proof}

Thus the threshold
$\Xi_{\nu_{s,c}}(A;q_{\mathrm{pop}})
    >
    2\eps_{\mathrm{crit}}$
is necessary for a uniform worst-case guarantee that exponential weighting
strictly amplifies the relative weight of a truly high-value local region under
only a supremum-norm critic-error assumption.  The result concerns the
\emph{weighting law}; translating this weighting guarantee into a guarantee for
the value of the barycentric latent executed by \QFRS{} additionally requires
the local regularity conditions developed in
Section~\ref{subsec:finite-cloud-tilting}.  The threshold is not claimed to be
necessary on every particular instance.

\subsection{True-score improvement under a learned local weighting}
\label{subsec:true-score-improvement-learned-tilt}

The robust amplification result in
Theorem~\ref{thm:exponential-tilting} is local to a separated high-value
subset of the particle law.  The next result gives a complementary global
statement: despite critic error, the learned exponential weighting can improve
the population score averaged over the entire local particle distribution.

\begin{corollary}[Population-score lower bound for learned local weighting]
    \label{cor:learned-tilt-true-score-improvement}
    Under the assumptions of part~\textnormal{(e)} of
    Theorem~\ref{thm:exponential-tilting}, fix a state $s\in \mathcal{S}$ and an anchor
    $c\in \mathcal{C}$.  Then
    \begin{equation}
        \label{eq:learned-tilt-true-score-lower}
        \int_{\gZ}
            q_{\mathrm{pop}}(s,z)
            \,\nu_{s,c,\betatilt}^{\hat q}(\mathrm dz)
        \ge
        \int_{\gZ}
            q_{\mathrm{pop}}(s,z)
            \,\nu_{s,c}(\mathrm dz)
        +
        \frac{1}{\betatilt}
        \KL
        \left(
            \nu_{s,c,\betatilt}^{\hat q}
            \,\middle\|\,
            \nu_{s,c}
        \right)
        -
        2\eps_{\mathrm{crit}}.
    \end{equation}
    Consequently, the learned weighting has strictly larger average population
    score than the untilted local law whenever
    \begin{equation}
        \label{eq:kl-gain-dominates-score-error}
        \KL
        \left(
            \nu_{s,c,\betatilt}^{\hat q}
            \,\middle\|\,
            \nu_{s,c}
        \right)
        >
        2\betatilt\eps_{\mathrm{crit}}.
    \end{equation}
\end{corollary}

\begin{proof}
    By Theorem~\ref{thm:exponential-tilting},
    $\nu_{s,c,\betatilt}^{\hat q}$ is the unique optimizer of the
    KL-regularized objective associated with the learned score $\hat q$.
    Since the untilted local law $\nu_{s,c}$ is feasible,
    \begin{equation}
        \label{eq:learned-weight-optimality-base-comparison}
        \int_{\gZ}
            \hat q(s,z)
            \,\nu_{s,c,\betatilt}^{\hat q}(\mathrm dz)
        -
        \frac{1}{\betatilt}
        \KL
        \left(
            \nu_{s,c,\betatilt}^{\hat q}
            \,\middle\|\,
            \nu_{s,c}
        \right)
        \ge
        \int_{\gZ}
            \hat q(s,z)
            \,\nu_{s,c}(\mathrm dz).
    \end{equation}

    The uniform critic-error assumption $\left\|
            \hat q(s,\cdot)
            -
            q_{\mathrm{pop}}(s,\cdot)
        \right\|_{L^\infty(\nu_{s,c})}
        \le
        \eps_{\mathrm{crit}}$
    implies
    \begin{equation}
        \label{eq:population-score-lower-under-learned-weight}
        \int_{\gZ}
            q_{\mathrm{pop}}(s,z)
            \,\nu_{s,c,\betatilt}^{\hat q}(\mathrm dz)
        \ge
        \int_{\gZ}
            \hat q(s,z)
            \,\nu_{s,c,\betatilt}^{\hat q}(\mathrm dz)
        -
        \eps_{\mathrm{crit}}.
    \end{equation}
    Indeed, the integrand satisfies
    $q_{\mathrm{pop}}(s,z)
        \ge
        \hat q(s,z)-\eps_{\mathrm{crit}}$
    $\nu_{s,c}$-almost everywhere, and
    $\nu_{s,c,\betatilt}^{\hat q}\ll\nu_{s,c}$ by
    Theorem~\ref{thm:exponential-tilting}.

    Analogously, we have that 
    \begin{equation}
        \label{eq:learned-score-lower-under-base-law}
        \int_{\gZ}
            \hat q(s,z)
            \,\nu_{s,c}(\mathrm dz)
        \ge
        \int_{\gZ}
            q_{\mathrm{pop}}(s,z)
            \,\nu_{s,c}(\mathrm dz)
        -
        \eps_{\mathrm{crit}}.
    \end{equation}

    Rearranging
    \eqref{eq:learned-weight-optimality-base-comparison} yields
    \[
        \int_{\gZ}
            \hat q(s,z)
            \,\nu_{s,c,\betatilt}^{\hat q}(\mathrm dz)
        \ge\;
        \int_{\gZ}
            \hat q(s,z)
            \,\nu_{s,c}(\mathrm dz)+
        \frac{1}{\betatilt}
        \KL
        \left(
            \nu_{s,c,\betatilt}^{\hat q}
            \,\middle\|\,
            \nu_{s,c}
        \right).
    \]
    Combining this inequality with
    \eqref{eq:population-score-lower-under-learned-weight}
    and
    \eqref{eq:learned-score-lower-under-base-law}
    yields
    \[
    \begin{aligned}
        \int_{\gZ}
            q_{\mathrm{pop}}(s,z)
            \,\nu_{s,c,\betatilt}^{\hat q}(\mathrm dz)
        \ge\;&
        \int_{\gZ}
            q_{\mathrm{pop}}(s,z)
            \,\nu_{s,c}(\mathrm dz)
+
        \frac{1}{\betatilt}
        \KL
        \left(
            \nu_{s,c,\betatilt}^{\hat q}
            \,\middle\|\,
            \nu_{s,c}
        \right)
        -
        2\eps_{\mathrm{crit}},
    \end{aligned}
    \]
    proving \eqref{eq:learned-tilt-true-score-lower}.

    Under the condition of 
    \eqref{eq:kl-gain-dominates-score-error},
   we have that  $
        \frac{1}{\betatilt}
        \KL(
            \nu_{s,c,\betatilt}^{\hat q}
           \|
            \nu_{s,c})
        -
        2\eps_{\mathrm{crit}}
        >
        0$,
    so the weighted population score is strictly larger than its value under
    the untilted local law.
\end{proof}

Corollary~\ref{cor:learned-tilt-true-score-improvement} is a statement about
the \emph{weighted average} of the population scores of local particles.  In
the implemented \QFRS{} policy, however, the tilted law is used to construct
the barycentric latent
\[
    \tilde z_{c,\betatilt}(s)
    =
    \int_{\gZ}
        z\,
        \nu_{s,c,\betatilt}^{\hat q}(\mathrm dz)
\]
in the population idealization, provided that
$\int_{\gZ}\|z\|_2\,\nu_{s,c}(\mathrm dz)<\infty$.
Under the bounded-score assumptions, the tilted density relative to
$\nu_{s,c}$ is bounded, so this condition also ensures that the tilted
barycenter exists. It holds for the Gaussian particle law. For a realized
finite cloud, the corresponding weighted sum is always well-defined.  Therefore
\eqref{eq:learned-tilt-true-score-lower} does not by itself imply an
improvement in
$q_{\mathrm{pop}}(s,\tilde z_{c,\betatilt}(s))$.
The finite-cloud results below show how local smoothness of the population
action value converts the weighted-score guarantee into a guarantee for the
executed barycentric refinement.

\subsection{Finite realized local particle clouds}
\label{subsec:finite-cloud-tilting}

The implemented local-refinement rule operates on a finite particle cloud
generated around a selected FRS anchor. Since the Gaussian perturbations are
resampled at each refinement step, this cloud is random.  In this subsection we
condition on a realized cloud and analyze the weighting and barycentric
refinement induced by that realization.  Conditional on the cloud, both the
exponential weights and the executed barycentric latent are deterministic.
High-probability guarantees over the random Gaussian construction itself are
developed separately in
Section~\ref{subsec:gaussian-particle-cloud-connection}.

Fix a state $s\in \mathcal{S}$ and an anchor $c\in \mathcal{C}$, and condition on a realized local
cloud
\[
    \mathcal B_c(s)
    :=
    (z_{c,0}(s), z_{c,1}(s), \ldots, z_{c,M}(s)).
\]
Different particle indices are treated as distinct even if two realized latent
vectors coincide.

Set
$q_{c,k}
    :=
    q\bigl(s,z_{c,k}(s)\bigr)$ for each $k\in \{0,\ldots, M\}$.
and let the vector $u$ defined such that 
$u(k):=1/(M+1)$ for $k\in\{0,\ldots, M\}$
denote uniform weighting of the indexed cloud.  Let
\[
    \Delta_{M+1}
    :=
    \left\{
        w\in[0,1]^{M+1}:
        \sum_{k=0}^{M}w_k=1
    \right\}.
\]

\begin{proposition}[Exact entropic weights on a realized local cloud]
    \label{prop:finite-cloud-softmax-duality}
    For every $\betatilt>0$, define
    \begin{equation}
        \label{eq:finite-cloud-softmax-weights}
        w_{c,k}^{q,\betatilt}(s)
        :=\textstyle
            \exp\left(
                \betatilt q_{c,k}
            \right)/
            \sum_{\ell=0}^{M}
            \exp\left(
                \betatilt q_{c,\ell}
            \right),
        \quad\forall\; k\in\{0,1,\ldots, M\}.
    \end{equation}
    Then the finite-cloud entropic evaluation
    \begin{equation}
        \label{eq:finite-cloud-entropic-value}
        \mathrm{CE}_{\betatilt}^{\mathrm{cloud}}(q;s,c)
        :=
        \frac{1}{\betatilt}
        \log
        \left[
            \frac{1}{M+1}
            \sum_{k=0}^{M}
            \exp\!\left(
                \betatilt q_{c,k}
            \right)
        \right]
    \end{equation}
    satisfies
    \begin{equation}
        \label{eq:finite-cloud-simplex-dual}
        \mathrm{CE}_{\betatilt}^{\mathrm{cloud}}(q;s,c)
        =
        \max_{w\in\Delta_{M+1}}
        \left\{
            \sum_{k=0}^{M}w_kq_{c,k}
            -
            \frac{1}{\betatilt}
            \KL(w\|u)
        \right\},
    \end{equation}
    where
    $\KL(w\|u)
        :=
        \sum_{k=0}^{M}
        w_k
        \log
        \frac{w_k}{u(k)}$.
    The unique optimizer is
    $w_c^{q,\betatilt}(s)
        :=
        \bigl(
            w_{c,k}^{q,\betatilt}(s)
        \bigr)_{k=0}^{M}$.
    In particular, the weights are non-negative:
    \begin{equation}
        \label{eq:finite-bank-positive-index-support}
        w_{c,k}^{q,\betatilt}(s)>0
        \qquad
        \text{for every }k\in \{0,1,\ldots, M\}.
    \end{equation}
    Moreover, the bounds hold:
    \begin{equation}
        \label{eq:finite-cloud-max-sandwich}
        \max_{0\le k\le M}q_{c,k}
        -
        \frac{\log(M+1)}{\betatilt}
        \le
        \mathrm{CE}_{\betatilt}^{\mathrm{cloud}}(q;s,c)
        \le
        \max_{0\le k\le M}q_{c,k},
    \end{equation}
    and
    \begin{equation}
        \label{eq:finite-cloud-expected-score-gap}
        0
        \le
        \max_{0\le k\le M}q_{c,k}
        -
        \sum_{k=0}^{M}
        w_{c,k}^{q,\betatilt}(s)q_{c,k}
        \le
        \frac{\log(M+1)}{\betatilt}.
    \end{equation}
\end{proposition}

\begin{proof}
Apply Theorem~\ref{thm:exponential-tilting} to the finite index set $[M]$, equipped with the uniform probability mass function $u(k)=1/(M+1)$, and score $k\mapsto q(s,z_{
c,k}(s))$.
    Its exponentially tilted law is
    \[
        w_{c,k}^{q,\betatilt}(s)
        =
            e^{\betatilt q_{c,k}}/\textstyle
            \sum_{\ell=0}^{M}
            e^{\betatilt q_{c,\ell}},
    \]
    proving \eqref{eq:finite-cloud-softmax-weights}.  The variational identity
    \eqref{eq:finite-cloud-simplex-dual} is the finite-space specialization of
    \eqref{eq:gibbs-variational-principle}, and uniqueness follows from
    Theorem~\ref{thm:exponential-tilting}.  Positivity follows immediately
    from the exponential form.

    Let
    $\bar{q}_{\max}
        :=
        \max_{0\leq k\leq M}q_{c,k}$.
    Since every summand in
    \eqref{eq:finite-cloud-entropic-value} is at most
    $e^{\betatilt \bar{q}_{\max}}$,
    \[
        \frac{1}{M+1}
        \sum_{k=0}^{M}
        e^{\betatilt q_{c,k}}
        \le
        e^{\betatilt \bar{q}_{\max}},
    \]
    which gives the upper bound in
    \eqref{eq:finite-cloud-max-sandwich}.  Since at least one summand equals
    $e^{\betatilt \bar{q}_{\max}}$,
    \[
        \frac{1}{M+1}
        \sum_{k=0}^{M}
        e^{\betatilt q_{c,k}}
        \geq
        \frac{1}{M+1}
        e^{\betatilt \bar{q}_{\max}},
    \]
    giving the lower bound.

    Finally, the optimizer identity gives
    \[
    \begin{aligned}
        \mathrm{CE}_{\betatilt}^{\mathrm{cloud}}(q;s,c)
        &=
        \sum_{k=0}^{M}
        w_{c,k}^{q,\betatilt}(s)q_{c,k}
        -
        \frac{1}{\betatilt}
        \KL
        \left(
            w_c^{q,\betatilt}(s)
            \,\middle\|\,
            u
        \right)\leq
        \sum_{k=0}^{M}
        w_{c,k}^{q,\betatilt}(s)q_{c,k}.
    \end{aligned}
    \]
    Combining this with the lower bound in
    \eqref{eq:finite-cloud-max-sandwich} yields
    \[
        \sum_{k=0}^{M}
        w_{c,k}^{q,\betatilt}(s)q_{c,k}
        \geq
        \bar{q}_{\max}
        -
        \frac{\log(M+1)}{\betatilt}.
    \]
    The reverse inequality follows because a convex combination cannot exceed
    the maximum component, proving
    \eqref{eq:finite-cloud-expected-score-gap}.
\end{proof}

The log-mean-exp in
\eqref{eq:finite-cloud-entropic-value} is the optimal value of the
\emph{auxiliary entropic weighting problem}.  It is not, in general, the
value of the barycentric latent executed by \QFRS{}, nor is it used as the
Bellman target in the implemented method.  The next results connect the
entropic weighting guarantees to the actual local refinement.

\begin{corollary}[Population-score error of learned local weights]
    \label{cor:finite-cloud-selection-error}
    Let
    $q_{c,k}^{\mathrm{pop}}
        :=
        q_{\mathrm{pop}}
        \bigl(s,z_{c,k}(s)\bigr)$
    be population scores and let
    $\hat q_{c,k}
        :=
        \hat q
        \bigl(s,z_{c,k}(s)\bigr)$
    satisfy
    \begin{equation}
        \label{eq:finite-bank-uniform-score-error}
        \max_{0\le k\le M}|
            \hat q_{c,k}
            -
            q_{c,k}^{\mathrm{pop}}|
        \leq
        \eps_{\mathrm{crit}}.
    \end{equation}
    Let
    $\widehat w_{c,k}^{\betatilt}(s)
        :=
            e^{\betatilt\hat q_{c,k}}
/\textstyle\sum_{\ell=0}^{M}
            e^{\betatilt\hat q_{c,\ell}}$.
    Then
    \begin{equation}
        \label{eq:finite-cloud-true-selection-error}
        \max_{0\le k\le M}
            q_{c,k}^{\mathrm{pop}}
        -
        \sum_{k=0}^{M}
            \widehat w_{c,k}^{\betatilt}(s)
            q_{c,k}^{\mathrm{pop}}
        \le
        \frac{\log(M+1)}{\betatilt}
        +
        2\eps_{\mathrm{crit}}.
    \end{equation}
\end{corollary}

\begin{proof}
    The uniform score-error bound gives
    \begin{equation}
        \label{eq:finite-bank-max-score-transfer}
        \max_k q_{c,k}^{\mathrm{pop}}
        \le
        \max_k \hat q_{c,k}
        +
        \eps_{\mathrm{crit}}.
    \end{equation}
    Proposition~\ref{prop:finite-cloud-softmax-duality}, applied to the learned
    scores, gives
    \begin{equation}
        \label{eq:finite-bank-learned-soft-gap}
        \max_k\hat q_{c,k}
        -
        \sum_{k=0}^{M}
        \widehat w_{c,k}^{\betatilt}(s)
        \hat q_{c,k}
        \le
        \frac{\log(M+1)}{\betatilt}.
    \end{equation}
    Finally, we have that 
    \begin{equation}
        \label{eq:finite-bank-average-score-transfer}
        \sum_{k=0}^{M}
        \widehat w_{c,k}^{\betatilt}(s)
        \hat q_{c,k}
        \le
        \sum_{k=0}^{M}
        \widehat w_{c,k}^{\betatilt}(s)
        q_{c,k}^{\mathrm{pop}}
        +
        \eps_{\mathrm{crit}},
    \end{equation}
    because the weights are nonnegative and sum to one.  Combining
    \eqref{eq:finite-bank-max-score-transfer}--
    \eqref{eq:finite-bank-average-score-transfer}
    proves \eqref{eq:finite-cloud-true-selection-error}.
\end{proof}

\paragraph{From weighted particle values to the executed barycenter.}
The preceding corollary controls the population score averaged under the
learned local weights.  \QFRS{}, however, executes the weighted latent
\begin{equation}
    \label{eq:finite-cloud-executed-barycenter}
    \tilde z_{c,\betatilt}(s)
    :=
    \sum_{k=0}^{M}
    \widehat w_{c,k}^{\betatilt}(s)
    z_{c,k}(s).
\end{equation}
The next proposition gives the regularity condition required to transfer the
weighted-score guarantee to this executed latent.

    Define the weighted particle variance
    \begin{equation}
        \label{eq:weighted-particle-variance}
        \mathcal V_{c,\betatilt}(s)
        :=
        \sum_{k=0}^{M}
        \widehat w_{c,k}^{\betatilt}(s)
        \left\|
            z_{c,k}(s)
            -
            \tilde z_{c,\betatilt}(s)
        \right\|_2^2.
    \end{equation}
\begin{proposition}[Value of the executed barycentric refinement]
    \label{prop:barycentric-refinement-value}
    Suppose $q_{\mathrm{pop}}(s,\cdot)$ is differentiable on a neighborhood of
    $\operatorname{conv}
        \mathcal B_c(s)$
    and has $L_{\mathrm{sm}}$-Lipschitz gradient there.  Consequently, for every
    $z,\tilde z\in\operatorname{conv}\mathcal B_c(s)$, the bound holds:
    \begin{equation}
        \label{eq:local-score-smoothness}
        q_{\mathrm{pop}}(s,\tilde z)
        \leq
        q_{\mathrm{pop}}(s,z)
        +
        \nabla_z q_{\mathrm{pop}}(s,z)^\top
        (\tilde z-z)
        +
        \frac{L_{\mathrm{sm}}}{2}
        \|\tilde z-z\|_2^2.
    \end{equation}
    Then
    \begin{equation}
        \label{eq:barycenter-weighted-value-lower}
        q_{\mathrm{pop}}
        \left(
            s,
            \tilde z_{c,\betatilt}(s)
        \right)
        \ge
        \sum_{k=0}^{M}
        \widehat w_{c,k}^{\betatilt}(s)
        q_{\mathrm{pop}}
        \left(
            s,
            z_{c,k}(s)
        \right)
        -
        \frac{L_{\mathrm{sm}}}{2}
        \mathcal V_{c,\betatilt}(s).
    \end{equation}
    Moreover, under the uniform critic-error condition
    \eqref{eq:finite-bank-uniform-score-error},
    \begin{equation}
        \label{eq:barycenter-best-particle-gap}
        \begin{aligned}
        &
        \max_{0\le k\le M}
        q_{\mathrm{pop}}
        \left(
            s,z_{c,k}(s)
        \right)
        -
        q_{\mathrm{pop}}
        \left(
            s,\tilde z_{c,\betatilt}(s)
        \right)\leq
        \frac{\log(M+1)}{\betatilt}
        +
        2\eps_{\mathrm{crit}}
        +
        \frac{L_{\mathrm{sm}}}{2}
        \mathcal V_{c,\betatilt}(s).
        \end{aligned}
    \end{equation}
\end{proposition}

\begin{proof}
    For each $k$, apply
    \eqref{eq:local-score-smoothness} with
    $z
        =
        \tilde z_{c,\betatilt}(s)$
        and
        $\tilde z
        =
        z_{c,k}(s)$.
    This yields
    \[
    \begin{aligned}
        q_{\mathrm{pop}}
        \left(
            s,z_{c,k}(s)
        \right)
        \leq&
        q_{\mathrm{pop}}
        \left(
            s,\tilde z_{c,\betatilt}(s)
        \right)+
        \nabla_zq_{\mathrm{pop}}
        \left(
            s,\tilde z_{c,\betatilt}(s)
        \right)^\top
        \left(
            z_{c,k}(s)
            -
            \tilde z_{c,\betatilt}(s)
        \right)\\ &\quad +
        \frac{L_{\mathrm{sm}}}{2}
        \left\|
            z_{c,k}(s)
            -
            \tilde z_{c,\betatilt}(s)
        \right\|_2^2 .
    \end{aligned}
    \]
    Multiply by
    $\widehat w_{c,k}^{\betatilt}(s)$ and sum over
    $k\in\{0,\ldots,M\}$ so that, by the definition of the barycenter, we deduce that
    \[
        \sum_{k=0}^{M}
        \widehat w_{c,k}^{\betatilt}(s)
        \left(
            z_{c,k}(s)
            -
            \tilde z_{c,\betatilt}(s)
        \right)
        =
        0.
    \]
    Hence the first-order term vanishes, yielding
    \[
        \sum_{k=0}^{M}
        \widehat w_{c,k}^{\betatilt}(s)
        q_{\mathrm{pop}}
        \left(
            s,z_{c,k}(s)
        \right)
        \le
        q_{\mathrm{pop}}
        \left(
            s,\tilde z_{c,\betatilt}(s)
        \right)
        +
        \frac{L_{\mathrm{sm}}}{2}
        \mathcal V_{c,\betatilt}(s).
    \]
    Rearranging proves
    \eqref{eq:barycenter-weighted-value-lower}.

    Corollary~\ref{cor:finite-cloud-selection-error} gives
    \[
        \max_k
        q_{\mathrm{pop}}
        \left(
            s,z_{c,k}(s)
        \right)
        \le
        \sum_{k=0}^{M}
        \widehat w_{c,k}^{\betatilt}(s)
        q_{\mathrm{pop}}
        \left(
            s,z_{c,k}(s)
        \right)
        +
        \frac{\log(M+1)}{\betatilt}
        +
        2\eps_{\mathrm{crit}}.
    \]
    Substituting
    \eqref{eq:barycenter-weighted-value-lower}
    proves \eqref{eq:barycenter-best-particle-gap}.
\end{proof}

\paragraph{Connection to the ideal reach-$\rho$ class.}
The preceding proposition compares the executed barycentric refinement with the
best particle in the realized local cloud.  The particle theory of
Section~\ref{sec:particle-bias-complexity-frontier}, however, is stated relative
to the ideal reach-$\rho$ neighborhood around the selected anchor.  We therefore
separate the remaining gap into two pieces: the error from representing this
continuous local neighborhood by finitely many sampled particles, and the error
from combining those particles through learned entropic weights.  The following
corollary makes this decomposition explicit.

\begin{corollary}[From ideal local reach to the executed refinement]
    \label{cor:local-reach-barycenter-error}
    Let
  $\gZ_{\rho,c}(s)
        :=
         B_2
        \left(
            z_c(s),\rho
        \right)$
    denote the ideal reach-$\rho$ neighborhood of the selected anchor.
    Define the realized local representation error
    \begin{equation}
        \label{eq:local-realized-representation-error}
        \eps_{\mathrm{rep}}
        \left(
            s,c,\mathcal B_c(s);\rho
        \right)
        :=
        \left[
            \sup_{z\in\gZ_{\rho,c}(s)}
                q_{\mathrm{pop}}(s,z)
            -
            \max_{0\le k\le M}
                q_{\mathrm{pop}}
                \left(
                    s,z_{c,k}(s)
                \right)
        \right]_+ .
    \end{equation}
    Under the smoothness assumption of
    Proposition~\ref{prop:barycentric-refinement-value} and the uniform critic-error condition~\ref{eq:finite-bank-uniform-score-error}, for $\betatilt>0$, we have that 
    \begin{equation}
        \label{eq:local-reach-executed-refinement-bound}
        \sup_{z\in\gZ_{\rho,c}(s)}
            q_{\mathrm{pop}}(s,z)
        -
        q_{\mathrm{pop}}
        \left(
            s,\tilde z_{c,\betatilt}(s)
        \right)
\leq
        \eps_{\mathrm{rep}}
        \left(
            s,c,\mathcal B_c(s);\rho
        \right)
        +
        \frac{\log(M+1)}{\betatilt}
        +
        2\eps_{\mathrm{crit}}
        +
        \frac{L_{\mathrm{sm}}}{2}
        \mathcal V_{c,\betatilt}(s).
    \end{equation}
\end{corollary}

\begin{proof}
    By definition of
    \eqref{eq:local-realized-representation-error}, we have that 
    \[
        \sup_{z\in\gZ_{\rho,c}(s)}
        q_{\mathrm{pop}}(s,z)
        \le
        \max_k
        q_{\mathrm{pop}}
        \left(
            s,z_{c,k}(s)
        \right)
        +
        \eps_{\mathrm{rep}}
        \left(
            s,c,\mathcal B_c(s);\rho
        \right).
    \]
    Subtract
    $q_{\mathrm{pop}}
        (s,\tilde z_{c,\betatilt}(s))$
    from both sides and apply
    \eqref{eq:barycenter-best-particle-gap} to conclude the proof.
\end{proof}

Thus, conditional on the selected anchor, the gap between the best behavior
available in the ideal reach-$\rho$ neighborhood and the latent actually
executed by \QFRS{} decomposes into finite-particle representation error,
critic error, finite-temperature weighting error, and barycentric interpolation
error.

\subsection{Sub-optimality bound for hierarchical \QFRS{}}
\label{subsec:hierarchical-anchorq-control}

The preceding results characterize local refinement within any fixed anchor
neighborhood. We now combine this local analysis with the support-gap
decomposition from Theorems~\ref{thm:general-mdp-abstraction-price}
and~\ref{thm:particle-bias-complexity-frontier} to obtain a sub-optimality
bound for the complete hierarchical \QFRS{} policy. At each state, the policy
samples a local particle cloud around every FRS anchor, forms one
value-tilted barycentric candidate per cloud, and then uses the critic to
select among these refined candidates. Only the selected candidate is
executed through the frozen VLA.

For each state $s\in\gS$ and anchor $c\in\gC$, define the ideal
reach-$\rho$ neighborhood
$\gZ_{\rho,c}(s):= B_2(z_c(s),\rho)$ and its best unrestricted
action value
\[
    Q_{\rho,c}^\star(s)
    :=
    \sup_{z\in\gZ_{\rho,c}(s)}Q^\star(s,z).
\]
The union of these local neighborhoods is
$\gZ_\rho(s):=\bigcup_{c\in\gC}\gZ_{\rho,c}(s)$, so
\begin{equation}
    \label{eq:reach-value-over-anchors}
    Q_\rho^\star(s)
    :=
    \sup_{z\in\gZ_\rho(s)}Q^\star(s,z)
    =
    \max_{c\in\gC}Q_{\rho,c}^\star(s).
\end{equation}
Recall the corresponding support gap
\begin{equation}
    \label{eq:reach-support-gap-recall}
    \Delta_\rho(s):=V^\star(s)-Q_\rho^\star(s).
\end{equation}

Let
\[
    \mathbf B_s
    :=(\mathbf B_{s,c})_{c\in\gC},
    \qquad
    \mathbf B_{s,c}
    :=(Z_{s,c,0},\ldots,Z_{s,c,M}),
\]
denote the joint collection of random local particle clouds, where
$Z_{s,c,0}=z_c(s)$. For a realization
$b=(b_c)_{c\in\gC}$, with $b_c=(z_{c,0},\ldots,z_{c,M})$, define
\[
    \widehat w_{c,k}(s,b)
    :=
    \frac{
        \exp\!\left(\betatilt Q_\psi(s,z_{c,k})\right)
    }{
        \sum_{\ell=0}^{M}
        \exp\!\left(\betatilt Q_\psi(s,z_{c,\ell})\right)
    },
    \qquad k=0,\ldots,M,
\]
where $\betatilt>0$. The corresponding local barycentric candidate is
\begin{equation}
    \label{eq:hierarchical-local-barycenters}
    \mathcal T_{\betatilt,c}(s,b)
    :=
    \sum_{k=0}^{M}\widehat w_{c,k}(s,b)z_{c,k}.
\end{equation}
Weights are normalized separately within each cloud; particles associated
with different anchors are not pooled into a common tilted distribution.

After constructing all candidates, the greedy selector chooses
\begin{equation}
    \label{eq:hierarchical-refined-selector}
    \widehat c(s,b)
    \in
    \arg\max_{c\in\gC}
    Q_\psi\!\left(s,\mathcal T_{\betatilt,c}(s,b)\right),
\end{equation}
with a fixed measurable tie-breaking rule. The executed latent is
\begin{equation}
    \label{eq:hierarchical-executed-barycenter}
    \mathcal T_{\betatilt}(s,b)
    :=
    \mathcal T_{\betatilt,\widehat c(s,b)}(s,b).
\end{equation}
Unlike a selector applied to the unrefined anchors, $\widehat c(s,b)$
depends on the realized particle clouds through their refined candidates.

Define the resulting \emph{refined-candidate selection error} by
\begin{equation}
    \label{eq:refined-candidate-selection-error}
    \eps_{\mathrm{sel}}(s,b)
    :=
    \max_{c\in\gC}
    Q^\star\!\left(s,\mathcal T_{\betatilt,c}(s,b)\right)
    -
    Q^\star\!\left(s,\mathcal T_{\betatilt}(s,b)\right)
    \geq 0.
\end{equation}
This measures the loss from choosing among the realized refined candidates
using the learned critic rather than their true values. Also define the
critic error at these candidates by
\begin{equation}
    \label{eq:hierarchical-selection-critic-error}
    \eps_{\mathrm{crit,sel}}(s,b)
    :=
    \max_{c\in\gC}
    \left|
        Q_\psi\!\left(s,\mathcal T_{\betatilt,c}(s,b)\right)
        -
        Q^\star\!\left(s,\mathcal T_{\betatilt,c}(s,b)\right)
    \right|.
\end{equation}
Greedy selection implies
\begin{equation}
    \label{eq:hierarchical-selection-error-bound}
    \eps_{\mathrm{sel}}(s,b)
    \leq 2\eps_{\mathrm{crit,sel}}(s,b).
\end{equation}
Accuracy at the sampled particles alone does not imply accuracy at their
barycenters, so this error is distinguished from the particle-level error
below.

For the joint realization $b$, define the representation error relative to
the union of ideal local neighborhoods by
\begin{equation}
    \label{eq:hierarchical-local-representation-error}
    \eps_{\mathrm{rep}}(s,b)
    :=
    \left[
        Q_\rho^\star(s)
        -
        \max_{\substack{c\in\gC\\0\leq k\leq M}}
        Q^\star(s,z_{c,k})
    \right]_+,
\end{equation}
and the maximum critic error on the sampled particles by
\begin{equation}
    \label{eq:hierarchical-cloud-critic-error}
    \eps_{\mathrm{crit}}(s,b)
    :=
    \max_{\substack{c\in\gC\\0\leq k\leq M}}
    \left|Q_\psi(s,z_{c,k})-Q^\star(s,z_{c,k})\right|.
\end{equation}
The maximum over all sampled particles in the representation error is an
analytical benchmark, not an additional action-selection step. The positive
part keeps the comparison valid even when particles fall outside the ideal
reach-$\rho$ neighborhoods.

For each cloud, define its tilted particle variance by
\[
    \mathcal V_{\betatilt,c}(s,b)
    :=
    \sum_{k=0}^{M}
    \widehat w_{c,k}(s,b)
    \left\|z_{c,k}-\mathcal T_{\betatilt,c}(s,b)\right\|_2^2,
\]
and use the uniform within-cloud bound
\begin{equation}
    \label{eq:hierarchical-cloud-variance}
    \mathcal V_{\betatilt}(s,b)
    :=\max_{c\in\gC}\mathcal V_{\betatilt,c}(s,b).
\end{equation}
This is the largest local tilted variance, not the variance of a distribution
that mixes particles across anchors.

\begin{assumption}
\label{ass:smoothness_mdp}
Consider the discounted latent-action MDP induced by the frozen VLA, with
bounded rewards and discount factor $\gamma\in[0,1)$. All sampled particles
and refined candidates are admissible latent actions. For every state $s$
and almost every realized joint collection $b$, the function
$Q^\star(s,\cdot)$ has $L_{\mathrm{sm}}$-Lipschitz gradient on
$\operatorname{conv}(b_c)$ for every $c\in\gC$.
\end{assumption}

\begin{theorem}[Sub-optimality decomposition for hierarchical \QFRS{}]
\label{thm:hierarchical-anchorq-control}
Suppose Assumption~\ref{ass:smoothness_mdp} holds. Fix the learned critic,
the state-conditioned particle-sampling rule, and $\betatilt>0$. Define the
stationary randomized policy $\pi_{\AQ}$ to sample $\mathbf B_s$, form all
local barycentric candidates, and execute
$\mathcal T_{\betatilt}(s,\mathbf B_s)$ using the greedy selector in
\eqref{eq:hierarchical-refined-selector}. Then, for every state $s$,
\begin{equation}
\label{eq:hierarchical-one-step-control-gap}
\begin{aligned}
    0
    &\leq
    V^\star(s)
    -
    \mathbb E_{\mathbf B_s}
    \left[Q^\star\!\left(s,\mathcal T_{\betatilt}(s,\mathbf B_s)\right)\right]
    \\
    &\leq
    \Delta_\rho(s)
    +\mathbb E_{\mathbf B_s}
        [\eps_{\mathrm{sel}}(s,\mathbf B_s)]
    +\mathbb E_{\mathbf B_s}
        [\eps_{\mathrm{rep}}(s,\mathbf B_s)]
    \\
    &\quad
    +\frac{\log(M+1)}{\betatilt}
    +2\mathbb E_{\mathbf B_s}
        [\eps_{\mathrm{crit}}(s,\mathbf B_s)]
    +\frac{L_{\mathrm{sm}}}{2}
        \mathbb E_{\mathbf B_s}
        [\mathcal V_{\betatilt}(s,\mathbf B_s)].
\end{aligned}
\end{equation}
Define
\begin{align*}
    \bar\Delta_\rho
    &:=\sup_s\Delta_\rho(s),
    &\bar\eps_{\mathrm{sel}}
    &:=\sup_s\mathbb E_{\mathbf B_s}
        [\eps_{\mathrm{sel}}(s,\mathbf B_s)],
    \\
    \bar\eps_{\mathrm{rep}}
    &:=\sup_s\mathbb E_{\mathbf B_s}
        [\eps_{\mathrm{rep}}(s,\mathbf B_s)],
    &\bar\eps_{\mathrm{crit}}
    &:=\sup_s\mathbb E_{\mathbf B_s}
        [\eps_{\mathrm{crit}}(s,\mathbf B_s)],
    \\
    \bar{\mathcal V}_{\betatilt}
    &:=\sup_s\mathbb E_{\mathbf B_s}
        [\mathcal V_{\betatilt}(s,\mathbf B_s)],
    &\bar\eps_{\mathrm{crit,sel}}
    &:=\sup_s\mathbb E_{\mathbf B_s}
        [\eps_{\mathrm{crit,sel}}(s,\mathbf B_s)].
\end{align*}
Whenever these quantities are finite, the policy satisfies
\begin{equation}
\label{eq:hierarchical-anchorq-control-bound}
    \left\|V^\star-V^{\pi_{\AQ}}\right\|_\infty
    \leq
    \frac{
        \bar\Delta_\rho
        +\bar\eps_{\mathrm{sel}}
        +\bar\eps_{\mathrm{rep}}
        +\betatilt^{-1}\log(M+1)
        +2\bar\eps_{\mathrm{crit}}
        +\tfrac{L_{\mathrm{sm}}}{2}\bar{\mathcal V}_{\betatilt}
    }{1-\gamma}.
\end{equation}
Moreover, $\bar\eps_{\mathrm{sel}}\leq
2\bar\eps_{\mathrm{crit,sel}}$, yielding a bound entirely in terms of the
support gap, representation error, temperature, particle-level critic error,
candidate-level critic error, and within-cloud variance.
\end{theorem}

\begin{proof}
Fix a state $s$ and a joint cloud realization $b$ for which
Assumption~\ref{ass:smoothness_mdp} holds. For each anchor $c$, optimality of
the Gibbs weights in the KL-regularized objective, compared with a point mass
on a highest-scoring particle, gives
\[
    \sum_{k=0}^{M}\widehat w_{c,k}(s,b)Q_\psi(s,z_{c,k})
    \geq
    \max_{0\leq k\leq M}Q_\psi(s,z_{c,k})
    -\frac{\log(M+1)}{\betatilt}.
\]
Using the particle-level critic error twice therefore yields
\begin{equation}
\label{eq:hierarchical-local-weighted-value}
    \sum_{k=0}^{M}\widehat w_{c,k}(s,b)Q^\star(s,z_{c,k})
    \geq
    \max_{0\leq k\leq M}Q^\star(s,z_{c,k})
    -\frac{\log(M+1)}{\betatilt}
    -2\eps_{\mathrm{crit}}(s,b).
\end{equation}
Smoothness on $\operatorname{conv}(b_c)$ and the barycenter identity
$\sum_k\widehat w_{c,k}(z_{c,k}-\mathcal T_{\betatilt,c})=0$ imply
\[
    Q^\star\!\left(s,\mathcal T_{\betatilt,c}(s,b)\right)
    \geq
    \sum_{k=0}^{M}\widehat w_{c,k}(s,b)Q^\star(s,z_{c,k})
    -\frac{L_{\mathrm{sm}}}{2}\mathcal V_{\betatilt,c}(s,b).
\]
Consequently, for every $c$,
\begin{equation}
\label{eq:hierarchical-local-bound}
\begin{aligned}
    Q^\star\!\left(s,\mathcal T_{\betatilt,c}(s,b)\right)
    \geq\;&
    \max_{0\leq k\leq M}Q^\star(s,z_{c,k})
    -\frac{\log(M+1)}{\betatilt}
    \\
    &-2\eps_{\mathrm{crit}}(s,b)
    -\frac{L_{\mathrm{sm}}}{2}\mathcal V_{\betatilt,c}(s,b).
\end{aligned}
\end{equation}
Apply this inequality to an anchor whose cloud contains a highest-valued
particle among all clouds. Since
$\mathcal V_{\betatilt,c}(s,b)\leq\mathcal V_{\betatilt}(s,b)$, it follows that
\begin{equation}
\label{eq:hierarchical-best-refined-candidate}
\begin{aligned}
    \max_{c\in\gC}
    Q^\star\!\left(s,\mathcal T_{\betatilt,c}(s,b)\right)
    \geq\;&
    \max_{\substack{c\in\gC\\0\leq k\leq M}}Q^\star(s,z_{c,k})
    -\frac{\log(M+1)}{\betatilt}
    \\
    &-2\eps_{\mathrm{crit}}(s,b)
    -\frac{L_{\mathrm{sm}}}{2}\mathcal V_{\betatilt}(s,b).
\end{aligned}
\end{equation}

For the complete policy, add and subtract the best ideal-neighborhood
value, the best sampled-particle value, and the best refined-candidate value:
\begin{equation}
\label{eq:hierarchical-exact-decomposition}
\begin{aligned}
    &V^\star(s)-Q^\star\!\left(s,\mathcal T_{\betatilt}(s,b)\right)
    \\
    &\quad=
    \Delta_\rho(s)
    +\left[
        Q_\rho^\star(s)
        -\max_{\substack{c\in\gC\\0\leq k\leq M}}Q^\star(s,z_{c,k})
    \right]
    \\
    &\qquad
    +\left[
        \max_{\substack{c\in\gC\\0\leq k\leq M}}Q^\star(s,z_{c,k})
        -\max_{c\in\gC}
        Q^\star\!\left(s,\mathcal T_{\betatilt,c}(s,b)\right)
    \right]
    +\eps_{\mathrm{sel}}(s,b).
\end{aligned}
\end{equation}
Bound the first bracket by its positive part
$\eps_{\mathrm{rep}}(s,b)$ and the second using
\eqref{eq:hierarchical-best-refined-candidate}. This gives
\begin{align*}
    V^\star(s)-Q^\star\!\left(s,\mathcal T_{\betatilt}(s,b)\right)
    \leq\;&
    \Delta_\rho(s)+\eps_{\mathrm{sel}}(s,b)
    +\eps_{\mathrm{rep}}(s,b)
    \\
    &+\frac{\log(M+1)}{\betatilt}
    +2\eps_{\mathrm{crit}}(s,b)
    +\frac{L_{\mathrm{sm}}}{2}\mathcal V_{\betatilt}(s,b).
\end{align*}
Taking expectation over the joint cloud collection proves the upper bound in
\eqref{eq:hierarchical-one-step-control-gap}. Its lower bound follows from
$Q^\star(s,z)\leq V^\star(s)$ for every admissible latent $z$.

To verify \eqref{eq:hierarchical-selection-error-bound}, let $c^+$ maximize
the true value among the realized refined candidates. Then
\begin{align*}
    \eps_{\mathrm{sel}}(s,b)
    &=
    Q^\star\!\left(s,\mathcal T_{\betatilt,c^+}(s,b)\right)
    -Q^\star\!\left(s,\mathcal T_{\betatilt,\widehat c(s,b)}(s,b)\right)
    \\
    &\leq
    Q_\psi\!\left(s,\mathcal T_{\betatilt,c^+}(s,b)\right)
    -Q_\psi\!\left(s,\mathcal T_{\betatilt,\widehat c(s,b)}(s,b)\right)
    +2\eps_{\mathrm{crit,sel}}(s,b)
    \\
    &\leq 2\eps_{\mathrm{crit,sel}}(s,b),
\end{align*}
where the last inequality uses greedy selection.

Finally, let $\gT^{\pi_{\AQ}}$ and $P^{\pi_{\AQ}}$ denote the Bellman and
transition operators of the stationary policy. Since
\[
    Q^\star(s,z)
    =r(s,z)+\gamma\mathbb E_{S'\sim P(\cdot\mid s,z)}[V^\star(S')],
\]
we have
\[
    \gT^{\pi_{\AQ}}V^\star(s)
    =\mathbb E_{\mathbf B_s}
    \left[Q^\star\!\left(s,\mathcal T_{\betatilt}(s,\mathbf B_s)\right)\right].
\]
Define $\xierr:=V^\star-\gT^{\pi_{\AQ}}V^\star$. The policy Bellman equation
gives
\[
    V^\star-V^{\pi_{\AQ}}
    =\xierr+\gamma P^{\pi_{\AQ}}(V^\star-V^{\pi_{\AQ}}),
\]
and hence
\begin{equation}
\label{eq:hierarchical-resolvent-identity}
    V^\star-V^{\pi_{\AQ}}
    =(I-\gamma P^{\pi_{\AQ}})^{-1}\xierr.
\end{equation}
Because $P^{\pi_{\AQ}}$ is a Markov operator,
\[
    \left\|(I-\gamma P^{\pi_{\AQ}})^{-1}\right\|_{\infty\to\infty}
    \leq\sum_{t=0}^{\infty}\gamma^t
    =\frac{1}{1-\gamma}.
\]
Taking the supremum of the one-step bound over states establishes
\eqref{eq:hierarchical-anchorq-control-bound}. Taking expectations and
suprema in \eqref{eq:hierarchical-selection-error-bound} also gives
$\bar\eps_{\mathrm{sel}}\leq2\bar\eps_{\mathrm{crit,sel}}$.
\end{proof}

Theorem~\ref{thm:hierarchical-anchorq-control} separates the performance loss
into errors arising at each stage of the refinement-first construction.
The term $\bar\Delta_\rho$ is structural: it measures the largest one-step
gap between the unrestricted optimum and the best latent in the union of
reach-$\rho$ neighborhoods around the FRS anchors. The representation term
$\bar\eps_{\mathrm{rep}}$ measures how well the joint collection of finite
particle clouds represents the best value in this union. There is no
additional term for committing to one anchor before refinement, because
every anchor is refined before selection.

The remaining terms distinguish local refinement from final selection.
The finite-temperature term $\log(M+1)/\betatilt$ is the price of soft
entropic weighting within each cloud. It depends on the number of particles
per cloud, not the total number across anchors, because tilting is local.
The term $2\bar\eps_{\mathrm{crit}}$ captures the effect of particle-level
critic errors on the tilted weights. The variance term
$(L_{\mathrm{sm}}/2)\bar{\mathcal V}_{\betatilt}$ accounts for the difference
between the weighted average of particle values and the value of their
weighted-average latent. Smoothness is needed only within each cloud's
convex hull, not between different anchor neighborhoods.

The term $\bar\eps_{\mathrm{sel}}$ is the price of choosing among the
resulting barycentric candidates using the learned critic. It vanishes when
the selector chooses a truly highest-valued refined candidate and is at most
$2\bar\eps_{\mathrm{crit,sel}}$. Particle-level and candidate-level critic
errors are distinct because a barycenter need not be one of the sampled
particles. Both are measured relative to the unrestricted optimal action
value $Q^\star$. They can include estimation error and bias from the
candidate-based bootstrap, and need not vanish merely because the critic
accurately fits its Bellman target. The theorem is therefore a performance
decomposition conditional on accuracy relative to $Q^\star$, not a guarantee
that training drives these terms to zero.

The bound exposes complementary tradeoffs. Increasing the ideal reach $\rho$
can reduce the structural support gap but enlarges the neighborhood union
that the particles must represent. For fixed anchors and realized clouds,
representation error is nondecreasing in $\rho$; adding particles while
retaining the existing clouds cannot increase that representation error.
These statements concern the representation term, not monotonic improvement
of the entire bound or the executed policy. Increasing $\betatilt$ reduces
the explicit soft-weighting term, but reliable refinement still requires
sufficiently accurate local value rankings, as quantified by
Theorem~\ref{thm:exponential-tilting}. Thus good performance requires both
useful local candidates and accurate selection among their refinements.

The same pointwise argument accommodates a common adaptive tilt
$\betatilt(s,b)>0$ shared across anchors. In that case, all weights,
barycenters, and variances use the realized tilt, and the temperature term in
the uniform value bound becomes
\[
    \sup_s\mathbb E_{\mathbf B_s}
    \left[\frac{\log(M+1)}{\betatilt(s,\mathbf B_s)}\right],
\]
provided it is finite. This allows uncertainty-dependent tilting without
changing the decomposition or dispensing with either critic-error term.

The theorem evaluates greedy selection after refinement. For an
$\epsilon$-greedy execution rule that selects an anchor uniformly with
probability $\epsilon$, define
\[
    D_{\mathrm{unif}}(s,b)
    :=
    \max_{c\in\gC}Q^\star\!\left(s,\mathcal T_{\betatilt,c}(s,b)\right)
    -\frac{1}{|\gC|}\sum_{c\in\gC}
        Q^\star\!\left(s,\mathcal T_{\betatilt,c}(s,b)\right).
\]
The same decomposition applies to the resulting randomized policy after
replacing the conditional selection error by
\[
    \eps_{\mathrm{sel}}^{\epsilon}(s,b)
    :=(1-\epsilon)\eps_{\mathrm{sel}}(s,b)
      +\epsilon D_{\mathrm{unif}}(s,b),
\]
and averaging over the selection randomness as well as the particle clouds.
This makes the additional exploration loss explicit rather than attributing
it to critic error.

\subsection{Connection to Theorems 1--3}
\label{subsec:tilting-connection-prior-theorems}

The preceding results complete the coarse-to-fine picture developed in
Theorems~\ref{thm:persistent-abstraction-separation}--
\ref{thm:particle-bias-complexity-frontier}.  Those results characterize the
action representation made available to online learning; the present section
characterizes how value information is used to refine that representation after
a coarse anchor has been selected.

\paragraph{Connection to Theorem~\ref{thm:persistent-abstraction-separation}.}
Theorem~\ref{thm:persistent-abstraction-separation} isolates the statistical
benefit of restricting the coarse online decision to a small set of
state-conditioned FRS anchors.  The local weighting mechanism does not alter
this first-stage identification problem: \QFRS{} still selects among the
$C$ coarse anchors before constructing a local refinement.  Particles then
restore additional expressivity only around the selected anchor.

Conditional on this local cloud, Theorem~\ref{thm:exponential-tilting} shows
that value information directs the refinement toward reliably better local
particles.  In particular, for a high-value local region $A$,
\begin{equation}
    \label{eq:connection-robust-local-amplification}
    \operatorname{Odds}_{
        \nu_{s,c,\betatilt}^{\hat q}
    }(A)
    \ge
    \exp
    \left(
        \betatilt
        \left[
            \Xi_{\nu_{s,c}}(A;q_{\mathrm{pop}})
            -
            2\eps_{\mathrm{crit}}
        \right]
    \right)
    \operatorname{Odds}_{\nu_{s,c}}(A).
\end{equation}
Thus the FRS construction reduces the coarse search problem, while particles
and value weighting introduce a controlled local refinement problem.  The
latter is useful only once the critic separates genuinely useful refinements
from inferior ones with sufficient accuracy.

\paragraph{Connection to Theorem~\ref{thm:general-mdp-abstraction-price}.}
Theorem~\ref{thm:general-mdp-abstraction-price} shows that restricting the
available latent actions incurs a structural support price. Value weighting
cannot by itself remove this price. It changes the relative influence of
particles within each realized local cloud and can thereby affect which
refined anchor candidate is ultimately selected, but it does not enlarge the
union of local regions accessible around the FRS anchors.

For each
anchor $c$, the barycentric refinement
\begin{equation}
\tilde z_{c,\betatilt}(s)
=
\sum_{k=0}^{M}
w_{c,\betatilt}(k\,|\, s),
z_{c,k}(s)
\end{equation}
need not coincide with any sampled particle; rather, it lies in their convex
hull. Consequently, if the realized local cloud around anchor $c$ lies in the
convex reach-$\rho$ neighborhood
$ B_2(z_c(s),\rho)$, then
\begin{equation}
\tilde z_{c,\betatilt}(s)
\in
 B_2(z_c(s),\rho).
\end{equation}
Thus, if this condition holds for every anchor, the candidate ultimately
selected for execution lies in
\begin{equation}
\gZ_\rho(s)
=
\bigcup_{c\in\gC}
 B_2(z_c(s),\rho).
\end{equation}
Local value weighting can therefore improve the refinement within each
accessible region and influence selection among the resulting candidates,
but it cannot overcome the structural gap between the union of these regions
and globally better latent behavior outside it. This distinction is captured
by the support term $\Delta_\rho$ in
Theorem~\ref{thm:hierarchical-anchorq-control}.

\paragraph{Connection to Theorem~\ref{thm:particle-bias-complexity-frontier}.}
Theorem~\ref{thm:particle-bias-complexity-frontier} characterizes the tradeoff
created by enlarging the local action representation. The ideal reach
$\rho$ controls how much latent behavior is accessible in the union of local
neighborhoods around the FRS anchors, while the implemented parameters
$(\sigma_p,M)$ determine the finite random clouds used to probe these
neighborhoods. The tilt intensity $\betatilt$ plays a different role: it does
not enlarge the accessible region, but controls how strongly particle values
influence each local refinement.

Theorem~\ref{thm:hierarchical-anchorq-control} combines these effects for the
actual hierarchical policy:
\begin{equation}
    \label{eq:connection-hierarchical-control-bound}
    \left\|
        V^\star-V^{\pi_{\AQ}}
    \right\|_\infty
    \le
    \frac{
        \bar\Delta_\rho
        +
        \bar\eps_{\mathrm{sel}}
        +
        \bar\eps_{\mathrm{rep}}
        +
        \betatilt^{-1}\log(M+1)
        +
        2\bar\eps_{\mathrm{crit}}
        +
        L_{\mathrm{sm}}
        \bar{\mathcal V}_{\betatilt}/2
    }{
        1-\gamma
    }.
\end{equation}
The terms have distinct origins: $\bar\Delta_\rho$ is the structural price of
restricting behavior to the union of the local action regions;
$\bar\eps_{\mathrm{rep}}$ is the finite-particle representation error across
those regions;
$\log(M+1)/\betatilt$ is the finite-temperature weighting error;
$2\bar\eps_{\mathrm{crit}}$ captures the effect of particle-level critic
errors on the local tilting;
$(L_{\mathrm{sm}}/2)\bar{\mathcal V}_{\betatilt}$ controls the discrepancy
between high weighted particle value and the value of the corresponding
barycentric latent; and
$\bar\eps_{\mathrm{sel}}$ is the error from selecting among the resulting
refined candidates using the learned critic. In particular,
$\bar\eps_{\mathrm{sel}}\leq 2\bar\eps_{\mathrm{crit,sel}}$, where
$\bar\eps_{\mathrm{crit,sel}}$ measures critic error at the refined
barycentric candidates.

The bound therefore exposes three coupled tuning axes. Increasing the ideal
reach $\rho$ can reduce structural support bias, but enlarges the region that
the particle clouds must represent. For fixed anchors and realized clouds,
adding particles while retaining the existing particles cannot increase the
representation error, although it may affect the other terms in the bound.
Increasing $\betatilt$ reduces the explicit finite-temperature term, but more
aggressive weighting is trustworthy only when local value differences are
large relative to particle-level critic error. Indeed,
\eqref{eq:tilted-density-ratio-robustness} shows that sensitivity of the
weighting law to critic perturbations grows with
$\betatilt\eps_{\mathrm{crit}}$, while
\eqref{eq:connection-robust-local-amplification} guarantees correct
high-value amplification when the true local margin exceeds
$2\eps_{\mathrm{crit}}$. Finally, because selection occurs only after each
anchor has been locally refined, reliable execution also requires accurate
value comparisons across the resulting barycentric candidates, as captured
by $\bar\eps_{\mathrm{sel}}$.

\subsection{Interpretation and novelty}
\label{subsec:tilting-interpretation-novelty}

The Gibbs variational identity and the associated exponential-family
derivatives are standard.  The contribution here is their use within the
specific coarse-to-fine structure of \QFRS{}, together with results that connect
the resulting risk-sensitive weights to the barycentric latent actually
executed by the robot.

First, the analysis identifies the precise random variable to which the
entropic functional is applied.  Conditional on state $s$ and selected anchor
$c$,
\[
    Y_{s,c}^{q}
    =
    q(s,Z_{s,c}),
    \qquad
    Z_{s,c}\sim\nu_{s,c},
\]
is the score of a random \emph{local refinement} of that anchor.  Positive
$\betatilt$ is therefore upper-tail seeking over local particle values, not
over coarse FRS commands, critic-head uncertainty, or realized environment
returns.  It becomes an exact entropic functional of realized return only
under the conditional transformation of
Proposition~\ref{prop:conditional-entropic-return-transform}.

Second, the exponential weights have an exact optimization interpretation.
The tilted local law $\nu_{s,c,\betatilt}^{q}$ is the unique optimizing change
of measure in the KL-regularized variational representation of the entropic
functional.  Thus the weights are not an ad hoc softmax heuristic: they are the
optimal risk-sensitive reweighting of the local particle law under a KL penalty
from its untilted reference distribution.  This interpretation also separates
the role of $\betatilt$ from the particle parameters: $(\sigma_{p},M)$
control which local refinements are represented, whereas $\betatilt$ controls
how aggressively value concentrates their influence.

Third, the high-value amplification result provides a reliability guarantee
for this local weighting.  If a nonempty proper subset of the particle indices has true value margin $\margin_{\mathcal H}$ over its complement and the critic error is bounded by $\eps_{\mathrm{crit}}$, then its odds relative to that complement are at least a factor 
\[
    \exp
    \left(
        \betatilt
        (\margin_{\mathcal H}-2\eps_{\mathrm{crit}})
    \right)
\]
times the corresponding odds under uniform weighting. The threshold
$\margin_{\mathcal H}>2\eps_{\mathrm{crit}}$ is worst-case sharp under a
uniform critic-error model by
Proposition~\ref{prop:critic-separation-threshold-sharpness}.  Thus aggressive
tilting is theoretically justified only when the critic resolves the relevant
local value differences reliably.

Fourth, the implemented policy does not sample from this tilted law.
It leverages the barycentric latent
\[
    \tilde z_{c,\betatilt}(s)
    =
    \sum_{k=0}^{M}
    w_{c,\betatilt}(k\mid s)\,
    z_{c,k}(s).
\]
This creates a new theoretical question absent from ordinary exponential
candidate selection: high weighted particle value need not imply high value at
their weighted-average latent.  Proposition~\ref{prop:barycentric-refinement-value}
shows that local smoothness closes this gap, with interpolation penalty
$L_{\mathrm{sm}}\cdot
    \mathcal V_{c,\betatilt}(s)/2$.
Hence the theory explicitly distinguishes the quality of the weighting rule
from the geometry of latent interpolation used to realize that weighting.

Finally, Theorem~\ref{thm:hierarchical-anchorq-control} combines the full
construction in a single sequential-control bound.  Its decomposition separates
the structural support bias, coarse-anchor error, finite-particle
representation error, critic error, finite-temperature weighting error, and
barycentric interpolation error.  

The resulting picture is that the anchor and particle construction determines the available local
behavior, while risk-sensitive weighting determines how that local capacity is
used.

\section{Uncertainty Adaptation: Epistemically adaptive local tilting}
\label{sec:uncertainty_adaptation}
The preceding section treated the exponential-tilting intensity $\betatilt$ as fixed.  There, positive $\betatilt$ was shown to control the strength of upper-tail weighting over local particle values: increasing $\betatilt$ places progressively more influence on high-value local refinements, while the limit $\betatilt\downarrow0$ recovers the un-tilted local particle law.  The same analysis also exposes the principal limitation of aggressive tilting with a learned critic.  If the local critic error is bounded by $\eps_{\mathrm{crit}}$, Theorem~\ref{thm:exponential-tilting} yields
    $e^{-2\betatilt\eps_{\mathrm{crit}}}
    \leq
        \mathrm d\nu_{s,c,\betatilt}^{\hat q}/
        \mathrm d\nu_{s,c,\betatilt}^{q_{\mathrm{pop}}}
    \leq
    e^{2\betatilt\eps_{\mathrm{crit}}}$.
Thus the effect of critic error on the tilted weighting law is governed by the product
$\betatilt\eps_{\mathrm{crit}}$.

A large fixed $\betatilt$ can therefore amplify spurious score differences precisely when the critic is least reliable, as may occur early in online learning or in weakly covered regions of the state--latent space.

This motivates adapting the tilting intensity to \emph{epistemic uncertainty} in the learned critic.  
For a fixed state $s$, selected anchor $c$, and realized local particle cloud, let  $\widehat u_{\mathrm{epi}}(s)\in [0,\infty)$ denote a finite scalar surrogate obtained from disagreement across the critic ensemble. We suppress its dependence on the selected anchor and realized cloud in the notation. We define
\begin{equation}
u_{\mathrm{epi}}(s)
:=
\max\{
\widehat u_{\mathrm{epi}}(s),
u_{\min}
\},
\qquad
\betatilt(s)
:=
1/
u_{\mathrm{epi}}(s).
\label{eq:adaptive-tilt-beta}
\end{equation}
The floor $u_{\min}>0$ bounds the maximum tilting intensity when the estimated uncertainty is spuriously small.

To distinguish the resulting adaptive weights from the fixed-$\betatilt$
weights defined in Proposition~\ref{prop:finite-cloud-softmax-duality}, we use
the shorthand 
\begin{equation}
\label{eq:adaptive_weights}
\widehat w_{c,k}^{\mathrm{ad}}(s)
:=
\widehat w_{c,k}^{\betatilt(s)}(s)
=
\frac{
\exp\left(
\betatilt(s)\hat q_{c,k}(s)
\right)
}{
\sum_{\ell=0}^{M}
\exp\left(
\betatilt(s)\hat q_{c,\ell}(s)
\right)
},
\qquad k=0,\ldots,M,
\end{equation}
where $\hat q_{c,k}(s)$ denotes the learned score of the $k$th particle in the realized local cloud, as in Section~\ref{subsec:finite-cloud-tilting}. Thus $\widehat w_{c,k}^{\mathrm{ad}}(s)$ is not a new weighting rule; it is the same entropic weighting analyzed above, with its tilting intensity now chosen state-wise from epistemic uncertainty. We similarly write
\[
\tilde z_c^{\mathrm{ad}}(s)
:=
\sum_{k=0}^{M}
\widehat w_{c,k}^{\mathrm{ad}}(s)
z_{c,k}(s)
\]
for the corresponding adaptive barycentric refinement. We similarly define the adaptive weighted particle variance by
\begin{equation}
    \label{eq:adaptive-weighted-particle-variance}
    \mathcal V_c^{\mathrm{ad}}(s)
    :=
    \sum_{k=0}^{M}
    \widehat w_{c,k}^{\mathrm{ad}}(s)
    \left\|
        z_{c,k}(s)
        -
        \tilde z_c^{\mathrm{ad}}(s)
    \right\|_2^2.
\end{equation}

The interpretation is direct.  Exponential weighting remains the upper-tail-seeking local refinement mechanism developed in Section~\ref{sec:exponential-tilting}; adaptation changes only \emph{how strongly that mechanism is applied}.  Large epistemic uncertainty yields weak tilting and keeps the weighting law close to the FRS local particle law.  As epistemic uncertainty contracts, the same mechanism becomes progressively more aggressive in amplifying high-value local refinements.

Throughout this section, $u_{\mathrm{epi}}$ refers to uncertainty in the learned score of a fixed particle.  It is distinct from the variation of $q(s,Z)$ across local particles analyzed by the entropic functional in the preceding section, and from aleatoric randomness in the realized return conditional on a fixed latent.

\subsection{Calibration of the epistemic scale}

Randomized-prior critic ensembles provide a practical surrogate for epistemic uncertainty, but ensemble disagreement is not automatically a calibrated confidence radius for deep Q-learning.  We therefore state explicitly the condition required by the adaptive analysis.

\begin{assumption}[Local epistemic calibration]
\label{ass:adaptive-epistemic-calibration}
Let $q_{\mathrm{pop}}$ denote the population local score and $\hat q$ the learned score.  There exists a finite constant $c_{\mathrm{cal}}>0$, independent of the state $s$ and selected anchor $c$, such that, for every state-anchor pair under consideration, 
\begin{equation}
\left\|
\hat q(s,\cdot)
-
q_{\mathrm{pop}}(s,\cdot)
\right\|_{
L^\infty(\nu_{s,c})
}
\le
c_{\mathrm{cal}}
u_{\mathrm{epi}}(s).
\label{eq:adaptive-epistemic-calibration}
\end{equation}
When interpreted probabilistically, this inequality is assumed to hold simultaneously over the state-anchor pairs under consideration on a common high-probability event. 
\end{assumption}

Assumption~\ref{ass:adaptive-epistemic-calibration} says that the ensemble-derived scale need not equal the critic error, but it must track that error up to a state-independent multiplicative constant.  Under this condition, choosing $\betatilt$ inversely proportional to $u_{\mathrm{epi}}$ normalizes the aggressiveness of the local tilt by the reliability of the critic.

\begin{theorem}[Epistemically adaptive local tilting]
\label{thm:adaptive-local-tilting}
Suppose Assumption~\ref{ass:bounded-tilting-score} and~\ref{ass:adaptive-epistemic-calibration} hold.  Fix a state $s$ and selected anchor $c$, and set
$u
    :=
    u_{\mathrm{epi}}(s)$ and 
    $\beta:=\betatilt(s)=
    1/u$.
Let
$\nu_{s,c,\beta}^{\hat q}$
and
$\nu_{s,c,\beta}^{q_{\mathrm{pop}}}$
denote the adaptive weighting laws constructed from the learned and population scores, respectively, using the same reference law $\nu_{s,c}$ and the same realized intensity $\beta=1/u_{\mathrm{epi}}(s)$. 

\begin{enumerate}[
label=\textnormal{\it\alph*.},
itemsep=0pt,
topsep=0pt,
leftmargin=*
]

\item \textbf{Uncertainty-weighted variational objective.}
The learned adaptive weighting law is the unique optimizer of
\begin{equation}
\label{eq:adaptive-variational-objective}
\nu_{s,c,\beta}^{\hat q}
=
\argmax_{\probmeas\ll \nu_{s,c}}
\left\{
\int_{\gZ}
\hat q(s,z)\probmeas(\mathrm dz)
-
u_{\mathrm{epi}}(s)
\KL(
\probmeas
\|
\nu_{s,c}
)
\right\}.
\end{equation}
Thus epistemic uncertainty is exactly the coefficient penalizing departure from the untilted FRS local-particle law.

\item \textbf{Scale-free sensitivity.}
The adaptive learned and population weighting laws satisfy
\begin{equation}
\label{eq:adaptive-density-ratio}
e^{-2c_{\mathrm{cal}}}
\leq
\frac{
\mathrm d
\nu_{s,c,\beta}^{\hat q}
}{
\mathrm d
\nu_{s,c,\beta}^{q_{\mathrm{pop}}}
}(z)
\leq
e^{2c_{\mathrm{cal}}}
\end{equation}
for $\nu_{s,c}$-almost every\footnote{For two probability measures $\mu$ and $\nu$, we write
$\mu\sim\nu$ to denote \emph{equivalence of measures}, i.e.,
$\mu\ll\nu$ and $\nu\ll\mu$, where $\mu\ll\nu$ denotes absolute continuity.
Equivalently, $\mu$ and $\nu$ have the same measure-zero sets.  In our setting,
the tilted law satisfies
\[
    \frac{\mathrm d\nu_{s,c,\beta}^{q}}
         {\mathrm d\nu_{s,c}}(z)
    =
    \exp\left(
        \beta\left(q(s,z)-\mathrm{CE}_{\beta}(q;s,c)\right)
    \right).
\]
Under the bounded-score assumption and finite $\beta$, this density is strictly
positive and finite $\nu_{s,c}$-almost everywhere.  Hence
$\nu_{s,c,\beta}^{q}\ll\nu_{s,c}$ and also
$\nu_{s,c}\ll\nu_{s,c,\beta}^{q}$, so
$\nu_{s,c,\beta}^{q}\sim\nu_{s,c}$.  Consequently, tilted laws constructed
from different bounded scores, such as $\hat q$ and $q_{\mathrm{pop}}$,
are also mutually equivalent and have the same null sets as $\nu_{s,c}$.} $z$, so that 
\begin{align}
\KL
\left(
\nu_{s,c,\beta}^{\hat q}
\|
\nu_{s,c,\beta}^{q_{\mathrm{pop}}}
\right)
&\leq
2c_{\mathrm{cal}},
\label{eq:adaptive-forward-kl}
\quad\text{and}\quad 
\KL
\left(
\nu_{s,c,\beta}^{q_{\mathrm{pop}}}
\|
\nu_{s,c,\beta}^{\hat q}
\right)
\leq
2c_{\mathrm{cal}}.
\end{align}
These bounds are independent of the absolute magnitude of $u:=u_{\mathrm{epi}}(s)$.

\item \textbf{Standardized high-value amplification.}
For any Borel set $A\subseteq\gZ$ with
$0<\nu_{s,c}(A)<1$,
\begin{equation}
\label{eq:adaptive-high-value-odds}
\operatorname{Odds}_{
\nu_{s,c,\beta}^{\hat q}
}(A)
\geq
\exp
\left(
\frac{
\Xi_{\nu_{s,c}}
(A;q_{\mathrm{pop}})
}{
u_{\mathrm{epi}}(s)
}
-
2c_{\mathrm{cal}}
\right)
\operatorname{Odds}_{\nu_{s,c}}(A).
\end{equation}
In particular, if for some $m\ge0$ we have 
\begin{equation}
\label{eq:adaptive-standardized-margin}
\Xi_{\nu_{s,c}}
(A;q_{\mathrm{pop}})
\ge
\bigl(
2c_{\mathrm{cal}}+m
\bigr)
u_{\mathrm{epi}}(s),
\end{equation}
then
\begin{equation}
\label{eq:adaptive-exp-amplification}
\operatorname{Odds}_{
\nu_{s,c,\beta}^{\hat q}
}(A)
\geq
e^m
\operatorname{Odds}_{\nu_{s,c}}(A).
\end{equation}

\end{enumerate}
\end{theorem}

\begin{proof}
Part~\textnormal{\it a} is Theorem~\ref{thm:exponential-tilting}\textnormal{\it (a)} with
$1/\betatilt
    =
    u_{\mathrm{epi}}(s)$.

For part~\textnormal{\it b}, Theorem~\ref{thm:exponential-tilting}\textnormal{\it (e)} gives us that 
\[
    e^{-2\beta\eps_{\mathrm{crit}}}
    \le
    \frac{
        \mathrm d
        \nu_{s,c,\beta}^{\hat q}
    }{
        \mathrm d
        \nu_{s,c,\beta}^{q_{\mathrm{pop}}}
    }
    \le
    e^{2\beta\eps_{\mathrm{crit}}}.
\]
Under Assumption~\ref{ass:adaptive-epistemic-calibration},
we have that 
$\eps_{\mathrm{crit}}
    \leq
    c_{\mathrm{cal}}
    u_{\mathrm{epi}}(s)$,
and with
$\beta
    =
    \frac1{
        u_{\mathrm{epi}}(s)
    }$,
we obtain
$\beta
    \eps_{\mathrm{crit}}
    \leq
    c_{\mathrm{cal}}$.

Substitution gives~\eqref{eq:adaptive-density-ratio}, and integrating the log-density ratio under either measure yields
\eqref{eq:adaptive-forward-kl}.

For part~\textnormal{\it c},
Theorem~\ref{thm:exponential-tilting}\textnormal{\it (e)} gives us that
\[
    \operatorname{Odds}_{
        \nu_{s,c,\beta}^{\hat q}
    }(A)
    \geq
    \exp
    \left(
        \beta
        \left(
            \Xi_{\nu_{s,c}}
            (A;q_{\mathrm{pop}})
            -
            2\eps_{\mathrm{crit}}
        \right)
    \right)
    \operatorname{Odds}_{\nu_{s,c}}(A).
\]
Using
$\beta=1/u_{\mathrm{epi}}(s)$
and
$\eps_{\mathrm{crit}}
\leq
c_{\mathrm{cal}}u_{\mathrm{epi}}(s)$
yields
\[
\begin{aligned}
    \beta
    \left(
        \Xi_{\nu_{s,c}}
        (A;q_{\mathrm{pop}})
        -
        2\eps_{\mathrm{crit}}
    \right)
    &\ge
    \frac{
        \Xi_{\nu_{s,c}}
        (A;q_{\mathrm{pop}})
    }{
        u_{\mathrm{epi}}(s)
    }
    -
    2c_{\mathrm{cal}},
\end{aligned}
\]
which proves~\eqref{eq:adaptive-high-value-odds}.  Under
\eqref{eq:adaptive-standardized-margin}, the exponent is at least $m$, proving
\eqref{eq:adaptive-exp-amplification}.
\end{proof}

The following corollary applies uncertainty adaptation to the barycentric refinement used in \QFRS{}. Recall that for a fixed state $s$ and coarse anchor $c$, we have  the  adaptive weights by
$\widehat w_{c,k}^{\mathrm{ad}}(s)
    :=
    \widehat w_{c,k}^{\betatilt(s)}(s)$ for each $k\in \{0,\ldots, M\}$.
The adaptive barycentric refinement is
$\tilde z_c^{\mathrm{ad}}(s)
    :=
    \sum_{k=0}^{M}
    \widehat w_{c,k}^{\mathrm{ad}}(s)
    z_{c,k}(s)$,
with weighted particle variance defined in \eqref{eq:adaptive-weighted-particle-variance}, where $\betatilt(s):=1/u_{\mathrm{epi}}(s)$.
\begin{corollary}[Value of the adaptive barycentric refinement]
\label{cor:adaptive-barycentric-refinement}
Fix a state $s$, a selected coarse anchor $c$, and its realized local particle
cloud
$\mathcal B_c(s)
    =
    \{z_{c,0}(s),\ldots,z_{c,M}(s)\}$.
Set
$\betatilt(s)
    :=
    1/u_{\mathrm{epi}}(s)$.
Suppose the local smoothness condition of
Proposition~\ref{prop:barycentric-refinement-value} holds on
$\operatorname{conv}\mathcal B_c(s)$, and suppose the critic error on the
realized cloud is calibrated to the epistemic scale---i.e., 
\begin{equation}
    \label{eq:adaptive-barycenter-calibration}
    \max_{0\le k\le M}
    \left|
        \hat q
        \left(
            s,z_{c,k}(s)
        \right)
        -
        q_{\mathrm{pop}}
        \left(
            s,z_{c,k}(s)
        \right)
    \right|
    \le
    c_{\mathrm{cal}}
    u_{\mathrm{epi}}(s).
\end{equation}
Then
\begin{equation}
    \label{eq:adaptive-barycenter-best-particle-gap}
    \max_{0\le k\le M}
    q_{\mathrm{pop}}
    \left(
        s,z_{c,k}(s)
    \right)
    -
    q_{\mathrm{pop}}
    \left(
        s,\tilde z_c^{\mathrm{ad}}(s)
    \right)
\leq
    u_{\mathrm{epi}}(s)
    \left(
        \log(M+1)
        +
        2c_{\mathrm{cal}}
    \right)
    +
    \frac{L_{\mathrm{sm}}}{2}
    \mathcal V_c^{\mathrm{ad}}(s).
\end{equation}
\end{corollary}

\begin{proof}
Apply Proposition~\ref{prop:barycentric-refinement-value} with
$\betatilt
    =
    \betatilt(s)
    =
    1/u_{\mathrm{epi}}(s)$.
The calibration condition
\eqref{eq:adaptive-barycenter-calibration} implies that the uniform
critic-error condition used in
\eqref{eq:barycenter-best-particle-gap} holds with
$\eps_{\mathrm{crit}}
    =
    c_{\mathrm{cal}}
    u_{\mathrm{epi}}(s)$.
Therefore
\begin{align*}
\max_{0\le k\le M}
q_{\mathrm{pop}}
\left(
    s,z_{c,k}(s)
\right)
-
q_{\mathrm{pop}}
\left(
    s,\tilde z_c^{\mathrm{ad}}(s)
\right)
\leq
\frac{\log(M+1)}{\betatilt(s)}
+
2c_{\mathrm{cal}}
u_{\mathrm{epi}}(s)
+
\frac{L_{\mathrm{sm}}}{2}
\mathcal V_c^{\mathrm{ad}}(s).
\end{align*}
Since
$1/\betatilt(s)
    =
    u_{\mathrm{epi}}(s)$,
the claimed bound follows.
\end{proof}

\paragraph{Interpretation.}
Corollary~\ref{cor:adaptive-barycentric-refinement} extends the adaptive
tilting guarantee from the particle weights to the latent actually executed by
the barycenter implementation.  The bound separates two effects:
\[
    \underbrace{
    u_{\mathrm{epi}}(s)
    \left(
        \log(M+1)+2c_{\mathrm{cal}}
    \right)
    }_{\text{epistemically tempered selection error}}
    \;+\;
    \underbrace{
    L_{\mathrm{sm}}
    \mathcal V_c^{\mathrm{ad}}(s)/2
    }_{\text{barycentric interpolation error}}.
\]
Thus the value-selection component contracts with calibrated epistemic
uncertainty, while the additional cost of executing a weighted latent rather
than sampling a particle is controlled by the local curvature of the
population score and the geometric spread of the tilted cloud.

Theorem~\ref{thm:adaptive-local-tilting} isolates the benefit of adapting the tilting intensity to epistemic uncertainty.  With fixed $\betatilt$, the sensitivity of the weighting law to critic error scales as
$\betatilt\eps_{\mathrm{crit}}$.
Yet, under the adaptive choice
$
    \betatilt(s)
    =
    1/u_{\mathrm{epi}}(s),
$
and a calibrated uncertainty scale
$\eps_{\mathrm{crit}}\lesssim u_{\mathrm{epi}}$,
this dependence becomes scale free.  The same normalization yields a standardized criterion for aggressive local weighting: a high-value region is amplified according to its population value separation measured in units of epistemic uncertainty.

Consequently, adaptation does not remove the upper-tail preference established in the preceding section.  Rather, it determines when that preference can be trusted.  Large uncertainty keeps the local weighting conservative; as uncertainty contracts, progressively smaller true value differences become sufficient to justify strong exponential amplification.

\subsection{Finite-cloud implication}

The same substitution gives an immediate implementation-level consequence for the finite particle cloud used by \QFRS{}.  In the fixed-temperature analysis, Corollary~\ref{cor:finite-cloud-selection-error} gives
us that \[
    \max_{0\le k\le M}
        q_{c,k}^{\mathrm{pop}}
    -
    \sum_{k=0}^{M}
        \widehat w_{c,k}^{\betatilt}
        q_{c,k}^{\mathrm{pop}}
    \leq
    \frac{\log(M+1)}{\betatilt}
    +
    2\eps_{\mathrm{crit}}.
\]

Under adaptive tilting, both terms are controlled by the same epistemic scale.

\begin{corollary}[Finite-cloud error under epistemically adaptive tilting]
\label{cor:adaptive-finite-cloud}
Fix a realized local particle cloud and assume \eqref{eq:adaptive-barycenter-calibration} holds on that cloud. Let $\betatilt(s)=1/u_{\mathrm{epi}}(s)$. Then 
\begin{equation}
\label{eq:adaptive-finite-cloud-gap}
\max_{0\leq k\leq M}
q_{c,k}^{\mathrm{pop}}
-
\sum_{k=0}^{M}
\widehat w_{c,k}^{\mathrm{ad}}(s)
q_{c,k}^{\mathrm{pop}}
\le
u_{\mathrm{epi}}(s)
\left(
\log(M+1)
+
2c_{\mathrm{cal}}
\right).
\end{equation}
If, in addition, the local smoothness assumptions of
Proposition~\ref{prop:barycentric-refinement-value} hold, then the executed barycentric refinement satisfies
\begin{equation}
\label{eq:adaptive-barycentric-gap}
\begin{aligned}
&
\max_{0\leq k\leq M}
q_{\mathrm{pop}}
\left(
s,z_{c,k}(s)
\right)
-
q_{\mathrm{pop}}
\left(
s,
\tilde z_{c}^{\mathrm{ad}}(s)
\right)
\leq
u_{\mathrm{epi}}(s)
\left(
\log(M+1)
+
2c_{\mathrm{cal}}
\right)
+
\frac{L_{\mathrm{sm}}}{2}
\mathcal V_{c}^{\mathrm{ad}}(s).
\end{aligned}
\end{equation}
\end{corollary}

\begin{proof}
Substitute
$\betatilt(s)
    =
    1/u_{\mathrm{epi}}(s)
$
and
$
    \eps_{\mathrm{crit}}
    \leq
    c_{\mathrm{cal}}
    u_{\mathrm{epi}}(s)
$
into
Corollary~\ref{cor:finite-cloud-selection-error}
and
Proposition~\ref{prop:barycentric-refinement-value}.
\end{proof}

Corollary~\ref{cor:adaptive-finite-cloud} gives the most direct finite-cloud interpretation of adaptive $\betatilt$.  Ignoring the separate barycentric interpolation term, the gap between the best represented local particle and the population value achieved by the learned weighting is bounded above by a quantity linear in $u_{\mathrm{epi}}(s)$, for fixed $M$ and $c_{\mathrm{cal}}$: 
\[
    \text{adaptive weighting error}
    =
    \mathcal O
    \left(
        u_{\mathrm{epi}}(s)
    \right).
\]
As calibrated epistemic uncertainty contracts, the finite-temperature
term and the upper bound on the critic-estimation contribution
contract together, until the uncertainty floor becomes active.

\subsection{Automatic annealing with learning}

The adaptive rule also induces an automatic transition from conservative to aggressive local refinement.  Suppose, for a fixed state, local particle law, and score landscape, the epistemic uncertainty decreases over training---i.e., $u_{\mathrm{epi},t}(s)
    \downarrow$.
Then
$\betatilt^{t}(s)
    =
    \frac1{
        u_{\mathrm{epi},t}(s)
    }
    \uparrow$.

By the monotonicity results of Theorem~\ref{thm:exponential-tilting},
the average score under the tilted weighting and the KL departure
from the local reference law are nondecreasing as uncertainty contracts.
Both increase strictly when the score is nonconstant
$\nu_{s,c}$-almost surely and the tilting intensity increases strictly.
Hence the same fixed exponential-tilting mechanism automatically moves from a weakly tilted, reference-like regime toward increasingly concentrated upper-tail refinement.

This statement isolates the effect of uncertainty adaptation: during actual online learning the critic scores and particle clouds also change, so the observed weighting statistics need not be monotone in training time.  Conditional on a fixed local score landscape, however, contraction of epistemic uncertainty necessarily strengthens the risk-seeking local tilt.

\paragraph{Role of the uncertainty floor.}
The floor in~\eqref{eq:adaptive-tilt-beta} imposes
$\betatilt(s)
    \le
    1/u_{\min}$.
It therefore caps the risk-seeking intensity, preventing a vanishing ensemble-disagreement estimate from making $\betatilt(s)$ diverge. 
 Once the floor becomes active, uncertainty-driven annealing saturates at this maximum intensity.

\end{document}